\documentclass[11pt]{article}

\usepackage[margin=1in]{geometry}
\usepackage[numbers,sort&compress]{natbib}
\usepackage{amsmath,amssymb,amsthm,mathtools}
\usepackage{bm}
\usepackage{booktabs}
\usepackage{array}
\usepackage{longtable}
\usepackage{graphicx}
\usepackage{placeins}
\usepackage{needspace}
\usepackage{xcolor}
\usepackage{tikz}
\usepackage{pgfplots}
\usepgfplotslibrary{groupplots}
\usetikzlibrary{arrows.meta,calc,positioning,decorations.pathreplacing,
  angles,quotes,shapes.geometric}
\usepackage{microtype}
\usepackage[colorlinks=true,linkcolor=blue!50!black,citecolor=blue!50!black,urlcolor=blue!50!black]{hyperref}
\hypersetup{
  pdftitle={The Fourth Quadrant: A Stylized View of Benign Misfitting},
  pdfauthor={Gireeja Ranade, Anant Sahai}
}

\newcommand{\notationgroup}[2]{%
  \multicolumn{3}{@{}l}{%
    \rule{0pt}{1.6\normalbaselineskip}%
    \bfseries\hyperref[#1]{#2}}\\*[-1pt]
  \cmidrule(lr){1-3}
}
\newcommand{\notationgroupplain}[1]{%
  \multicolumn{3}{@{}l}{%
    \rule{0pt}{1.6\normalbaselineskip}%
    \bfseries #1}\\*[-1pt]
  \cmidrule(lr){1-3}
}

\pgfplotsset{compat=1.18}

\definecolor{sigblue}{RGB}{36,82,139}
\definecolor{nuisance}{RGB}{182,84,46}
\definecolor{calibfill}{RGB}{242,233,205}
\definecolor{softgray}{RGB}{90,90,90}
\tikzset{
    figurebox/.style={draw,rounded corners=4pt,inner sep=6pt,align=left,fill=white},
    calloutbox/.style={figurebox,fill=calibfill},
    flow/.style={->,thick,rounded corners=3pt},
}

\newcommand{\FigStylizedDesign}{%
\begin{tikzpicture}[x=1cm,y=1cm,every node/.style={font=\small}]
    \def\cw{0.72}
    \def\rh{0.44}
    \coordinate (M0) at (1.0,1.45);

    \node at (2.45,6.1) {$\Xt=[\vec{x}_1,\vec{x}_2,\vec{x}_3,\vec{x}_4]$};
    \foreach \c/\lab in {0/$\vec{x}_1$,1/$\vec{x}_2$,2/$\vec{x}_3$,3/$\vec{x}_4$}{
        \node[font=\footnotesize] at ($(M0)+(\c*\cw+0.5*\cw,4.25)$) {\lab};
    }

    \foreach \c in {0,...,3}{
        \foreach \r in {0,...,8}{
            \draw ($(M0)+(\c*\cw,\r*\rh)$) rectangle ++(\cw,\rh);
        }
    }

    \foreach \c in {0,...,3}{
        \fill[sigblue!22] ($(M0)+(\c*\cw,8*\rh)$) rectangle ++(\cw,\rh);
    }
    \fill[nuisance!25] ($(M0)+(0*\cw,6*\rh)$) rectangle ++(\cw,2*\rh);
    \fill[nuisance!25] ($(M0)+(1*\cw,4*\rh)$) rectangle ++(\cw,2*\rh);
    \fill[nuisance!25] ($(M0)+(2*\cw,2*\rh)$) rectangle ++(\cw,2*\rh);
    \fill[nuisance!25] ($(M0)+(3*\cw,0*\rh)$) rectangle ++(\cw,2*\rh);

    \foreach \c in {0,...,3}{
        \draw ($(M0)+(\c*\cw,8*\rh)$) rectangle ++(\cw,\rh);
    }
    \draw ($(M0)+(0*\cw,6*\rh)$) rectangle ++(\cw,2*\rh);
    \foreach \r in {0,...,5}{ \draw ($(M0)+(0*\cw,\r*\rh)$) rectangle ++(\cw,\rh); }

    \draw ($(M0)+(1*\cw,4*\rh)$) rectangle ++(\cw,2*\rh);
    \foreach \r in {0,...,3,6,7}{ \draw ($(M0)+(1*\cw,\r*\rh)$) rectangle ++(\cw,\rh); }

    \draw ($(M0)+(2*\cw,2*\rh)$) rectangle ++(\cw,2*\rh);
    \foreach \r in {0,1,4,5,6,7}{ \draw ($(M0)+(2*\cw,\r*\rh)$) rectangle ++(\cw,\rh); }

    \draw ($(M0)+(3*\cw,0*\rh)$) rectangle ++(\cw,2*\rh);
    \foreach \r in {2,...,7}{ \draw ($(M0)+(3*\cw,\r*\rh)$) rectangle ++(\cw,\rh); }

    \foreach \c in {0,...,3}{
        \node at ($(M0)+(\c*\cw+0.5*\cw,8*\rh+0.5*\rh)$) {$\sqrt{\gamma}$};
    }

    \node at ($(M0)+(0.5*\cw,7*\rh)$) {$\vec{v}_1$};
    \node at ($(M0)+(1.5*\cw,5*\rh)$) {$\vec{v}_2$};
    \node at ($(M0)+(2.5*\cw,3*\rh)$) {$\vec{v}_3$};
    \node at ($(M0)+(3.5*\cw,1*\rh)$) {$\vec{v}_4$};

    \draw[decorate,decoration={brace,amplitude=4pt,mirror},softgray] (0.48,4.97) -- (0.48,5.41)
        node[midway,left=6pt,font=\footnotesize,align=right] {shared\\spike};

    \node[font=\footnotesize,align=right,softgray] at (0.0,3.21) {orthogonal\\nuisance};

    \draw[decorate,decoration={brace,amplitude=4pt,mirror},softgray] (0.78,4.09) -- (0.78,4.97);
    \draw[decorate,decoration={brace,amplitude=4pt,mirror},softgray] (0.78,3.21) -- (0.78,4.09);
    \draw[decorate,decoration={brace,amplitude=4pt,mirror},softgray] (0.78,2.33) -- (0.78,3.21);
    \draw[decorate,decoration={brace,amplitude=4pt,mirror},softgray] (0.78,1.45) -- (0.78,2.33);

    \node at (6.65,5.7) {$\vec{\alpha}^\top=(\alpha_1,\alpha_2,\alpha_3,\alpha_4)$};
    \coordinate (A0) at (6.15,2.05);
    \foreach \r/\lab in {3/$\alpha_1$,2/$\alpha_2$,1/$\alpha_3$,0/$\alpha_4$}{
        \draw ($(A0)+(0,\r*0.62)$) rectangle ++(1.1,0.62);
        \node at ($(A0)+(0.55,\r*0.62+0.31)$) {\lab};
    }
    \node at (5.2,3.27) {$\times$};

    \draw[flow] (7.45,3.27) -- (8.55,3.27) node[midway,above] {$\vec{u}=\Xt\vec{\alpha}$};
    \coordinate (U0) at (9.1,1.45);
    \fill[sigblue!18] ($(U0)+(0,8*\rh)$) rectangle ++(1.95,\rh);
    \foreach \r in {0,...,7}{
        \fill[nuisance!18] ($(U0)+(0,\r*\rh)$) rectangle ++(1.95,\rh);
    }
    \foreach \r in {0,...,8}{
        \draw ($(U0)+(0,\r*\rh)$) rectangle ++(1.95,\rh);
    }
    \node at ($(U0)+(0.975,8*\rh+0.5*\rh)$) {$\sqrt{\gamma}\,s$};
    \foreach \r/\lab in {6/$\sqrt n\,\alpha_1$,7/$\sqrt n\,\alpha_1$,4/$\sqrt n\,\alpha_2$,5/$\sqrt n\,\alpha_2$,2/$\sqrt n\,\alpha_3$,3/$\sqrt n\,\alpha_3$,0/$\sqrt n\,\alpha_4$,1/$\sqrt n\,\alpha_4$}{
        \node[font=\scriptsize] at ($(U0)+(0.975,\r*\rh+0.5*\rh)$) {\lab};
    }

    \node at ($(U0)+(0.975,4.25)$) {$\vec{u}$};

    \draw[decorate,decoration={brace,amplitude=4pt},softgray] (11.15,5.41) -- (11.15,4.97)
        node[midway,right=6pt,softgray] {$u[1]$};
    \draw[decorate,decoration={brace,amplitude=4pt},softgray] (11.15,4.97) -- (11.15,1.45)
        node[midway,right=6pt,softgray] {$\vec{u}_{-1}$};

    \node[calloutbox,anchor=north west,text width=11.9cm] at (0.9,0.82) {
        Since $u[1]=\sqrt{\gamma}\,s$ and $\|\vec{u}_{-1}\|_2^2=d\,q$,
        both errors depend on $\vec{\alpha}$ only through
        $s=\textstyle\sum_i\alpha_i$ and $q=\textstyle\sum_i\alpha_i^2$:
        \[
            \Ttest(\vec{u})=(\gamma s-1)^2+d\,q,
            \qquad
            \Ttrain(\vec{u})=(\gamma s-1)^2
            +\tfrac{2d s}{n}(\gamma s-1)
            +\tfrac{d^2}{n}\,q.
        \]
    };
\end{tikzpicture}%
}

\newtheorem{theorem}{Theorem}[section]
\newtheorem{proposition}[theorem]{Proposition}
\newtheorem{lemma}[theorem]{Lemma}
\newtheorem{corollary}[theorem]{Corollary}
\theoremstyle{definition}
\newtheorem{remark}[theorem]{Remark}

\theoremstyle{plain}

\DeclareMathOperator{\diag}{diag}
\DeclareMathOperator{\median}{median}

\newcommand{\R}{\mathbb{R}}
\newcommand{\E}{\mathbb{E}}
\newcommand{\one}{\vec{1}}
\newcommand{\Xt}{\bm{X}}
\newcommand{\Ttrain}{E_{\rm train}}
\newcommand{\Ttest}{E_{\rm test}}
\newcommand{\calE}{\mathcal{E}}
\newcommand{\calN}{\mathcal{N}}

\title{The Fourth Quadrant: A Stylized View of Benign Misfitting}
\author{Gireeja Ranade \quad Anant Sahai\\
University of California, Berkeley\\
\texttt{\{ranade,sahai\}@eecs.berkeley.edu}}
\date{}

\begin{document}
\maketitle

\begin{abstract}
Training error is what we can observe on a training set; test error is the
quantity we actually care about.
This paper gives a simple setting in which 
minimizing the training error would lead us astray.
We study supervised linear regression with traditional squared-error loss in 
a deterministic $(d+1)$-dimensional single-spike model.  
Each stylized training vector has the same informative spike coordinate, of amplitude
$\sqrt{\gamma}$ with $\gamma>1$.
The remaining directions are nuisance, and 
the nuisance components of distinct training vectors all have equal norm and 
are mutually orthogonal.
The training labels are all $1$. 
Fresh test points are drawn from
 \(\vec{x}_{\rm test} \sim \calN(\vec{0},\diag(\gamma,1,\ldots,1))\), 
with the noise-free test labels being the normalized spike coordinate 
\(x_{\rm test}[1]/\sqrt{\gamma}\).
We first compute the exact train--test tradeoff over linear predictors in the
training span, the class naturally reached by zero-initialized linear gradient
methods.
When the spike direction is unknown, enlarging to all linear predictors in the full space does not improve this tradeoff, either on average over directions consistent with the data or against the worst such direction.
We exhibit a range of training-set
sizes in which every span predictor that generalizes well must
fit the training data \emph{worse} than the predictor that always outputs zero.
We call this regime \emph{benign misfitting}, or the fourth quadrant.
Here
test performance \(\Ttest\) is good but training performance \(\Ttrain\) is
bad, so good generalization \emph{requires} a negative generalization gap.
At its sharpest, \(\Ttest=o(1)\) while
\(\Ttrain\to\infty\), with the zero predictor normalized to
\(\Ttrain=\Ttest=1\).
The best span predictor begins to
generalize when \(n\gg d/\gamma^2\), while interpolation does not generalize
until the later threshold \(n\gg d/\gamma\).  
In the intervening window \(d/\gamma^2 \ll n \ll d/\gamma\), useful
prediction within the linear span lies beyond interpolation: predictions
on the training points overshoot the labels.
We show that one-pass
stochastic gradient descent (SGD), with a large constant learning rate,
reaches small test error throughout this window---matching the best span
predictor up to a logarithmic factor.
We also verify directly that it indeed
has \emph{large} empirical training error (despite the descent premise in its
name).
Finally, we show that
the unavoidable nuisance component responsible for the training misfit also
controls the predictor's adversarial sensitivity.
\end{abstract}

\section{Introduction}
\label{sec:introduction}

The traditional train--test picture has three familiar regions.
A predictor
can underfit, fit well, or overfit---a story usually told through model
capacity.
But the train--test plane itself has two coordinates: empirical
training error and test error.
Modern benign-overfitting
results show that interpolation---exactly fitting all the training labels---need
not mean poor test error
\cite{belkin2019reconciling,belkin2020twomodels,belkin2021fit,
bartlett2020benign,hastie2022surprises,muthukumar2020harmless,
mallinar2022taxonomy}.
This paper studies a different possibility: good test prediction may
first appear in a region where empirical training performance is bad.
We
normalize both errors so that the predictor that always outputs zero has
\(\Ttrain=\Ttest=1\), and we call the fourth quadrant \(\Ttest<1<\Ttrain\)
\emph{benign misfitting}: good test error
paired with bad empirical training error.
Under the sign convention
\(\Ttest-\Ttrain\), the generalization gap is negative.

Figure~\ref{fig:quadrants} places this in a wider taxonomy.
Benign
overfitting and tempered overfitting%
\footnote{The term \emph{tempered overfitting} and the
benign/tempered/catastrophic taxonomy were introduced by
\citet{mallinar2022taxonomy}.
The paradigmatic example of tempered
overfitting is one-nearest-neighbor classification or regression with equally noisy training
and test labels.
With enough training data, the 1-NN rule approaches roughly
twice the Bayes error rate: it generalizes reasonably (small test error), interpolates
(zero training error), and has training error below the irreducible
noisy-label test level.}
both lie in the good-fit quadrant: they
generalize well, and they are traditionally called overfitting only because
their generalization gap is positive --- training error sits significantly below test
error.
The fourth quadrant breaks the pattern: good test error is instead paired with bad empirical fit.
In ordinary practice, training error above a natural baseline like the zero-predictor's performance would often look
like a bug.
In this stylized example, it is the price of using the shared signal early.%
\footnote{In colloquial use, ``underfitting'' has bad test error and bad training error
while ``overfitting'' has bad test error and good training error.
But what is bad?
What is good?  
The quadrant labels in Figure~\ref{fig:quadrants} must therefore be relative to some 
baseline.
For us in this figure, ``underfitting'' in the upper-right quadrant means both
errors exceed those of the zero predictor, which is admittedly a more extreme sense than any
ordinary colloquial use of \emph{underfitting}.  
Following the terminological approach of \cite{mallinar2022taxonomy}, one could argue that 
losses this extreme should be called ``catastrophic,'' but for simplicity we are omitting 
this finer distinction.
We use the zero-predictor
baseline because our fourth-quadrant theorem proves the strong statement
that useful span predictors can have training error above it. }


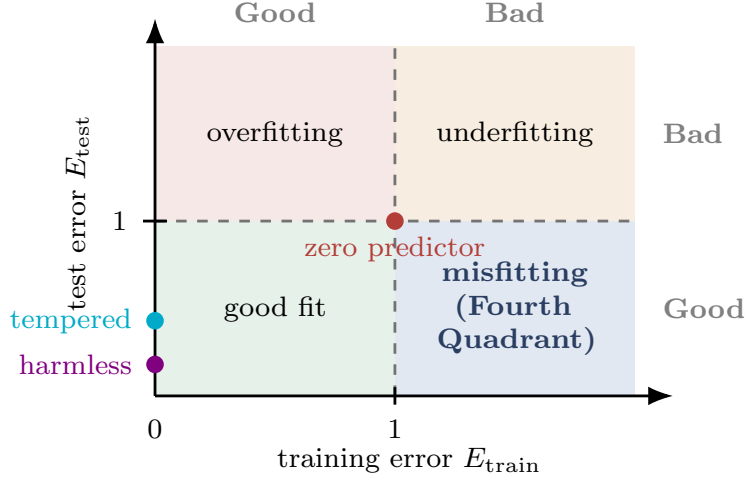
\begin{figure}[t]
    \centering
    \resizebox{0.62\textwidth}{!}{\begin{tikzpicture}[x=2.4cm,y=1.75cm]
    \definecolor{goodgreen}{RGB}{48,135,84}
    \definecolor{alertred}{RGB}{178,62,55}
    \definecolor{warnorange}{RGB}{201,124,42}
    \definecolor{softblue}{RGB}{76,112,176}
    \definecolor{harmlesspurple}{RGB}{128,0,128}
    \definecolor{temperedcyan}{RGB}{0,170,200}

    \fill[goodgreen!12] (0,0) rectangle (1,1);
    \fill[alertred!12] (0,1) rectangle (1,2);
    \fill[warnorange!14] (1,1) rectangle (2,2);
    \fill[softblue!16] (1,0) rectangle (2,1);

    \draw[black!55,thick,dashed] (1,0) -- (1,2);
    \draw[black!55,thick,dashed] (0,1) -- (2,1);

    \node[font=\scriptsize,align=center,text width=2.0cm] at (0.5,0.5)
        {good fit};
    \node[font=\scriptsize,align=center,text width=2.0cm] at (0.5,1.5)
        {overfitting};
    \node[font=\scriptsize,align=center,text width=2.0cm] at (1.5,1.5)
        {underfitting};
    \node[font=\scriptsize\bfseries,align=center,text width=2.05cm,softblue!55!black] at (1.5,0.5)
        {misfitting\\(Fourth Quadrant)};

    \draw[-{Latex[length=2.5mm]},thick] (0,0) -- (2.16,0);
    \draw[-{Latex[length=2.5mm]},thick] (0,0) -- (0,2.16);
    \node[font=\scriptsize,anchor=north] at (0,-0.06) {0};

    \draw[thick] (1,-0.05) -- (1,0.05);
    \node[font=\scriptsize,anchor=north] at (1,-0.06) {1};
    \draw[thick] (-0.05,1) -- (0.05,1);
    \node[font=\scriptsize,anchor=east] at (-0.06,1) {1};

    \node[font=\scriptsize,anchor=north] at (1.06,-0.22) {training error $E_{\rm train}$};
    \node[font=\scriptsize,rotate=90,anchor=south] at (-0.22,1.05) {test error $E_{\rm test}$};

    \fill[alertred] (1,1) circle (2.5pt);
    \node[font=\scriptsize,alertred,anchor=north,align=center] at (1,0.99) {zero predictor};

    \fill[harmlesspurple] (0,0.18) circle (2.5pt);
    \node[font=\scriptsize,harmlesspurple,anchor=east] at (-0.04,0.18) {harmless};

    \fill[temperedcyan] (0,0.43) circle (2.5pt);
    \node[font=\scriptsize,temperedcyan,anchor=east] at (-0.04,0.43) {tempered};

    \node[font=\scriptsize\bfseries,black!50,anchor=south] at (0.5,2.06) {Good};
    \node[font=\scriptsize\bfseries,black!50,anchor=south] at (1.5,2.06) {Bad};

    \node[font=\scriptsize\bfseries,black!50,anchor=west] at (2.06,1.5) {Bad};
    \node[font=\scriptsize\bfseries,black!50,anchor=west] at (2.06,0.5) {Good};
\end{tikzpicture}}
    \caption[A four-quadrant train--test taxonomy]{A baseline-relative four-quadrant train--test taxonomy.
    The dashed lines mark the
    zero-predictor baselines \(\Ttrain=1\) and \(\Ttest=1\), whose intersection
    is the zero predictor: moving toward the origin from a dashed line
    beats this baseline on that axis (left of \(\Ttrain=1\) for training,
    below \(\Ttest=1\) for testing).
    The purple dot represents benign overfitting
    (harmless interpolation) and the cyan dot represents tempered
    overfitting: both have small training error and small test
    error, hence both lie in the good-fit quadrant, even though their
    generalization gaps are positive.
    Benign misfitting occupies the
    good-test/worse-than-zero-train quadrant, where the generalization gap is negative.}
    \label{fig:quadrants}
\end{figure}

The setup in this paper is intentionally elementary.
There are \(n\le d\)
training points in \(\R^{d+1}\).
Each has the same signal coordinate
\(\sqrt{\gamma}\), with \(\gamma>1\), so the spike direction carries more variance than any
single nuisance direction, while the nuisance
components are mutually orthogonal and have squared norm \(d\).
Each training
label is the normalized spike coordinate and is therefore equal to \(1\).
A
fresh test point $\vec{x}_{\rm test}$ is drawn from the zero-mean single-spike Gaussian
\(\calN(\vec{0},\diag(\gamma,1,\ldots,1))\), 
with its true test label being the normalized spike
coordinate \(x_{\rm test}[1]/\sqrt{\gamma}\).
The factor \(1/\sqrt{\gamma}\) gives that label unit variance as \(\gamma\) varies and therefore keeps the squared-error scale fixed%
\footnote{\interlinepenalty=10000\relax
Multiplying every label and predictor output by the same nonzero factor multiplies both squared errors and their zero-predictor baselines by its square.
The comparisons with those baselines are therefore unchanged.}.
Remark~\ref{rem:sensor-model} gives a noisy sensor model that produces the single-spike test law and explains why this normalization is natural.
The best possible predictor here
simply reads off the appropriately normalized spike direction, the
\emph{spike oracle}%
\footnote{The displayed first coordinate is expressed in a basis aligned with the latent spike direction and need not be an observed sparse coordinate.
Note that
any common orthogonal rotation preserves the span geometry and hence the
train--test tradeoff calculation in this paper.
Appendix~\ref{app:latent-rotation-span-complete} formalizes what can be
concluded when that common orientation is unknown.
Appendix~\ref{app:rotation-aware} returns to rotation for the distinct
question of coordinatewise adversarial sensitivity.}
\(\vec{w}^{\star}=(\tfrac{1}{\sqrt{\gamma}},0,\ldots,0)^\top\).
But that oracle is not in the
training span.

We therefore ask what can be done by \emph{span predictors}:
linear predictors built from the training vectors themselves.
We study this
class because, starting from zero, linear gradient methods build their
predictors out of a linear combination of the training examples.%
\footnote{Kernel methods also return span predictors by the representer
theorem \cite{kimeldorf1971some,schoelkopf2001generalized}.
For
underdetermined least squares, zero-initialized gradient descent with a
suitably small learning rate selects the minimum-Euclidean-norm interpolator
\cite{gunasekar2018characterizing}, and related implicit-bias results for
classification select max-margin directions within the training span
\cite{soudry2018implicit}.}

The exact frontier yields two distinct sample thresholds\footnote{\label{fn:threshold-numbers}For
\(\gamma=100\) and \(d=50{,}000\), the two asymptotic crossover thresholds are
\(d/\gamma^2=5\) and \(d/\gamma=500\).
The exact curves are smooth.
At
\(n=500\), the interpolator's test error is \(0.2525\), and
Corollary~\ref{cor:sample-complexity} gives \(n_{\rm mis}\approx15\) for the
best span predictor to attain that same target.}.
The best span predictor, meaning the
predictor in the training span with minimum test error, starts to generalize
once \(n\gg d/\gamma^2\).
The minimum-norm interpolator, which fits every
training label exactly, does not generalize until the later threshold
\(n\gg d/\gamma\).
We call the intervening interval the \emph{benign-misfitting
window}\footnote{One concrete way to read the window asymptotically is
polynomial scaling: if \(d=D\), \(\gamma=D^b\), and \(n=D^a\) with \(0<b<1\)
and \(a>0\), then
\(d/\gamma^2\ll n\ll d/\gamma\) becomes \(1-2b<a<1-b\).
The lower endpoint is
the best-span threshold.
The upper endpoint is the interpolation threshold.}:
\begin{equation}\label{eq:benign-misfitting-window}
    d/\gamma^2\ll n\ll d/\gamma,
    \qquad\text{equivalently}\qquad
    \gamma n\ll d\ll\gamma^2 n.
\end{equation}
In this window, there are span predictors with vanishing test error, but every
such predictor must have diverging empirical training error.
Thus the span
becomes useful before the interpolator does---and the gap is large: for a
fixed test error, the interpolation threshold sits a factor of order
\(\gamma\) in samples beyond the first-useful-span threshold.

There is no label noise in this setup: the spike oracle fits every
training point and every test point.
The misfit comes from the span
representation of predictors.
Interpolation under-calibrates the shared
spike.
Getting good calibration inside the span requires the predictor to
deliberately overshoot the training labels on average --- and the efficient
predictors overshoot every single one.
Put bluntly: if test error is
what we care about, we must
be prepared to sacrifice average performance on the training data.

The core phenomenon is geometric.
Using the \emph{combined signal} --- the spike information shared across
the training points --- forces the predictor to carry
a leftover nuisance component, which we call the \emph{nuisance residue}.
Fresh random test points see only its average squared size, but the training
examples see the nuisance residue in the same nuisance directions that created it.
In this paper's stylized calculations, this becomes an extra factor \(d/n\).
We call this asymmetry \emph{residue amplification}.
The same examples that reveal the shared signal are the examples on which a
good span predictor must be willing to accept large residuals.

This counterintuitive train-test tradeoff is algorithmically attainable.
One-pass
(i.e.\ single epoch) stochastic gradient descent (SGD), run from zero with a suitably chosen large
constant learning rate, reaches the good-test part of the span throughout the
benign-misfitting window.
Over most of the window (unless \(\gamma\) is too
big, as explained in sense~(e) of
Subsection~\ref{subsec:spike-strength-senses}), the tuned near-optimal
learning rate is
small on the spike coordinate but large relative to what repeated-example
stability would allow --- stability when the same example is revisited
(i.e.\ multi-epoch training).
It therefore overshoots individual training
points by the per-example magnitude needed to build up good spike
calibration.
At the tuned one-pass learning rate, its empirical training error
becomes large and can even rise throughout the pass.

The same nuisance residue also separates random-test performance from
adversarial sensitivity.
A random test input sees only the residue's average
squared size, but a label-preserving adversary can align a test perturbation
with the residue and tap the same kind of amplification that the training points feel,
exploiting the residue's full norm.
Small random test error can therefore
arrive before low adversarial sensitivity.
Driving that sensitivity down
requires more training samples.

Although the necessity of misfitting is proved for linear predictors in the
training span, Section~\ref{sec:one-sample} gives the one-sample intuition for why
the same necessity must extend in spirit to all learned linear predictors, as
long as the spike direction is unknown.
Appendix~\ref{app:latent-rotation-span-complete} proves
that every ambient global-linear predictor is dominated by its span
projection under the orthogonally invariant uniform average over compatible
worlds and worst-compatible risk.
Consequently the exact train--test frontier and its
learner-output-norm consequences extend, under those criteria, to every
procedure that emits one global linear predictor.%
\footnote{Appendix~\ref{app:improper} shows that nonlinear aggregation can
circumvent the necessity of large training error, and
Appendix~\ref{app:sco-comparison} gives the corresponding norm-budget
statement: in the window \eqref{eq:benign-misfitting-window}, every linear span
predictor with small test error must be outside the spike oracle's Euclidean norm
ball.}
This paper's stylized training design is a deliberate caricature:
it replaces the nearly orthogonal nuisance vectors of a high-dimensional
Gaussian sample with exactly orthogonal vectors of exactly equal length,
keeping the test law's essential geometry of one strong direction among many equal,
featureless nuisance directions.
Holding this geometry fixed reduces computing the 
train--test tradeoff to orthogonality, Cauchy--Schwarz, and one quadratic
formula.
The existence of the fourth quadrant is thus a direct algebraic fact about
high-dimensional spans --- not a byproduct of concentration bounds, pathologies
in function classes, or delicate random-matrix edge effects.
The
same-distribution companion \citep{ranade2026samedistribution} draws training
and test data from the same Gaussian law.
Exact orthogonality disappears, but
the threshold separation, forced training misfit, and one-pass SGD mechanism
persist.

The rest of the paper is organized as follows.
Section~\ref{sec:one-sample} gives a one-sample warm-up, develops its sensor backstory, and explains intuitively why an unknown spike direction supplies no useful off-span bet.
Sections~\ref{sec:setup-frontier}--\ref{sec:robustness} develop the exact
span frontier, the one-pass SGD result, and the adversarial-sensitivity
calculation.
Section~\ref{sec:discussion} explains why the regime is ``attractive'' but hard
to diagnose.
Section~\ref{sec:related} gives related scholarly perspectives.
The appendices give
extensions and further commentary, including: the fresh-versus-reused stability boundaries
(Appendix~\ref{app:stability-boundaries}), further one-pass SGD material including learning rate robustness and how the training error can ascend during SGD
(Appendix~\ref{app:sgd-refinements}), a rotation-aware adversarial calculation
(Appendix~\ref{app:rotation-aware}), the formal global-linear
span-completeness result under unknown orientation
(Appendix~\ref{app:latent-rotation-span-complete}), a nonlinear improper
scheme that escapes the span constraint (Appendix~\ref{app:improper}), and a comparison with
stochastic-convex-optimization perspectives (Appendix~\ref{app:sco-comparison}).

For a first reading, Sections~\ref{sec:one-sample}--\ref{sec:discussion}
carry the main narrative.
The related-work section and appendices are
modular.

\section{A one-sample warm-up}
\label{sec:one-sample}

The entire mechanism is already visible with one stylized labeled training point $(\vec{x}, y)$.
Let
\begin{equation}\label{eq:one-sample-data}
    \vec{x}=
    \begin{bmatrix}
        \sqrt{\gamma}\\
        \vec{v}
    \end{bmatrix}
    \in\R^{d+1},
    \qquad
    \|\vec{v}\|_2^2=d,
    \qquad
    y=1.
\end{equation}
Here the label \(y=x[1]/\sqrt{\gamma}\).
Since \(\|\vec{x}\|_2^2=\gamma+d\), a single SGD
step\footnote{When a learning rate is used in an SGD update, we use
the half-squared one-sample loss so that the gradient has no extra factor of
two.
This convention only fixes the gradient normalization.
The empirical and
test errors in \eqref{eq:one-sample-train-loss} and
\eqref{eq:one-sample-test-error} are ordinary squared residuals.} from zero
with learning rate \(\eta\) returns the span predictor \(\vec{u}=\eta\vec{x}\).
Its prediction on the training point is
\(\vec{u}^{\top}\vec{x}=\eta\|\vec{x}\|_2^2=\eta(\gamma+d)\).
Its squared
empirical error on the training point is
\begin{equation}\label{eq:one-sample-train-loss}
    \Ttrain(\eta)=((\gamma+d)\eta-1)^2.
\end{equation}
For a fresh test point
\[
    \vec{x}_{\rm test}=
    \begin{bmatrix}
        \sqrt{\gamma}\,y_{\rm test}\\
        \vec{z}_{\rm test}
    \end{bmatrix},
\]
the test label \(y_{\rm test}\sim \calN(0,1)\) and the
test nuisance vector \(\vec{z}_{\rm test}\sim \calN(\vec{0},\bm{I}_d)\) are
independent.
The test error is
\begin{equation}\label{eq:one-sample-test-error}
    \Ttest(\eta)
    =
    \E[(\eta(\gamma y_{\rm test}+\vec{v}^{\top}\vec{z}_{\rm test})-y_{\rm test})^2]
    =
    (\gamma\eta-1)^2+d\eta^2.
\end{equation}
Here \(\gamma\eta\) is the spike calibration, while \(d\eta^2\) is the
nuisance variance.
The test-optimal learning rate is thus
\begin{equation}\label{eq:one-sample-eta-opt}
    \eta_{\rm opt}=\frac{\gamma}{\gamma^2+d}.
\end{equation}
At this learning rate,
\begin{equation}\label{eq:one-sample-opt-values}
    \Ttest(\eta_{\rm opt})=\frac{d}{\gamma^2+d},
    \qquad
    \Ttrain(\eta_{\rm opt})
    =
    \left(\frac{d(\gamma-1)}{\gamma^2+d}\right)^2.
\end{equation}

Figure~\ref{fig:one-sample-rate} plots this one-sample tradeoff for a
representative parameter choice.
For comparison,
\begin{equation}\label{eq:one-sample-rates}
    \eta_{\rm interp}=\frac{1}{\gamma+d},
    \qquad
    \eta_{\rm cal}=\frac{1}{\gamma},
    \qquad
    \eta_{\rm edge}=\frac{2}{\gamma+d}.
\end{equation}
The interpolating rate \(\eta_{\rm interp}\) fits the training point exactly.
The calibrated rate \(\eta_{\rm cal}=1/\gamma\) sets the spike coefficient to
the oracle's.
The edge rate \(\eta_{\rm edge}\) is the repeated-example stability
boundary --- the one-example version of the classical \(2/L\)
gradient-descent threshold for a quadratic \cite{polyak1987introduction}.
Here \(L=\|\vec{x}\|_2^2=\gamma+d\) for the one-sample half-squared loss.
Along sequences with \(\gamma,d\to\infty\) in the few-shot window \(\gamma\ll d\ll\gamma^{2}\)
of Footnote~\ref{fn:few-shot-window},
\[
    \eta_{\rm opt}\approx\eta_{\rm cal}\approx\frac1\gamma,
    \qquad
    \frac{\eta_{\rm opt}}{\eta_{\rm edge}}\approx\frac{d}{2\gamma}\to\infty.
\]
If the same example were reused, the residual would be multiplied at each step by
\[
    1-\eta\|\vec{x}\|_2^2=1-\eta(\gamma+d).
\]
Repeated-example stability therefore requires \(|1-\eta(\gamma+d)|<1\),
equivalently \(0<\eta<2/(\gamma+d)\), matching \(\eta_{\rm edge}\).
Thus the rate that calibrates the spike is far beyond the repeated-example
stability boundary in the few-shot window.

The window \(\gamma\ll d\ll\gamma^{2}\) is worth a name of its own.
We call it the
\emph{few-shot window}, and a spike satisfying it \emph{few-shot-strong}.%
\footnote{\label{fn:few-shot-window}The ``few-shot-strong'' name reflects that a handful of
examples can already generalize, while interpolation still does not.
The window is
\eqref{eq:benign-misfitting-window} read at \(n\) of constant order, which is why the one-example
calculation of this section is a caricature of the many-sample story rather than a toy separate
from it.
}
In the few-shot window, \eqref{eq:one-sample-opt-values} gives
\[
    \Ttest(\eta_{\rm opt})\approx\frac{d}{\gamma^2}\ll1,
    \qquad
    \Ttrain(\eta_{\rm opt})\approx\left(\frac{d}{\gamma}\right)^2\gg1.
\]
The training vector is dominated by its nuisance component and is nearly
orthogonal to the spike direction, so interpolation shrinks the spike too
much.
Getting good spike calibration, meaning the correct
coefficient on the spike direction, necessarily amplifies the nuisance
component relative to interpolation.
The prediction on the training point is
then pushed past the label \(1\).

The one-sample risk level \(d/\gamma^2\) and the few-shot window
\(\gamma\ll d\ll\gamma^{2}\) are the \(n=1\) counterparts of the many-sample
best-span risk \(d/(\gamma^2n)\) and sample-count window
\(d/\gamma^2\ll n\ll d/\gamma\) studied in Section~\ref{sec:setup-frontier}.

\begin{figure}[t]
    \centering
    \includegraphics[width=0.72\textwidth]{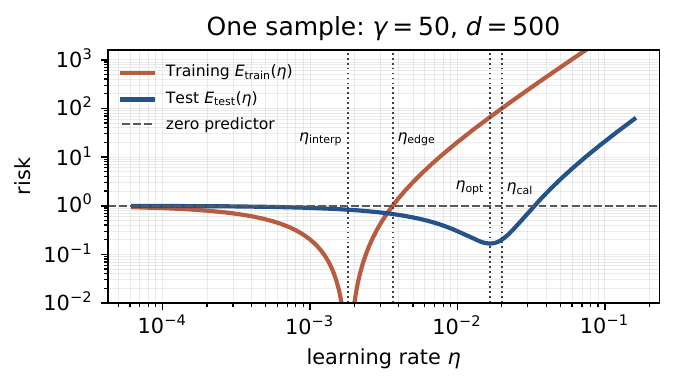}
    \caption{One-sample one-step learning-rate tradeoff.
    The red and blue curves show
    empirical training error and test error as functions of the learning rate,
    and the horizontal dashed line is the zero-predictor baseline.
    The interpolating step fits the sample but leaves test error near baseline.
    The calibrated/test-optimal step has small test error and squared empirical
    error far above baseline.
    The repeated-example stability boundary is
    \(\eta_{\rm edge}=2/(\gamma+d)\).}
    \label{fig:one-sample-rate}
\end{figure}

\FloatBarrier

\subsection{Misfitting as a natural consequence of noisy measurements}
\label{subsec:sensor-misfitting}
A noisy sensor backstory gives physical intuition for the
one-sample calculation.

\begin{remark}[Where a single-spike test law can come from]
\label{rem:sensor-model}
Let \(Y\sim\calN(0,1)\) be the quantity of interest and let there be
\(m_{\rm sens}\) noisy measurements of it, \(S_j=Y+N_j\) for
\(j=1,\ldots,m_{\rm sens}\), with
\(N_j\stackrel{\rm iid}{\sim}\calN(0,\sigma^2)\) independent of \(Y\).
A
hypothetical analyst can divide each reading by \(\sigma\) and then choose
orthogonal coordinates.
Take
\(x_{\rm test}[1]=(\sigma\sqrt{m_{\rm sens}})^{-1}
\sum_{j=1}^{m_{\rm sens}}S_j\) as the first coordinate.
The remaining
\(m_{\rm sens}-1\) coordinates are orthonormal linear combinations of the
standardized readings whose coefficients sum to zero, so the common quantity
\(Y\) cancels out of them.
The resulting test vector has the single-spike
structure used in the one-sample calculation:
\begin{equation}\label{eq:sensor-spike-law}
    \vec x_{\rm test}\sim
    \calN\!\left(\vec 0,\ \diag(\gamma,1,\ldots,1)\right),
    \qquad
    \gamma=1+\frac{m_{\rm sens}}{\sigma^2},
    \qquad
    d=m_{\rm sens}-1.
\end{equation}
By \eqref{eq:sensor-spike-law}, \(\gamma\) can be large even when each sensor's noise variance
\(\sigma^2\) is large, provided \(m_{\rm sens}/\sigma^2\) is large.
The parameterization in \eqref{eq:sensor-spike-law} reaches any \(\gamma>1\)
and any \(d\ge1\).
To realize such a pair, take \(m_{\rm sens}=d+1\) sensors with noise variance
\(\sigma^2=(d+1)/(\gamma-1)\).

The normalized sensor average \(y_{\rm test}\) is not the physical quantity
of interest.
That quantity is \(Y\).
Gaussian conditioning gives
\begin{equation}
    Y
    =
    \sqrt{1-\frac{1}{\gamma}}\;y_{\rm test}
    +
    \frac{\xi}{\sqrt{\gamma}},
    \qquad
    \xi\sim\calN(0,1).
\label{eq:sensor-conditional}
\end{equation}
The residual \(\xi\) in \eqref{eq:sensor-conditional} is independent of every
sensor reading, so the minimum mean squared error for predicting \(Y\) from
those readings is \(1/\gamma=\sigma^2/(m_{\rm sens}+\sigma^2)\).
(Notice that this goes to zero as sensors are added at fixed
per-sensor noise variance \(\sigma^2\).)
Up to a scalar factor, however, \(y_{\rm test}\) is the best prediction of \(Y\) that the measurements allow:
\begin{equation}\label{eq:sensor-conditional-mean}
    \E[\,Y\mid\vec x_{\rm test}\,]
    =
    \sqrt{1-\tfrac1\gamma}\;y_{\rm test}.
\end{equation}
If \(\widehat y\) is any predictor of
\(y_{\rm test}\) based on the sensor readings, then
\(\sqrt{1-1/\gamma}\,\widehat y\) predicts \(Y\) with mean squared error
\(1/\gamma+(1-1/\gamma)\E[(\widehat y-y_{\rm test})^2]\).
Thus, by \eqref{eq:sensor-conditional-mean}, predicting \(y_{\rm test}\) asks
how close a predictor comes to the best
anyone could do, and leaves out the error that nobody can avoid.%
\footnote{If observations of \(Y\) are used as labels in the sensor model, dividing those labels by \(\sqrt{1-1/\gamma}\) gives the corresponding normalized sensor label plus independent Gaussian noise of variance \(1/(\gamma-1)\).
Subsection~\ref{subsec:one-complete-pass} analyzes the corresponding
additive-noise mechanism on the paper's deterministic stylized design.}

\end{remark}

\paragraph{One reference measurement as the training point.}
The one-sample calculation has a physical interpretation in this backstory when the number of sensors becomes large.
Allow the sensors to differ in sensitivity and polarity.
What happens when the individual sensor responses are unknown?
Write \(\vec s_{\rm resp}\in\R^{m_{\rm sens}}\) for what the sensors
would read without noise at a physical value of one, and suppose
\(\|\vec s_{\rm resp}\|^2=m_{\rm sens}\).
For a fresh physical quantity \(Y\sim\calN(0,1)\), the sensor
readings are \(\vec S=Y\vec s_{\rm resp}+\vec N\), where
\(\vec N\sim\calN(\vec0,\sigma^2\bm I_{m_{\rm sens}})\) is
independent of \(Y\).
Then \(\operatorname{Cov}(\vec S/\sigma)
=\bm I_{m_{\rm sens}}+
\vec s_{\rm resp}\vec s_{\rm resp}^{\top}/\sigma^2\).
This covariance has eigenvalue
\(1+m_{\rm sens}/\sigma^2=\gamma\) along \(\vec s_{\rm resp}\),
and its other \(m_{\rm sens}-1=d\) eigenvalues are one.
Rotating \(\vec s_{\rm resp}/\sqrt{m_{\rm sens}}\) onto the first
coordinate axis gives \eqref{eq:sensor-spike-law}, with first
coordinate
\(\vec s_{\rm resp}^{\top}\vec S/(\sigma\sqrt{m_{\rm sens}})\).
Thus different sensor sensitivities and polarities give the same
single-spike test law.

Suppose we know the total squared response
\(\|\vec s_{\rm resp}\|^2=m_{\rm sens}\), but not the individual
responses \(\vec s_{\rm resp}[j]\).
Measure a physical reference quantity whose value is known to be one.
The reference sensor readings are
\(\vec S_0=\vec s_{\rm resp}+\vec N_0\), where
\(\vec N_0\sim\calN(\vec0,\sigma^2\bm I_{m_{\rm sens}})\).
The noise realizations \(\vec N\) in the fresh sensor readings
\(\vec S\) are independent of the realizations \(\vec N_0\) that
determined \(\vec S_0\), and both noise vectors are independent of
the fresh physical quantity \(Y\).
A natural way to estimate \(Y\) is by multiplying each new sensor reading by
that sensor's earlier reading for the reference quantity and averaging:
\begin{equation}\label{eq:sensor-product-average}
    \frac{\vec S_0^{\top}\vec S}{m_{\rm sens}}
    =\frac1{m_{\rm sens}}\sum_{j=1}^{m_{\rm sens}}
       \vec S_0[j]\vec S[j].
\end{equation}
The prediction in \eqref{eq:sensor-product-average} is obtained from one SGD step from zero on the physical reference example \((\vec S_0,1)\).
For the half-squared loss, rate \(1/m_{\rm sens}\) gives the predictor \(\vec S_0/m_{\rm sens}\) in the original sensor units.
Dividing only the readings by \(\sigma\), while keeping the physical reference label \(1\) and prediction target \(Y\), gives the same prediction from a step at rate \(\sigma^2/m_{\rm sens}\).
That rate equals \(\eta_{\rm cal}(1+\sigma^2/m_{\rm sens})\), where \(\eta_{\rm cal}=1/\gamma\) is the calibrated rate in \eqref{eq:one-sample-rates}.
To express its output as a predictor of \(y_{\rm test}\), divide the physical prediction by \(\sqrt{1-1/\gamma}\), as in \eqref{eq:sensor-conditional}.

The correspondence with the deterministic training vector in \eqref{eq:one-sample-data} is asymptotic, because the physical reference \(\vec S_0\) is random.
At fixed per-sensor noise variance, after standardizing the readings, the reference's spike coefficient and nuisance energy approach \(\sqrt\gamma\) and \(d\), respectively, up to factors tending to one as sensors are added.

Applying \eqref{eq:sensor-product-average} to the reference or fresh sensor readings gives
\begin{equation}\label{eq:sensor-product-expansion}
    \frac{\vec S_0^{\top}\vec S}{m_{\rm sens}}
    =
    \begin{cases}
    1+\frac{2\vec s_{\rm resp}^{\top}\vec N_0}{m_{\rm sens}}
      +\frac1{m_{\rm sens}}\sum_j\vec N_0[j]^2,
      &\vec S=\vec S_0,\\[6pt]
    Y+\frac{\vec s_{\rm resp}^{\top}\vec N+
             Y\vec s_{\rm resp}^{\top}\vec N_0}{m_{\rm sens}}
      +\frac1{m_{\rm sens}}\sum_j\vec N_0[j]\vec N[j],
      &\text{fresh }\vec S.
    \end{cases}
\end{equation}
For fresh sensor readings, each product
\(\vec N_0[j]\vec N[j]\) in \eqref{eq:sensor-product-expansion}
has mean zero because the earlier and fresh sensor noises are
independent and have zero mean.
On the reference readings, that product
becomes \(\vec N_0[j]^2\), whose mean is \(\sigma^2\).
The terms linear in the noises have mean zero in both cases.
Adding sensors therefore averages away the fresh noise products,
but not the noise squares produced by reusing the reference readings.%
\footnote{To compare the random reference readings with the deterministic point in \eqref{eq:one-sample-data}, split the earlier sensor noise \(\vec N_0\) along the clean response
\(\vec s_{\rm resp}\), giving
\(\vec S_0=(1+\bar N_0)\vec s_{\rm resp}+\vec N_{0,\perp}\),
where
\(\bar N_0=\vec s_{\rm resp}^{\top}\vec N_0/m_{\rm sens}
\sim\calN(0,\sigma^2/m_{\rm sens})\) and
\(\vec s_{\rm resp}^{\top}\vec N_{0,\perp}=0\).
At fixed \(\sigma^2\), the extra projection \(\bar N_0\) vanishes
as sensors are added.
It is the earlier record's counterpart of
the averaged sensor noise
\(\bar N=\vec s_{\rm resp}^{\top}\vec N/m_{\rm sens}\)
that produces the irreducible uncertainty in
\eqref{eq:sensor-conditional}.
For the fresh random physical quantity \(Y\),
\(y_{\rm test}=\sqrt{(\gamma-1)/\gamma}\,(Y+\bar N)\)
and
\(\xi=[Y-(\gamma-1)\bar N]/\sqrt\gamma\).
The reference quantity, by contrast, is known to equal one.
From \eqref{eq:sensor-product-expansion}, independence gives
\[
\begin{aligned}
 \E\!\left[\left(\frac{\vec S_0^{\top}\vec S}{m_{\rm sens}}-Y\right)^2\right]
   &=\frac{2\sigma^2+\sigma^4}{m_{\rm sens}},\\
 \E\!\left[\left(\frac{\|\vec S_0\|^2}{m_{\rm sens}}-(1+\sigma^2)\right)^2\right]
   &=\frac{4\sigma^2+2\sigma^4}{m_{\rm sens}}.
\end{aligned}
\]
Both mean squared differences tend to zero for fixed \(\sigma^2\),
and more generally when \(\sigma^2=o(\sqrt{m_{\rm sens}})\).}

At fixed per-sensor noise variance, the fresh prediction
\eqref{eq:sensor-product-average} approaches \(Y\) in mean square.
Applied to the reference readings, it instead approaches
\(1+\sigma^2\), rather than the known value one.
Increasing the feature dimension by adding equally noisy sensors therefore
makes the expected squared error on
fresh physical quantities tend to zero, while the expected squared
error on the reference measurement approaches the positive constant
\(\sigma^4\).
If \(\sigma^2>1\), that constant is greater than the
zero predictor's squared error one on the reference measurement.

For this same rule \eqref{eq:sensor-product-average} to make the
error on the reference measurement
vanish as well as the fresh error, the per-sensor noise variance
would have to tend to zero.
In the original sensor units, the interpolating learning rate \(1/\|\vec S_0\|^2\) instead gives the physical prediction \(\vec S_0^{\top}\vec S/\|\vec S_0\|^2\).
This rule predicts exactly one on the reference readings, so its training error is zero.
At fixed positive per-sensor noise variance, its fresh prediction tends in mean square to \(Y/(1+\sigma^2)\) as sensors are added.%
\footnote{For \(m_{\rm sens}>4\), let \(\ell_0=\|\vec S_0\|^2\) and \(a_{\rm sens}=\vec S_0^{\top}\vec s_{\rm resp}/\ell_0\).
Conditioning on \(\vec S_0\) and averaging gives
\[
\E\!\left[\left(\frac{\vec S_0^{\top}\vec S}{\ell_0}-\frac{Y}{1+\sigma^2}\right)^2\right]
=\E\!\left[\left(a_{\rm sens}-\frac{1}{1+\sigma^2}\right)^2\right]+\sigma^2\E[\ell_0^{-1}].
\]
Since \(\vec S_0^{\top}\vec s_{\rm resp}/m_{\rm sens}\to1\) and \(\ell_0/m_{\rm sens}\to1+\sigma^2\) in mean square, \(a_{\rm sens}\to1/(1+\sigma^2)\) in probability.
The noncentral chi-square law \(\ell_0/\sigma^2\sim\chi_{m_{\rm sens}}^2(m_{\rm sens}/\sigma^2)\) is a Poisson mixture of central chi-square laws with at least \(m_{\rm sens}\) degrees of freedom.
Cauchy--Schwarz and their inverse moments therefore give
\[
\begin{aligned}
\E[a_{\rm sens}^4]
   &\le m_{\rm sens}^2\E[\ell_0^{-2}]
   \le\frac{m_{\rm sens}^2}
           {\sigma^4(m_{\rm sens}-2)(m_{\rm sens}-4)},\\
\sigma^2\E[\ell_0^{-1}]
   &\le\frac1{m_{\rm sens}-2}.
\end{aligned}
\]
The first bound makes the family of squared coefficients \(a_{\rm sens}^2\), indexed by the sensor count \(m_{\rm sens}>4\), uniformly integrable, so \(\E[(a_{\rm sens}-1/(1+\sigma^2))^2]\to0\), and the second bound removes the fresh-noise term.
This proves the claimed mean-square limit.}
The deterministic interpolating rate \(\eta_{\rm interp}\) in \eqref{eq:one-sample-rates} has the same limiting signal attenuation.

By contrast, if we already knew the individual sensor responses
\(\vec s_{\rm resp}\), we could estimate the physical quantity
\(Y\) from fresh sensor readings \(\vec S\) using
\(\vec s_{\rm resp}^{\top}\vec S/m_{\rm sens}\), with mean squared
error
\(\E[(\vec s_{\rm resp}^{\top}\vec S/m_{\rm sens}-Y)^2]
 =\sigma^2/m_{\rm sens}\).
This vanishes under the condition \(\sigma^2\ll m_{\rm sens}\).
The additional noise product
\(\vec N_0[j]\vec N[j]\) in \eqref{eq:sensor-product-expansion}
instead contributes \(\sigma^4/m_{\rm sens}\) to the fresh error.
The rule based on one noisy record therefore needs the stronger
condition \(\sigma^2\ll\sqrt{m_{\rm sens}}\) for its fresh error
to vanish.

We can even allow each sensor to become noisier as the sensor
count grows.
How far can this go while fresh prediction improves
and the error on the reference measurement grows?
Consider the window
\begin{equation}\label{eq:sensor-one-sample-window}
    1\ll\sigma^2\ll\sqrt{m_{\rm sens}}.
\end{equation}
The condition \(\sigma^2\gg1\) makes the average noise-square
contribution on the reference readings dominate the noiseless value
one.
The condition \(\sigma^2\ll\sqrt{m_{\rm sens}}\) still lets the
independent noise products on fresh readings average to zero.%
\footnote{For the sensor readings of Remark~\ref{rem:sensor-model},
\(d=m_{\rm sens}-1\) and
\(\gamma=1+m_{\rm sens}/\sigma^2\) by \eqref{eq:sensor-spike-law}.
The few-shot window \(\gamma\ll d\ll\gamma^2\), obtained by setting
\(n=1\) in \eqref{eq:benign-misfitting-window}, reads
\eqref{eq:sensor-one-sample-window}.
Substituting into \eqref{eq:one-sample-opt-values} gives
\(\Ttest(\eta_{\rm opt})\to0\) and
\(\Ttrain(\eta_{\rm opt})\to\sigma^4\) at fixed \(\sigma^2\).
For comparison, the interpolating rate gives
\(\Ttest(\eta_{\rm interp})
  =\tfrac{m_{\rm sens}-1}{m_{\rm sens}}
   \bigl(\sigma^2/(1+\sigma^2)\bigr)^2\).}
Along the window in \eqref{eq:sensor-one-sample-window}, the
expected squared error on fresh physical quantities tends to zero,
while the expected squared error on the reference measurement grows
without bound.

\subsection{Could an arbitrary linear predictor generically do better than
span-restricted ones?
An intuitive argument for why not.}

\begin{figure}[!t]
    \centering
    \input{BenignMisfittingStylized_figures/fig_one_sample_compatible_worlds.tikz}
    \caption[Why an unknown spike makes every off-span move a bet]{Why an
    unknown spike makes every off-span move a bet.
    The illustration uses a
    convenient Pythagorean triple
    \((\sqrt{\gamma},\sqrt d,\|\vec x\|_2)=(9,40,41)\).
    Here \(d=1600\),
    so the ambient feature space is \(\mathbb R^{1601}\), and the
    hypersphere represented schematically by the purple circle in Panel~B is
    \(1599\)-dimensional.
    Panel~A decomposes the single observed training vector into orthogonal
    spike and nuisance components in two mirrored ways.
    Panel~B illustrates
    the entire hypersphere of spike directions represented among hidden
    worlds that give the same observed vector and label by a purple circle
    slice.
    Panel~C
    varies a signed off-span component \(\zeta\) around two different span
    predictors: the minimum-norm interpolator in~(i) on the left and the
    test-error-minimizing span predictor in~(ii) on the right.
    The prediction
    on the single training vector stays fixed as \(\zeta\) changes.
    Positive and negative values of \(\zeta\) favor opposite mirror
    worlds.
    The green uniform pair average is minimized at the balanced
    choice \(\zeta=0\), which
    removes the off-span component.}
    \label{fig:rotspan-compatible-worlds}
\end{figure}

Suppose that, a priori, we know nothing about the true spike direction.
All we have is one labeled sample.
The geometry of
\eqref{eq:one-sample-data} says the observed vector \(\vec x\) is the sum
of a spike component \(\sqrt{\gamma}\,\vec\theta\) and a nuisance
component \(\sqrt d\,\vec e_\perp\), with \(\vec\theta\) and
\(\vec e_\perp\) orthogonal unit vectors.
In right-triangle terms, we
hold the exact hypotenuse, and the two legs are hidden.
Whatever the
true spike direction is, its mirror image across the line defined by the observed vector
completes a second right triangle with the same hypotenuse.
Panel~A of
Figure~\ref{fig:rotspan-compatible-worlds} draws both.
The two worlds,
labeled \(+\) and \(-\), generate exactly the same observed vector and
label.\footnote{The representation in \eqref{eq:one-sample-data} writes
the decomposition in an aligned basis, but nothing in the observation
itself reveals that basis.
The ambiguity is also not an artifact of the
deterministic design.
Even if \(\vec x\) had been drawn from the spiked
Gaussian test model behind \eqref{eq:one-sample-test-error}, both worlds
would assign the observed vector the same probability density, because in
each the spike coordinate of \(\vec x\) is \(\sqrt{\gamma}\), the
orthogonal remainder has squared length \(d\), and the density is
isotropic around the spike.}
Nothing in the data can say which triangle is real.%
\footnote{\label{fn:one-sample-sensor-reflection}In the sensor backstory of Subsection~\ref{subsec:sensor-misfitting}, a physical reference quantity known to equal one produces the single observed vector of readings \(\vec S_0=\vec s_{\rm resp}+\vec N_0\), with \(\|\vec s_{\rm resp}\|^2=m_{\rm sens}\) known but the individual sensor responses unknown.
For \(\vec S_0\ne\vec0\), reflecting \(\vec s_{\rm resp}\) across the line defined by \(\vec S_0\) gives an allowed response with the same squared norm and the same required noise length \(\|\vec S_0-\vec s_{\rm resp}\|\).
The two candidates therefore give \(\vec S_0\) the same likelihood under independent Gaussian noise with mean zero and common variance \(\sigma^2>0\).
Neither this observation nor a prior on \(\vec s_{\rm resp}\) unchanged by rotations and reflections favors one candidate over the other.}

An off-span component of a linear predictor has zero inner product with
\(\vec x\), so it changes no training prediction and no training
error.\footnote{The minimum-norm interpolator discards its off-span
component for the same reason.
That component changes no fitted value
and can only add length.
The mirror argument in the text says that,
under symmetric ignorance, an off-span component likewise can only add
average test risk.}
Starting from any span predictor, an off-span move can improve test risk
only by betting on one triangle.
Panel~C varies such a signed bet
\(\zeta\) around two different span predictors.
Whenever the bet lowers
test risk in the world it favors, the resulting test risk increase in the mirror
world is larger.
In both subpanels of Panel~C, the green curve, the average of the two worlds' test risks, is minimized at the balanced choice \(\zeta=0\), which is the span predictor.

Since nothing we know favors one triangle over its mirror, we model the two
worlds as equally plausible, and on average every off-span bet
loses.\footnote{The willingness to average is not essential.
Let an
adversary choose which triangle is real after seeing the predictor.
Because the worse of the two risks is at least their average, every
off-span bet also loses against the adversary.}
Therefore, the train--test Pareto frontier
for global linear predictors is attained by span predictors when we have no
a priori information about the true spike direction.
Appendix~\ref{app:latent-rotation-span-complete} makes
this reasoning exact and proves the span-completeness statement for every
sample size the model allows.
Consequently, throughout the rest of the
paper we restrict attention to span predictors when studying the
fundamental train--test tradeoff among global linear predictors.
\FloatBarrier

\section{The deterministic orthogonal frontier with many training samples}
\label{sec:setup-frontier}

\subsection{Setup}
\label{subsec:setup}

Throughout the deterministic orthogonal model, assume \(n\le d\).
Unless
otherwise stated, fourth-quadrant claims also assume \(\gamma>1\).
For any vector
\(\vec{z}\in\R^{d+1}\), write \(z[1]\) for its spike coordinate and \(\vec{z}_{-1}\) for
its remaining nuisance coordinates:
\begin{equation}\label{eq:coordinate-convention}
    \vec{z}=
    \begin{bmatrix}
        z[1]\\
        \vec{z}_{-1}
    \end{bmatrix}.
\end{equation}

Choose deterministic vectors \(\vec{v}_1,\ldots,\vec{v}_n\in\R^d\) satisfying
\begin{equation}\label{eq:orthogonal-vectors}
    \vec{v}_i^{\top}\vec{v}_j=0 \quad (i\ne j),
    \qquad
    \|\vec{v}_i\|_2^2=d.
\end{equation}
Such a family exists when \(n\le d\).
Define
\begin{equation}\label{eq:training-data}
    \vec{x}_i=
    \begin{bmatrix}
        \sqrt{\gamma}\\ \vec{v}_i
    \end{bmatrix}
    \in\R^{d+1},
    \qquad
    \Xt=[\vec{x}_1,\ldots,\vec{x}_n].
\end{equation}
Every training label is
\begin{equation}\label{eq:training-labels}
    y_i=\frac{x_i[1]}{\sqrt{\gamma}}=1.
\end{equation}
A fresh test point has law%
\footnote{With a non-flat nuisance test covariance --- trace normalized to
\(d\), covariance dominated by the spike strength, and the training nuisance
directions mutually orthogonal with equal norms --- the interpolation
threshold remains a consequence of the Euclidean training geometry alone,
while the first-useful-span threshold becomes spectrum-dependent under a
weighted-isotropy hypothesis on the nuisance Gram; see the discussion at
Section~\ref{par:nonflat-tails}.}
\begin{equation}\label{eq:test-law}
    \bm{\Sigma}:=\diag(\gamma,1,\ldots,1),
    \qquad
    \vec{x}_{\rm test}\sim\calN(\vec{0},\bm{\Sigma}),
    \qquad
    y_{\rm test}=x_{\rm test}[1]/\sqrt{\gamma}.
\end{equation}
All feature coordinates are measured in units of the nuisance standard
deviation.
Thus \(\gamma\) is the spike-to-nuisance variance ratio.%
\footnote{Starting from a raw spike variance \(\sigma_{\rm sp}^2\) and a raw
nuisance variance \(\sigma_{\rm nu}^2\), rescale the raw input \(\vec x\) by
\(\frac{1}{\sigma_{\rm nu}}\) to obtain
\(\widetilde{\vec x}:=\vec x/\sigma_{\rm nu}\).
The covariance of
\(\widetilde{\vec x}\) becomes
\(\diag(\sigma_{\rm sp}^2/\sigma_{\rm nu}^2,1,\ldots,1)\), so
\[
    \gamma=\frac{\sigma_{\rm sp}^2}{\sigma_{\rm nu}^2},
    \qquad
    \frac{x[1]}{\sigma_{\rm sp}}
    =
    \frac{\widetilde x[1]}{\sqrt\gamma}.
\]
Thus the rescaling leaves the normalized target unchanged.}

Recall the noisy sensor backstory in Remark~\ref{rem:sensor-model}.
Changing \(\gamma\) changes how good the measurements are, not what is being
measured.
The factor \(1/\sqrt{\gamma}\) in \eqref{eq:test-law} is what keeps
the label on the same scale as \(Y\).
The spike feature grows like
\(\sqrt{\gamma}\), and \(Y\) does not.

The deterministic design is an exact high-dimensional caricature of the test
law: independent Gaussian nuisance vectors are nearly orthogonal and have
norms near \(\sqrt d\), whereas here both properties hold exactly.
We also
freeze the signal coordinate at one representative nonzero value.%
\footnote{An intercept or bias term could fit the constant training signal,
but it would fail on the test distribution, where the signal coordinate is
random and the target has mean zero.}
Thus this paper's model is a fixed-design geometric example: the training
inputs are deterministic, while the test point is random.
It is not an
i.i.d.\ statistical model.

For \(\vec{\alpha}\in\R^n\), write \(\vec{u}=\Xt\vec{\alpha}\), and define for any such span predictor $\vec{u}$:
\begin{equation}\label{eq:risk-defs}
    \Ttrain(\vec{u})=\frac1n\sum_{i=1}^n(\vec{u}^{\top}\vec{x}_i-1)^2,
    \qquad
    \Ttest(\vec{u})=\E[(\vec{u}^{\top}\vec{x}_{\rm test}-y_{\rm test})^2].
\end{equation}
We call \(\Ttrain\) and \(\Ttest\) the training and test errors.
Throughout
the paper, ``error'' and ``risk'' are otherwise used interchangeably for
these squared-loss quantities; fixed phrases such as ``empirical-risk
minimization'' retain their standard names.
The zero predictor has \((\Ttrain,\Ttest)=(1,1)\), since each training label
is \(1\) and \(\E[y_{\rm test}^2]=1\).

Define the coefficient sum \(s\) and coefficient energy \(q\) by
\begin{equation}\label{eq:sq-defs}
    s:=\sum_{i=1}^n\alpha_i,
    \qquad
    q:=\sum_{i=1}^n\alpha_i^2.
\end{equation}
The coefficient sum determines the predictor's spike coordinate \(u[1] = \sqrt\gamma\,s\), while
\(dq\) turns out to be the nuisance variance.

Orthogonality also fixes what a span predictor does on a training point.
By
\eqref{eq:orthogonal-vectors} and \eqref{eq:training-data} the training vectors satisfy
\(\vec{x}_j^{\top}\vec{x}_i=\gamma\) for \(j\ne i\) and \(\vec{x}_i^{\top}\vec{x}_i=\gamma+d\),
so
\begin{equation}\label{eq:train-prediction}
    \vec{u}^{\top}\vec{x}_i=\gamma s+d\,\alpha_i
    \qquad(1\le i\le n).
\end{equation}
Every training point sees the same shared term \(\gamma s\) of \eqref{eq:train-prediction}.
What distinguishes one training point from another is the second term of
\eqref{eq:train-prediction}, that point's own coefficient $\alpha_i$ amplified by \(d\).
This
single identity produces the training error in Lemma~\ref{lem:two-scalar}.
It also produces both
the per-example residuals and the reuse-stability boundary discussed in Section~\ref{sec:sgd}.

The picture to keep in mind is Figure~\ref{fig:block-construction}: all
training samples share the same spike row, while their nuisance blocks are
disjoint.%
\footnote{Because the nuisance law in \eqref{eq:test-law} is rotation
invariant, \eqref{eq:orthogonal-vectors} lets the analyst choose nuisance
coordinates so that the row-oriented stylized training design is
\[
    \Xt^\top
    =
    [\,\sqrt{\gamma}(1,\ldots,1)^\top\quad
      \sqrt d\,\bm I_n\quad
      \bm0_{n\times(d-n)}\,].
\]
After dropping the unused zero columns, this is exactly the scaled-identity
augmented-feature construction in
\citet[supplementary material, Appendix~B,
Eqs.~(13)--(15)]{muthukumar2020harmless}, specialized to one retained original
feature.
Those equations connect
minimum-\(\ell_2\)-norm interpolation on the augmented features to positive
ridge regression on the original feature block.
This construction inspired
the four-sample picture in Figure~\ref{fig:block-construction}.
The correspondence is confined to the training design.
The auxiliary ridge
features are zero at test time, whereas the nuisance coordinates here remain
random under \eqref{eq:test-law}.
For the related asymptotic construction
using random added features, see
\citet[Theorem~1]{nr_kobak2020}.}
When \(d\) is divisible by \(n\), a concrete realization partitions the
nuisance coordinates into \(n\) equal blocks and puts \(\vec{v}_i\) in the
\(i\)th block.
The four-sample construction in Figure~\ref{fig:block-construction} also
shows why these two scalars $s$ and $q$ are the only quantities that enter the errors.

\begin{figure}[t]
    \centering
    \resizebox{0.95\textwidth}{!}{\FigStylizedDesign}
    \caption{Block construction of the orthogonal design for four samples.
    The first row is the shared spike, while the nuisance coordinates are
    partitioned into orthogonal blocks.
    Orthogonality reduces the analysis of
    a span predictor \(\vec{u}=\Xt\vec{\alpha}\) to the coefficient sum \(s\) and
    coefficient energy \(q\), equivalently the spike coordinate
    \(\sqrt\gamma\,s\) and nuisance variance \(dq\).}
    \label{fig:block-construction}
\end{figure}

Define the spike oracle by
\begin{equation}\label{eq:oracle-def}
    \vec{w}^{\star}=(1/\sqrt{\gamma},0,\ldots,0)^\top.
\end{equation}
This oracle achieves zero train and test error, but it is not in the training span.\footnote{If
\(\Xt\vec{\alpha}=\vec{w}^{\star}\), the zeros in the nuisance coordinates and the 
linear independence of the \(\{\vec{v}_i\}_{i=1}^n\) would force
\(\alpha_i=0\) for every \(i\), while the spike coordinate requires
\(\sum_i\alpha_i=1/\gamma\), a contradiction.}


We compare span predictors through their train--test pairs
\((\Ttrain(\vec{u}),\Ttest(\vec{u}))\).
One pair \emph{dominates} another if it is no
larger in either coordinate and is strictly smaller in at least one.
A span
predictor is \emph{Pareto-optimal} if its pair is undominated.
The following
reduction identifies the Pareto-optimal predictors.

\begin{lemma}[Scalar reduction]
\label{lem:two-scalar}
For a generic span predictor \(\vec{u}=\Xt\vec{\alpha}\), using the notation $s = \sum_i \alpha_i$ and $q = \sum_i \alpha_i^2$ 
from \eqref{eq:sq-defs}, 
\begin{equation}\label{eq:scalar-test-error}
    \Ttest(\vec{u})=\underbrace{(\gamma s-1)^2}_{\text{spike bias}^2}+\underbrace{dq}_{\text{nuisance variance}}
\end{equation}
and
\begin{equation}\label{eq:scalar-train-error}
    \Ttrain(\vec{u})=(\gamma s-1)^2+
    \frac{2ds}{n}(\gamma s-1)+
    \frac{d^2}{n}q.
\end{equation}
For fixed \(s\), both errors are minimized by \(q=s^2/n\), attained exactly
when \(\alpha_1=\cdots=\alpha_n=s/n\).
Hence every Pareto-optimal span
predictor is symmetric.
\end{lemma}

\begin{proof}
The signal coordinate of \(\vec{u}=\sum_i\alpha_i\vec{x}_i\) is
\(\sum_i\alpha_i\sqrt{\gamma}=\sqrt{\gamma}s\), while orthogonality of the
\(\vec{v}_i\) gives \(\|\vec{u}_{-1}\|_2^2=\sum_i\alpha_i^2\|\vec{v}_i\|_2^2=dq\).
The test error decomposes into spike bias and
nuisance variance:
\begin{equation}\label{eq:test-error-decomposition}
    \Ttest(\vec{u})=(\sqrt{\gamma}s-1/\sqrt{\gamma})^2\gamma+\|\vec{u}_{-1}\|_2^2
    =(\gamma s-1)^2+dq.
\end{equation}
By \eqref{eq:train-prediction}, the residual on training point \(i\) is
\((\gamma s-1)+d\alpha_i\), and expanding the average of its square gives
\eqref{eq:scalar-train-error}.
Finally, Cauchy--Schwarz gives
\(q\ge s^2/n\), with equality for
equal coefficients.
If \(q>s^2/n\), replacing \(\vec{\alpha}\) by the
equal-coefficient vector with the same \(s\) strictly decreases both
\(\Ttrain\) and \(\Ttest\).
Hence no such point can be Pareto-optimal.
\end{proof}

\phantomsection\label{par:residue-amplification}%
\paragraph{Residue amplification: one extra power.}
In the raw coefficients, the nuisance quadratic enters the test error as
\(dq\) in \eqref{eq:test-error-decomposition}, while the pure nuisance quadratic in the training error enters as
\(d^2q/n\) in \eqref{eq:scalar-train-error}.
Thus training sees one extra factor\footnote{In 
the random same-distribution companion model \cite{ranade2026samedistribution},
\(d\bm{I}_n\) is replaced by the nuisance Gram matrix
\(\bm{W}=\Xt_{-1}^{\top}\Xt_{-1}\), where \(\Xt_{-1}\) denotes the
nuisance-coordinate rows of the data matrix.
The nuisance contribution to test
risk contains \(\vec{\alpha}^{\top}\bm{W}\vec{\alpha}\), while the pure nuisance quadratic in
empirical loss contains \(\vec{\alpha}^{\top}\bm{W}^2\vec{\alpha}/n\).} \(d/n\).%

\subsection{Exact Pareto frontier}
\label{subsec:frontier}

We now compute, for each level of training error, the smallest test error a
span predictor can achieve: the \emph{train--test tradeoff curve}, i.e.\ the set of undominated
train--test pairs \((t,g)\) in the sense above (here \(t=\Ttrain\),
\(g=\Ttest\)).

By Lemma~\ref{lem:two-scalar}, every Pareto-optimal predictor is symmetric,
so \(q=s^2/n\) and both errors are functions of the coefficient sum \(s\)
alone:
\begin{equation}\label{eq:symmetric-scalar-risks}
    \Ttest(s)=(\gamma s-1)^2+\frac{d}{n}s^2,
    \qquad
    \Ttrain(s)=\left(\left(\gamma+\tfrac{d}{n}\right)s-1\right)^2.
\end{equation}
On each training point the symmetric predictor has the common residual
\[
    \vec{u}^{\top}\vec{x}_i-1
    =
    \left(\gamma+\frac dn\right)s-1.
\]
Write this \emph{signed training residual} as
\begin{equation}\label{eq:frontier-h-def}
    h:=\left(\gamma+\tfrac{d}{n}\right)s-1,
\end{equation}
so that \(\Ttrain=h^2\).
Solving \eqref{eq:frontier-h-def} for \(s\) and
substituting into \(\Ttest\) in \eqref{eq:symmetric-scalar-risks} yields 
the \emph{signed-residual test-risk curve}
\begin{equation}\label{eq:frontier-signed-risk}
    \widetilde g(h)
    :=
    \frac{(d-\gamma n h)^2+n d\,(1+h)^2}{(d+\gamma n)^2}.
\end{equation}
For every training-error level \(t>0\), the equation \(h^2=t\) has the two
signed residuals \(h=\pm\sqrt t\).
Define the \emph{overshooting} and
\emph{undershooting} branches
\begin{equation}\label{eq:frontier-branch-defs}
    g_+(t):=\widetilde g(+\sqrt t),
    \qquad
    g_-(t):=\widetilde g(-\sqrt t).
\end{equation}
At \(t=0\), the two branches coincide at the interpolator.
Thus \(\widetilde g\)
is parameterized by the signed residual \(h\), whereas \(g_\pm\) are
parameterized by the training error \(t=h^2\).

\begin{theorem}[Exact Pareto frontier]
\label{thm:pareto}
Assume \(\gamma>1\).
The Pareto frontier over all span predictors is
\begin{equation}\label{eq:pareto-set}
    \{(t,g_+(t)):0\le t\le t^\star\},
\end{equation}
where
\begin{equation}\label{eq:frontier-endpoints}
    t^\star=\frac{(\gamma-1)^2 d^2}{(d+\gamma^2 n)^2},
    \qquad
    g_{\min}=\frac{d}{d+\gamma^2 n},
    \qquad
    g_{\rm int}=\frac{d\,(d+n)}{(d+\gamma n)^2},
\end{equation}
and
\begin{equation}\label{eq:frontier-endpoint-form}
    g_+(t)=g_{\min}+(g_{\rm int}-g_{\min})
    \left(1-\sqrt{t/t^\star}\right)^2.
\end{equation}
Thus the frontier runs from the minimum-norm interpolator
$(0,g_{\rm int})$ to the test-optimal span predictor
$(t^\star,g_{\min})$.
As $t$ increases from $0$ to $t^\star$,
$g_+(t)$ decreases from $g_{\rm int}$ to $g_{\min}$: along the efficient
branch, better test error is obtained by accepting more training error.%
\footnote{The multispike companion \citep{ranade2026multispike} asks how a fixed amount of
training misfit should be distributed among several informative directions.
The stronger directions improve first as more misfit is allowed.}
Equivalently, expanding the same curve gives
\begin{equation}\label{eq:frontier}
    g_+(t)=
    \frac{(d-\gamma n\sqrt t)^2+n d\,(1+\sqrt t)^2}{(d+\gamma n)^2},
    \qquad 0\le t\le t^\star.
\end{equation}
\end{theorem}

\begin{proof}
By Lemma~\ref{lem:two-scalar}, Pareto-optimal predictors are symmetric, so
the symmetric scalar risks \eqref{eq:symmetric-scalar-risks} apply.
Fix a
training error \(t>0\).
By \eqref{eq:frontier-h-def}, exactly two signed
residuals give this training error, \(h=\pm\sqrt t\), with test errors
\(g_\pm(t)\) from \eqref{eq:frontier-branch-defs}.
Their difference is
\begin{equation}\label{eq:frontier-branch-comparison}
    g_+(t)-g_-(t) = \widetilde{g}(+\sqrt t)-\widetilde{g}(-\sqrt t)
    =
    -\frac{4nd\,\sqrt t\,(\gamma-1)}{(d+\gamma n)^2}<0
\end{equation}
because \(\gamma>1\).
At every positive training error, the overshooting
branch has strictly smaller test error, so the undershooting branch is
dominated.
Expanding the numerator of \eqref{eq:frontier-signed-risk} as a
quadratic in the signed residual \(h\) gives
\[
    n(d+\gamma^2n)h^2
    -
    2nd(\gamma-1)h
    +
    d(d+n).
\]
Completing the square gives
\begin{equation}\label{eq:frontier-completed-square}
    \widetilde g(h)
    =
    \frac{d}{d+\gamma^2 n}
    +
    \frac{n\,(d+\gamma^2 n)}{(d+\gamma n)^2}
    (h-h^\star)^2,
    \qquad
    h^\star:=\frac{(\gamma-1)d}{d+\gamma^2 n}.
\end{equation}
Set \(t^\star:=(h^\star)^2\), which matches \eqref{eq:frontier-endpoints}.
On
\(0\le h\le h^\star\), \eqref{eq:frontier-completed-square} shows that
test error \(\widetilde g(h)\) decreases while training error \(t=h^2\) increases, so no two points
of this branch segment dominate one another.
For \(h>h^\star\), both
\(t=h^2\) and \(\widetilde g(h)\) increase, so every such point is
dominated by \((t^\star,g_{\min})\).
This identifies \(t^\star\),
\(g_{\min}=\widetilde g(h^\star)=d/(d+\gamma^2 n)\), and
\(g_{\rm int}=g_+(0)\).
Finally, the identity
\begin{equation}\label{eq:frontier-endpoint-gap}
    g_{\rm int}-g_{\min}
    =
    \frac{n\,(d+\gamma^2 n)}{(d+\gamma n)^2}\,t^\star,
\end{equation}
together with \eqref{eq:frontier-completed-square} at \(h=+\sqrt t\), shows
that the completed-square form \eqref{eq:frontier-completed-square} is exactly the endpoint form
\eqref{eq:frontier-endpoint-form}.
\end{proof}

\paragraph{Reading the frontier.}
The formulas become more transparent if we
define the sample ratio \(\rho\) and the signal-to-nuisance ratio \(R\) as
\begin{equation}\label{eq:scale-identities}
    \rho:=\frac{\gamma n}{d},
    \qquad
    R:=\frac{\gamma^2n}{d}=\gamma\rho.
\end{equation}
The endpoints \eqref{eq:frontier-endpoints} depend on \((d,n,\gamma)\) only
through \(\rho\) and \(R\):
\begin{equation}\label{eq:frontier-endpoints-rho-R}
    g_{\min}=\frac{1}{1+R},
    \qquad
    t^\star=\left(\frac{R-\rho}{\rho(1+R)}\right)^2,
    \qquad
    g_{\rm int}=\frac{R+\rho^2}{R(1+\rho)^2}.
\end{equation}
In particular, the minimum attainable test error $g_{\min}$ is governed by the single signal-to-nuisance ratio \(R\).

Now further define the \emph{spike calibration} \(r\) for a predictor \(\vec u\) as
\begin{equation}\label{eq:spike-calibration}
    r:=\gamma s = \sqrt\gamma u[1].
\end{equation}
The \(\vec{u}\) predictor's spike coordinate \(u[1]\) is \(r/\sqrt\gamma\), so \(r=1\) matches the
oracle's spike coordinate \(1/\sqrt\gamma\).
The bias term \((\gamma s-1)^2\)
in \eqref{eq:symmetric-scalar-risks} is therefore the squared calibration
error, and the two errors in \eqref{eq:symmetric-scalar-risks} take the form
\begin{equation}\label{eq:calibrated-frontier-form}
    \Ttest=(r-1)^2+\frac{r^2}{R},
    \qquad
    \Ttrain=\left(r\left(1+\frac1\rho\right)-1\right)^2.
\end{equation}

The frontier immediately turns the fourth-quadrant picture into a necessary
condition: in the benign-misfitting window, any span predictor with small test error
must pay large empirical error.

\begin{corollary}[Fourth-quadrant necessity]
\label{cor:fourth-quadrant-necessity}
Assume the benign-misfitting window \eqref{eq:benign-misfitting-window},
equivalently --- using the notation from \eqref{eq:scale-identities} --- the 
signal-to-nuisance ratio \(R=\gamma^{2}n/d\to\infty\) and the sample ratio
\(\rho=\gamma n/d\to0\).
If a 
family of span predictors
\(\vec{u}\) satisfies
\(\Ttest(\vec{u})\to0\), then \(\Ttrain(\vec{u})\to\infty\).
More generally, for any
fixed \(\varepsilon\in(0,1)\), any span predictor with
\(\Ttest(\vec{u})\le\varepsilon+o(1)\) must have
\begin{equation}\label{eq:fourth-quadrant-training-lower}
    \Ttrain(\vec{u})\ge
    \frac{(1-\sqrt{\varepsilon})^2+o(1)}{\rho^2}.
\end{equation}

In fact, for any \(\rho>0\) and any \(\varepsilon>0\), every span predictor
\(\vec{u}\) with \(\Ttest(\vec{u})\le\varepsilon\)
satisfies\footnote{Under the hypothesis \(\gamma>1\) of
Theorem~\ref{thm:pareto}, knowing the signal-to-nuisance ratio \(R\) of
\eqref{eq:scale-identities} as well allows us to sharpen \eqref{eq:fourth-quadrant-finite}.  
Minimizing \(\Ttrain\) subject to \(\Ttest\le\varepsilon\)
inverts the frontier \eqref{eq:frontier-endpoint-form}, and that optimization
analysis gives the attained value
\[
    \min_{\Ttest(\vec{u})\le\varepsilon}\Ttrain(\vec{u})
    =
    \frac{\max\Bigl\{0,R-\rho-(1+\rho)
          \sqrt{R\bigl((1+R)\varepsilon-1\bigr)}\Bigr\}^{2}}
         {\rho^{2}(1+R)^{2}}
\]
for every \(\varepsilon\ge g_{\min}=1/(1+R)\), below which no span predictor
meets the target at all.
The bound \eqref{eq:fourth-quadrant-finite} is this
value's \(R\to\infty\) limit at fixed \(\rho\), and is no larger at any finite
\(R\), the slack being the nuisance variance \(r^{2}/R\) of
\eqref{eq:calibrated-frontier-form} that the proof of Corollary~\ref{cor:fourth-quadrant-necessity} discards.
Both sides are zero once \(\varepsilon\ge g_{\rm int}\).}
\begin{equation}\label{eq:fourth-quadrant-finite}
    \Ttrain(\vec{u})
    \ge
    \frac{\max\bigl\{1-(1+\rho)\sqrt{\varepsilon},\,0\bigr\}^{2}}{\rho^{2}}.
\end{equation}
\end{corollary}

In words: once the span is useful but interpolation is still not useful, the
training points become the wrong diagnostic---a span predictor with small
fresh-test error must look bad on the very points used to construct it.
We call this the \emph{misfitting obstruction}.

\begin{proof}
The proof uses no optimization.
Small test error forces the spike calibration
\(r\) near one, and then the training residual is of size \(1/\rho\).

Write \(r=\gamma s=\sqrt{\gamma}\,u[1]\) for the spike calibration
\eqref{eq:spike-calibration}, with \(s=\sum_i\alpha_i\) and \(q=\sum_i\alpha_i^2\)
from \eqref{eq:sq-defs}.
From Lemma~\ref{lem:two-scalar},
\[
    \Ttest(\vec{u})=(r-1)^2+dq\ge (r-1)^2,
\]
because \(q\ge0\).
So \(\Ttest(\vec{u})\le\varepsilon\) gives
\begin{equation}\label{eq:calibration-from-test}
    r\ge 1-\sqrt{\varepsilon}.
\end{equation}

Lemma~\ref{lem:two-scalar} also gives \(q\ge s^2/n\).
Substituting that lower
bound into the training error
\(\Ttrain=(\gamma s-1)^{2}+\frac{2ds}{n}(\gamma s-1)+\frac{d^{2}}{n}q\)
of \eqref{eq:scalar-train-error}, and using
\(d/n=\gamma/\rho\) from \eqref{eq:scale-identities} together with
\(s=r/\gamma\) from \eqref{eq:spike-calibration},
\begin{equation}\label{eq:train-lower-in-r}
    \Ttrain(\vec{u})
    \ge
    (r-1)^2+\frac{2r(r-1)}{\rho}+\frac{r^2}{\rho^2}
    =
    \left(r\left(1+\frac1\rho\right)-1\right)^2.
\end{equation}

Fix \(\rho>0\) and \(\varepsilon>0\).
If \((1+\rho)\sqrt{\varepsilon}\ge1\), the
right side of \eqref{eq:fourth-quadrant-finite} is zero and there is nothing to
prove.
Otherwise \eqref{eq:calibration-from-test} gives
\[
    r\left(1+\frac1\rho\right)-1
    \ \ge\
    (1-\sqrt{\varepsilon})\left(1+\frac1\rho\right)-1
    =
    \frac{1-(1+\rho)\sqrt{\varepsilon}}{\rho}
    \ >\ 0.
\]
Both sides are positive, so squaring preserves the inequality, and
\eqref{eq:train-lower-in-r} gives \eqref{eq:fourth-quadrant-finite}.

Now consider what happens when \(\rho\to0\).
For the fixed-\(\varepsilon\)
claim, a family with \(\Ttest(\vec{u})\le\varepsilon+o(1)\) satisfies
\eqref{eq:fourth-quadrant-finite} at target \(\varepsilon+o(1)\).
Because
\(\varepsilon\) is fixed in \((0,1)\), the numerator of
\eqref{eq:fourth-quadrant-finite} at that target is
\(\max\bigl\{1-(1+\rho)\sqrt{\varepsilon}+o(1),\,0\bigr\}^{2}\).
As
\(\rho\to0\), the max operation in \eqref{eq:fourth-quadrant-finite} becomes
inactive since \(1-\sqrt{\varepsilon}>0\).
This immediately gives
\eqref{eq:fourth-quadrant-training-lower}.

For the vanishing-test-error claim, a family with \(\Ttest(\vec{u})\to0\)
eventually has \(\Ttest(\vec{u})\le1/4\).
Applying
\eqref{eq:fourth-quadrant-finite} at \(\varepsilon=1/4\) gives
\[
    \Ttrain(\vec{u})
    \ge
    \frac{\max\bigl\{1-(1+\rho)/2,\,0\bigr\}^{2}}{\rho^2}
    =
    \frac{\max\{1-\rho,\,0\}^{2}}{4\rho^2}
    \to\infty,
\]
using \(\rho\to0\).
Letting \(\varepsilon\downarrow0\) in
\eqref{eq:fourth-quadrant-finite} sharpens the leading constant to one.
\end{proof}

\medskip
\noindent\fbox{%
\begin{minipage}{0.94\textwidth}
By the calibrated risks \eqref{eq:calibrated-frontier-form}, calibration
\(r\approx1\) is what test prediction wants.
But inside the span, this
calibration makes the nuisance residue contribute a term of size \(1/\rho\) on
each training point, hence squared training error of order \(1/\rho^2\).
The
same residue contributes only \(1/R\) to fresh-test risk---smaller by
the residue-amplification factor \(d/n=\gamma/\rho\).
\end{minipage}}
\medskip

In the window
\begin{equation}\label{eq:misfitting-window}
    \frac{1}{\gamma}\ll\rho\ll 1
    \qquad\Longleftrightarrow\qquad
    \gamma n\ll d\ll\gamma^2n,
\end{equation}
the endpoints in \eqref{eq:frontier-endpoint-form} pull apart:
 \(g_{\rm int}\to1\) (the interpolator becomes useless), \(g_{\min}\to0\) 
 (the best span predictor becomes accurate), and
\(t^\star\to\infty\) (at the cost of a diverging training error).
Here
\(\rho=n/(d/\gamma)\) measures the sample size against the interpolation
threshold \(d/\gamma\), while \(R=n/(d/\gamma^2)\) measures it against the
first-useful-span threshold \(d/\gamma^2\).
The span contains good test
predictors, but they live above the zero-predictor training-error baseline.

\paragraph{Survival and contamination.}
The calibration--residue decomposition above is the deterministic one-spike
specialization of the survival--contamination framework of
\citet{muthukumar2020harmless,muthukumar2021classification}.
For a span
predictor \(\vec{u}\), define
\begin{equation}\label{eq:survival-contamination}
    \mathrm{SU}(\vec{u}):=r = \sqrt{\gamma} u[1],
    \qquad
    \mathrm{CN}(\vec{u}):=\sqrt{dq} = \|\vec{u}_{-1}\|_2.
\end{equation}
The spike calibration \(r\) is the \emph{signal-survival factor}, whose
squared-error target is \(r=1\).
The \emph{fresh-test contamination}
\(\sqrt{dq}\) is the standard deviation of the nuisance prediction on a
fresh test point.
The test risk decomposes as
\begin{equation}\label{eq:survival-risk-decomposition}
    \Ttest(\vec{u})
    =
    \bigl(1-\mathrm{SU}(\vec{u})\bigr)^2
    +
    \mathrm{CN}(\vec{u})^2.
\end{equation}
For a symmetric predictor at positive calibration \(r\),
\begin{equation}\label{eq:symmetric-survival-contamination}
    \mathrm{CN}(r)
    =
    \frac{r}{\sqrt{R}},
    \qquad
    \frac{\mathrm{CN}(r)}{\mathrm{SU}(r)}
    =
    \frac{1}{\sqrt{R}}.
\end{equation}
Thus the frontier calculation can be read either as calibration plus
nuisance residue or as survival plus contamination.

At the test-optimal span predictor, minimizing the calibrated test risk
\eqref{eq:calibrated-frontier-form} over \(r\) gives the optimal calibration
\[
    r^\star
    =
    \frac{R}{1+R}.
\]
Therefore, by \eqref{eq:survival-contamination} and
\eqref{eq:symmetric-survival-contamination},
\begin{equation}
\label{eq:best-span-survival-contamination}
    \mathrm{SU}_{\rm best}
    =
    \frac{R}{1+R},
    \qquad
    \mathrm{CN}_{\rm best}
    =
    \frac{\sqrt R}{1+R}.
\end{equation}
Once the sample count enters the benign-misfitting window \(R\to\infty\),
\[
    \bigl(
        \mathrm{SU}_{\rm best},
        \mathrm{CN}_{\rm best}
    \bigr)
    \longrightarrow
    (1,0).
\]
Thus the best span predictor succeeds in exactly the regression sense
identified by the survival--contamination framework: the shared signal
survives with the correct calibration while fresh-test contamination
vanishes.

\subsection{Scaled interpolation and negative ridge}
\label{subsec:scaled-interpolation}

By Lemma \ref{lem:two-scalar}, every Pareto-optimal predictor is a scalar multiple of the minimum-norm
interpolator.
Let \(\one\in\R^n\) denote the all-ones vector.
Because
\(\one\one^\top\one=n\one\), the all-ones vector is an eigenvector of
\(\Xt^\top\Xt=d\bm{I}_n+\gamma\one\one^\top\) with eigenvalue \(d+\gamma n\).
Hence
\begin{equation}\label{eq:gram-inverse}
    \Xt^\top\Xt=d\bm{I}_n+\gamma\one\one^\top,
    \qquad
    (\Xt^\top\Xt)^{-1}\one=\frac{1}{d+\gamma n}\one.
\end{equation}
Therefore the minimum-norm interpolator is
\begin{equation}\label{eq:interpolator-def}
    \vec{u}_{\rm int}=\Xt\vec{\alpha}_{\rm int},
    \qquad
    \vec{\alpha}_{\rm int}=\frac{1}{d+\gamma n}\one.
\end{equation}

Notice that the interpolator's spike calibration is
\begin{equation}\label{eq:rint-def}
    r_{\rm int}= \sqrt{\gamma} u_{\rm int}[1]=\frac{\gamma n}{d+\gamma n}=\frac{\rho}{1+\rho} < 1.
\end{equation}

For \(c\in\R\), define the scaled interpolator
\[
    \vec{u}_c:=c\,\vec{u}_{\rm int}.
\]
The interpolator predicts \(1\) on every training point, so \(\vec{u}_c\)
predicts \(c\).
Hence
\begin{equation}\label{eq:scaled-train-error}
    \Ttrain(\vec{u}_c)=(c-1)^2.
\end{equation}
Equation~\eqref{eq:scaled-train-error} exposes the asymmetry behind the fourth
quadrant.
Ordinary shrinkage, $0\le c\le1$, can only move the training error
up to the zero-predictor baseline.
Restoring spike calibration instead
requires \emph{anti-shrinkage}, \(c>1\), and anti-shrinkage can make
\(\Ttrain(\vec{u}_c)\) arbitrarily large.
Crossing the zero-predictor
baseline \(\Ttrain=1\) takes \(c>2\), and the test-optimal scaling
\(c^\star\) of \eqref{eq:cstar-def} in the benign-misfitting window lands
well past \(2\), pushing predictors into the fourth quadrant.%
\footnote{\citet{freeman2026shrinkage} calls a related random-design
phenomenon the \emph{Inflation Property}: with \(d/n\to\infty\), a scalar
\(c>1\) multiplying the minimum-norm interpolator can improve
generalization.
Subsection~\ref{subsec:related-span-rescaling} places this
next to James--Stein shrinkage and negative ridge.}
  In the ridge-regression perspective of
\eqref{eq:ridge-scale}--\eqref{eq:lambda-star}
below, the same anti-shrinkage move appears as a negative ridge penalty.%
\footnote{Appendix~\ref{app:sco-comparison} gives a norm-budget version of
the same point.
In the benign-misfitting window, the minimum-norm
interpolator remains well inside the oracle norm ball, while any span predictor
with small test error must use a larger Euclidean norm budget.}

Every Pareto-optimal predictor lies on the scaled-interpolator ray
\[
    \{\,c\,\vec{u}_{\rm int}:c\ge0\,\}.
\]
The scalar \(c\) is both the ray multiplier and the common training
prediction.
Its relations to the signed residual \(h\) and calibration \(r\)
are
\begin{equation}\label{eq:c-r-relation}
    h=c-1,
    \qquad
    r=c\,r_{\rm int}.
\end{equation}
The test-optimal scalar is
\begin{equation}\label{eq:cstar-def}
    c^\star=1+h^\star=\frac{\gamma(1+\rho)}{1+R},
\end{equation}
and \(c^\star-1=(\gamma-1)/(1+R)\).

In these coordinates, the frontier of Theorem~\ref{thm:pareto} is the segment
\(1\le c\le c^\star\) of the ray.
Ordinary shrinkage moves below the
interpolator, \(c<1\), onto the dominated branch \(g_{-}\).

\begin{remark}[Why the strong-signal condition matters]
\label{rem:strong-spike}
The branch comparison \eqref{eq:frontier-branch-comparison} above is the
strong-signal case \(\gamma>1\), where \(c^\star>1\) and the tuned ridge
penalty \eqref{eq:lambda-star} below is negative.
When \(\gamma<1\), the
branch ordering reverses.
Then \(c^\star<1\) and \(\lambda^\star>0\), so
ordinary positive ridge shrinkage, not overscaling past interpolation, is
the useful direction.
At \(\gamma=1\) the branches tie, \(c^\star=1\), and
\(\lambda^\star=0\).
The interpolator itself is then already test-optimal,
and neither shrinkage nor overscaling helps.
Thus the necessary
fourth-quadrant phenomenon studied here---where the useful Pareto branch lies
beyond interpolation---is a strong-signal phenomenon.
\end{remark}

Kernel ridge regression%
\footnote{Strictly speaking, we might more accurately call this a
span-restricted algebraic continuation of kernel ridge regression: for a
negative penalty, the representer theorem no longer justifies restricting to
the span, and the unrestricted ambient objective would be unbounded along
directions orthogonal to the span.
But the meaning here should be amply
clear.}
is the kernel, or dual, form of ridge regression: the
representer theorem writes its predictor in terms of coefficients on the
training examples \cite{schoelkopf2001generalized}.
With penalty
\(\lambda\), it gives the following coordinate on the same scalar path.
With
the dual convention of using \(\vec{\alpha}\) for the weights on the training
examples, the problem to be solved is
\begin{equation}\label{eq:kernel-ridge-objective}
    \min_{\vec{\alpha}} \frac1n\|\bm{K}\vec{\alpha}-\one\|_2^2
    +\lambda\,\vec{\alpha}^{\top}\bm{K}\vec{\alpha},
    \qquad
    \bm{K}=\Xt^\top\Xt.
\end{equation}
The solution%
\footnote{This inverse gives the unique minimizer exactly when
\(\bm K/n+\lambda\bm I_n\succ0\).
Once \(\bm K/n+\lambda\bm I_n\) has a
negative eigenvalue, \eqref{eq:kernel-ridge-objective} is unbounded below.
Every \(\lambda>-d/n\) suffices, and for \(n\ge2\) this threshold is exact.
The tuned penalty \eqref{eq:lambda-star} sits safely inside, with margin
\(d/(n\gamma)\).}
of \eqref{eq:kernel-ridge-objective} is
\[
    \vec{\alpha}
    =
    (\bm{K}+n\lambda\bm{I}_n)^{-1}\one.
\]
The vector \(\one\) is an eigenvector of \(\bm{K}\) with eigenvalue
\(d+\gamma n\).
Therefore all coefficients are equal:
\[
    \alpha_i
    =
    \frac{1}{d+\gamma n+n\lambda}.
\]
The common training prediction is
\begin{equation}\label{eq:ridge-scale}
    c(\lambda)=\frac{d+\gamma n}{d+n\lambda+\gamma n}.
\end{equation}
Solving \(c(\lambda)=c^\star\) gives
\begin{equation}\label{eq:lambda-star}
    \lambda^\star=-\frac dn\left(1-\frac1\gamma\right)<0,
    \qquad\text{using the standing }\gamma>1\text{ (Remark~\ref{rem:strong-spike}).}
\end{equation}
Notice that at \(\lambda^\star\), the kernel-space quadratic remains strictly convex.%
\footnote{For \(n=1\) the eigenvalue \(d/n\) has multiplicity zero.
The
kernel is then the single number \(\gamma+d\), and the shift leaves
\(\gamma+d/\gamma>0\), so strict convexity holds for every \(n\ge1\).}
For \(n\ge2\), the scaled kernel \((1/n)\Xt^\top\Xt\) has eigenvalues \(d/n\), with
multiplicity \(n-1\), and \(\gamma+d/n\), so adding \(\lambda^\star\bm I_n\)
leaves smallest eigenvalue
\[
    \frac dn-\frac dn\left(1-\frac1\gamma\right)
    =
    \frac{d}{n\gamma}>0.
\]
Thus this anti-shrinkage (and, in the benign-misfitting window
\eqref{eq:benign-misfitting-window}, the resulting large training error) is not an artifact
of an ill-conditioned feature matrix the way that catastrophically large test error for traditional
under-regularized least-squares is.
Here, the Gram
matrix of the nuisance components is perfectly isotropic, with all nuisance
eigenvalues equal.
The obstruction is not collinearity but the
span's limited capacity to synthesize calibrated spike signal without carrying
nuisance residue.
In ridge terms, this requires a negative penalty and the test-optimal point lies
past interpolation \cite{nr_kobak2020,nr_wu2020,nr_tsigler2023}.
Subsection~\ref{subsec:related-span-rescaling} returns to this
anti-shrinkage connection.

Why might a negative Gram shift help with noisy measurements?
First express our test-optimal shift in the original sensor units.
Remark~\ref{rem:sensor-model} divides the sensor readings
\(\vec S\) by \(\sigma\) to make the noise variance one.
The subsequent
rotation leaves inner products unchanged.
Expressing our stylized
features in the original sensor units therefore multiplies the Gram
matrix \(\bm K\) and its shift \(n\lambda^\star\bm I\) by
\(\sigma^2\).
Substituting \(d=m_{\rm sens}-1\) and
\(\gamma=1+m_{\rm sens}/\sigma^2\) into \eqref{eq:lambda-star} gives
\begin{equation}\label{eq:lambda-star-raw}
    \sigma^{2}\,n\lambda^{\star}
    =-\,m_{\rm sens}\sigma^{2}\,
      \frac{m_{\rm sens}-1}{m_{\rm sens}+\sigma^{2}}.
\end{equation}
To compare the shift in \eqref{eq:lambda-star-raw} with the classical errors-in-variables correction, treat the \(m_{\rm sens}\) sensors as the observations of a regression through the origin.
The physical quantity \(Y\) is the unknown slope in \(\vec S[j]=Y\,\vec s_{\rm resp}[j]+\vec N[j]\).
The regressor \(\vec s_{\rm resp}[j]\) is observed with measurement noise as the reference reading \(\vec S_0[j]=\vec s_{\rm resp}[j]+\vec N_0[j]\).
The classical known-noise errors-in-variables (EIV) correction subtracts the expected noise contribution \(m_{\rm sens}\sigma^2\) from the Gram scalar \(\|\vec S_0\|^2\) of these noisy regressors \citep[Sec.~2.2, Eq.~(2.2.20) and Ex.~9(a), pp.~107, 124]{fuller1987measurement}.
We call our test-error-minimizing shift in \eqref{eq:lambda-star-raw}
the tuned shift.
For \(n=1\), the tuned shift and the classical correction subtract the same amount to leading order as sensors are added at fixed per-sensor noise variance.%
\footnote{For \(n=1\), the stylized Gram scalar in sensor units is
\(\sigma^2(\gamma+d)=m_{\rm sens}(1+\sigma^2)=\E\|\vec S_0\|^2\).
The tuned shift \eqref{eq:lambda-star-raw} acts on this deterministic scalar, whereas the classical correction subtracts \(m_{\rm sens}\sigma^2\) from the observed scalar \(\|\vec S_0\|^2\).
The ratio of the tuned subtraction to the classical subtraction is
\((m_{\rm sens}-1)/(m_{\rm sens}+\sigma^2)\), which tends to one when
\(m_{\rm sens}\to\infty\) and \(\sigma^2/m_{\rm sens}\to0\).}

\begin{remark}[The frontier is one ray]\label{rem:frontier-ray}
Lemma~\ref{lem:two-scalar} says more than that optimal predictors are
symmetric.
Every Pareto-optimal span predictor, the minimum-norm
interpolator, the kernel-ridge path \eqref{eq:ridge-scale}, and the
test-optimal predictor are scalar multiples of one another: points of a
single ray through the flat average \(\tfrac1n\sum_i\vec x_i\),
distinguished only by their radius --- the \emph{interpolation ray}, in the
terminology of the same-distribution companion
\citep{ranade2026samedistribution}.
\end{remark}

\begin{remark}[Ridge and rescaling coincide only in this stylized
geometry]\label{rem:ridge-scalar-degeneracy}
The identity of the ridge path with pure rescaling is special to this
paper's model.
Because the nuisance geometry is exactly isotropic, \(\bm K\) has
at most two distinct eigenvalues, and the ridge resolvent
\((\bm K+n\lambda\bm I_n)^{-1}\) acts as a scalar on the ray.
With several
informative directions of unequal strength the two operations separate, as
Subsection~\ref{subsec:related-span-rescaling} discusses.
``Negative ridge'' and ``scaling past interpolation'' should therefore be
identified only within this paper.
\end{remark}
\subsection{Sample complexity}
\label{subsec:sample-complexity}

The exact frontier also shows why the fourth quadrant matters operationally:
for a fixed test-error target, the best span predictor can require far fewer
samples than interpolation.

Here and below, a sample threshold is an asymptotic crossover order, not a
finite-sample jump.
For fixed \(d\) and \(\gamma\), the exact loss curves in
Figure~\ref{fig:phase-sample} are smooth.

\begin{corollary}[Sample-complexity separation]
\label{cor:sample-complexity}
Assume \(\gamma>1\), and fix a target test error\footnote{\label{fn:target-window}%
The lower endpoint on \(\varepsilon\) is exactly the condition that the
interpolation threshold stay within the model's range, \(n_{\rm int}\le d\).
At the orthogonal-model boundary \(n=d\) --- equivalently \(\rho=\gamma\) and
\(R=\gamma\rho=\gamma^2\) --- \eqref{eq:frontier-endpoints-rho-R} gives
\(g_{\rm int}=(\gamma^2+\gamma^2)/\bigl(\gamma^2(1+\gamma)^2\bigr)=2/(1+\gamma)^2\),
so this is the interpolator's test error at the boundary and hence the smallest
target that interpolation reaches within \(n\le d\).
It is not a lower bound
for the span: the same boundary gives \(g_{\min}=1/(1+\gamma^2)\) for the best
span predictor, and \(2/(1+\gamma)^2>1/(1+\gamma^2)\) whenever \(\gamma>1\), so
the stated range also forces \(\varepsilon>1/(1+\gamma^2)\) and therefore
\(n_{\rm mis}\le d\) with room to spare.}
\(2/(1+\gamma)^2 \le \varepsilon<1\).

Let \(n_{\rm mis}(\varepsilon)\) and \(n_{\rm int}(\varepsilon)\) denote the
real-valued sample counts at which, respectively, the best span predictor and
the minimum-norm interpolator first attain test error \(\varepsilon\).

For every fixed \(\varepsilon\in(0,1)\), as \(\gamma\to\infty\),
\begin{equation}
\label{eq:sample-complexity-separation}
    \frac{n_{\rm int}(\varepsilon)}
         {n_{\rm mis}(\varepsilon)}
    =
    \frac{\gamma\sqrt{\varepsilon}}
         {1+\sqrt{\varepsilon}}
    \bigl(1+o(1)\bigr).
\end{equation}
Thus interpolation requires a factor of order \(\gamma\) more samples:
\(n_{\rm mis}(\varepsilon)\) is of order \(d/\gamma^2\), whereas
\(n_{\rm int}(\varepsilon)\) is of order \(d/\gamma\).

Exactly,
\begin{equation}
\label{eq:n-mis}
    n_{\rm mis}(\varepsilon)
    =
    \frac{d(1-\varepsilon)}
         {\varepsilon\gamma^2},
\end{equation}
and
\begin{equation}
\label{eq:n-int}
    n_{\rm int}(\varepsilon)
    =
    \frac{d}{\gamma}\,
    \rho_{\rm int}(\varepsilon,\gamma),
\end{equation}
where \(\rho_{\rm int}(\varepsilon,\gamma)\) is the positive root of
\begin{equation}
\label{eq:interp-quadratic}
    \varepsilon\rho^2
    +(2\varepsilon-1/\gamma)\rho
    -(1-\varepsilon)
    =
    0,
\end{equation}
namely
\begin{equation}
\label{eq:rho-int}
    \rho_{\rm int}(\varepsilon,\gamma)
    =
    \frac{
        1/\gamma-2\varepsilon+
        \sqrt{(2\varepsilon-1/\gamma)^2+
              4\varepsilon(1-\varepsilon)}
    }{2\varepsilon}.
\end{equation}
Their exact ratio is
\begin{equation}
\label{eq:n-int-n-mis-ratio}
    \frac{n_{\rm int}(\varepsilon)}
         {n_{\rm mis}(\varepsilon)}
    =
    \frac{\varepsilon\gamma}{1-\varepsilon}\,
    \rho_{\rm int}(\varepsilon,\gamma).
\end{equation}
Actual integer sample counts are obtained by rounding upward, provided the
rounded values remain within \(n\le d\).
\end{corollary}

\begin{proof}
By \eqref{eq:frontier-endpoints-rho-R}, the best span predictor has test
error \(g_{\min}=1/(1+R)\).
Using \(R=\gamma\rho\) from
\eqref{eq:scale-identities}, this is
\[
    g_{\min}=\frac{1}{1+\gamma\rho}.
\]
Thus the best span predictor reaches target test error \(\varepsilon\)
exactly when
\[
    \frac{1}{1+\gamma\rho}\le\varepsilon,
\]
which gives
\[
    \rho\ge \rho_{\rm mis}(\varepsilon)
    :=
    \frac{1-\varepsilon}{\varepsilon\gamma}.
\]
Substituting this lower bound for \(\rho\) into the ratio-to-samples relation
\(n=d\rho/\gamma\) from \eqref{eq:scale-identities} gives
\[
    n_{\rm mis}(\varepsilon)
    =
    \frac{d}{\gamma}\rho_{\rm mis}(\varepsilon)
    =
    \frac{d(1-\varepsilon)}{\varepsilon\gamma^2},
\]
which is \eqref{eq:n-mis}.

For interpolation, \eqref{eq:frontier-endpoints-rho-R} gives
\[
    g_{\rm int}
    =
    \frac{R+\rho^2}{R(1+\rho)^2}.
\]
Substituting \(R=\gamma\rho\) from \eqref{eq:scale-identities} yields
\[
    g_{\rm int}(\rho)
    =
    \frac{1+\rho/\gamma}{(1+\rho)^2}.
\]
Thus interpolation reaches the target test error \(\varepsilon\) exactly when
\[
    \frac{1+\rho/\gamma}{(1+\rho)^2}\le\varepsilon.
\]
Clearing denominators gives the quadratic \eqref{eq:interp-quadratic},
equivalently
\[
    \varepsilon\rho^2+(2\varepsilon-1/\gamma)\rho-(1-\varepsilon)\ge0.
\]
The equality case of the quadratic \eqref{eq:interp-quadratic} has one
positive and one negative root, because
the product of its roots (its constant term divided by its leading coefficient)
is
\[
    -\frac{1-\varepsilon}{\varepsilon}<0.
\]
Therefore interpolation reaches the target exactly when
\[
    \rho\ge\rho_{\rm int}(\varepsilon,\gamma),
\]
where \(\rho_{\rm int}(\varepsilon,\gamma)\) is the positive root in
\eqref{eq:rho-int}.
Converting back to samples with \(n=d\rho/\gamma\) gives
\eqref{eq:n-int}.

Finally, \eqref{eq:rho-int} gives
\begin{equation}\label{eq:rho-int-asymptotic}
    \rho_{\rm int}(\varepsilon,\gamma)
    \to
    \frac{1-\sqrt{\varepsilon}}{\sqrt{\varepsilon}}
\end{equation}
as \(\gamma\to\infty\).
Combining \eqref{eq:n-int-n-mis-ratio} with
\eqref{eq:rho-int-asymptotic} gives \eqref{eq:sample-complexity-separation}.
\end{proof}

Thus, throughout the target-error range in Corollary \ref{cor:sample-complexity}, overshooting
past interpolation is not merely a conceptual curiosity: in this model it is the
sample-efficient way to reach the target test error.
At first attainment the
sample-efficient predictor has training error exactly \((\gamma-1)^{2}\varepsilon^{2}\), so
once \(\varepsilon>1/(\gamma-1)\) it sits in the fourth quadrant proper (i.e.\ with
training error worse than that of the zero predictor). 

Notice that there is also a gap in test error at fixed sample size.
As visible in the
log-scale panel of Figure~\ref{fig:phase-sample}A, for any fixed~\(\rho\) the
interpolator's test error exceeds \(g_{\min}\) by a ratio, not just an
additive amount.
From \eqref{eq:frontier-endpoints},
\begin{equation}
\label{eq:fixed-sample-error-ratio}
    \frac{g_{\rm int}}{g_{\min}}
    =
    \frac{(1+\rho/\gamma)(1+\gamma\rho)}
         {(1+\rho)^2}
    =
    \frac{\gamma\rho}{(1+\rho)^2}
    \bigl(1+o(1)\bigr),
    \qquad
    \gamma\to\infty
    \quad\text{with \(\rho\) fixed}.
\end{equation}
At the interpolation crossover \(\rho=1\),
\eqref{eq:fixed-sample-error-ratio} is of order~\(\gamma\):
the interpolator's test error is a factor of order~\(\gamma\) larger than the
best span predictor's.
The interpolator achieves its low training error by
under-calibrating the spike (\(r_{\rm int}=\rho/(1+\rho)<1\)), which costs
spike bias \((r_{\rm int}-1)^2\) that the best span predictor
shrinks to \((1/(1+R))^{2}\).\footnote{The same under-calibration also makes the interpolator less
adversarially sensitive than the test-optimal predictor: the minimum RMS
adversarial sensitivity at fixed calibration (Theorem~\ref{thm:absolute-rms})
is linear in the calibration and is attained by symmetric coefficients, and
the interpolator is symmetric, so substituting \(r_{\rm int}\) gives its
sensitivity \(\delta\sqrt n/(1+\rho)\)
(in the perturbation normalization of Section~\ref{sec:robustness}).
As a
function of real-valued \(n\), that sensitivity is not monotone: it rises until the interpolation
threshold \(n=d/\gamma\), peaks there, and only then falls.}

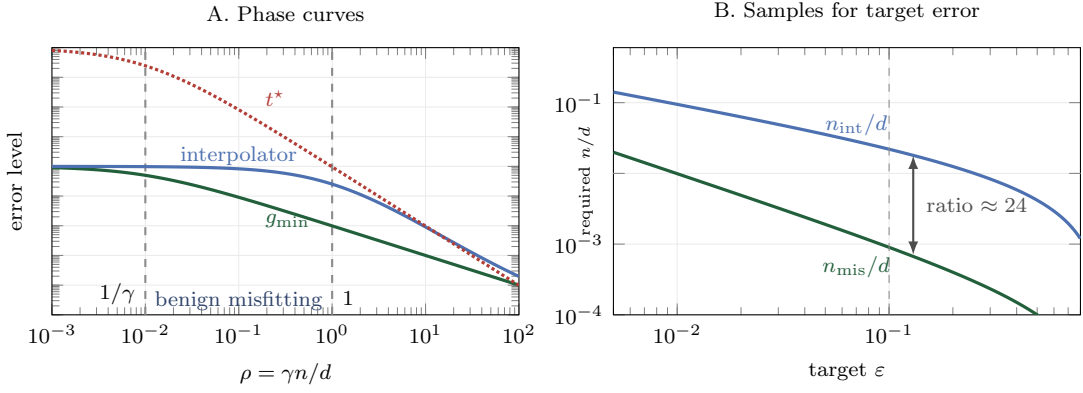
\begin{figure}[t]
    \centering
    \begin{tikzpicture}
\definecolor{goodgreen}{RGB}{48,135,84}
\definecolor{alertred}{RGB}{178,62,55}
\definecolor{softblue}{RGB}{76,112,176}
\definecolor{windowfill}{RGB}{235,244,255}
\begin{groupplot}[
    group style={group size=2 by 1, horizontal sep=1.25cm},
    width=0.47\textwidth,
    height=0.31\textwidth,
    tick label style={font=\scriptsize},
    label style={font=\scriptsize},
    title style={font=\scriptsize},
    grid=major,
    grid style={black!8},
    clip mode=individual
]
\nextgroupplot[
    title={A. Phase curves},
    xmode=log,
    ymode=log,
    xmin=0.001,xmax=100,
    ymin=1e-5,ymax=1e4,
    xlabel={$\rho=\gamma n/d$},
    ylabel={error level},
    ytick={1e-5,1e-4,1e-3,1e-2,1e-1,1e0,1e1,1e2,1e3,1e4},
    yticklabel=\empty
]
    \addplot[windowfill,draw=none] coordinates {(0.01,1e-5) (1,1e-5) (1,1e4) (0.01,1e4)} -- cycle;
    \addplot[black!45,dashed,thick] coordinates {(0.01,1e-5) (0.01,1e4)};
    \addplot[black!45,dashed,thick] coordinates {(1,1e-5) (1,1e4)};
    \addplot[goodgreen!70!black,very thick,domain=0.001:100,samples=300] {1/(1+100*x)};
    \addplot[softblue,very thick,domain=0.001:100,samples=300] {(1+x/100)/(1+x)^2};
    \addplot[alertred,very thick,densely dotted,domain=0.001:100,samples=300] {(99/(1+100*x))^2};
    \node[font=\scriptsize,goodgreen!70!black,anchor=north west] at (axis cs:0.15,0.055) {$g_{\min}$};
    \node[font=\scriptsize,softblue,anchor=south] at (axis cs:0.1,0.5) {interpolator};
    \node[font=\scriptsize,alertred,anchor=south west] at (axis cs:0.15,50) {$t^\star$};
    \node[font=\scriptsize,align=center,softblue!60!black] at (axis cs:0.1,3e-5) {benign misfitting};
    \node[font=\scriptsize,anchor=south east] at (axis cs:0.01,1e-5) {$1/\gamma$};
    \node[font=\scriptsize,anchor=south west] at (axis cs:1,1e-5) {$1$};

\nextgroupplot[
    title={B. Samples for target error},
    xmode=log,
    ymode=log,
    xmin=0.005,xmax=0.8,
    ymin=0.0001,
    ymax=0.6,
    xlabel={target \(\varepsilon\)},
    ylabel={\raisebox{0pt}{\fontsize{6}{7}\selectfont required \(n/d\)}},
    ylabel style={at={(axis description cs:-0.02,0.5)}},
    ytick={1e-4,1e-3,1e-1},
    yticklabels={$10^{-4}$,$10^{-3}$,$10^{-1}$},
    extra y ticks={1e-2},
    extra y tick labels={},
    extra y tick style={grid style={black!8}}
]
    \addplot[softblue,very thick,domain=0.005:0.8,samples=220]
        {1/(100*((2*x-1/100+sqrt((1/100-2*x)^2+4*x*(1-x)))/(2*(1-x))))};
    \addplot[goodgreen!70!black,very thick,domain=0.005:0.8,samples=220]
        {(1-x)/(x*100^2)};
    \node[font=\scriptsize,softblue,anchor=south] at (axis cs:0.07,0.025) {$n_{\rm int}/d$};
    \node[font=\scriptsize,goodgreen!70!black,anchor=north] at (axis cs:0.07,0.00095) {$n_{\rm mis}/d$};
    \addplot[black!45,dashed] coordinates {(0.1,0.0001) (0.1,0.6)};
    \draw[{Latex[length=1.8mm]}-{Latex[length=1.8mm]},thick,black!70]
        (axis cs:0.13,0.0007) -- (axis cs:0.13,0.018)
        node[midway,anchor=west,font=\scriptsize,xshift=1pt] {ratio $\approx 24$};
\end{groupplot}
\end{tikzpicture}
    \caption[Exact sample-complexity curves for \(\gamma=100\)]{Exact sample-complexity curves for \(\gamma=100\).
    Panel~A plots the interpolator test error (blue), the best span test error
    \(g_{\min}\) (green), and the best-span training error \(t^\star\) (red
    dotted) as functions of \(\rho=\gamma n/d\).
    The benign-misfitting window lies between the best-span threshold (at \(\rho = \frac{1}{\gamma}\)) and the
    interpolation threshold (at \(\rho = 1\)).
    Panel~B maps each target error \(\varepsilon\) to the required normalized
    sample counts \(n_{\rm mis}/d\) and \(n_{\rm int}/d\).
    The arrow marks the multiplicative sample saving at the displayed target.

    The dashed vertical lines in Panel~A mark the two asymptotic crossover
    locations; the actual finite curves are visibly smooth across the crossovers.
    The panels use \(\rho=\gamma n/d\) and \(n/d\), so fixing \(\gamma\) fixes
    the curves without choosing a numerical value of \(d\).
    Taking \(d=50{,}000\) converts the horizontal coordinates to the sample
    counts in Footnote~\ref{fn:threshold-numbers}.}
    \label{fig:phase-sample}
\end{figure}
\FloatBarrier

\Needspace{9\baselineskip}
\section{One-pass SGD reaches the misfitting region}
\label{sec:sgd}

The exact frontier in Theorem~\ref{thm:pareto} describes what is possible in
the training span.
We now show that one pass of SGD on the clean training
labels can reach a span predictor with small test error and large
training error.

Figure~\ref{fig:stylized-sgd-trace} shows a representative run of SGD.
During the first pass over the data, the test error
falls steadily while the full empirical training error rises.
The reuse epochs
--- everything past the first epoch --- are where large-learning-rate training
looks visibly unstable: we see gradient spikes, and the test error degrades.
Shuffling the training data just moves the spikes around.
Note
here that a classically stable learning rate for multi-epoch training is not
useful: its test performance barely improves over the zero-predictor baseline,
despite the fact that its training error steadily improves with no visible gradient
spikes or training instability.

\begin{figure}[t!]
    \centering
    \includegraphics[width=0.9\textwidth]{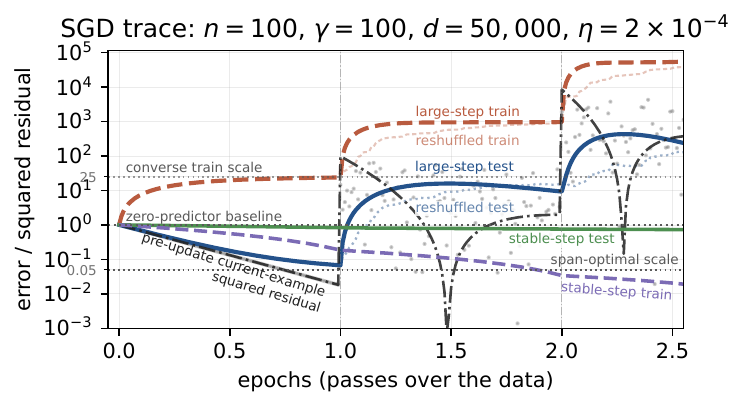}
    \caption[Analytic stylized SGD trace showing the diagnostic split]{Analytic SGD trace showing the diagnostic split.
    Parameters: \(n=100\) training points, \(\gamma=100\), \(d=50{,}000\); the
    large learning rate is \(\eta=2\times10^{-4}\), the stable comparator uses
    \(\eta=10^{-5}\), and three passes over the same data are shown.
    All
    curves are computed exactly from the deterministic Gram geometry.

    During the first epoch, the large-step test error (solid blue) falls while
    the large-step full empirical training error (dashed red) rises above the
    zero-predictor baseline.
    The dash-dotted black curve is the pre-update
    squared residual on the current training example, not a running average.
    The
    stable-rate comparator's test and training errors are the (essentially
    flat) solid green and (steadily improving) dashed purple curves.
    Dotted
    horizontal lines mark the zero-predictor baseline and two asymptotic
    reference scales: the span-optimal test scale \(d/(\gamma^{2}n)=0.05\) and
    the converse training scale \((d/(\gamma n))^{2}=25\).

    Solid curves in later epochs reuse the same cyclic order.
    Faint curves and
    gray points show a reshuffled run.
    Shuffled reuse changes the detailed
    oscillations but still degrades both training and test performance.

    So what makes the test error improve in the first epoch, makes the
    training error rise throughout it, and makes multi-epoch training so bad?
    What helps in the first epoch is participation: by \eqref{eq:one-pass-neff}
    the traced pass has an effective participation \(n_{\rm eff}\approx76\) of
    the \(n=100\) examples (\(n_{\rm eff}\) is defined in \eqref{eq:neff-def},
    and Remark~\ref{rem:participation} explains why participation, not raw
    sample size, is what controls the nuisance variance), which is why the
    first-epoch test error lands quantitatively so close to the span-optimal
    scale.
    The qualitative behaviors are explained by the learning rate:
    \(\eta\gamma=0.02\) is small on the shared spike coordinate, while
    \(\eta(\gamma+d)=10.02\) is far beyond the repeated-example stability
    threshold \(2\) of \eqref{eq:one-sample-rates}.
    Each update multiplies
    the residual on the example it just used by \(1-\eta(\gamma+d)\approx-9\),
    so the training error rises throughout the pass as it accumulates those
    stored overshoots.
    Fresh examples are stable, reused examples are
    violently not, and that instability is the cause of the gradient spiking
    in the reuse epochs (Appendix~\ref{app:stability-boundaries}).}
    \label{fig:stylized-sgd-trace}
\end{figure}

The point is not that a large learning rate finds a better empirical training error
minimizer.
Instead, the final iterate of one-pass SGD at a large learning rate \(\eta\) can
reach a useful span point that deliberately avoids empirical minimization.
The learning rate
alone selects a one-pass span point past the interpolating solution, with much better spike
calibration.

\subsection{Exact one-pass geometry}
\label{subsec:sgd-exact-geometry}

Run one pass of SGD from
\(\vec{w}_0=\vec{0}\), in order \(1,\ldots,n\), with constant learning rate \(\eta>0\).
We use the half-squared loss \(\frac12(\vec{w}^{\top}\vec{x}_i-y_i)^2\), so the one-sample
update is
\begin{equation}\label{eq:sgd-update}
    \vec{w}_i=\vec{w}_{i-1}+\eta(1-\vec{w}_{i-1}^\top \vec{x}_i)\vec{x}_i.
\end{equation}
After \(t\) updates, write the SGD iterate as
\begin{equation}\label{eq:sgd-span-rep}
    \vec w_t=\sum_{i=1}^t\alpha_i^{(t)}\vec x_i,
\end{equation}
where \(\alpha_i^{(t)}\) is the coefficient on example \(i\).
Define the
running coefficient sum $s_t$ and spike residual $b_t$ by
\begin{equation}\label{eq:sgd-state-defs}
    s_t:=\sum_{i=1}^t\alpha_i^{(t)},
    \qquad
    b_t:=1-\gamma s_t.
\end{equation}
Before example \(i\) is used, its nuisance direction has not appeared in the
iterate, so \(\vec{w}_{i-1}^\top \vec{x}_i=\gamma s_{i-1}\). 
Thus \(\gamma s_t\) is the prediction contributed by the shared spike, and
\(b_t\) is the spike residual seen by the next fresh example.
This shows that the spike residual obeys a scalar recursion: \(b_t = (1-\eta\gamma)b_{t-1}\).

The predictor's recursion thus has two pieces.
The shared spike residual evolves
geometrically, because every fresh example has the same spike coordinate, and it contracts in
magnitude exactly when \(0<\eta\gamma<2\).
Meanwhile, the nuisance variance accumulates one fresh orthogonal direction at a time,
because an example's nuisance direction has not appeared before it is used by SGD.
If an example is instead \emph{reused}, its own stored nuisance coefficient re-enters
the update, restoring the full multiplier \(1-\eta\|\vec x_i\|_2^2=1-\eta(\gamma+d)\).
Repeated-example stability therefore requires \(0<\eta<2/(\gamma+d)\), the one-sample
boundary of \eqref{eq:one-sample-rates}, while the shared spike contracts for every
\(0<\eta\gamma<2\).
Appendix~\ref{app:stability-boundaries} quantifies this contrast at
general \(t\).

\begin{lemma}[Single-spike SGD is a geometric series]
\label{lem:single-geometric}
Run one pass of \eqref{eq:sgd-update} from \(\vec w_0=\vec0\) with constant learning rate
\(\eta>0\), and define the one-step spike-residual multiplier
\begin{equation}\label{eq:mu-def}
    \mu:=1-\eta\gamma .
\end{equation}
Then for every \(0\le t\le n\) the coefficients \eqref{eq:sgd-span-rep} and the state variables
\eqref{eq:sgd-state-defs} are
\begin{equation}\label{eq:single-geometric-coeffs}
    \alpha_i^{(t)}
    =
    \begin{cases}
        \eta\,\mu^{\,i-1}, & 1\le i\le t,\\[2pt]
        0, & t<i\le n,
    \end{cases}
    \qquad
    b_t=\mu^{t},
    \qquad
    s_t=\frac{1-\mu^{t}}{\gamma},
\end{equation}
and the coefficient energy is
\begin{equation}\label{eq:single-geometric-energy}
    q_t:=\sum_{i=1}^{n}\bigl(\alpha_i^{(t)}\bigr)^{2}
    =
    \eta^{2}\sum_{k=0}^{t-1}\mu^{2k}
    =
    \eta^{2}\,\frac{1-\mu^{2t}}{1-\mu^{2}}
    \qquad(\mu^{2}\ne1).
\end{equation}
Consequently, by \eqref{eq:scalar-test-error},
\begin{equation}\label{eq:single-geometric-risk}
    \Ttest(\vec w_t)
    =
    \underbrace{\mu^{2t}}_{\text{spike bias}^2}
    +
    \underbrace{\frac{d\eta}{\gamma(2-\eta\gamma)}\bigl(1-\mu^{2t}\bigr)}_{\text{nuisance variance}},
    \qquad
    0<\eta\gamma<2 ,
\end{equation}
and the nuisance variance is bounded uniformly in the length of the pass:
\begin{equation}\label{eq:nuisance-floor}
    \Ttest(\vec w_t)
    \le
    \mu^{2t}+\frac{d\eta}{\gamma(2-\eta\gamma)}
    \qquad\text{for every }0\le t\le n,
    \qquad 0<\eta\gamma<2 .
\end{equation}
\end{lemma}

\begin{proof}
Everything follows from one observation: \emph{a fresh example cannot see the nuisance directions
already written into the iterate.}
By \eqref{eq:sgd-update} the iterate \(\vec w_{t-1}\) lies in
the span of \(\vec x_1,\ldots,\vec x_{t-1}\), whose nuisance directions
\(\vec v_1,\ldots,\vec v_{t-1}\) are orthogonal to \(\vec v_t\), and the nuisance subspace is orthogonal to
the spike direction.
Hence
\begin{equation}\label{eq:fresh-inner-product}
    \vec w_{t-1}^{\top}\vec x_t = \underbrace{(\sqrt{\gamma} s_{t-1})\cdot \sqrt{\gamma}}_{\text{spike term}} + \underbrace{0}_{\text{nuisance term}} = \gamma s_{t-1}
    \qquad(1\le t\le n).
\end{equation}
A fresh example therefore responds to the iterate exactly as it would if the \(d\) nuisance
dimensions were absent and the task were the one-dimensional one of fitting the spike alone.

\emph{The spike residual.}
The $\eta(1-\vec{w}_{t-1}^\top \vec{x}_t)$ term in \eqref{eq:sgd-update} is
therefore \(\eta(1-\gamma s_{t-1})=\eta\,b_{t-1}\) by \eqref{eq:fresh-inner-product} and the definition \eqref{eq:sgd-state-defs} of \(b_{t-1}\).
Consequently, the coefficient $\alpha_t^{(t)} = \eta\,b_{t-1}$ since it is $\vec{x}_t$ that is being added into the SGD iterate with that multiplier.

Applying the definition of $s_t$ as the sum of the $\alpha_i$ coefficients, we get \(s_t=s_{t-1}+\eta\,b_{t-1}\).
This gives us:
\begin{equation}\label{eq:b-recursion}
    b_t
    =
    1-\gamma s_t
    =
    \bigl(1-\gamma s_{t-1}\bigr)-\eta\gamma\,b_{t-1}
    =
    (1-\eta\gamma)\,b_{t-1}
    =
    \mu\,b_{t-1} .
\end{equation}
This is a scalar recursion with no nuisance in it at all.
Since \(\vec w_0=\vec0\) gives
\(s_0=0\) and \(b_0=1\), induction on \(t\) using \eqref{eq:b-recursion} gives \(b_t=\mu^{t}\)
immediately, and then \(s_t=(1-b_t)/\gamma=(1-\mu^{t})/\gamma\).
These are the second and third
claims of \eqref{eq:single-geometric-coeffs}.

\emph{The coefficients now fall out.}
All the $\alpha_i$ start out as zero.
The update at
step \(t\) changes only the coefficient of \(\vec x_t\), setting it to \(\eta\,b_{t-1}=\eta\,\mu^{\,t-1}\),
and no later step revisits \(\vec x_t\) --- so the stored value \(\eta\,\mu^{\,t-1}\) is final.
Example
\(i>t\) has not been seen at time \(t\), so \(\alpha_i^{(t)}=0\).
Together these give the first
claim of \eqref{eq:single-geometric-coeffs}.

\emph{Energy.}
Squaring and summing those coefficients gives the first equality of
\eqref{eq:single-geometric-energy}.
The finite geometric sum
\(\sum_{k=0}^{t-1}\mu^{2k}=(1-\mu^{2t})/(1-\mu^{2})\) gives the second.

\emph{Test Error.}
Substituting \(s=s_t\) from \eqref{eq:single-geometric-coeffs} and \(q=q_t\) from
\eqref{eq:single-geometric-energy} into the test-error decomposition
\eqref{eq:scalar-test-error} gives
\(\Ttest(\vec w_t)=(\gamma s_t-1)^{2}+dq_t=\mu^{2t}+dq_t\).
With \(\mu=1-\eta\gamma\) from
\eqref{eq:mu-def},
\begin{equation}\label{eq:one-minus-mu-squared}
    1-\mu^{2}
    =
    (1-\mu)(1+\mu)
    =
    \eta\gamma\,(2-\eta\gamma),
\end{equation}
so \(dq_t=d\eta^{2}(1-\mu^{2t})/\bigl(\eta\gamma(2-\eta\gamma)\bigr)
=\frac{d\eta}{\gamma(2-\eta\gamma)}\bigl(1-\mu^{2t}\bigr)\), which completes the proof of
\eqref{eq:single-geometric-risk}.
Finally \(0<\eta\gamma<2\) gives \(|\mu|<1\) and hence
\(1-\mu^{2t}\le1\), which yields \eqref{eq:nuisance-floor}.
\end{proof}

Lemma~\ref{lem:single-geometric} identifies the geometric coefficient profile
produced by one pass.
We next isolate exactly what a departure from a
symmetric profile costs.

\begin{remark}[Participation: how to understand a non-symmetric profile]
\label{rem:participation}
Consider any \(\vec\alpha\) with
\(s=\sum_i\alpha_i\ne0\) that defines a span predictor $\vec{u} = \sum_i\alpha_i \vec{x}_i$.

We can define the effective \emph{participation count}
\begin{equation}\label{eq:neff-def}
    n_{\rm eff}(\vec\alpha)
    :=
    \frac{s^{2}}{q}
    =
    \frac{\bigl(\sum_{i=1}^{n}\alpha_i\bigr)^{2}}{\sum_{i=1}^{n}\alpha_i^{2}},
\end{equation}
using the notation \(s=\sum_i\alpha_i\) and \(q=\sum_i\alpha_i^{2}\) from
\eqref{eq:sq-defs}.
For nonnegative coefficients, this ratio is the
classical effective sample size of survey sampling \citep{kish1965survey};
here the coefficients may be signed.
The orthogonal nuisance directions make the representation
\(\vec u=\Xt\vec\alpha\) introduced before \eqref{eq:risk-defs} unique, so we also write
\(n_{\rm eff}(\vec u):=n_{\rm eff}(\vec\alpha)\) and apply the count to predictors directly.
The Cauchy--Schwarz step in Lemma~\ref{lem:two-scalar}'s
proof, \(q\ge s^{2}/n\), gives the upper bound in
\begin{equation}\label{eq:neff-range}
    0<n_{\rm eff}(\vec\alpha)\le n,
\end{equation}
with equality on the right precisely for equal coefficients.
For \emph{nonnegative profiles}
\(\alpha_i\ge0\) --- the case relevant below --- the cross terms in \(s^{2}\) are nonnegative, so
\(s^{2}\ge q\) and
\begin{equation}\label{eq:neff-range-nonneg}
    1\le n_{\rm eff}(\vec\alpha)\le n
    \qquad(\alpha_i\ge0\text{ for all }i),
\end{equation}
with equality on the left precisely when a single coefficient is nonzero.
For such nonnegative profiles,
\(n_{\rm eff}\) counts how many training examples effectively share the load, and \(n/n_{\rm eff}\ge1\)
measures how far the profile sits from flat.
For signed profiles \eqref{eq:neff-def} is still the
quantity that enters the risks below, but it is then an algebraic ratio and not a count.%
\footnote{For nonnegative profiles, the \(n_{\rm eff}\) from \eqref{eq:neff-def} counts participating \emph{training examples}.
The effective ranks \(R_k\) of \eqref{eq:effective-rank-families} count participating
\emph{nuisance directions}.
These effective quantities are built the same way, but answer very different
questions.
Nothing in this section relates them.}

Recall from \eqref{eq:spike-calibration} the notation that $r$ is the calibration, i.e. $r = \sqrt{\gamma} u[1] = s \gamma$.
Substituting \(q=s^{2}/n_{\rm eff}\) from \eqref{eq:neff-def} and \(s=r/\gamma\) from
\eqref{eq:spike-calibration} into \eqref{eq:scalar-test-error} and
\eqref{eq:scalar-train-error}, and using \(\rho\) and \(R\) from \eqref{eq:scale-identities},
gives the general form of \eqref{eq:calibrated-frontier-form}:
\begin{equation}\label{eq:neff-frontier-form}
    \Ttest
    =
    \underbrace{(r-1)^{2}}_{\text{spike bias}^2}
    +
    \underbrace{\frac{n}{n_{\rm eff}}\cdot\frac{r^{2}}{R}}_{\text{nuisance variance}},
    \qquad
    \Ttrain
    =
    \underbrace{(r-1)^{2}}_{\text{spike bias}^2}
    +
    \underbrace{\frac{2r(r-1)}{\rho}}_{\text{cross term}}
    +
    \underbrace{\frac{n}{n_{\rm eff}}\cdot\frac{r^{2}}{\rho^{2}}}_{\text{amplified nuisance variance}} .
\end{equation}
Term by term, \(dq=dr^{2}/(\gamma^{2}n_{\rm eff})=\frac{n}{n_{\rm eff}}\frac{r^{2}}{R}\),
\(\frac{d^{2}}{n}q=\frac{n}{n_{\rm eff}}\frac{r^{2}}{\rho^{2}}\), and
\(\frac{2ds}{n}(\gamma s-1)=\frac{2r(r-1)}{\rho}\), all by \eqref{eq:scale-identities}.

Two features of \eqref{eq:neff-frontier-form} are used throughout Section~\ref{sec:sgd}.
The performance
of a span predictor with \(s\ne0\) is described completely\footnote{At \(s=0\) participation is undefined and different coefficient energies
can give different risks.}, for \emph{both} training and test loss, by two numbers: its
calibration \(r\) from \eqref{eq:spike-calibration} and its participation \(n_{\rm eff}\) from
\eqref{eq:neff-def}.
Notice that the same factor \(n/n_{\rm eff}\) multiplies the nuisance variance and the amplified
nuisance variance --- though not the cross term \(2r(r-1)/\rho\) --- so a departure from flatness
inflates those two by an identical factor.
Equivalently, since
\(\frac{n}{n_{\rm eff}}\frac{r^{2}}{R}=\frac{r^{2}d}{\gamma^{2}n_{\rm eff}}\), a perfectly
calibrated predictor (\(r=1\)) carries nuisance variance \(d/(\gamma^{2}n_{\rm eff})\), which is the
flat-profile formula \eqref{eq:flat-average-risk} of Subsection~\ref{subsec:flat-averaging} with the sample count replaced by
\(n_{\rm eff}\). 

\end{remark}

\subsection{One complete pass}
\label{subsec:one-complete-pass}

\begin{theorem}[One pass reaches the fourth quadrant, within a logarithm of the frontier on both axes]
\label{thm:one-pass-sgd}
Consider a family of problems with sample ratio \(\rho=\gamma n/d\to0\) and
signal-to-nuisance ratio \(R=\gamma^{2}n/d\to\infty\), which is the benign-misfitting window
\eqref{eq:benign-misfitting-window}.
Write \(\vec w_n(\eta)\) for the final
iterate of one SGD pass of \eqref{eq:sgd-update} from \(\vec w_0=\vec0\) at
constant learning rate \(\eta\).
Choose
\begin{equation}\label{eq:eta-pass-def}
    \eta_{\rm pass}
    :=
    \frac{\log R}{2\gamma n} = \frac{\log\frac{\gamma^2 n}{d}}{2\gamma n}.
\end{equation}
If further
\[
    (\log R)^2=o(n),
\]
then
\begin{equation}\label{eq:one-pass-fourth-quadrant}
\begin{aligned}
    \Ttest\bigl(\vec w_n(\eta_{\rm pass})\bigr)
    &=
    \frac{1}{R}\bigl(1+o(1)\bigr)
    +
    \frac{\log R}{4R}\bigl(1+o(1)\bigr)
    \longrightarrow0,
    \\
    \Ttrain\bigl(\vec w_n(\eta_{\rm pass})\bigr)
    &=
    \frac{\log R}{4\rho^{2}}\bigl(1+o(1)\bigr)
    \longrightarrow\infty .
\end{aligned}
\end{equation}
The exact frontier endpoints from
\eqref{eq:frontier-endpoints-rho-R} are
\[
    g_{\min}=\frac{1}{1+R},
    \qquad
    t^\star=
    \left(\frac{R-\rho}{\rho(1+R)}\right)^2.
\]
Therefore both risks in \eqref{eq:one-pass-fourth-quadrant} are a factor
\(\tfrac14\log R\,(1+o(1))\) from their corresponding frontier endpoints.
\end{theorem}

The theorem follows from an exact finite-\(n\) calculation.
We state that
calculation separately because the later learning-rate, ascent, mini-batch,
and noisy-label arguments reuse it.

\begin{proposition}[One-pass calibration, participation, and risks]
\label{prop:one-pass-risk-summary}
Run one pass of the SGD update \eqref{eq:sgd-update} from \(\vec w_0=\vec0\) with
\(0<\eta\gamma<2\).
Recall the multiplier
\(\mu=1-\eta\gamma\) from \eqref{eq:mu-def}, and write
\begin{equation}\label{eq:Lambda-def}
    \Lambda:=\eta\gamma n.
\end{equation}
We refer to \(\Lambda\) as the calibration
budget\footnote{\label{fn:calibration-budget}The name reflects a natural bound.
Exact spike calibration means \(s=1/\gamma\) by
\eqref{eq:spike-calibration}.
Lemma~\ref{lem:single-geometric} gives
\(\lvert\alpha_i^{(n)}\rvert=\eta\lvert\mu\rvert^{i-1}\le\eta\) when
\(0<\eta\gamma<2\).
Therefore
\(\lvert s_n\rvert\le n\eta\) and so
\(\gamma\lvert s_n\rvert\le\Lambda\).
Thus
\(\Lambda\) is the upper bound \(n\eta\), expressed in units of the
coefficient sum required for exact calibration.
Subsection~\ref{subsec:lr-robustness} uses this interpretation.}.

The final iterate has calibration
\begin{equation}\label{eq:one-pass-calibration}
    r_n
    :=
    \gamma s_n
    =
    \sqrt{\gamma}\,w_n[1]
    =
    1-\mu^n
\end{equation}
and participation
\begin{equation}\label{eq:one-pass-neff}
    n_{\rm eff}(\vec w_n)
    =
    \frac{2-\eta\gamma}{\eta\gamma}
    \cdot
    \frac{1-\mu^n}{1+\mu^n}.
\end{equation}
On the branch \(0<\eta\gamma\le1\), the coefficients in
\eqref{eq:single-geometric-coeffs} are nonnegative, so
\eqref{eq:one-pass-neff} is an example count in the sense of
\eqref{eq:neff-range-nonneg}.
On the rest of the stable range it remains the
algebraic ratio in \eqref{eq:neff-def}.

The exact finite-\(n\) risks are
\begin{equation}\label{eq:one-pass-exact-risks}
\begin{aligned}
    \Ttest(\vec w_n)
    &=
    \underbrace{\mu^{2n}}_{\text{spike bias}^2}
    +
    \underbrace{\frac{\Lambda}{R}
    \cdot\frac{1-\mu^{2n}}{2-\eta\gamma}}_{\text{nuisance variance}},
    \\[6pt]
    \Ttrain(\vec w_n)
    &=
    \underbrace{\mu^{2n}}_{\text{spike bias}^2}
    \underbrace{{}-\frac{2\mu^n(1-\mu^n)}{\rho}}_{\text{cross term}}
    +
    \underbrace{\frac{\Lambda}{\rho^2}
    \cdot\frac{1-\mu^{2n}}{2-\eta\gamma}}_{\text{amplified nuisance variance}}.
\end{aligned}
\end{equation}

On a family satisfying
\[
    \Lambda\to\infty,
    \qquad
    \frac{\Lambda}{n}\to0,
\]
the spike residual and participation satisfy
\begin{equation}\label{eq:one-pass-mu-decay}
    \mu^n\le e^{-\Lambda}=o(1)
\end{equation}
and
\begin{equation}\label{eq:one-pass-neff-asymp}
    n_{\rm eff}(\vec w_n)
    =
    \frac{2n}{\Lambda}\bigl(1+o(1)\bigr).
\end{equation}
If the family also satisfies \(\Lambda^2/n\to0\), then
\begin{equation}\label{eq:one-pass-contraction}
    \mu^n=e^{-\Lambda}\bigl(1+o(1)\bigr).
\end{equation}
\end{proposition}

\begin{proof}
To analyze the risks, it is useful to leverage the concept of participation.
Lemma~\ref{lem:single-geometric} at \(t=n\) gives
\begin{equation}\label{eq:one-pass-sq}
    s_n=\frac{1-\mu^n}{\gamma},
    \qquad
    q_n=
    \eta^2\frac{1-\mu^{2n}}{1-\mu^2}.
\end{equation}
The formula for \(s_n\) in \eqref{eq:one-pass-sq} gives
\eqref{eq:one-pass-calibration}.
Substituting \eqref{eq:one-pass-sq} into
\eqref{eq:neff-def}, and using
\(1-\mu^2=\eta\gamma(2-\eta\gamma)\) from
\eqref{eq:one-minus-mu-squared}, gives
\[
\begin{aligned}
    n_{\rm eff}(\vec w_n)
    &=
    \frac{(1-\mu^n)^2}{\gamma^2}
    \cdot
    \frac{\eta\gamma(2-\eta\gamma)}
         {\eta^2(1-\mu^n)(1+\mu^n)}
    \\
    &=
    \frac{2-\eta\gamma}{\eta\gamma}
    \cdot
    \frac{1-\mu^n}{1+\mu^n},
\end{aligned}
\]
which is \eqref{eq:one-pass-neff}.

Substitute \(r_n=1-\mu^n\) from
\eqref{eq:one-pass-calibration} and the participation
\eqref{eq:one-pass-neff} into \eqref{eq:neff-frontier-form}.
The nuisance
factor shared by the two risks becomes
\begin{equation}\label{eq:one-pass-nuisance-factor}
\begin{aligned}
    \frac{n}{n_{\rm eff}(\vec w_n)}r_n^2
    &=
    \frac{n\eta\gamma}{2-\eta\gamma}
    \cdot
    \frac{1+\mu^n}{1-\mu^n}
    \cdot
    (1-\mu^n)^2
    \\
    &=
    \frac{\Lambda}{2-\eta\gamma}(1-\mu^{2n}).
\end{aligned}
\end{equation}
The spike-calibration error is \(r_n-1=-\mu^n\).
Substituting
\(r_n-1=-\mu^n\) and \eqref{eq:one-pass-nuisance-factor} into
\eqref{eq:neff-frontier-form} gives \eqref{eq:one-pass-exact-risks}.

It remains to verify the asymptotic consequences.
On the stated small-rate
branch, \(\eta\gamma=\Lambda/n\to0\), so
\(\mu=1-\Lambda/n\in[0,1)\) for all sufficiently large \(n\).
The inequality
\(\log(1-x)\le-x\) for \(0\le x<1\) gives
\[
    n\log\mu
    =
    n\log\left(1-\frac{\Lambda}{n}\right)
    \le-\Lambda.
\]
Exponentiating gives \eqref{eq:one-pass-mu-decay}.
Substituting
\(\mu^n=o(1)\) and \(\eta\gamma=o(1)\) into
\eqref{eq:one-pass-neff} gives \eqref{eq:one-pass-neff-asymp}.

If \(\Lambda^2/n\to0\), then
\[
    n\log\left(1-\frac{\Lambda}{n}\right)
    =
    -\Lambda
    +
    O\left(\frac{\Lambda^2}{n}\right)
    =
    -\Lambda+o(1).
\]
Exponentiating gives \eqref{eq:one-pass-contraction}.
\end{proof}

\begin{proof}[Proof of Theorem~\ref{thm:one-pass-sgd}]
At the learning rate \eqref{eq:eta-pass-def}, the definition
\eqref{eq:Lambda-def} gives
\(\Lambda=\eta_{\rm pass}\gamma n=\tfrac12\log R\).
Therefore \(\Lambda\to\infty\), and the hypothesis \((\log R)^{2}=o(n)\) of Theorem~\ref{thm:one-pass-sgd} gives
\(\Lambda^2/n=(\log R)^2/(4n)\to0\).
Since \(\Lambda\to\infty\), we also
have
\(\Lambda/n=(\Lambda^2/n)/\Lambda\to0\).
The small-rate conclusions
\eqref{eq:one-pass-mu-decay}--\eqref{eq:one-pass-contraction} apply.
In
particular,
\[
    \mu^{2n}
    =
    e^{-2\Lambda}\bigl(1+o(1)\bigr)
    =
    \frac{1}{R}\bigl(1+o(1)\bigr).
\]
Also \(\eta_{\rm pass}\gamma=o(1)\) and \(\mu^n=o(1)\).
The nuisance term
of the exact test risk \eqref{eq:one-pass-exact-risks} is therefore
\[
    \frac{\Lambda}{R}
    \cdot
    \frac{1-\mu^{2n}}{2-\eta_{\rm pass}\gamma}
    =
    \frac{\log R}{4R}\bigl(1+o(1)\bigr).
\]
This proves the test-risk claim in
\eqref{eq:one-pass-fourth-quadrant}.

For the training risk, the spike-bias term, the cross term, and the amplified nuisance
variance in \eqref{eq:one-pass-exact-risks} have respective orders
\[
    \frac{1}{R},
    \qquad
    O\left(\frac{1}{\rho\sqrt R}\right),
    \qquad
    \frac{\log R}{4\rho^2}\bigl(1+o(1)\bigr).
\]
The ratios of the spike-bias term and the cross term to the amplified nuisance
variance satisfy
\[
    O\left(\frac{\rho^2}{R\log R}\right)\longrightarrow0,
    \qquad
    O\left(\frac{\rho}{\sqrt R\log R}\right)\longrightarrow0,
\]
because \(\rho\to0\) and \(R\to\infty\).
The amplified nuisance variance
therefore gives the training-risk claim in
\eqref{eq:one-pass-fourth-quadrant}.

Finally, \eqref{eq:frontier-endpoints-rho-R} gives
\(g_{\min}=1/(1+R)=(1+o(1))/R\) and
\(t^\star=((R-\rho)/(\rho(1+R)))^2=(1+o(1))/\rho^2\).
Comparing these exact
endpoints and their asymptotic forms with
\eqref{eq:one-pass-fourth-quadrant} gives the common factor
\(\tfrac14\log R\,(1+o(1))\).
\end{proof}

The pass ends with calibration \(r_n\to1\) by
\eqref{eq:one-pass-mu-decay}, and with
\(r_n=1-e^{-\Lambda}(1+o(1))\) under the stronger hypothesis of
\eqref{eq:one-pass-contraction}.
That is far past the interpolator's
\(r_{\rm int}=\rho/(1+\rho)\to0\) of \eqref{eq:rint-def}.
Moreover, every
training prediction overshoots its own training label: the end-of-pass signed
residual on example \(i\) is
\[
    h_i:=\vec w_n^{\top}\vec x_i-1=-\mu^{n}+d\eta\,\mu^{\,i-1}>0
    \qquad(1\le i\le n)
\]
whenever \(1/d<\eta<1/\gamma\), which the tuned rate \(\eta_{\rm pass}\) from \eqref{eq:eta-pass-def} satisfies under the
theorem's hypotheses, by the computation in Appendix~\ref{app:stability-boundaries}.
In this
precise sense, one-pass SGD at a large learning rate has an algorithmic
inductive bias toward the anti-shrinkage of
Subsection~\ref{subsec:scaled-interpolation}.
The dynamics, not an explicit
(anti-)penalty, select the overshooting side of interpolation.

\begin{remark}[The logarithmic gap in Theorem~\ref{thm:one-pass-sgd} is a participation deficit]
\label{rem:one-pass-log-participation}
At the tuned rate \(\eta_{\rm pass}\) of \eqref{eq:eta-pass-def}, the pass carries calibration
budget \(\Lambda_{\rm pass}=\tfrac12\log R\) by \eqref{eq:Lambda-def}, so its
per-update step on the spike is
\begin{equation}\label{eq:eta-pass-step}
    \eta_{\rm pass}\gamma=\frac{\Lambda_{\rm pass}}{n}=\frac{\log R}{2n} .
\end{equation}
At this learning rate \(\eta_{\rm pass}\), \eqref{eq:one-pass-neff-asymp} gives
\begin{equation}\label{eq:eta-pass-participation-deficit}
    n_{\rm eff}\bigl(\vec w_n(\eta_{\rm pass})\bigr)
    =
    \frac{4n}{\log R}\bigl(1+o(1)\bigr),
    \qquad
    \frac{n}{n_{\rm eff}(\vec w_n(\eta_{\rm pass}))}
    =
    \frac14\log R\bigl(1+o(1)\bigr).
\end{equation}
The final SGD iterate also has \(r_n=1-o(1)\) by
\eqref{eq:one-pass-mu-decay}.
Substituting
\eqref{eq:eta-pass-participation-deficit} into the participation form
\eqref{eq:neff-frontier-form} gives the nuisance variance
\(\frac{\log R}{4R}(1+o(1))\) and the amplified nuisance variance
\(\frac{\log R}{4\rho^2}(1+o(1))\) in
\eqref{eq:one-pass-fourth-quadrant}.

The source of the deficit is the geometric profile
\eqref{eq:single-geometric-coeffs}.
Once \(\mu^n=o(1)\) and
\(\eta\gamma=o(1)\), \eqref{eq:one-pass-neff-asymp} becomes
\[
    n_{\rm eff}(\vec w_n)
    =
    \frac{2}{\eta\gamma}\bigl(1+o(1)\bigr).
\]
The length of the pass no longer appears explicitly.
A fixed learning rate
sets a geometric decay length, so later examples receive negligible
coefficients after enough decay lengths have passed.
Increasing the pass
length then improves calibration without restoring a flat coefficient
profile.

This calculation identifies the mechanism behind the logarithm at the
displayed rate.
Can another constant learning rate avoid it?
\end{remark}

It cannot.
Optimizing over stable constant learning rates, \(0<\eta\gamma<2\),
does not eliminate the logarithmic gap as long as \(\log R=o(n)\).
From \eqref{eq:one-pass-exact-risks},
\[
    \Ttest(\vec w_n)\ \ge\ \mu^{2n}+\frac{\Lambda}{2R}\bigl(1-\mu^{2n}\bigr).
\]
For \(\eta\gamma\le\tfrac12\), \(\mu^{2n}\ge e^{-4\Lambda}\), so a sublogarithmic
budget leaves appreciable spike bias.
For \(\eta\gamma>\tfrac12\),
\(\Lambda>n/2\gg\log R\), so either the spike bias \(\mu^{2n}\) is itself
appreciable or the nuisance term is already of order \(\Lambda/R\gg\log R/R\).
Optimizing this tradeoff forces a logarithmic budget, \(\Lambda=\Theta(\log R)\),
and the minimum over stable constant rates is \(\Theta\bigl(\tfrac{\log R}{R}\bigr)\).
Conversely, at \(\eta_{\rm pass}\) --- where \(\eta\gamma\le1\) once
\(n\ge\log R\) --- the spike bias is \(\mu^{2n}\le1/R\) while the nuisance term of
\eqref{eq:one-pass-exact-risks} is at most \(\Lambda_{\rm pass}/R=\frac{\log R}{2R}\),
so the same order is attained.
The logarithm comes from the geometric participation
profile, not from the particular displayed learning rate.
The full calculation,
including the sharp leading constant, is omitted here.

Looking back at Figure~\ref{fig:stylized-sgd-trace}: the plotted run uses
budget \(\Lambda=2\), a constant-factor overshoot of the tuned value
\(\tfrac12\log R\approx1.5\) of \eqref{eq:eta-pass-def} at the figure's
\(R=20\), and its step satisfies \((\eta\gamma)^{2}n=0.04\), so \(\Lambda\) is
the spike contraction exponent \(-n\log|1-\eta\gamma|\) to within \(2\%\).
By
\eqref{eq:one-pass-neff} the tuned rate would participate
\(n_{\rm eff}\approx85\) of the \(n=100\) examples versus the plotted
\(\approx76\), while the asymptotic \(4n/\log R\approx134\) exceeds \(n\)
itself: at this small \(R\) the traced run shows the mechanism without any
observable logarithmic participation deficit that the tuned rate carries at
large \(R\).

Is the logarithmic penalty for one-pass SGD fundamental to all algorithmic
approaches?
It is not.
As Appendix~\ref{app:sgd-refinements}
(Subsection~\ref{subsec:flat-averaging}) shows, the logarithm is a suboptimality of
constant-learning-rate SGD.
A natural sample-splitting scheme --- average all but one
training point and use the held-out point to set the calibration --- pays no
logarithmic penalty and attains the best-span test error up to a \(1+o(1)\) factor.

The tuned rate is also not a knife edge.
Every learning rate in the multiplicatively
wide window \(\frac{1}{\gamma n}\ll\eta\ll\frac{\gamma}{d}\) drives the test error to
zero\footnote{On the stable branch \(0<\eta\gamma\le1\), which Corollary~\ref{cor:learning-rate-robustness}
assumes.
The displayed window does not by itself cap \(\eta\gamma\) when
\(\gamma^{2}/d\) is large.} (Corollary~\ref{cor:learning-rate-robustness} in
Appendix~\ref{app:sgd-refinements}), and any log-spaced hyperparameter sweep whose
range brackets \(\eta_{\rm pass}\) contains a successful learning rate for all
sufficiently large problems\footnote{On the many-sample branch \(\log R=o(n)\),
on which Remark~\ref{rem:constant-factor-lr} operates.} (Remark~\ref{rem:constant-factor-lr}).
There is also a finite-\(n\), non-asymptotic version of this robustness.
For any
target \(0<\varepsilon<1\), every learning rate with \(0<\eta\gamma\le1\) and
\[
    \frac{\log(2/\varepsilon)}{2\gamma n}\ \le\ \eta\ \le\
    \min\Bigl\{\frac{\varepsilon\gamma}{2d},\ \frac{1}{\gamma}\Bigr\}
\]
achieves test error \(\Ttest(\vec w_n)\le\varepsilon\): allocate \(\varepsilon/2\)
to each term of the safe-window bound \(e^{-2\Lambda}+\frac{\Lambda}{R}\) of
\eqref{eq:lr-robust}.
This is a \emph{sufficient} certified set, not the exact set
of successful rates.
It is nonempty exactly when
\(\min\{\varepsilon R,\,2n\}\ge\log(2/\varepsilon)\), and its endpoint ratio is
then \(\min\{\varepsilon R,2n\}/\log(2/\varepsilon)\ge1\).

The calculation also tolerates noisy training labels.
Let the training labels be
\[
    y_i=1+\sigma_\xi\xi_i,
    \qquad
    \xi_i\stackrel{\rm iid}{\sim}\calN(0,1),
\]
write \(\E_\xi\) for expectation over this training-label noise, and write
\(\Ttest^{\rm clean}(\vec w_n)\) for the test risk of the clean-label iterate
(\(\sigma_\xi=0\)) at the same learning rate.
At \(\eta_{\rm pass}\), with
\(\Lambda_{\rm pass}=\tfrac12\log R\), assuming \(R\to\infty\) and
\(\Lambda_{\rm pass}^{2}/n=o(1)\), the clean-test risk satisfies
\begin{equation}\label{eq:noisy-full-pass-risk-main}
    \E_\xi\,\Ttest(\vec w_n)
    =
    \Ttest^{\rm clean}(\vec w_n)
    +
    \sigma_\xi^2\frac{\Lambda_{\rm pass}^2}{R}
    +
    \sigma_\xi^2\frac{\Lambda_{\rm pass}}{2n}\bigl(1+o(1)\bigr).
\end{equation}
The two noise routes are distinct.
The direct term records noisy nuisance energy
written into fresh orthogonal directions, which the isotropic test law averages.
The signal-path term records noise fed through the shared spike recursion,
averaged by the geometric weights over an effective window of about
\(n/\Lambda_{\rm pass}\) labels.
The feedback between the two routes is lower
order under these assumptions.
Growing noise is therefore tolerated exactly when
\(\sigma_\xi^2\Lambda_{\rm pass}^2/R\to0\) and
\(\sigma_\xi^2\Lambda_{\rm pass}/n\to0\): both \(R/\Lambda_{\rm pass}^{2}\) and
\(n/\Lambda_{\rm pass}\) must keep up with the noise variance.
In particular the
expected clean-test risk can vanish even when \(\sigma_\xi^{2}\) diverges with the
dimension.
The clean tuned rate is, however, no longer the right rate when the
noise is strong.
Heuristically, in the strong-signal case \(d\ll\gamma^{2}\),
balancing the spike-bias and signal-path terms
\(e^{-2\Lambda}+\sigma_\xi^{2}\Lambda/(2n)\) gives budget
\(\tfrac12\log(4n/\sigma_\xi^{2})\), which drops below \(\Lambda_{\rm pass}\) once
\(\sigma_\xi^{2}\gamma^{2}/d>4\): a smaller learning rate keeps a longer
moving-average window that better washes out the label noise.

\subsection{Ascending empirical training error}
\label{subsec:sgd-misfit}

Theorem~\ref{thm:one-pass-sgd} says the pass \emph{ends} with large empirical training error.
Something stronger holds at \(\eta_{\rm pass}\) itself, under exactly the hypotheses of Theorem~\ref{thm:one-pass-sgd}:
\(\rho\to0\), \(R\to\infty\), and \((\log R)^{2}=o(n)\).
At that rate the budget is
\(\Lambda=\Lambda_{\rm pass}=\tfrac12\log R\) by \eqref{eq:eta-pass-step}, so
\(\Lambda\to\infty\), \(\Lambda/n\to0\), and \(\Lambda/R\to0\).
It actually turns out that 
the full empirical training error  increases at every single step, while the fresh-test
error decreases at every single step.
Subsection~\ref{subsec:ascent-derivation} of Appendix~\ref{app:sgd-refinements}
carries out the step-by-step computation.

How is it that test error gets small while training error grows for SGD?
Isn't this supposed to be a descent algorithm?
The basic answer is that an
iterative algorithm motivated by descent has no reason to keep descending
once its learning rate is beyond the stability range that matters --- and
what stability actually means in this setting is a bit subtle.
The fresh-versus-reused stability
contrast in Appendix~\ref{app:stability-boundaries} answers this question.
At the tuned rate
of \eqref{eq:eta-pass-def}, a reused training point sees an amplified impact
because its nuisance direction is now stored within the iterate being
updated, driving the training and test errors sharply upward.

A reader who has met the catapult effect of \citet{lewkowycz2020large} may wonder whether that is
what is happening here.
There, the \emph{training} loss first spikes upward {\em and} then falls
back, because the large step eventually throws the iterate into a region of lower curvature.
That is not
what is happening here.
The quadratic geometry in this model is fixed, so there is no curvature
to recover.
The empirical training error grows monotonically because the tuned one-pass rate is
not merely past the repeated-example stability boundary but far past it.
By
\eqref{eq:eta-pass-edge-ratio} the ratio of the tuned rate to that boundary diverges.
That
margin is what lets the stored-nuisance growth term \(\Lambda/\rho^{2}\) dominate the descending
terms in the step difference \eqref{eq:climbing-uphill}.
Merely crossing that boundary would not force the ascent.
At \(\eta_{\rm edge}\) itself the
first update flips the sign of the residual on the example it uses while leaving that residual's
magnitude unchanged.
The same update shrinks the residuals of the \(n-1\) examples whose
nuisance directions the iterate does not yet store.
So the full training loss falls at first,
and by continuity it still falls slightly beyond the boundary.
Instability permits the ascent.
The divergent margin \eqref{eq:eta-pass-edge-ratio} produces it.

Looking back at Figure~\ref{fig:stylized-sgd-trace} with the section's tools:
across the first epoch the test error falls toward the span-optimal scale
\(d/(\gamma^{2}n)=1/R\), while the training error climbs past the
zero-predictor baseline toward the converse scale \(1/\rho^{2}\) whose order
Corollary~\ref{cor:fourth-quadrant-necessity} forces (the stable-rate
comparator reduces training error even under reuse, but it barely moves the
shared spike, so its test error stays near the baseline).
Both reference lines are asymptotic scales.
At the figure's finite
\(R=20\) and \(\rho=0.2\) the exact endpoints of
\eqref{eq:frontier-endpoints-rho-R} are \(g_{\min}=1/21\) and
\(t^\star\approx22.2\), about \(5\%\) and \(12\%\) from the plotted
\(1/R=0.05\) and \(1/\rho^{2}=25\).  
Section~\ref{sec:discussion} returns to this Figure's trace in order to discuss a diagnostic split:
fresh-example diagnostics can improve even as the final empirical training
loss silently rises.

\paragraph{Why the large learning rate helps here.}
There is a particularly simple sense in which a larger learning rate can
help generalization here.
It does not help by finding a better empirical minimizer.
In the benign-misfitting window, interpolation has zero empirical error and is
therefore a global empirical-risk minimizer.
But the minimum-norm
interpolating point in the training span is the wrong point for test error.
The useful one-pass iterate is different in kind: it has large empirical error
and improves test error by refusing empirical minimization as the target.
Concretely, the larger one-pass step refuses the per-example zero-residual fit:
it overshoots each training point's nuisance direction, but over one pass it
builds up the correct coefficient on the shared spike.

The size of this overshoot is not arbitrary:
\eqref{eq:fourth-quadrant-training-lower} in
Corollary~\ref{cor:fourth-quadrant-necessity} shows that a span predictor with
small test error must have training error of order at least \(1/\rho^2\), and
hence a root-mean-square training residual of order at least \(1/\rho\).
In many familiar
large-learning-rate explanations, the large step helps choose, sample, or
stabilize near a better minimum.
Here it helps because the useful span
predictor is not an empirical-risk minimizer at all.%
\footnote{This is also why one must be careful with per-example gradient
clipping.
The hazard is not a large pre-update residual.
On the fresh
pass, the residual driving the update on example \(i\) is the controlled
shared spike residual \(b_{i-1}\) of \eqref{eq:sgd-state-defs}.
The large
training residual on that example is \emph{created} by its own update,
through the stored coefficient, and is seen only afterward or upon reuse.
What a clipping rule does see is the gradient norm
\(b_{i-1}\|\vec x_i\|_2=b_{i-1}\sqrt{\gamma+d}\), large because the feature
vector is long.
Too aggressive a threshold shrinks the stored coefficient
below the \(\eta\,b_{i-1}\) of \eqref{eq:single-geometric-coeffs} that the
one-pass calculation needs --- in effect cutting the learning rate on the
fresh steps that build calibration --- while a threshold that only catches
the far larger reuse-epoch spikes would safely leave the first pass untouched.}
Subsection~\ref{subsec:related-large-steps} contrasts this elementary mechanism with
other large-learning-rate phenomena.

\section{Adversarial sensitivity}
\label{sec:robustness}

Let \(\vec{\theta}\in\R^{d+1}\) denote the unit spike direction.
In the aligned coordinates of
Sections~\ref{sec:one-sample}--\ref{sec:sgd}, \(\vec{\theta}=\vec{e}_1\).
The true label is \(y=\vec{\theta}^{\top}\vec{x}/\sqrt{\gamma}\).
We now use the nuisance
residue to ask how sensitive a learned predictor is to an adversary.
To make
this a fair question we fix the truth: the adversary may perturb the input, but
only in directions that leave the true label unchanged.
A large output change
is then attributable to the learned predictor's sensitivity, not to a change in
the correct answer.
We call such a perturbation \emph{label-preserving}.
Since
the label is the spike coordinate, label-preserving perturbations are exactly
those orthogonal to the spike direction.
The predictor's \emph{adversarial
sensitivity}\footnote{In classification, an adversarial example is often
described as a small perturbation that preserves the intended label while
changing the model's output.
Here that intended-label condition is explicit
rather than informal: $y=\vec{\theta}^{\top}\vec{x}/\sqrt{\gamma}$, so
$\vec{\Delta}\perp\vec{\theta}$ preserves the label exactly.} is the
largest output change a label-preserving perturbation of a given size can cause
\cite{szegedy2014intriguing,goodfellow2015explaining,madry2018towards,tsipras2019robustness,ilyas2019adversarial}.
From here on, unperturbed test error means the original test error \(\Ttest\).

We first isolate the nuisance residue, compute its worst label-preserving
response, and read off the resulting sample thresholds
(Sections~\ref{subsec:label-preserving}--\ref{subsec:phase-ladder}).
We then examine training-point witnesses
(Section~\ref{subsec:training-adversarial}).

\subsection{Label-preserving perturbations and the role of nuisance residue}
\label{subsec:label-preserving}

A perturbation is \emph{label-preserving} when it moves the input without
changing the correct answer --- here that means moving along nuisance
directions, leaving the spike coordinate alone.
Because the observed
coordinates need not align with the latent spike, we state the calculation in
coordinate-free form.
Appendix~\ref{app:rotation-aware} uses the same representation to analyze
coordinatewise $\ell_\infty$ perturbations after a generic rotation.
Write the training points as
\begin{equation}\label{eq:rotation-aware-data}
    \vec{x}_i=\sqrt{\gamma}\,\vec{\theta}+\vec{v}_i,
    \qquad
    \vec{v}_i\perp\vec{\theta},
    \qquad
    \vec{v}_i^{\top}\vec{v}_j=d\,\mathbf{1}\{i=j\}.
\end{equation}
For a span predictor \(\vec{u}=\Xt\vec{\alpha}\), recall that its
\emph{spike calibration} and \emph{nuisance residue} are defined by
\begin{equation}\label{eq:calibration-residue}
    r:=\gamma\one^\top\vec{\alpha},
    \qquad
    \widetilde{\vec{v}}:=\sum_i\alpha_i \vec{v}_i.
\end{equation}
Then
\begin{equation}\label{eq:rotation-decomposition}
    \vec{u}=\frac{r}{\sqrt{\gamma}}\vec{\theta}+\widetilde{\vec{v}},
    \qquad
    \widetilde{\vec{v}}\perp\vec{\theta},
    \qquad
    \Ttest(\vec{u})=(r-1)^2+\|\widetilde{\vec{v}}\|_2^2.
\end{equation}
The oracle is \(\vec{\theta}/\sqrt{\gamma}\), corresponding to \(r=1\) and
\(\widetilde{\vec{v}}=\vec{0}\).
Meanwhile, every span predictor with nonzero calibration is forced to carry a nuisance residue
\(\widetilde{\vec{v}}\).
At fixed \(r\), Cauchy--Schwarz gives
\begin{equation}\label{eq:fixed-r-nuisance-lower}
    \|\widetilde{\vec{v}}\|_2^2\ge \frac{dr^2}{\gamma^2n},
\end{equation}
with equality for symmetric coefficients.

\subsection{RMS adversarial sensitivity}
\label{subsec:rms-sensitivity}

\begin{theorem}[RMS adversarial sensitivity]
\label{thm:absolute-rms}
Let \(\vec{u}\) be a span predictor and let \(\delta\ge0\).
Define the \emph{root-mean-square (RMS) adversarial sensitivity} of
\(\vec{u}\) at per-coordinate perturbation size \(\delta\) by
\begin{equation}\label{eq:A2-def}
    A_2(\vec{u};\delta)
    :=
    \sup_{\vec{\Delta}\perp\vec{\theta},\ \|\vec{\Delta}\|_2\le\delta\sqrt d}
    |\langle \vec{u},\vec{\Delta}\rangle|.
\end{equation}
Here the radius \(\delta\sqrt d\) treats \(\delta\) as a per-coordinate RMS
perturbation size\footnote{Using a fixed Euclidean radius would shift the
sample thresholds in \eqref{eq:robustness-thresholds} by powers of \(d\).}.
Then
\begin{equation}\label{eq:A2-response}
    A_2(\vec{u};\delta)=\delta\sqrt d\,\|\widetilde{\vec{v}}\|_2.
\end{equation}
At fixed calibration \(r\), the minimum possible value for \(A_2(\vec{u};\delta)\) over span predictors \(\vec{u}\) is
\begin{equation}\label{eq:A2-fixed-r-response}
    \delta |r|\frac{d}{\gamma\sqrt n}.
\end{equation}
Consequently, among span predictors at fixed calibration \(r\), keeping the
adversarial sensitivity below output tolerance \(\tau>0\) requires
\begin{equation}\label{eq:A2-sample-lower}
    n\ge \frac{\delta^2 r^2 d^2}{\gamma^2\tau^2}.
\end{equation}
\end{theorem}

\begin{proof}
Because \(\vec{\Delta}\perp\vec{\theta}\),
\(\langle \vec{u},\vec{\Delta}\rangle=\langle \widetilde{\vec{v}},\vec{\Delta}\rangle\).
By Cauchy--Schwarz,
the supremum over \(\|\vec{\Delta}\|_2\le\delta\sqrt d\) is achieved when
\(\vec{\Delta}\) is aligned with \(\widetilde{\vec{v}}\), giving
\[
    A_2(\vec{u};\delta)=\delta\sqrt d\,\|\widetilde{\vec{v}}\|_2.
\]
By \eqref{eq:fixed-r-nuisance-lower}, every span predictor with calibration
\(r\) has \(\|\widetilde{\vec{v}}\|_2\ge |r|\sqrt d/(\gamma\sqrt n)\), with equality for
symmetric coefficients.
Combining this lower bound with
\eqref{eq:A2-response} gives the fixed-calibration sensitivity
\eqref{eq:A2-fixed-r-response}.
Requiring \eqref{eq:A2-fixed-r-response} to
be at most \(\tau\) gives the sample threshold.
\end{proof}

Appendix~\ref{app:rotation-aware} repeats this calculation for coordinatewise
\(\ell_\infty\) perturbations after a generic rotation, where the same nuisance
residue is what the perturbation aligns with.

The whole adversarial-sensitivity calculation is the following norm comparison:
\[
    \begin{aligned}
    \text{fresh random test error contribution: }&\|\widetilde{\vec{v}}\|_2^2,\\
    \text{worst label-preserving adversarial output change: }&
    \delta\sqrt d\,\|\widetilde{\vec{v}}\|_2.
    \end{aligned}
\]
This is the adversarial version of residue amplification:
random testing averages the residue, while the adversary aligns with it.

Thus the minimum achievable sensitivity at calibration \(r\), attained by
the symmetric profile, is linear in \(|r|\): better signal calibration hands
an aligned perturbation a larger nuisance residue to exploit.
The only way to have both small unperturbed test error and low
adversarial sensitivity is to have enough samples that the nuisance residue
itself is small --- which is precisely the gap between \(n_{\rm clean}\) and
\(n_{\rm adv}\) in \eqref{eq:robustness-thresholds}.

\subsection{Sample thresholds and the robustness phase ladder}
\label{subsec:phase-ladder}

A clean-test target gives a calibration floor with no optimization: if a span
predictor has \(\Ttest(\vec u)\le\varepsilon<1\), then
\(\|\widetilde{\vec v}\|_2^2\ge0\) and \eqref{eq:rotation-decomposition} give
\((r-1)^2\le\varepsilon\), hence
\begin{equation}\label{eq:clean-calibration-floor}
    r\ge 1-\sqrt\varepsilon>0.
\end{equation}

For a fixed clean-test target \(0<\varepsilon<1\), perturbation radius \(\delta>0\),
and output tolerance \(\tau>0\), three real-valued sample-count thresholds organize the calculation.
Let \(n_{\rm clean}\) denote the sample threshold at which the best span predictor
first reaches the clean target.
This is the threshold
\(n_{\rm mis}(\varepsilon)\) of Corollary~\ref{cor:sample-complexity}, of order
\(d/\gamma^2\).
Let \(n_{\rm int}\) denote the corresponding interpolation threshold, likewise
the exact threshold of that corollary.
It is of order \(d/\gamma\).
Finally, let \(n_{\rm adv}(\varepsilon,\delta,\tau)\) denote the sample-count
threshold at which some span predictor first meets the clean target and the
output tolerance simultaneously.
Its lower bound comes from
\eqref{eq:A2-sample-lower} after using \(r\ge 1-\sqrt\varepsilon\) from
\eqref{eq:clean-calibration-floor}.
Suppressing constants depending only on the fixed target,
\begin{equation}\label{eq:robustness-thresholds}
    n_{\rm clean}\asymp \frac{d}{\gamma^2},
    \qquad
    n_{\rm int}\asymp \frac{d}{\gamma},
    \qquad
    n_{\rm adv}(\varepsilon,\delta,\tau)
    \gtrsim \frac{\delta^2 (1-\sqrt\varepsilon)^2 d^2}{\gamma^2\tau^2}
    \asymp \frac{\delta^2d^2}{\gamma^2\tau^2}.
\end{equation}
Actual sample counts are obtained by rounding up and respecting the boundary
\(n\le d\).
This lower bound is attained at the same order.
The symmetric profile
attains equality in \eqref{eq:fixed-r-nuisance-lower}.
At full calibration
\(r=1\), it has clean test error \(d/(\gamma^{2}n)\) and adversarial
sensitivity exactly \(\delta d/(\gamma\sqrt n)\) by \eqref{eq:A2-response}.
It therefore meets both the clean target and the tolerance once \(n\) exceeds
a constant multiple of \(n_{\rm clean}\) and of
\(\delta^{2}d^{2}/(\gamma^{2}\tau^{2})\).
For \(\tau\le\delta\sqrt d\) the
second requirement dominates, so
\(n_{\rm adv}\asymp\delta^{2}d^{2}/(\gamma^{2}\tau^{2})\) with constants
depending only on the targets.
The ratio
\[
    \frac{n_{\rm adv}}{n_{\rm clean}}
    \asymp
    \frac{\delta^2 d}{\tau^2}
\]
shows that the extra samples needed for adversarial robustness scale with the
ambient dimension \(d\), the factor behind the norm
comparison in Section~\ref{subsec:rms-sensitivity}.\footnote{Related clean-versus-robust
sample-complexity and universal-law comparisons are discussed in
Section~\ref{subsec:related-adversarial}.}
Read from left to right as \(n\) grows.
Before \(n_{\rm clean}\), the span is not yet useful for clean prediction.
Up to a constant fraction%
\footnote{\label{fn:rho-star}Quantitatively,
\eqref{eq:clean-calibration-floor} makes \(r\ge1-\sqrt\varepsilon\) necessary
for the clean target, while \eqref{eq:calibrated-frontier-form} together with
\(q\ge s^2/n\) makes \(r\le 2\rho/(1+\rho)\) necessary for \(\Ttrain\le1\).
The two are incompatible precisely when
\(\rho<\rho^{\star}:=(1-\sqrt\varepsilon)/(1+\sqrt\varepsilon)\), so no span
predictor meets the clean target inside the good-fit quadrant while
\(\gamma n/d<\rho^{\star}\).
For \(\gamma=100\), \(d=50{,}000\) and the target
\(\varepsilon=0.2525\) of Footnote~\ref{fn:threshold-numbers}, this gives
\(\rho^{\star}\approx0.33\), against \(\rho=1\) at \(n_{\rm int}\).}
of \(n_{\rm int}\), useful prediction lies only in the fourth quadrant.
The
good-fit quadrant is empty of useful predictors there because good spike
calibration forces large empirical training error.
Closer to \(n_{\rm int}\)
the good-fit quadrant fills in as well.
After \(n_{\rm int}\), interpolation can have small clean test error.
Only after \(n_{\rm adv}\) can a calibrated useful span predictor have that
residue averaged down enough for small aligned adversarial sensitivity.

This ladder orders thresholds within the stylized orthogonal model's boundary
\(n\le d\).
Depending on the scaling, adjacent thresholds may be close, may
trade places, or may lie beyond the boundary.
The pictured order holds when
the tolerance \(\tau\) sits below \(\delta\sqrt{d/\gamma}\), the sensitivity
that calibrated span predictors still have at the interpolation threshold
\eqref{eq:A2-fixed-r-response}, yet above \(\delta\sqrt d/\gamma\), the
sensitivity they retain at the boundary \(n=d\).
Its invariant message is that small aligned
sensitivity requires additional ambient-dimensional averaging beyond clean
random-test accuracy.%
\footnote{If the initial ``span not useful'' phase is
asymptotically long, then \(n_{\rm adv}\) lies beyond the boundary.
Conversely, if the
adversarial-sensitivity transition occurs before \(n=d\), then
\(n_{\rm clean}\) is already only constant-order.
Up to constants depending on
\(\varepsilon,\delta,\tau\), this is the condition \(d\lesssim\gamma^{2}\), the
upper half of the few-shot window \(\gamma\ll d\ll\gamma^{2}\) of
Footnote~\ref{fn:few-shot-window}.
With the lower half \(\gamma\ll d\) in
force as well, the spike is few-shot-strong, meaning strong enough that a
handful of examples already generalizes while
interpolation still does not.}

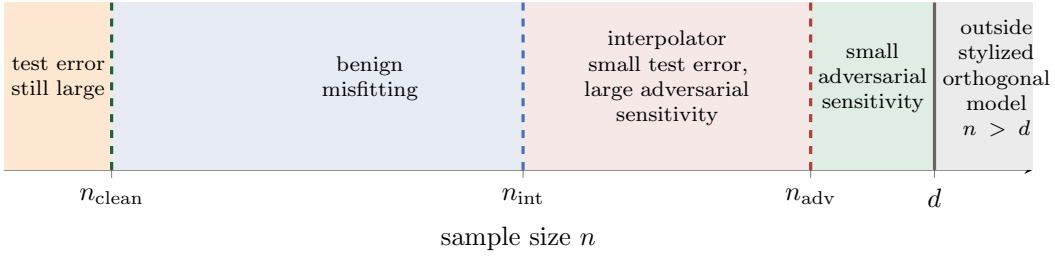
\begin{figure}[htbp]
    \centering
    \begin{tikzpicture}
    \definecolor{goodgreen}{RGB}{48,135,84}
    \definecolor{alertred}{RGB}{178,62,55}
    \definecolor{softblue}{RGB}{76,112,176}
    \begin{axis}[
        width=0.92\textwidth,
        height=0.23\textwidth,
        xmode=log,
        xmin=3e-5, xmax=3,
        ymin=0, ymax=1,
        ytick=\empty,
        xtick={1e-4,1e-2,0.25,1},
        xticklabels={$n_{\rm clean}$,$n_{\rm int}$,$n_{\rm adv}$,$d$},
        tick label style={font=\small},
        xlabel={sample size \(n\)},
        label style={font=\small},
        axis x line=bottom,
        axis y line=none,
        clip=false,
        major tick length=6pt,
    ]
        \fill[orange!18] (axis cs:3e-5,0) rectangle (axis cs:1e-4,1);
        \fill[softblue!14] (axis cs:1e-4,0) rectangle (axis cs:1e-2,1);
        \fill[alertred!12] (axis cs:1e-2,0) rectangle (axis cs:0.25,1);
        \fill[goodgreen!15] (axis cs:0.25,0) rectangle (axis cs:1,1);
        \fill[black!8] (axis cs:1,0) rectangle (axis cs:3,1);
        \draw[goodgreen!70!black,very thick,dashed] (axis cs:1e-4,0) -- (axis cs:1e-4,1);
        \draw[softblue,very thick,dashed] (axis cs:1e-2,0) -- (axis cs:1e-2,1);
        \draw[alertred,very thick,dashed] (axis cs:0.25,0) -- (axis cs:0.25,1);
        \draw[black!60,very thick] (axis cs:1,0) -- (axis cs:1,1);
        \node[font=\scriptsize,align=center,text width=1.7cm] at (axis cs:5.5e-5,0.55)
            {test error\\still large};
        \node[font=\scriptsize,align=center,text width=2cm] at (axis cs:1.8e-3,0.55)
            {benign\\misfitting};
        \node[font=\scriptsize,align=center,text width=2.8cm] at (axis cs:5e-2,0.55)
            {interpolator\\small test error,\\large adversarial\\sensitivity};
        \node[font=\scriptsize,align=center,text width=2.0cm] at (axis cs:0.5,0.55)
            {small\\adversarial\\sensitivity};
        \node[font=\scriptsize,align=center,text width=2.0cm] at (axis cs:2,0.55)
            {outside stylized\\orthogonal model\\$n>d$};
    \end{axis}
    \end{tikzpicture}
    \caption{Schematic phase ordering for sample size \(n\), not to scale,
    pictured in the demanding-tolerance regime
    \(\delta\sqrt d/\gamma\ll\tau\ll\delta\sqrt{d/\gamma}\),
    where \(n_{\rm adv}\) lies past \(n_{\rm int}\) but before the
    boundary \(d\).
    Clean span prediction begins at \(n_{\rm clean}\).
    For \(n\) up to a
    constant fraction of \(n_{\rm int}\), the useful span predictors are
    \emph{only} in the benign-misfitting quadrant, because within the span
    calibrating the spike well enough to be useful forces large empirical
    training error.
    Interpolation reaches small clean test error at
    \(n_{\rm int}\) but can still have large adversarial sensitivity until the
    adversarial-sensitivity threshold \(n_{\rm adv}\).
    The final gray region
    begins at \(n>d\), where the stylized orthogonal setup used in this
    paper no longer applies.}
    \label{fig:robustness-phases}
\end{figure}
\FloatBarrier

\subsection{Training points as adversarial examples}
\label{subsec:training-adversarial}

This adversarial sensitivity is already visible in the training data.
For
the symmetric predictor at calibration \(r\), each training point has residual
\begin{equation}\label{eq:robust-training-residual}
    \vec{u}^{\top}\vec{x}_i-1=(r-1)+\frac{r}{\rho},
    \qquad
    \rho=\frac{\gamma n}{d}.
\end{equation}
For a unit-RMS probe along \(\vec{v}_i\), the signed term \(r/\rho\) is exactly the
signed output change from that point's nuisance direction.
More generally,
since \(\|\delta \vec{v}_i\|_2=\delta\sqrt d\) and \(\vec{v}_i\perp\vec{\theta}\), the
perturbation \(\delta \vec{v}_i\) is allowed under \eqref{eq:A2-def} and gives
\[
    \left\langle \widetilde{\vec{v}},\delta \vec{v}_i\right\rangle
    =
    \delta\frac{r}{\rho},
    \qquad
    \left|\left\langle \widetilde{\vec{v}},\delta \vec{v}_i\right\rangle\right|
    =
    \delta\frac{|r|}{\rho}.
\]
In the benign-misfitting window \eqref{eq:misfitting-window}, on the useful
branch \(r\approx1\), the nuisance term \(r/\rho\) in
\eqref{eq:robust-training-residual} dominates the empirical residual.
The
worst RMS adversary combines the \(n\) orthogonal training-point nuisance
projections and gains a factor \(\sqrt n\), giving the adversarial sensitivity
\(\delta |r| d/(\gamma\sqrt n)\) in Theorem~\ref{thm:absolute-rms}.

In this precise sense, within the benign-misfitting window the training points
are their own adversarial examples: each one exposes a label-preserving nuisance
direction along which the calibrated span predictor is sensitive, and the
predictor gets it wrong.
The misfit is not incidental --- it is the predictor's
adversarial sensitivity, read off one training point at a time.
Note that the same residuals also carry a privacy-relevant signature.
Within the window a
training set member's residual is unusually \emph{large}, so large errors can identify
training membership as effectively as unusually small ones.

The same object is being measured in four ways.
Training points see the
residue in their own directions.
Random test points average it.
Adversarial
perturbations align with it.
Privacy noise has to be big enough to hide it.
Thus the coordinate calculation
of Appendix~\ref{app:rotation-aware} is not tied to a special axis-aligned
oracle: after a generic rotation, the same span geometry is dense in the
observed basis.

\section{Discussion: attractive but hard to see}
\label{sec:discussion}

The fourth quadrant is not merely strange; it is attractive.
A method that
targets test error is pulled toward the benign-misfitting window by a diverging
sample-complexity advantage: small test error becomes possible at
\(d/\gamma^2\), while interpolation must wait until \(d/\gamma\).
The
difficulty is what happens once a method is there -- the regime can be hard to
see in the curves a practitioner normally plots during training.

One pass of SGD exposes three distinct empirical observables: the running
pre-update loss seen during the pass, the final iterate's loss on the
training set the pass has already consumed, and a held-out test loss.
In the
deterministic
orthogonal model, the pre-update loss at step \(i\) evaluates
\(\vec{w}_{i-1}\) on a training point whose nuisance direction has not yet
appeared in the iterate.
It sees the current spike residual \(b_{i-1}\), but
not the nuisance variance already written into earlier orthogonal directions.
Thus a running-loss curve can look good while the final iterate has large
residuals on the examples that produced it: online loss tests fresh nuisance
directions, while final training loss tests directions already written into
the predictor.
Figure~\ref{fig:stylized-sgd-trace} shows this split
directly: the fresh-test curve falls during the first epoch, the full
empirical training error (dashed red) rises, while the dash-dotted pre-update
current-example residual --- precisely the running-loss diagnostic --- stays
low throughout the first epoch.
But reuse in later epochs exposes the hidden residuals as
spikes in that running-loss itself.

Put plainly: a held-out test loss can look small, the running training loss
can look small, and only the trained model evaluated on its own training set
would report the large residuals.
Standard training traces usually show the
first two diagnostics, not the third.
If such an example is reused at the same
large learning rate, the hidden residual can surface as a gradient
spike%
\footnote{The same sign reversal has a restricted privacy analogue: members can
be recognizable through large residuals rather than small losses.}
--- a large residual on a training point \emph{is} a spike in that point's
gradient, and these are exactly the loss spikes visible in the reuse epochs of
Figure~\ref{fig:stylized-sgd-trace}.

A practitioner may simply say: ``gradient spikes happen; that is why we clip.''
In this geometry, clipping is not an innocuous modification: a threshold
aggressive enough to shrink the fresh-pass updates suppresses precisely the
overshoots that carry the calibration mechanism.

Because of this diagnostic split, the phenomenon is also easy to miss during hyperparameter search.
Learning rates are often swept over orders of magnitude.
A successful large
step may be found empirically \cite{gilmer2022loss} before one has a reason
to expect that it is stable for fresh examples but unstable for reused ones.
Corollary~\ref{cor:learning-rate-robustness} explains why this is
compatible with ordinary log-spaced tuning: in the simple \(\eta\gamma\le1\)
branch, a grid need only land in the broad interval
\[
    \frac{1}{\gamma n}\ll\eta\ll\frac{\gamma}{d}.
\]
An explicit finite-\(n\) certified interval of learning rates for a given
target is also available, with endpoint ratio
\(\min\{\varepsilon R,2n\}/\log(2/\varepsilon)\), and
Remark~\ref{rem:constant-factor-lr} explains the constant-factor
behavior\footnote{With noisy training labels the same lower edge holds, with extra
upper bounds by \eqref{eq:noisy-lr-bound-main}: the rate must also avoid writing
too much noisy nuisance energy and too much signal-path noise.}.

The split between fresh-example loss and reused-example loss becomes an exact
conditional-mean identity in the
same-distribution companion model.
The next i.i.d.\ labeled training example
is independent of the previous iterate.
Conditional on that iterate, its
expected pre-update loss equals that iterate's test error.
Yet the final
iterate's training
loss --- measured on all training points \emph{after} they have been incorporated ---
can still diverge.
The broader mechanism is that a low-dimensional informative
direction can be hidden inside a high-dimensional nuisance span, and early span
predictors that improve fresh-point prediction may need to behave differently
on training points than on fresh random points.

So are the training points misleading?
Misleadingness, like beauty, is in the
eye of the beholder.
Here, the beholder is a learner that treats empirical fit
inside the training span as the target.
The training examples are not
misleading because their labels are corrupted or because they are intrinsically
harder than test examples.
After all, the same spike oracle fits them
exactly.
They
mislead only a learner that conflates fitting the training points with doing
well on new ones.
Such a learner under-calibrates the spike, because the
sample-specific nuisance directions help fit the individual training points.
Those same directions cannot help on fresh points.
The way to use the
shared signal well inside the span is to accept residuals on the very
examples that revealed it.

This preprint is the first of four.
Its purpose is to tell the core story in
the simplest exact example.
The second preprint \citep{ranade2026samedistribution} draws training
and test data from the same Gaussian distribution and shows that the threshold
separation, forced training misfit, and one-pass SGD mechanism persist without
exact orthogonality.
Even with random training points, the
minimum-norm interpolator points in the right direction but stops at the wrong
radius.

The third preprint \citep{ranade2026multispike} allows several informative
directions of unequal strength.
The basic span obstruction and fresh-pass
mechanism survive, but a new issue appears: one common rescaling cannot
generally calibrate every direction by the amount it needs, and one learning
rate must serve them all.
But large-learning-rate SGD can still work and have
the right kind of anti-shrinkage implicit bias in many regimes.

The fourth preprint \citep{ranade2026oja} studies
the analogous unsupervised question through Oja's classical online method for
principal component analysis \cite{oja1982simplified}.
The present paper
stays with the deterministic model because it makes the shared geometry
visible without concentration arguments.

\section{Related perspectives}
\label{sec:related}

We group related work by how each perspective relates to the train--test plane,
the span geometry, the large-step dynamics, and the adversarial/privacy
readings developed above.
The subsections are largely modular: readers may enter through the perspective
closest to their background.

\subsection{Train--test taxonomy and generalization gaps}

\paragraph{Benign overfitting, double descent, and the train--test plane.}
The benign-overfitting literature asks how interpolation can coexist with small
test error.
Useful entry points include survey and overview treatments
\cite{belkin2021fit,dar2021farewell}.
The double-descent viewpoint has been
developed for linear least squares and related overparameterized models
\cite{belkin2019reconciling,belkin2020twomodels}, broadened to kernels and
nonparametric rules
\cite{belkin2018kernel,belkin2018overfitting,belkin2019does}, and made
precise for linear regression
\cite{bartlett2020benign,hastie2022surprises}.
Among the closest antecedents to the present setting
are harmless-interpolation analyses
\cite{muthukumar2019harmless,muthukumar2020harmless}.
Empirical studies,
jamming-transition models, teacher--student setups, and historical reviews
provide complementary context
\cite{nakkiran2020deep,spigler2019jamming,advani2020highdim,
loog2020prehistory}.

The question here is complementary.
Classical bias--variance and double-descent
curves plot test error against model complexity, a different pair of axes from
the training-error/test-error plane used here.
In this paper the
training-span class is fixed, and the question is where predictors in that one
class sit in the plane of training error and test error.
The benign-overfitting literature showed that interpolation can still
generalize.
The point here is that good test error can require moving
\emph{above} the training-error
baseline.
Tradeoff curves between training and test error appear there too:
near-interpolators with training error below the noise floor pay rapid norm
growth \citep{wang2024nearinterp}, and optimistic-rates bounds apply at any
prescribed training-error level \citep{zhou2024optimistic}.
Norm growth has
a deterministic analogue here: by Appendix~\ref{app:sco-comparison}, useful
span predictors must leave the spike oracle's Euclidean norm ball by a factor
of order \(1/\rho\) in squared norm.
The frontier itself is different in
kind: it is an exact converse forcing useful predictors above the
zero-predictor baseline.
In the benign, tempered, and catastrophic taxonomy of
\citet{mallinar2022taxonomy}, which classifies positive-gap interpolation by
how much it costs in test error, the fourth quadrant is unusual along an
orthogonal axis: its generalization gap is \emph{negative}, with empirical
training risk exceeding test risk.

\paragraph{Uniform convergence and the wrong empirical risk diagnostic.}
Uniform-convergence and stability analyses seek to explain generalization by bounding
the gap between population risk and empirical risk, sometimes even for
interpolators in benign-overfitting settings \cite{koehler2021uniform}.
Nagarajan and Kolter give a different warning: for some overparameterized models,
uniform-convergence bounds can be unable to explain the predictors selected by
gradient descent \cite{nagarajan2019uniform}.
The present calculation is a
separate, more literal warning about the empirical risk diagnostic itself.
In the
benign-misfitting window, the useful span predictors are not empirically good
predictors with a hard-to-bound gap.
They have deliberately bad empirical
error.
Thus the issue is not the tightness of a generalization-gap bound;
it is that small empirical risk can point to the wrong
region of the train--test plane.

\paragraph{When the implicit regularization is too strong.}
The benign-overfitting question is not why an overparameterized model can reach
zero training error; that is the premise.
The question is why doing so need
not harm test error.
In the linear benign-overfitting picture, the answer is
that the high-dimensional tail can act as an implicit regularizer: when the
covariance geometry is favorable, interpolation uses many nuisance directions
in a way that absorbs label noise or sample fluctuations at small test cost.
Benign overfitting occurs when this implicit regularization is at the right
level.

The benign-misfitting window is the opposite side of the same geometry.
By \(r=c\,r_{\rm int}\) from \eqref{eq:c-r-relation}, setting $c=1$ shows that the minimum-norm
interpolator has spike calibration
\[
    r_{\rm int}
    =
    \frac{\gamma n}{d+\gamma n}
    =
    \frac{\rho}{1+\rho}.
\]
Thus when $\rho\ll1$, interpolation severely under-calibrates the spike.
The nuisance directions are not failing to regularize.
Instead, they are
over-regularizing the signal.

\paragraph{Two effective ranks.}
Write $\|\bm{A}\|_{\rm op}:=\sup_{\|\vec{x}\|_2=1}\|\bm{A}\vec{x}\|_2$ for the
operator norm.
For symmetric $\bm{A}$ this equals the largest absolute
eigenvalue, and for a covariance the largest eigenvalue.
\citet{bartlett2020benign} describe a spectral tail with two effective-rank
families: the first measures how broad the remaining tail is relative to its
largest direction, and the second how many directions effectively share its
mass.
Throughout this discussion, the \emph{subscripted} symbols \(r_k,R_k\)
denote those effective ranks.
They are unrelated to this paper's un-subscripted
spike calibration \(r\) and signal-to-nuisance ratio \(R=\gamma^2n/d\) from
\eqref{eq:scale-identities}.

For the finite-dimensional covariance matrices considered here, let
\(\lambda_1\ge\lambda_2\ge\cdots\ge0\) be the eigenvalues.%
\footnote{\citet[Definition~3]{bartlett2020benign} write the families as
\(r_k\) and \(R_k\) with \(\lambda_1\ge\lambda_2\ge\cdots\) for the
covariance eigenvalues; we match that notation so quantities can be compared
directly across papers.
Within the present paper, this eigenvalue notation
\(\lambda_j\) is unrelated to the ridge parameter \(\lambda\) and the
negative-ridge value \(\lambda^{\star}\) of \eqref{eq:lambda-star}.}
For every \(k\ge0\) with \(\lambda_{k+1}>0\), define
\begin{equation}\label{eq:effective-rank-families}
    r_k(\bm{\Sigma}):=\frac{\sum_{j>k}\lambda_j}{\lambda_{k+1}},
    \qquad
    R_k(\bm{\Sigma}):=\frac{\bigl(\sum_{j>k}\lambda_j\bigr)^2}
                          {\sum_{j>k}\lambda_j^2}.
\end{equation}
Here \(r_k\) is a mass-to-peak ratio---the spectral mass remaining after the
first \(k\) directions, relative to the largest remaining eigenvalue---while
\(R_k\) is a participation ratio---how many directions effectively share that
remaining mass.

\paragraph{Scholarly placement.}
In \citet{bartlett2020benign}, the \(r_k\) family helps select a critical
head--tail split and \(R_k\) controls participation in the selected tail.
One of their sufficient benignity conditions is
\(r_0(\bm{\Sigma})/n\to0\), and the present window sits on the opposite side
of it.\footnote{Applied under the conditions of benign misfitting, where
\(r_0(\bm{\Sigma})/n\to\infty\), their critical-rank selection rule would in
fact return \(k^\star=0\), placing the signal coordinate itself in the tail.}
The later analysis of \citet{nr_tsigler2023} retains this head--tail
decomposition but makes the conditioning of a regularized tail training Gram
matrix the central structural object.
This removes the independent-coordinate
assumption, supplies sharp signal-dependent bias bounds, and extends the
analysis to ridge regression.
Both are same-distribution random-design
analyses, in contrast with the fixed-design train--test geometry studied
here.

\paragraph{The flat single-spike covariance.}
For
\(\bm{\Sigma}=\diag(\gamma,1,\ldots,1)\), for which
\(\operatorname{tr}(\bm{\Sigma})=\gamma+d\), \(\|\bm{\Sigma}\|_{\rm op}=\gamma\),
and \(\operatorname{tr}(\bm{\Sigma}^2)=\gamma^2+d\), the \(k=0\) quantities are
\begin{equation}\label{eq:flat-effective-ranks}
\begin{aligned}
    r_0(\bm{\Sigma})
    &=
    \frac{\operatorname{tr}(\bm{\Sigma})}{\|\bm{\Sigma}\|_{\rm op}}
    =1+\frac d\gamma,\\
    R_0(\bm{\Sigma})
    &=
    \frac{\operatorname{tr}(\bm{\Sigma})^2}
         {\operatorname{tr}(\bm{\Sigma}^2)}
    =\frac{(\gamma+d)^2}{\gamma^2+d}.
\end{aligned}
\end{equation}
Substituting the two members of \eqref{eq:flat-effective-ranks} gives
\(r_0(\bm{\Sigma})^2/R_0(\bm{\Sigma})=(\gamma^2+d)/\gamma^2\), and therefore
\begin{equation}\label{eq:flat-effective-rank-thresholds}
    \frac{r_0(\bm{\Sigma})^2}{R_0(\bm{\Sigma})}-1
    =\frac d{\gamma^2},
    \qquad
    r_0(\bm{\Sigma})-1=\frac d\gamma,
\end{equation}
and the benign-misfitting window can be written exactly as
\begin{equation}\label{eq:flat-effective-rank-window}
    \frac{r_0(\bm{\Sigma})^2}{R_0(\bm{\Sigma})}-1
    \ll n\ll
    r_0(\bm{\Sigma})-1.
\end{equation}

\paragraph{Scope.}
In this stylized model these two quantities are exactly the two asymptotic
sample-count thresholds:
\(r_0(\bm{\Sigma})^2/R_0(\bm{\Sigma})-1\) is the first-useful-span threshold and
\(r_0(\bm{\Sigma})-1\) is the interpolation threshold.
These identities
translate the two deterministic thresholds into the effective-rank notation of
\citet{bartlett2020benign}.
They do not apply that paper's general theorem.

Both families nevertheless matter algebraically: the full-spectrum mass-to-peak
ratio sets the interpolation threshold, while removing the leading spike's
squared spectral contribution sets the first-useful-span threshold.
Indeed the
present window has \(r_0(\bm{\Sigma})/n\to\infty\), the opposite side of one of
the sufficient benignity conditions of \citet{bartlett2020benign}
--- these are sufficient, not necessary, conditions.
The same algebraic distinction extends to a non-flat nuisance tail.

\paragraph{Non-flat nuisance tails.}\label{par:nonflat-tails}
Let \(\bm\Omega\) denote the nuisance test covariance, normalized by
\[
    \operatorname{tr}(\bm\Omega)=d,
    \qquad
    \bm0\preceq\bm\Omega\preceq\gamma\bm I_d,
\]
and let
\[
    \bm V=[\vec v_1,\ldots,\vec v_n],
    \qquad
    \bm G_\Omega:=\bm V^{\top}\bm\Omega\bm V
\]
be the training nuisance matrix and the covariance-weighted nuisance Gram.
Two conclusions are weighted-isotropy-free: the interpolation threshold
\(n\gg d/\gamma\) depends only on the Euclidean training geometry, not on the
orientation of the training directions relative to \(\bm\Omega\), and, on the
low-sample branch \(\rho=\gamma n/d\to0\), the forced-misfit converse needs only
positive semidefiniteness of \(\bm\Omega\) together with the same Euclidean
training geometry.
The nuisance
spectrum alone, however, does not determine the deterministic test frontier.
The orientation of the training directions relative to the covariance
matters: Euclidean orthogonality is not weighted orthogonality.
Under
covariance-weighted approximate isotropy of the nuisance Gram,
\[
    \bigl\|M_2^{-1}\bm G_\Omega-\bm I_n\bigr\|_{\rm op}\le\delta_{\rm tail},
    \qquad
    M_2=\operatorname{tr}(\bm\Omega^2),
    \qquad
    \limsup\delta_{\rm tail}<1,
\]
the first-useful-span threshold becomes
\(n\gg\operatorname{tr}(\bm\Omega^2)/\gamma^2\).
Under the trace normalization,
\(d\le\operatorname{tr}(\bm\Omega^{2})\le\gamma d\), with equality on the left
exactly for the flat tail: any non-flatness raises the first-useful-span
threshold and narrows the benign-misfitting window.
The operator-norm condition
sandwiches \(\bm G_\Omega\) between \((1\mp\delta_{\rm tail})M_2\bm I_n\), so
scalar comparison risks bracket the attainable risk.
The threshold then
acquires its effective-rank reading --- burden squared over nuisance-tail
participation --- which is how favorable tail participation can coexist with
an unfavorable global mass-to-peak ratio \(r_0\).
We conjecture that this
effective-rank translation persists for suitably regular covariance-matched
random orientations of the training nuisance directions --- for instance,
training nuisance directions drawn independently from a suitably regular
distribution with covariance \(\bm\Omega\).
The regularity that matters is
concentration of both Grams.
The Euclidean Gram \(\bm V^{\top}\bm V\) must
concentrate near \(d\,\bm I_n\), and the weighted Gram
\(\bm V^{\top}\bm\Omega\bm V\) must concentrate near \(M_2\bm I_n\).
This is a conjecture, not a theorem.
Because the spectrum alone
cannot determine the frontier, some orientation control is necessary.

\paragraph{Benign underfitting in stochastic convex optimization.}
A closer terminological neighbor is benign underfitting in stochastic convex
optimization (SCO).
\citet{koren2022benign} construct examples in which a
particular algorithm can achieve good population performance while its
empirical risk remains large.
The mechanism here is different: in the
deterministic span model, the train--test tradeoff curve itself forces every
span predictor with small test error into the fourth quadrant.
Appendix~\ref{app:sco-comparison} explains why the present mechanism differs.

\paragraph{Proper and improper learning.}
The split-and-median construction in Appendix~\ref{app:improper} is a
deliberately simple functional improper learner: it aggregates linear span
predictors and outputs a nonlinear predictor.
Such proper/improper separations
recur in learning theory
\cite{hanneke2016optimal,bousquet2020proper,montasser2019vc,larsen2023bagging,
foster2018logistic}.
Appendix~\ref{app:sco-comparison} highlights a
separate bounded-comparator-ball perspective: the sample-efficient span
predictor remains linear, but it must use a larger Euclidean norm budget than
the spike oracle.

\paragraph{Negative generalization gaps.}
Negative generalization gaps are not new, but their usual sources are
different from the one here.
They can arise when training is deliberately
hardened by aggressive augmentation
\cite{touvron2019fixing,zhang2018mixup,yun2019cutmix,cubuk2019autoaugment},
data pruning \cite{paul2021dataDiet,sorscher2022beyond,toneva2019empirical},
or distributionally robust objectives
\cite{sagawa2020groupDRO,sinha2018certifying}.
Here the gap is forced by
span geometry, not by making the training problem harder.

\paragraph{Learning from a noisy reference.}
The product average in \eqref{eq:sensor-product-average}, introduced
as a natural way to estimate \(Y\), is a classical noisy-reference
correlator.
An earlier noisy observation supplies
the reference for a new one.
\citet{sp_roe1961} compute its
distribution and moments.
Their aligned, independent-noise calculation already contains the noise-product contribution responsible for the \(\sigma^4/m_{\rm sens}\) term in the fresh prediction error.
Equation~\eqref{eq:sensor-product-expansion} shows how reusing the reference replaces the mean-zero fresh products \(\vec N_0[j]\vec N[j]\) with noise squares \(\vec N_0[j]^2\), whose mean \(\sigma^2\) produces the misfit on the reference measurement.

Known pilot symbols or pilot tones similarly let a receiver measure
an unknown channel before interpreting data
\citep{sp_bello1962,sp_cavers1991}.
To connect channel sounding
with the sensor backstory in Remark~\ref{rem:sensor-model},
consider well-separated pulses passing through a fixed multipath
channel.
A unit-amplitude pilot gives noisy samples of the
channel's pulse response \(\vec s_{\rm resp}\).
A later pulse of amplitude \(Y\) gives noisy samples of
\(Y\vec s_{\rm resp}\).
These received waveform samples play
the role of the sensor readings.
\citet{sp_rush1964} analyzes
a binary-message version of this construction under a Gaussian prior for
the channel output.
His receiver estimates that output from the
reference and correlates the message observation with a filtered
version of the estimate.

\subsection{Span geometry, classification, and rescaling}
\label{subsec:related-span-rescaling}

\paragraph{Classification and feature-learning relatives.}
The same single-spike geometry connects directly to the
survival--contamination viewpoint of
\citet{muthukumar2020harmless,muthukumar2021classification}.
For a one-sparse noiseless Gaussian model,
\citet[Proposition~17]{muthukumar2021classification} give
\[
    \mathcal{R}
    =
    (1-\mathrm{SU})^2+\mathrm{CN}^2,
    \qquad
    \mathcal{C}
    =
    \frac12
    -
    \frac1\pi
    \arctan\!\left(
        \frac{\mathrm{SU}}{\mathrm{CN}}
    \right),
\]
where \(\mathcal{R}\) is squared regression risk and \(\mathcal{C}\) is the
error obtained by comparing the sign of the regression prediction with the
sign of the target.
Thus regression consistency requires
\(\mathrm{SU}\to1\) and \(\mathrm{CN}\to0\), whereas consistency of the
induced sign classifier requires only
\(\mathrm{SU}/\mathrm{CN}\to+\infty\).
When survival is eventually
positive---as it is for the interpolator, since
\(\mathrm{SU}_{\rm int}=\rho/(1+\rho)>0\)---this is equivalently
\(\mathrm{CN}/\mathrm{SU}\to0\).
The exact arctangent formula is Gaussian-specific, but
\citet{mcrae2022structured} show that the same survival-versus-residual
principle and the classification--regression consistency separation persist
for bounded orthonormal-system features, including Fourier systems.

In the present orthogonal model,
\[
    \mathrm{SU}_{\rm int}
    =
    r_{\rm int}
    =
    \frac{\rho}{1+\rho},
    \qquad
    \frac{\mathrm{CN}_{\rm int}}
         {\mathrm{SU}_{\rm int}}
    =
    \frac{1}{\sqrt{R}}.
\]
Inside the benign-misfitting window,
\[
    \mathrm{SU}_{\rm int}\longrightarrow0,
    \qquad
    \frac{\mathrm{CN}_{\rm int}}
         {\mathrm{SU}_{\rm int}}
    \longrightarrow0.
\]
The interpolator therefore loses too much signal survival for squared-error
regression, even though its contamination is still smaller than its surviving
signal by enough for the induced sign classifier to be consistent.

The same-distribution companion \citep{ranade2026samedistribution} makes the
flat-tail geometric link exact: its population-risk-optimal span predictor
is a positive multiple of the minimum-norm interpolator.
Positive rescaling
leaves the sign of every prediction unchanged while restoring the signal
survival required for squared-error regression.
Related high-dimensional interpolation geometries have also been analyzed in
multiclass settings
\cite{wang2021multiclass,subramanian2022multiclass,wu2023multiclass}.
For a related binary classification model connecting feature lifting,
interpolation, Gibbs phenomenon, and adversarial examples, see
\citet{narang2021classification}.

Prediction scale also matters in logistic regression \citep{wu2025earlylogistic}.
In a well-specified overparameterized Gaussian logistic model, appropriately stopped GD consistently estimates conditional probabilities as the sample size grows, whereas continuing GD on a fixed separable training sample drives the parameter norm and population logistic risk to infinity.
The scale error has the opposite sign from the one here: late logistic iterates make probability predictions too extreme, whereas our interpolator attenuates the numerical regression target and must be scaled up.

\paragraph{Training points experience amplified contamination.}
The contamination is not experienced in the same way by training and test
points.
For any span predictor,
\(\mathrm{CN}(\vec{u})^2=dq\)
is the fresh-test contamination variance, whereas the average squared
nuisance response on the training points is
\[
    \frac{1}{n}\sum_{i=1}^{n}(d\alpha_i)^2
    =
    \frac{d^2q}{n}
    =
    \frac{d}{n}\,\mathrm{CN}(\vec{u})^2.
\]
Thus training amplifies the same nuisance component by the factor \(d/n\) in
squared magnitude.
Along the symmetric Pareto ray, \(\alpha_i=s/n\), so every training point
sees nuisance response \(d\alpha_i=r/\rho\), while a fresh test point sees
contamination standard deviation \(r/\sqrt{R}\).
Benign misfitting is therefore the price of restoring signal survival when
the examples that supplied the signal also probe their own contamination with
this amplified strength.

\paragraph{Index learning.}
The premise of a low-dimensional informative direction under a
high-dimensional nuisance is also that of the single-index and multi-index
feature-learning literature, where gradient methods must recover such a
direction from high-dimensional data
\cite{dudeja2018singleindex,bietti2022singleindex,abbe2022merged,
abbe2023leap,bruna2025multiindex}.

\paragraph{Negative ridge and anti-shrinkage.}
A previous subsection described the interpolator as over-regularizing the
signal.
The ridge coordinate gives the algebraic version of that statement.
Negative ridge and related anti-shrinkage phenomena arise when the implicit
regularization of the overparameterized fit is already too strong
\cite{nr_kobak2020,nr_wu2020,nr_tsigler2023}.
\citet{freeman2026shrinkage}
isolates a closely related \emph{inflation} phenomenon: in anisotropic
random-design linear regression with $d/n\to\infty$, multiplying the
$\ell_2$ minimum-norm interpolator by a scalar $c>1$ can improve
generalization error.
That work treats scalar inflation as distinct from
negative ridge.
Here the deterministic single-spike geometry puts the useful
branch on an exact scalar path, and the same sign-reversed move appears in
ridge coordinates as a negative penalty.
Thus, in this symmetric
deterministic one-spike geometry, negative ridge is not a separate static
path: it is the ridge coordinate for undoing the interpolator's
under-calibration of the spike.

In survival--contamination terms, the problem is insufficient survival rather
than excessive fresh-test contamination.
Positive scaling by \(c\) sends
\(\mathrm{SU}\mapsto c\,\mathrm{SU}\) and
\(\mathrm{CN}\mapsto c\,\mathrm{CN}\).
The minimum-norm interpolator has
\(\mathrm{SU}_{\rm int}\asymp\rho\) in the benign-misfitting window.
Scaling
by \(c\asymp1/\rho\) restores order-one survival.
The resulting contamination
is of order \(1/\sqrt{R}\), which vanishes because \(n\) is already beyond the
first-useful-span threshold \(d/\gamma^2\), i.e.\ \(R\to\infty\).
In the present one-spike geometry, negative ridge realizes this same
anti-shrinkage exactly on the interpolation ray.
One-pass large-learning-rate
SGD is not a pure scalar rescaling---its coefficients carry geometric
weights---but it has the same statistical purpose: it builds signal survival
toward one while keeping fresh-test contamination small, with the logarithmic
cost quantified in Corollary~\ref{cor:learning-rate-robustness}.

\paragraph{Classical attenuation and subtractive corrections.}
For negative ridge, the normal equation associated with
\eqref{eq:kernel-ridge-objective} is
\((\bm K+n\lambda\bm I_n)\vec\alpha=\one\).
When \(\lambda<0\), this subtracts \(n|\lambda|\bm I_n\) from the
Gram matrix.
Such subtractive corrections have classical precedents
in measurement-error regression.

The loss of signal also has a classical name.
With one explanatory
variable and independent additive measurement error, the population
least-squares slope is shrunk toward zero relative to the slope
relating the noise-free variables.
This is called \emph{attenuation}
\citep[Sec.~1.1.1, Eq.~(1.1.6), p.~3]{fuller1987measurement},
or \emph{regression dilution} \citep{eiv_frost2000}.
In the regression through the origin across sensors of Subsection~\ref{subsec:scaled-interpolation}, the physical interpolating rule \(\vec S_0^{\top}\vec S/\|\vec S_0\|^2\) estimates the noise-free slope \(Y\).
With clean regressors of unit mean square and independent measurement noise of fixed variance \(\sigma^2\), its fitted slope tends to \(Y/(1+\sigma^2)\) as sensors are added (Subsection~\ref{subsec:sensor-misfitting}).
The attenuation factor \(1/(1+\sigma^2)\) is also the limit of the deterministic \(r_{\rm int}=\gamma/(\gamma+d)\) in \eqref{eq:rint-def}, under the sensor parameterization.

Subsection~\ref{subsec:scaled-interpolation} shows that, for one training point, the tuned shift recovers the leading subtraction of the classical known-noise EIV correction as sensors are added at fixed per-sensor noise variance \citep{fuller1987measurement}.

In standard total least squares with one response and the usual
singular-value gap, the normal equation subtracts the squared
smallest singular value of the augmented data matrix.
\citet{nr_golub1980} derived that shift and described total least
squares as deregularizing, a kind of reverse ridge regression.
\citet{nr_golub1999} formulate Tikhonov regularization within total
least squares.
\citet{nr_markovsky2007} review the TLS subtraction's
errors-in-variables interpretation.

\paragraph{James--Stein and sign-reversed rescaling.}
The scalar path also mirrors the classical James--Stein idea of choosing a
multiple of a raw estimator to reduce squared-error risk
\cite{stein1956inadmissibility,james1961estimation,efron1977stein}
(see \cite{efron2016jamesstein,draper1979ridge} for the classical connection
between James--Stein and ridge).
The sign
is reversed here: ordinary shrinkage moves a predictor toward zero, while the
useful strong-signal span predictor is scaled upward, past interpolation.
The classical shape is recognizable, but the direction of movement is reversed.

\paragraph{From scalar attenuation to directional distortion.}
In the present one-spike model, survival is a single number: insufficient
survival is a global attenuation, and one scalar inflation can repair it.
With several informative directions of unequal strength, the directions can
survive by different amounts, so the minimum-norm interpolator can distort
the recovered signal by changing the relative sizes of its informative
components---and one common multiplier can no longer restore them all.
The
multispike companion \citep{ranade2026multispike} studies the
direction-dependent recalibration this calls for---\emph{equalization}---and
how ridge and one-pass large-learning-rate SGD provide it, in regimes that
overlap but do not coincide.

\subsection{Large learning rates and mini-batches}
\label{subsec:related-large-steps}

Two distinctions organize this subsection.
First, a stochastic gradient can
be used as fresh information about a population objective or as one step
toward optimizing a fixed empirical objective.
Second, a large learning rate can matter through several different mechanisms---freshness, stochastic noise, spectral reweighting, curvature adaptation, or solution selection.
The stylized calculation of this paper isolates freshness.
Table~\ref{tab:large-step-mechanisms} summarizes the distinctions.

\begin{table}[ht]
\centering\small
\caption{Mechanisms by which large or stochastic steps can help generalization.
The columns refer to stochastic gradient noise, changes in curvature, and training-data reuse.
The present paper isolates the last row.}
\label{tab:large-step-mechanisms}
\begin{tabular}{@{}>{\raggedright\arraybackslash}p{0.26\linewidth}ccc >{\raggedright\arraybackslash}p{0.30\linewidth}@{}}
\toprule
\textbf{Account}
& \textbf{Noise?}
& \textbf{Curvature?}
& \textbf{Reuse?}
& \textbf{Main mechanism}
\\
\midrule
Small-step implicit ridge
& No & No & Usually
& Shrinkage along an empirical optimization path
\\[2pt]
Spectral reweighting
& Not nec. & No & Yes
& Step size changes residual weighting of fixed empirical directions
\\[2pt]
SGD-noise / flat-minimum
& Yes & Sometimes & Usually
& Noise-biased exploration or selection
\\[2pt]
Edge of stability / catapult
& Not nec. & Yes & Yes
& Instability coupled to evolving curvature
\\[2pt]
\textbf{Present fresh pass}
& \textbf{No} & \textbf{No} & \textbf{No}
& \textbf{Fresh nuisance directions accumulate shared signal}
\\
\bottomrule
\end{tabular}
\end{table}

\paragraph{Stochastic approximation versus empirical-risk minimization.}
A useful older distinction is between stochastic approximation and
sample-average optimization---called sample-average approximation (SAA) in the
stochastic-programming literature
\cite{robbins1951stochastic,nemirovski2009robust}.
Stochastic approximation
uses noisy first-order information to move toward the minimizer of an expected
objective.
Sample-average optimization freezes a finite sample and optimizes the
resulting empirical average, the learning analogue of empirical risk
minimization (ERM).
In learning language, this
is the distinction between using fresh stochastic gradients as information
about population risk and treating SGD as a solver for a fixed empirical-risk
problem.

The stylized one-pass calculation sits at this boundary.
During the first
pass, the nuisance directions are fresh and the shared spike direction is
stable.
In that precise sense the update is SA-like: it uses each fresh
example to accumulate the shared signal, rather than to solve the frozen
finite-sample problem.
If the same examples are reused, the procedure becomes
finite-sample fitting, and the same learning rate is far outside
repeated-example stability
\eqref{eq:stability-fresh}--\eqref{eq:stability-reuse}.
The empirical residuals left behind are not a
failure of the fresh pass; they are the sample-specific nuisance directions
that random testing averages away.

This distinction is not only philosophical.
In stochastic convex
optimization, ERM can pay a dimension-dependent sample-complexity penalty
even in settings where stochastic first-order methods or other non-ERM
procedures attain the useful population-level performance
\cite{shalevshwartz2010sco,feldman2016sco,carmon2024erm}.
Closely related
SGD-versus-GD separations show that stochastic gradient methods can generalize
better than full-batch gradient descent on the empirical loss, and that adding
regularization to the empirical objective need not close the gap
\cite{amir2021sgd}.
\citet{sekhari2021sgd} further study the roles of implicit
regularization, batch size, and multiple epochs, showing both separations from
regularized ERM and regimes where multiple passes help.
Recent work also shows
that full-batch GD itself can inherit dimension-dependent ERM-like sample
complexity in nonsmooth SCO \cite{livni2024gd}.
The related \emph{Never Go Full Batch} result gives a complementary full-batch
first-order lower bound: full-batch methods can require
\(\Omega(1/\varepsilon^4)\) iterations or dimension-dependent samples to
generalize in SCO, while SGD achieves \(O(1/\varepsilon^2)\)
\cite{amir2021neverfullbatch}.

Large-scale-learning analyses
make a related computational-statistical point: exact empirical minimization can
be the wrong use of data and computation when approximate stochastic
optimization already reaches the statistically useful error level
\cite{bottou2008tradeoffs,bottou2018optimization}.
The present model is not an
SCO lower-bound construction, but it reinforces the same warning in an exactly
solvable squared-loss geometry: the training span contains useful predictors
before empirical minimization within that span is useful.

The Deep Bootstrap framework gives a modern deep-learning version of this
online/offline comparison by coupling empirical-risk training to an idealized
online process \cite{nakkiran2021deepbootstrap}.
The benign-misfitting window
is a useful counterpoint: here the fresh first-pass dynamics and the repeated
empirical dynamics should not track each other.

\paragraph{One-pass SGD on least squares.}
There is also a direct theory of one-pass SGD on random-design least squares.
Exact high-dimensional dynamics are known for proportional dimensions
\citep{collinswoodfin2024onepass}.
Constant-stepsize SGD with tail averaging
can overfit benignly \citep{zou2023benign}.
The unaveraged last iterate is
characterized for decaying stepsizes \citep{wu2022last}, and power-law
spectra yield scaling laws \citep{lin2024scaling}.
Instance-wise comparisons of expected population excess risk exhibit regimes where tail-averaged SGD with a constant learning rate needs only logarithmically more samples than ridge, and instances where optimally tuned ridge needs nearly quadratically more samples than SGD \citep{zou2021benefits}.%
\footnote{For the Gaussian comparison, logarithmic sample overhead is established under the bounded signal-to-noise and spectral conditions of \citet[Corollary~5.1]{zou2021benefits}, with a ridge estimator whose expected excess risk is at most the label-noise variance.}
On a fixed interpolable sample, \citet{zou2022multipass} compare GD with uniform with-replacement SGD at matched zero initialization, iteration count, and constant learning rate.
Conditional on the sample, SGD's expected population excess risk equals GD's excess risk plus a nonnegative fluctuation term.
\citet{wu2026riskcomparisons} construct spiked random-design families, including noiseless ones, where geometrically decayed one-pass SGD has polynomially smaller expected population excess risk than GD with a constant learning rate at every iteration count within the analyzed learning-rate range.%
\footnote{In Theorem~4.2 of the full version of \citet{wu2026riskcomparisons}, GD uses the averaged half-squared empirical loss, with learning rate at most the reciprocal of the largest eigenvalue of the empirical second-moment matrix.
SGD uses a geometrically decaying schedule and outputs the final iterate.
The ridge comparisons in that paper use nonnegative penalties.
Their hard family keeps the signal contribution to the label at unit variance.
Writing its original ambient dimension as \(D\ge n^2\), multiplying the inputs by \(\sqrt D\) gives population covariance \(\diag(\gamma,1,\ldots,1)\), with \(\gamma=D\,n^{-0.9}\) and target coefficient \(1/\sqrt\gamma\).
Our nuisance dimension is \(d=D-1\), so \(\rho=(D/(D-1))n^{1/10}\asymp n^{1/10}\to\infty\), outside the benign-misfitting window \eqref{eq:benign-misfitting-window}.
For the corresponding deterministic training geometry, \eqref{eq:frontier-endpoints} gives \(g_{\rm int}=(1+n/d)/(1+\rho)^2\asymp n^{-1/5}\), matching the exponent of their lower bound for GD.}
The one-pass result here uses a constant learning rate and the unaveraged final iterate (Theorem~\ref{thm:one-pass-sgd}).

The SGD--ridge formulation of \citet{zou2021benefits} permits negative penalties, and our test-error-minimizing span predictor is the negative-ridge point \eqref{eq:lambda-star}.
This paper asks a further question: what does accurate prediction require of the training error itself?
In the window \eqref{eq:benign-misfitting-window}, every span predictor with vanishing test error must have diverging training error (Corollary~\ref{cor:fourth-quadrant-necessity}).
Restoring the shared signal's calibration within the training span forces this increase because training points experience amplified contamination.

The next two groups of works keep the training sample fixed and examine how the learning rate changes the optimization path.

\paragraph{Large steps on fixed quadratic losses.}
In interpolable least squares, GD and SGD from zero approach the same minimum-norm interpolator when their training error tends to zero, but their paths to it can differ.
Gradient flow removes error rapidly in directions where the training loss is steep and slowly where it is flat.
Large GD steps can instead overshoot in the steep directions, keeping that part of the error oscillating while the flatter directions improve.
Early stopping can therefore select different predictors at a common training-loss target.
Training and test losses need not penalize the remaining error directions in the same way.
Leaving error in a direction penalized heavily by training but lightly by testing can then give lower test error at the same training-loss level.

\citet{nakkiran2020annealing} makes this path change explicit in a noiseless linear-regression example.
Large initial GD steps move the iterate to a new starting point for subsequent gradient flow.
The gradient field is unchanged, but the new flow trajectory reaches a prescribed training loss with lower population loss than gradient flow from the original initialization.
\citet{beugnot2022large} obtain a related benefit for convergent GD on quadratic objectives while keeping the large learning rate constant throughout.
Overshooting leaves the error concentrated in a different direction, which can improve the evaluation risk at a common objective stopping threshold.
\citet{wu2021direction} study repeated-epoch mini-batch SGD with a moderate initial learning rate followed by annealing.
The initial phase changes the direction of the remaining error, allowing early stopping to improve test risk at a prescribed training-loss level.

\paragraph{Small-step paths and modified losses.}
Classical early stopping is already a form of implicit regularization for
small-step GD or gradient flow, closely related to ridge-like shrinkage in
least squares~\cite{ali2019continuous}.
For full-batch gradient descent,
\citet{barrett2021implicit} use backward error analysis to derive an implicit
gradient regularization term penalizing large loss gradients.
For
random-shuffling SGD, \citet{smith2021origin} derive a modified loss containing
a minibatch-gradient regularizer whose strength depends on the learning rate
and batch size.
Stochastic-gradient-flow analyses of least squares identify
ridge-like regularization effects of constant-learning-rate mini-batch
SGD~\cite{ali2020implicit}.
In particular, early-stopped small-step SGD acts as
an implicit \emph{positive} ridge regularizer.
In the benign-misfitting
window, the sign reverses: the large-learning-rate one-pass iterate moves past
interpolation toward the test-optimal point reached by the negative-ridge
penalty \eqref{eq:lambda-star} of
Section~\ref{subsec:scaled-interpolation}, but it reaches that point
through geometric weights \eqref{eq:single-geometric-coeffs}, paying the
logarithmic geometric-weighting cost isolated in Subsection~\ref{subsec:flat-averaging}, rather than following the
exact symmetric
negative-ridge path.
A related label-noise SGD model of
\citet{blanc2020implicit} studies dynamics near zero training error and
identifies an implicit regularizer built from squared parameter-gradient norms
on the training examples.
Algorithmic-stability analyses provide a different
route: stochastic gradient methods can generalize when the algorithm is stable
over few passes~\cite{hardt2016train}.

\paragraph{SGD noise, flat minima, and saddle escape.}
In neural-network optimization, larger steps have been studied through accounts
of SGD noise and approximate Bayesian inference
\cite{jastrzebski2017three,mandt2017sgd,smith2018bayesian}, cyclical and
super-convergent schedules \cite{smith2017cyclical,smith2019super}, flatness
and sharpness of minima,\footnote{This flat-minima literature is a point of
contrast, not the mechanism used here; see, for example,
\cite{hochreiter1997flat,keskar2017largebatch,dinh2017sharp}.}
regularization from an initial large learning rate \cite{li2019towards}, and
late-decay explanations for sustained large rates \cite{ren2024understanding}.
\citet{andriushchenko2023sparse} observe sparse feature learning during large-step loss stabilization and analyze the role of sustained SGD noise.
\citet{lu2024benignoscillation} study a two-layer convolutional model trained by cyclic single-example SGD, whose trajectory is deterministic conditional on the data and initialization.
Under an assumption that residual magnitudes remain bounded away from zero on strong-feature examples, SGD with a large learning rate learns a weak shared feature that SGD with a small learning rate leaves near initialization.
A related refinement is that SGD noise is not merely a scalar temperature: its
covariance can be structured, and that structure can bias the learning dynamics
toward flatter or lower-noise directions
\cite{zhu2019anisotropic,haochen2021shape}.
A separate optimization line studies how stochastic or perturbed gradient
methods escape saddle points
efficiently~\cite{ge2015escaping,jin2017escape}.
That is also not the mechanism
here, because the stylized objective is not using noise to find a better local
minimum.

\paragraph{Large-learning-rate stability and solution selection.}
Step size and batch size can affect which solution is selected.
Catapult and
edge-of-stability
dynamics
\cite{lewkowycz2020large,cohen2021gradient,ahn2022understanding} and large-step
implicit bias \cite{nacson2022implicit,wang2025goodregularity} have been
studied primarily for full-batch gradient descent.
\citet{cohen2021gradient} identify edge-of-stability dynamics where the
sharpness hovers near \(2/\eta\), while \citet{lewkowycz2020large} describe
catapult dynamics in which the training loss grows initially before curvature
drops.
Theory work has begun to explain EoS through mechanisms such as
self-stabilization and implicit flows along sharpness-controlled manifolds
\cite{damian2023selfstabilization,arora2022understandingeos}.
\citet{ahn2022understanding} study related unstable-convergence
behavior beyond the classical \(2/L\) descent condition.

For separable logistic regression, \citet{wu2023logisticedge} prove convergence in direction to the maximum-margin classifier for any fixed GD learning rate under a support-vector nondegeneracy condition.
\citet{wu2024largelogistic} obtain accelerated empirical-loss minimization without that nondegeneracy assumption, using a constant learning rate chosen for the iteration budget.%
\footnote{Theorem~4 of \citet{wu2024largelogistic} also bounds time averages of population logistic and classification risks along online-SGD iterates, using fresh independent samples from a bounded population with a uniform positive separation margin.
These are not last-iterate guarantees.
The authors leave open whether larger learning rates improve the SGD rates beyond logarithmic factors.}
Acceleration with a large learning rate also holds for separable logistic regression with sufficiently weak \(\ell_2\) regularization, where the objective has a finite minimizer \citep{wu2025regularizedlogistic}.
The analyses of \citet{wu2024largelogistic} and \citet{wu2025regularizedlogistic} use the small curvature of logistic loss at large positive margins to prove eventual monotone descent of the objective while allowing initial nonmonotone behavior.%
\footnote{Large adaptive GD steps yield empirical-loss guarantees for averaged or best iterates \citep{zhang2026adaptive}.
Nonlinear extensions of the constant-learning-rate analysis include non-homogeneous two-layer networks under suitable activation and separability assumptions \citep{cai2024nonhomogeneous}.}
Constructed nonseparable logistic datasets can admit stable cycles even when their finite minimizer is locally stable \citep{meng2024nonseparable}.
The present quadratic has fixed curvature, and its useful fresh pass improves test risk while leaving training squared error above the zero-predictor baseline.

Mini-batch SGD needs a different notion of what must be stable:
\citet{leejang2023new} introduce an
interaction-aware sharpness depending on the batch-gradient distribution, while
\citet{andreyev2024edge} replace the Hessian eigenvalue by Batch Sharpness.
In diagonal linear networks, stochasticity and large learning rates can change
the solution selected by the dynamics~\cite{pesme2021implicit,even2023sgd}.
Dynamical-stability analyses also show that SGD and GD can select different
stable minima, with learning rate and batch size playing different
roles~\cite{wu2018sgd}, and that SGD stability can impose trace- or
Frobenius-type constraints rather than the spectral constraint associated with
full-batch GD~\cite{wu2023implicit}.

\paragraph{Freshness before reuse.}
Here the large step works because each update encounters a new nuisance
direction while acting on the same shared spike.
Once the data order is fixed,
the trajectory is deterministic, and the useful effect appears before any
training direction is revisited.%
\footnote{This is distinct from Polyak--Ruppert iterate averaging, which
averages SGD iterates to reduce stochastic variance around a target
\cite{polyak1992acceleration,ruppert1988efficient}.
Here averaging would give
more weight to the early, under-calibrated part of the pass.
On the small-rate
branch \(\Lambda\to\infty\) with \(\Lambda^{2}/n\to0\), the averaged spike
residual is of order \(1/\Lambda\), while the final iterate has residual
\(e^{-\Lambda}(1+o(1))\).}
Thus the large step is not SGD noise wandering among minima, not a modified-loss
minimizer, not curvature recovery, and not selection among empirical solutions:
the benefit is already present in a deterministic, fixed-order first pass.

Read as a hypothesis for larger models, this suggests that the fresh-example
update, rather than the already-written empirical Hessian, may be the relevant
stability condition for the first epoch.
A large-step first pass can then
operate above the
frozen full-batch sharpness line without requiring progressive sharpening or
selection among empirical minima.

The formal claim remains specific to this span geometry.
Other least-squares and
kernel regimes can benefit from multiple passes
\cite{pillaud2018statistical,lin2025datareuse}.
Here, reuse at the same rate probes nuisance directions already written into
the iterate.
Random reshuffling changes the multi-pass convergence dynamics
on a finite sum \cite{mishchenko2020random}, but not that reuse problem.
Related
rapid-overfitting results give a complementary warning \cite{vansover2025rapid},
while Figure~\ref{fig:stylized-sgd-trace} shows the exact train--test split
in the present model. 

\paragraph{One-step full-batch GD and incremental projection.}
There is also a revealing full-batch one-step calculation.
In this exact
geometry, one full-batch gradient step on the averaged half-squared empirical
loss from zero with learning rate \(1/\gamma\) gives
\((1/(\gamma n))\sum_i \vec{x}_i\), the calibrated flat-averaging benchmark
of Subsection~\ref{subsec:flat-averaging}.
But this is a fragile one-shot
update: the same learning rate is far beyond the finite-matrix stability
threshold \(2n/(d+\gamma n)\) for that half-squared
loss, which is \(2n/d\,(1+o(1))\) in the benign-misfitting window, since
\(\gamma^{-1}\cdot(d+\gamma n)/(2n)=1/(2\rho)+1/2\to\infty\).
One-pass SGD
is useful not because it injects regularizing noise, but because it stretches
this explosive full-batch move into sequential updates on fresh nuisance
directions.

Optimization readers may also view the fixed-order pass as incremental gradient
descent, closely related to Kaczmarz row-action
methods~\cite{strohmer2009randomized,needell2014kaczmarz}.
The sign of the interpretation is
different: a Kaczmarz step uses roughly \(1/\|\vec{x}_i\|^2\) to project toward
fitting the current equation, while the useful rate here is larger by a factor
of order \(\Lambda/\rho\) and deliberately over-projects the current example.

A separate full-batch curvature comparison points in the same direction.
Matched against the population curvature accumulated over the same \(n\)
examples,
\[
    \frac{\lambda_{\max}(\Xt\Xt^\top)}{\lambda_{\max}(n\bm\Sigma)}
    =
    \frac{d+\gamma n}{\gamma n}
    =
    1+\frac1\rho
    \longrightarrow\infty ,
\]
a ratio in which every normalization cancels.
This compares the
finite-sample and population quadratics as full-batch objects.

\paragraph{Clipping and effective step reduction.}
This also places the calculation next to work on training instability and
curvature control.
\citet{gilmer2022loss} study learning-rate warmup,
initialization, normalization, and gradient clipping as ways to avoid or escape
high-curvature regions that would otherwise prevent large learning rates.
In
their experiments, gradient clipping can act like a reduction of the effective
learning rate in high-curvature regions.
In our stylized model, that
observation has a complementary interpretation: reducing the effective
per-example step would suppress precisely the overshoot that lets one-pass SGD
calibrate the spike at low sample sizes.
The same point connects to
clipping-bias analyses in private gradient methods, where clipping can change
the effective update direction or optimization profile rather than merely
reducing sensitivity
\cite{chen2020understandingclipping,song2020evading}.

\paragraph{Mini-batch scaling.}
The mini-batch version of the calculation in Subsection~\ref{subsec:minibatch-scaling} also touches the literature on
scaling learning rates with batch size.
Large-minibatch ImageNet training
popularized the linear scaling rule for SGD, together with learning-rate warmup
\cite{goyal2017accurate}.
Other work relates learning-rate schedules and
batch-size schedules through an SGD noise-scale picture \cite{smith2018dont},
and argues that the learning-rate-to-batch-size ratio influences the width of
the minima reached by SGD \cite{jastrzebski2017three}.
The
critical-batch-size literature instead asks when increasing the batch size
stops giving proportional gains; gradient-noise-scale estimates predict the
largest useful batch size across many workloads
\cite{mccandlish2018empirical}, while large empirical studies show that
batch-size effects depend strongly on tuning, optimizer, model, and compute
budget \cite{shallue2018measuring}.
Stochastic differential equation (SDE)
analyses give accounts of these
scaling rules for SGD and lead to different rules for adaptive methods such as
Adam and RMSprop \cite{malladi2022sdes}.
For least squares itself,
mini-batching and tail averaging admit a sharp finite-sample theory
\citep{jain2018parallelizing}, and batch scaling laws under power-law spectra
now extend from one-pass SGD to data reuse \citep{chen2026sketched}.

The linear rule suggested by our
stylized model has a different source.
With averaged mini-batch gradients, a
batch of size \(B_{\rm mb}\) divides each example's contribution by
\(B_{\rm mb}\).
To preserve the overshoot needed to calibrate the spike, the
learning rate must grow by the same factor.
Thus the same linear scaling
appears here, in this one-pass small-learning-rate regime, without invoking gradient
noise, flat minima, or sampling temperature: it is a signal-calibration
requirement.

\subsection{Caricatures and signal-versus-nuisance analogies}

\paragraph{Exact ML caricatures.}
Among the closest ML antecedents are the Fourier and ``ultra-toy'' models used in the
harmless-interpolation line
\cite{muthukumar2019harmless,muthukumar2020harmless,muthukumar2021classification}.
Those papers use deliberately simplified feature families to isolate
the extent to which the signal survives overparameterized interpolation and
how much high-dimensional nuisance directions contaminate prediction.
The present model follows the same spirit, but asks a different question: once
the training span contains useful signal, should a squared-loss predictor in
that span be seeking an empirical fit at all?
A broader methodological cousin is the exact analysis of deep linear networks
by \citet{saxe2014exact}, which also invokes exact orthogonality to expose the
dynamics of learning in a tractable setting.
Toy models separating robust and non-robust features are discussed with
the adversarial-sensitivity literature in
Section~\ref{subsec:related-adversarial}.

\paragraph{Spiked covariance models and unlabeled detection.}
A classical line of high-dimensional statistics studies data with one
unusually strong direction and asks when that direction can be detected from
unlabeled random samples
\cite{johnstone2001distribution,baik2005phase}.
The present paper does not
need that random-matrix theory.
It fixes the special direction and makes the
nuisance geometry exactly orthogonal, so the span obstruction follows from
elementary inner-product calculations.
The same-distribution
companion \citep{ranade2026samedistribution} returns to random Gaussian samples and uses the unlabeled detection
problem as a point of comparison for the supervised thresholds studied here.
The interpolation ray of Remark~\ref{rem:frontier-ray} gives the geometric
bridge: \(\one\) is the top eigenvector of
\(\bm K=d\bm I_n+\gamma\one\one^\top\), so that ray points along the top
empirical principal direction, and the whole train--test tradeoff of
Section~\ref{sec:setup-frontier} is a one-dimensional choice of how far to
travel in the direction unsupervised PCA would select.
The Oja companion
\citep{ranade2026oja} studies iterative one-pass recovery of the analogous
spike direction from unlabeled data, and the multispike companion
\citep{ranade2026multispike} determines when several such rays collapse to
one.

\paragraph{Why use a deterministic caricature?}
The use of an exact stylized geometry follows the methodological tradition of
simplified Euclidean and deterministic models in information theory, from
classical channel geometry and orthogonal signaling
\cite{shannon1949communication,kotelnikov1959optimum,
wozencraft1965principles,gallager1968information} to local Euclidean
approximations of information geometry
\cite{borade2008euclidean,huang2015euclideanNetworks,huang2024universal} and
deterministic caricatures of Gaussian channels
\cite{etkin2008gaussian,avestimehr2011wireless}.
The analogy is
methodological rather than literal: simple geometries are useful when they
turn the governing tradeoff into a calculation.

\paragraph{Signal versus nuisance degrees of freedom.}
Wideband and noncoherent communication theory also gives language for
the difference between learning useful signal and tracking nuisance degrees of
freedom
\cite{telatar2000capacityWideband,medard2002bandwidth,
subramanian2002broadband,verdu2002spectral,zheng2002grassmann,
hassibi2003training}.
In those problems, as here, spending effort on every
observed degree of freedom can be the wrong goal.
The analogy is that useful
combined signal can be easier to recover than every nuisance degree of
freedom that appeared in the observations.

\subsection{Information, recognizability, and privacy}

Information-theoretic generalization bounds control train--test gaps through
mutual information, conditional mutual information, or related stability
quantities
\cite{russo2016controlling,xu2017information,bu2020tightening,
steinke2020conditional}.
Our calculation is complementary: it gives an exact
deterministic frontier rather than a bound for a particular algorithm, and it
shows that the useful gap can have the opposite sign from classical
overfitting.

\paragraph{Memorization, recognizability, and privacy.}
Another line of work asks when individual training samples must remain visible
in the learned model.
\citet{feldman2020memorization} shows that memorizing
rare or noisy examples can be necessary for near-optimal accuracy in long-tailed
distributions, and emphasizes the privacy risk of such memorization.
\citet{brown2021memorization} prove a stronger necessity: in simple prediction
problems, every sufficiently accurate learner must encode substantial
information about many training examples, even information that is irrelevant to
the task.
This connects to information-based memorization views
\cite{attias2024information}.
Empirical work on unintended memorization and training-data extraction shows the
operational side of the same concern: learned models can reveal rare or unique
training examples when queried appropriately
\cite{carlini2019secret,carlini2021extracting}.

Membership-inference work gives a complementary operational lens.
\citet{shokri2017membership} study attacks that distinguish training points
from nonmembers, and \citet{yeom2018privacy} analyze how such risk can follow
from overfitting.
The formal framework behind such questions is differential
privacy \cite{dwork2006calibrating,dwork2014algorithmic}, which enters learning
through private empirical-risk minimization and private deep learning
\cite{chaudhuri2011differentially,abadi2016deep}.\footnote{There is also a large
literature on utility--privacy tradeoffs and private release of high-dimensional
statistics, including private PCA and linear-query lower bounds
\cite{sankar2013utility,dwork2014analyze,chaudhuri2013near,hardt2010geometry,
kattis2017lower,kamath2022new}, and, for clipped-update optimizers,
edge-of-stability-like sharpness dynamics \cite{hussain2025optimizer}.}
Recent empirical work also studies membership leakage beyond
classical overfitting: \citet{khalil2025membership} find that, even in
well-generalized models, the vulnerable samples are often geometrically
atypical within their classes.
A membership attack normally keys on members
looking easy, whether through per-example calibrated losses
\citep{carlini2022firstprinciples} or through the robustness of the model's
predictions to perturbation \citep{choquette2021labelonly}.
In the fourth
quadrant the sign reverses: members can be
recognizable through unusually large residuals rather than unusually small
losses.

The spike oracle would not need to remember any individual nuisance direction.
That recognizability appears because a calibrated
span predictor uses the training examples themselves as its representation: the
shared signal helps prediction, while the nuisance residue leaves a
sample-level trace.  
(Indeed,
the basic prerequisite for having training residuals of the same order as the
desired test residuals is having a factor \(\gamma\) more samples than what is
needed for good generalization --- as the comparison with interpolation has
shown in this paper.)

\subsection{Adversarial sensitivity and non-robustness}
\label{subsec:related-adversarial}

\paragraph{Geometric robust/non-robust feature analogy.}
The adversarial-sensitivity calculation above places the nuisance
residue next to the adversarial-examples literature
\cite{szegedy2014intriguing,goodfellow2015explaining,madry2018towards,
tsipras2019robustness,ilyas2019adversarial,schmidt2018adversarial,
raghunathan2020understanding,hao2024surprising,ribeiro2023overparameterized}.
Misfit, adversarial sensitivity, and test accuracy share one geometric root ---
the nuisance residue --- and differ only in how they probe it: the
training points in their own directions, an aligned adversary along the
residue, a random test input on average.
The parallel with the
robust/non-robust feature distinction of \citet{ilyas2019adversarial} is
geometric.
A label-preserving perturbation is orthogonal to the spike, so
it cannot move the spike oracle at all: the oracle's response is exactly zero,
and the spike is a perfectly robust feature.
A calibrated predictor
constrained to the training span must instead carry the nuisance residue,
which is not predictive of the fresh label.
In the proof, this is just a
norm-duality calculation: the worst allowed perturbation points in the
direction of the nuisance residue.

The coordinate system matters.
In the spike-aligned coordinates used for the
RMS calculation, the nuisance residue is exactly label-preserving.
After a generic rotation, each coordinate
carries a small amount of spike signal together with nuisance variance, so the
same geometry can look like predictive-but-non-robust evidence across the
coordinates.
For the minimum-norm interpolator specifically,
\citet{chinot2022robustness} study related robustness questions directly:
random test inputs average the residue, while adversarial perturbations choose
its direction.

\paragraph{Clean accuracy versus robustness thresholds.}
The sample-threshold ladder in Section~\ref{subsec:phase-ladder} also echoes
clean-versus-robust sample-complexity separations
\cite{schmidt2018adversarial} and robust linear-regression tradeoffs
\cite{javanmard2020precise,xing2021adversarially,
ribeiro2023overparameterized,dohmatob2023robustGeneral,scetbon2023robust}.
A close high-level parallel is the universal-law line of
\citet{bubeck2021universalLaw}, where smooth interpolation can require a
dimension-dependent overparameterization penalty.
The analogy here is only at
the level of the ambient-dimension factor: low random test error begins around
\(d/\gamma^2\) samples, while low RMS adversarial sensitivity can require order
\(d^2/\gamma^2\) samples, up to the perturbation and tolerance constants.
Our
resource is samples rather than parameters, and our mechanism is nuisance
residue rather than a smoothness constraint.
The closest ambient-dimension parallels are the linear-accumulation view of
small coordinate effects \cite{goodfellow2015explaining} and geometric
no-free-lunch limits on robustness
\cite{fawzi2018analysis,gilmer2018adversarial,
dohmatob2019generalizedNoFreeLunch}.

\paragraph{Non-robustness from algorithmic selection.}
A related line studies non-robustness as a consequence of implicit bias.
\citet{frei2023doubleedged} show that gradient flow in two-layer ReLU networks
can select accurate but non-robust classifiers even when robust classifiers
exist.
\citet{li2025featureAveraging} identify feature averaging as a
mechanism by which gradient descent can combine predictive directions into a
non-robust solution.
Other work asks how simplicity bias, architecture, or
robust-ERM implicit bias changes this conclusion
\cite{shah2020pitfalls,min2024implicit,tsilivis2024price}.
The mechanism here
is different: those works study which classifier the training dynamics or
architecture selects, while the low-sample obstruction here is already present
in the train--test/adversarial geometry of the span.
Algorithmic bias may
decide whether a method selects such an accurate but non-robust predictor, but
it is not the source of the residue.

\appendix

\section{Fresh and reused examples have different stability boundaries}
\label{app:stability-boundaries}

The coefficient and state identities \eqref{eq:single-geometric-coeffs} of
Lemma~\ref{lem:single-geometric} are exact for every \(\eta>0\), with the closed geometric-sum
form \eqref{eq:single-geometric-energy} read under its stated condition \(\mu^{2}\ne1\).
It is
the uniform nuisance bound \eqref{eq:nuisance-floor} that needs \(0<\eta\gamma<2\), while
the risk formula \eqref{eq:single-geometric-risk} is stated on that range but in fact needs only
\(\mu^{2}\ne1\).
It is worth seeing what that condition is and is not.
Three stability boundaries must be distinguished.
Two of them are
numerically close in this model and the third is far from both, which is exactly the asymmetry
that one pass exploits.

\emph{A fresh example} contributes only through the shared spike, so by the induction in
Lemma~\ref{lem:single-geometric} the shared residual obeys \(b_t=(1-\eta\gamma)b_{t-1}\)
exactly --- a scalar recursion, not an approximation.
It contracts if and only if
\begin{equation}\label{eq:stability-fresh}
    0<\eta<\frac{2}{\gamma}=\frac{2}{\lambda_{\max}(\bm\Sigma)} ,
    \qquad\text{using the standing assumption }\gamma>1\text{ of Subsection~\ref{subsec:setup}.}
\end{equation}

\emph{A reused example} is a different situation, because by the time it comes back around it
already carries a coefficient of its own because it has influenced the SGD iterate.
Present \(\vec x_i\) again to an iterate \(\vec w=\sum_{j}\alpha_j\vec x_j\) that has already
seen it, so that \(\alpha_i\ne0\) in \eqref{eq:train-prediction} and the prediction on that
example is \(\gamma s+d\alpha_i\).
A fresh example would have contributed only the first of
those two terms.
Now look at the signed residual \eqref{eq:frontier-h-def}, as
\(h_i:=\vec w^{\top}\vec x_i-1\).
Using \(b=1-\gamma s\) from \eqref{eq:sgd-state-defs},
\begin{equation}\label{eq:reuse-residual}
    h_i
    =
    \underbrace{-\,b\vphantom{d\alpha_i}}_{\text{shared spike residual}}
    \;+\;
    \underbrace{d\,\alpha_i}_{\text{its own stored nuisance}} .
\end{equation}
It is \eqref{eq:reuse-residual}, not just the shared \(b\), that drives the update, and of
its two terms the stored one can be much larger.
Specializing \eqref{eq:reuse-residual} to the end of a full pass, using
\eqref{eq:single-geometric-coeffs} at \(t=n\),
\begin{equation}\label{eq:reuse-residual-onepass}
    h_i=-\mu^{n}+d\eta\,\mu^{\,i-1}
    \qquad(1\le i\le n).
\end{equation}
Whenever \(1/d<\eta<1/\gamma\), which includes the rate
\eqref{eq:eta-pass-def} under the hypotheses of
Theorem~\ref{thm:one-pass-sgd}, \(0<\mu<1\) and \(d\eta>1\), so the second term of \eqref{eq:reuse-residual-onepass}
satisfies \(d\eta\,\mu^{\,i-1}>\mu^{\,i-1}\ge\mu^{\,n}\) and hence \(h_i>0\) for every example
the pass has used.
The positive signed residual means that the used example is \emph{overshot},
due to its own stored nuisance.

Now apply \eqref{eq:sgd-update} to that reused \(\vec x_i\).
Since
\(1-\vec w^{\top}\vec x_i=-h_i\), the update is \(\vec w\mapsto\vec w-\eta h_i\vec x_i\), which
subtracts \(\eta h_i\) from \(\alpha_i\).
Because that is the only coefficient that moves, this
subtracts \(\eta h_i\) from \(s\) as well.
Both terms of \eqref{eq:reuse-residual} therefore
move, and they move the same way:
\begin{equation}\label{eq:reuse-multiplier}
    h_i'
    =
    \bigl(\gamma(s-\eta h_i)-1\bigr)+d\bigl(\alpha_i-\eta h_i\bigr)
    =
    h_i-\eta(\gamma+d)h_i
    =
    \bigl(1-\eta\|\vec x_i\|_2^{2}\bigr)h_i .
\end{equation}
While the actual residual \eqref{eq:reuse-residual} depends on the whole iterate through \(s\) and
\(\alpha_i\), the multiplier $\bigl(1-\eta\|\vec x_i\|_2^{2}\bigr)$ in \eqref{eq:reuse-multiplier}
does not.
This is the same as the one-example calculation of Section~\ref{sec:one-sample}, and it
is contractive if and only if
\begin{equation}\label{eq:stability-reuse}
    0<\eta<\frac{2}{\gamma+d}=\eta_{\rm edge},
\end{equation}
the boundary already pointed out in \eqref{eq:one-sample-rates}.

\emph{Full-batch gradient descent} on the summed least-squares objective is a third situation,
requiring \(\eta<2/\lambda_{\max}(\Xt\Xt^{\top})=2/(d+\gamma n)\).
In the benign-misfitting window \(\gamma n\ll d\),
so this and the reuse boundary \eqref{eq:stability-reuse} are both \((2/d)(1+o(1))\): conceptually distinct,
but numerically close together.

Notice that it is the \emph{fresh} boundary that stands apart, by a factor
\begin{equation}\label{eq:stability-gap}
    \frac{2/\gamma}{2/(\gamma+d)}
    =
    1+\frac{d}{\gamma}
    \;\gg\;
    n
    \qquad\text{in the benign-misfitting window \eqref{eq:benign-misfitting-window},}
\end{equation}
since that window is exactly \(\gamma n\ll d\ll\gamma^{2}n\).
This means this model
has a vast interval of learning rates --- more than \(n\)-fold wide --- that contract
the shared spike residual on fresh examples while amplifying any example that is revisited.
The rate \eqref{eq:eta-pass-def} lies inside this interval under the
hypotheses of Theorem~\ref{thm:one-pass-sgd}.
This is what makes the
first epoch different in kind from any potential later ones.%
\footnote{For readers meeting this material for the first time: the full-batch
condition is the textbook gradient-descent threshold for a quadratic
\cite{polyak1987introduction}.
The fresh-example condition \eqref{eq:stability-fresh} is
spiritually the same as the mean-convergence condition \(0<\eta<2/\lambda_{\max}(\bm\Sigma)\) of the LMS
recursion in adaptive filtering \cite{widrow1985adaptive}; the correspondence is an analogy,
since no convergence of an empirical covariance to \(\bm\Sigma\) is used anywhere here.
The
reuse condition \eqref{eq:stability-reuse} is the relaxed-projection condition: each update is a
relaxed Kaczmarz step with relaxation parameter \(\lambda=\eta\|\vec x_i\|_2^{2}\), and
repeatedly applied relaxed projections converge exactly when \(0<\lambda<2\);
this is the classical projection-method line from the original Kaczmarz
iteration to the relaxed algebraic reconstruction technique (ART)
\citep{kaczmarz1937angenaeherte,herman2009fundamentals}, with the randomized
variant in \citet{strohmer2009randomized}.
Constant-learning-rate SGD
revisiting a frozen training set is an incremental gradient method in the sense of
\citet{bertsekas2011incremental}.}

\section{More about one-pass SGD: strong spikes, learning-rate robustness, training error ascends, how to avoid the log penalty, and mini-batch scaling}
\label{app:sgd-refinements}
\enlargethispage{1pt}

This appendix continues the one-pass analysis of Section~\ref{sec:sgd}.
It
explains the inequivalent senses in which the spike can be strong, shows that
the useful learning rates form a multiplicatively wide window (also under
label noise when \(\sigma_\xi^{2}=o(\min(n,R))\)), computes the step-by-step ascent of the empirical training
error along the tuned many-sample pass, exhibits a sample-splitting scheme
showing that the logarithmic penalty is algorithmic rather than geometric,
and derives the mini-batch scaling analogue.

\subsection{What ``the spike is strong'' can mean}
\label{subsec:spike-strength-senses}

The word \emph{strong} does several different jobs in this paper, and this subsection is where the
last of them arrives.
Separating them is worth a comment, because they are not variations in
degree; they are comparisons of \(\gamma\) against different things.

\emph{(a) Strong signal:} \(\gamma>1\), the standing assumption of
Subsection~\ref{subsec:setup}.
The spike direction then carries more variance
in the test law \eqref{eq:test-law} than any one nuisance direction.
By Remark~\ref{rem:strong-spike} this is
what makes the useful Pareto branch the overshooting one.
It involves neither \(n\) nor \(d\).

The remaining senses all compare \(\gamma\) against \(n\) and \(d\) together, and the first
three of them are \emph{statistical}: they say which predictors can generalize, and each is a
power law.

\emph{(b) Strong enough that some span predictor generalizes:} \(n\gg d/\gamma^{2}\),
equivalently \(\gamma \gg \sqrt{d/n}\), the lower edge of \eqref{eq:benign-misfitting-window}.
The best
span predictor then reaches \(g_{\min}=(1+o(1))/R\to0\) by
\eqref{eq:frontier-endpoints-rho-R}.

\emph{(c) Strong enough that interpolation itself generalizes:} \(n\gg d/\gamma\), equivalently
\(\gamma \gg d/n\), the upper edge of \eqref{eq:benign-misfitting-window}.
This is when the
interpolator's calibration \(r_{\rm int}=\rho/(1+\rho)\) of \eqref{eq:rint-def} finally reaches
one.
Senses (b) and (c) are separated by exactly a factor \(\gamma\) in the
real-valued crossover scale for \(n\), and that separation \emph{is} the
benign-misfitting window.

\emph{(d) Strong enough to support few-shot learning:} \(\gamma\ll d\ll\gamma^{2}\), the
\emph{few-shot window} of Footnote~\ref{fn:few-shot-window}.
Here (b)'s threshold
\(d/\gamma^{2}\) has dropped below a handful of examples while (c)'s threshold \(d/\gamma\) is
still above it, so this is \eqref{eq:benign-misfitting-window} read at \(n\) of constant order.
It is the practically pointed case: few-shot learning can unlock long before interpolation is any
good, which is what the one-example calculation of Section~\ref{sec:one-sample} is a caricature
of.

The last sense is different in kind, and it is the one this subsection wants to call out specifically.

\emph{(e) Strong enough that one pass can no longer simultaneously take small
steps and be close to optimal:} \(n\lesssim\log R\).
Equivalently,
\(\gamma\ge\sqrt{d/n}\;e^{cn}\) for some constant \(c>0\), so this happens
only when the spike is exponentially strong in the number of samples \(n\).
This one is not about which predictors exist.
It is about the step size the tuned rate asks for.
That rate carries budget \(\Lambda_{\rm pass}=\tfrac12\log R\) in
\(n\) updates, and its step \eqref{eq:eta-pass-step} equals \(1\) exactly at
\(n=\tfrac12\log R\) and \(2\) exactly at \(n=\tfrac14\log R\).
So the condition
\begin{equation}\label{eq:very-strong-spike-condition}
    \underbrace{n}_{\text{updates available}}
    \ \lesssim\
    \underbrace{\log R}_{\text{budget the tuned rate wants}}
\end{equation}
is where \eqref{eq:eta-pass-step} stops being small, and it therefore denies the many-sample condition \(\log R=o(n)\) under which the
tuned rate attains the optimal \(\Theta\bigl(\tfrac{\log R}{R}\bigr)\) order.
The two endpoints are exact.
At \(n=\tfrac12\log R\) the tuned rate reaches
\(\eta\gamma=1\), the last point of the branch on which participation counts examples, and it
leaves that branch for smaller \(n\).
At \(n=\tfrac14\log R\) it reaches \(\eta\gamma=2\),
which the stable range \(0<\eta\gamma<2\) excludes.
Past there an exact one-line
calculation takes over --- one example is nearly enough: at \(\eta\gamma=1\) a
single update zeros the spike residual
outright --- no asymptotics --- after which every later example arrives fully
calibrated and contributes nothing, so the pass has participation
\(n_{\rm eff}=1\) exactly.
The resulting test error is \(n/R=d/\gamma^{2}\),
within a logarithmic factor of the best any span predictor can do.

Condition (e) is \emph{algorithmic}: by \eqref{eq:b-recursion} each
update multiplies the spike residual by a fixed factor.
Whenever that factor has fixed magnitude
strictly between zero and one, driving the residual down to a given negative power of \(R\) costs
a number of updates proportional to \(\log R\).
The degenerate case \(\eta\gamma=1\), where one
update zeros the residual, is exactly the one-shot calculation above.
The comparison is
therefore against a logarithm rather than a power, and what it constrains is what one
constant-rate pass can build.
Two asymptotics of the one-pass analysis become unavailable past (e), and quoting either there is a
mistake: the participation deficit of \eqref{eq:one-pass-neff-asymp} needs \(\Lambda/n\to0\),
and the optimal-order claim for the tuned budget needs \(\log R=o(n)\).

\subsection{Learning-rate robustness}
\label{subsec:lr-robustness}

The rate \(\eta_{\rm pass}\) of Theorem~\ref{thm:one-pass-sgd} is written
explicitly because the model is exact.
It attains the optimal \(\Theta\bigl(\tfrac{\log R}{R}\bigr)\) order 
under the mild condition \(\log R=o(n)\), and it is certainly not a knife edge.
Using the calibration budget \(\Lambda\) of \eqref{eq:Lambda-def} the
essence of the whole story is one line.
By the exact full-pass risks
\eqref{eq:one-pass-exact-risks} with the contraction asymptotics
\eqref{eq:one-pass-contraction}, the test error is \(e^{-2\Lambda}+\Lambda/(2R)\) to leading
order on the narrower branch \(\Lambda\to\infty\) with \(\Lambda^{2}/n\to0\) that
the contraction asymptotics \eqref{eq:one-pass-contraction} need,
and at most \(e^{-2\Lambda}+\Lambda/R\) throughout \(0<\eta\gamma\le1\), which is the form Corollary~\ref{cor:learning-rate-robustness} below proves and uses.
Either way the
first term needs \(\Lambda\) large and the second needs \(\Lambda\) small compared with
\(R\).
That is a lot of flexibility.

\begin{corollary}[The safe window is \(1\ll\Lambda\ll R\)]
\label{cor:learning-rate-robustness}
Run one pass of the SGD update \eqref{eq:sgd-update} with \(0<\eta\gamma\le1\), and
recall the calibration budget \(\Lambda=\eta\gamma n\) from \eqref{eq:Lambda-def} and the
signal-to-nuisance ratio \(R=\gamma^{2}n/d\) from \eqref{eq:scale-identities}.
Then
\begin{equation}\label{eq:lr-robust}
    \Ttest(\vec w_n)\ \le\ e^{-2\Lambda}+\frac{\Lambda}{R}.
\end{equation}
Consequently, for a family of problems with \(R\to\infty\) and
\begin{equation}\label{eq:lr-safe-window}
    1\ \ll\ \Lambda\ \ll\ R,
    \qquad\text{equivalently}\qquad
    \frac{1}{\gamma n}\ \ll\ \eta\ \ll\ \frac{\gamma}{d},
\end{equation}
we have \(\Ttest(\vec w_n)\to0\).
In the benign-misfitting window
\eqref{eq:benign-misfitting-window} we also have \(\Ttrain(\vec w_n)\to\infty\).
\end{corollary}

\begin{proof}
Start from the exact test error
\(\Ttest(\vec w_n)=\mu^{2n}+\frac{\Lambda}{R}\cdot\frac{1-\mu^{2n}}{2-\eta\gamma}\)
of \eqref{eq:one-pass-exact-risks}.
For the spike term,
\(0\le1-\eta\gamma\le e^{-\eta\gamma}\) on the closed interval \(0<\eta\gamma\le1\) gives
\(\mu^{2n}\le e^{-2\eta\gamma n}=e^{-2\Lambda}\) by \eqref{eq:mu-def} and
\eqref{eq:Lambda-def}.
Note that this is also well defined at the endpoint \(\eta\gamma=1\),
where in fact \(\mu=0\).

For the nuisance term, \(0<\eta\gamma\le1\) gives \(2-\eta\gamma \ge 1\), and \(|\mu|<1\)
gives \(1-\mu^{2n}\le1\), so that term is at most \(\Lambda/R\).
Adding the two gives
\eqref{eq:lr-robust}.
Under \eqref{eq:lr-safe-window} both terms vanish.

The two forms of the window \eqref{eq:lr-safe-window} agree because \(\Lambda=\eta\gamma n\) by \eqref{eq:Lambda-def} and
\(\Lambda/R=d\eta/\gamma\) by \eqref{eq:scale-identities}.
Finally \(\vec w_n\) is then a family
of span predictors with vanishing test error, so
Corollary~\ref{cor:fourth-quadrant-necessity} immediately gives \(\Ttrain(\vec w_n)\to\infty\) within the benign-misfitting window.
\end{proof}

The two edges of \eqref{eq:lr-safe-window} control different terms.
The lower
edge \(\Lambda\gg1\) is the calibration requirement.
Exact calibration
requires coefficient sum \(s=1/\gamma\) by
\eqref{eq:spike-calibration}.
A budget below one cannot reach \(s=1/\gamma\) by
the bound in Footnote~\ref{fn:calibration-budget}.

Along a family with \(n\to\infty\) --- the case Corollary~\ref{cor:learning-rate-robustness} is about --- a budget that stays bounded cannot
calibrate either.
If \(\Lambda\le M\) then
\(\mu^{n}=(1-\Lambda/n)^{n}\ge(1-M/n)^{n}\to e^{-M}>0\), so the calibration formula
\(r_n=1-\mu^{n}\) of \eqref{eq:one-pass-calibration} stays bounded away from one.
At bounded \(n\) this second fact fails, because
\(\eta\gamma\to1\) drives \(\mu^{n}\to0\) at budget \(\Lambda\to n\).
That is precisely
the very-strong-spike regime, where a single example at \(\eta\gamma=1\) is nearly
enough.

The upper edge \(\Lambda\ll R\) caps the nuisance variance written into the SGD iterates during the pass.
Notice that in the \(\eta\) form, the upper edge $\frac{\gamma}{d}$ involves no \(n\).
This is not a typo but a restatement of \eqref{eq:nuisance-floor}.
At fixed \(\eta\), accumulated nuisance is bounded independently of
pass length in the noiseless label case.

The two edges of \eqref{eq:lr-safe-window} separate exactly when \(R\to\infty\), which is exactly
the asymptotic condition for the span to become useful.
The intersection of the window \eqref{eq:lr-safe-window}
with the branch \(0<\eta\gamma\le1\) needs more.
That branch caps \(\Lambda\le n\) by
\eqref{eq:Lambda-def}, so the lower edge \(\Lambda\gg1\) forces \(n\to\infty\) as well.
A
family with bounded \(n\) and \(R\to\infty\) is exactly the conceptually distinct
very-strong-spike case, which the exact one-shot calculation at \(\eta\gamma=1\)
handles instead.

With the noisy training labels \(y_i=1+\sigma_\xi\xi_i\) of
Subsection~\ref{subsec:one-complete-pass}, the lower edge is unchanged, while
label noise adds upper-bound constraints:
\begin{equation}\label{eq:noisy-lr-bound-main}
    \E_\xi\,\Ttest(\vec w_n)
    \ \le\
    e^{-2\Lambda}+\frac{\Lambda}{R}
    +\sigma_\xi^2\Bigl(\frac{2\Lambda^2}{R}+\frac{\Lambda}{n}\Bigr),
    \qquad 0<\eta\gamma\le1 .
\end{equation}
The reason is that label noise is
steadily written into fresh orthogonal directions by all \(n\) examples, rather than being
geometrically discounted.
Notice that the final $\frac{\sigma_\xi^2 \Lambda}{n}$ term does not 
block asymptotically zero test error as long as $\sigma_\xi^2 \ll n$.
Combined
with $\sigma_\xi^2 \ll R$, which keeps the $\frac{2\sigma_\xi^2\Lambda^2}{R}$
term small at some $\Lambda\to\infty$, a suitable learning rate can be found. 

\begin{remark}[Constant-factor misspecification and hyperparameter sweeps]
\label{rem:constant-factor-lr}
Work on the many-sample branch \(\log R=o(n)\), where the tuned rate
\(\eta_{\rm pass}\) of \eqref{eq:eta-pass-def} attains the optimal order.
A sweep will not land on it exactly, so take
\(\eta=c\,\eta_{\rm pass}\) for a fixed constant \(c>0\).
This is a calibration budget
\(\Lambda=\tfrac{c}{2}\log R\).
The branch hypothesis makes \(\eta\gamma=c\log R/(2n)\le1\)
automatic for all large \(n\), so Corollary~\ref{cor:learning-rate-robustness} applies.
Then
\eqref{eq:lr-robust} gives
\begin{equation}\label{eq:lr-c-bound}
    \Ttest(\vec w_n)
    \ \le\
    R^{-c}+\frac{c\log R}{2R}.
\end{equation}
Every fixed multiplicative misspecification therefore still gives vanishing test error.
An
underestimate (\(c<1\)) leaves the slower spike-bias decay \(R^{-c}\).
An overestimate
(\(c>1\)) writes a constant multiple more nuisance variance.
Equivalently, by
\eqref{eq:one-pass-neff-asymp}, it has a constant factor fewer examples participating.
That second
reading needs \(\Lambda\to\infty\) and \(\Lambda/n\to0\), which the branch condition already
supplies.
So any log-spaced sweep whose range brackets \(\eta_{\rm pass}\) and whose adjacent
grid points differ by a fixed factor contains a successful learning rate for all sufficiently
large problems.

The contrast with the very-strong-spike branch is sharp.
There, on the branch where
eventually \(n<R\), the useful rates form a \emph{certified} tolerance band
\(|1-\eta\gamma|\le\tau_n=(n/R)^{1/(2n)}\) about \(\eta\gamma=1\) --- a sufficient
band, not the exact set of useful rates --- of fixed relative width when
\(n\asymp\log R\), sharpening only when \(\log R/n\to\infty\).
Here the successful
range is multiplicatively unbounded.
\end{remark}

\subsection{Ascending empirical training error}
\label{subsec:ascent-derivation}

Work at the tuned rate \(\eta_{\rm pass}\) of \eqref{eq:eta-pass-def}, under the
hypotheses of Theorem~\ref{thm:one-pass-sgd}: \(\rho\to0\), \(R\to\infty\), and
\((\log R)^{2}=o(n)\), so that \(\Lambda=\Lambda_{\rm pass}=\tfrac12\log R\) satisfies
\(\Lambda\to\infty\), \(\Lambda/n\to0\), and \(\Lambda/R\to0\).
All three are used
below.

The general-\(t\) training error follows from Lemma~\ref{lem:two-scalar} exactly as the test
error did.
Example by example it is the residual \(h_i=-b+d\alpha_i\) of
\eqref{eq:reuse-residual} squared and averaged,
with \(\alpha_i=0\) on the \(n-t\) examples the pass has not yet reached.
Substituting
\(s=s_t\) from \eqref{eq:single-geometric-coeffs} and \(q=q_t\) from
\eqref{eq:single-geometric-energy} into the exact training error
\(\Ttrain=(\gamma s-1)^{2}+\frac{2ds}{n}(\gamma s-1)+\frac{d^{2}}{n}q\)
of \eqref{eq:scalar-train-error}, and using
\(\gamma s_t-1=-\mu^{t}\) together with \(d/(n\gamma)=1/\rho\) from
\eqref{eq:scale-identities},
\begin{equation}\label{eq:sgd-train-error-exact}
    \Ttrain(\vec w_t)
    =
    \mu^{2t}
    -
    \frac{2}{\rho}\bigl(1-\mu^{t}\bigr)\mu^{t}
    +
    \frac{d^{2}}{n}\,q_t .
\end{equation}

Appendix~\ref{app:stability-boundaries} showed this learning rate is small on the shared spike,
\(\eta_{\rm pass}\gamma=\Lambda/n=\log R/(2n)\) by \eqref{eq:eta-pass-step}, which is \(o(1)\)
on the many-sample branch \(\log R=o(n)\), while its effect on the
example currently being used, \(\eta_{\rm pass}\|\vec x_i\|_2^2=\Lambda/n+\Lambda/\rho\),
diverges.
This is the precise sense in which the update overshoots individual training points.
Compared with the repeated-example stability boundary \(\eta_{\rm edge}=2/(\gamma+d)\) of
\eqref{eq:stability-reuse},
\begin{equation}\label{eq:eta-pass-edge-ratio}
    \frac{\eta_{\rm pass}}{\eta_{\rm edge}}
    =
    \frac{\eta_{\rm pass}}{2/(\gamma+d)}
    =
    \frac12\eta_{\rm pass}(\gamma+d)
    =
    \left(\frac{\Lambda}{2}+o(\Lambda)\right)\frac{d}{\gamma n}
    \to\infty.
\end{equation}
In edge-of-stability language \cite{cohen2021gradient}, the full-batch
empirical sharpness is \(\lambda_{\max}(\Xt\Xt^\top)=d+\gamma n\), dominated
in this window by the nuisance contribution \(d\) since \(\rho\to0\) (that
literature states sharpness for the averaged loss, which divides by \(n\)).
The first pass remains controlled because each
nuisance direction is encountered only once, while the spike direction is shared
across all updates.

\paragraph{The ascent, step by step.}
At the tuned rate \(\eta_{\rm pass}\) of \eqref{eq:eta-pass-def}, SGD indeed ascends rather than
descends the full-batch empirical training error.
Write \(b_t=\mu^{t}\) for the spike
residual after \(t\) updates, as in \eqref{eq:single-geometric-coeffs}.
Substituting
\(q_t=\eta^{2}\sum_{k<t}\mu^{2k}\) from \eqref{eq:single-geometric-energy} into
\eqref{eq:sgd-train-error-exact}, and using \(d^{2}\eta^{2}/n=\eta\gamma\Lambda/\rho^{2}\) from
\eqref{eq:Lambda-def} and \eqref{eq:scale-identities},
\begin{equation}\label{eq:sgd-train-error-path}
    \Ttrain(\vec w_t)
    =
    b_t^{2}
    -
    \frac{2}{\rho}\,b_t\bigl(1-b_t\bigr)
    +
    \frac{\eta\gamma\,\Lambda}{\rho^{2}}\sum_{k=0}^{t-1}b_k^{2} .
\end{equation}
Since \(b_{t+1}=(1-\eta\gamma)b_t\) we have
\(b_{t+1}^{2}-b_t^{2}=-\eta\gamma(2-\eta\gamma)b_t^{2}\).
Substituting these step differences
into \eqref{eq:sgd-train-error-path} and factoring out \(\eta\gamma\) yields
\begin{equation}\label{eq:climbing-uphill}
    \Ttrain(\vec{w}_{t+1})-\Ttrain(\vec{w}_t)
    =
    \eta\gamma\left[
    \frac{2b_t}{\rho}
    +
    b_t^{2}
    \left\{
    \frac{\Lambda}{\rho^{2}}-\Bigl(1+\frac{2}{\rho}\Bigr)\bigl(2-\eta\gamma\bigr)
    \right\}
    \right].
\end{equation}
The term \(2b_t/\rho\) inside the outer bracket of \eqref{eq:climbing-uphill} is positive,
since \(b_t>0\) by \eqref{eq:single-geometric-coeffs} on the branch \(0<\eta\gamma<1\).
Inside
the braces of \eqref{eq:climbing-uphill}, \(\Lambda/\rho^{2}\) dominates
\((1+2/\rho)(2-\eta\gamma)\le2+4/\rho\), because their ratio is
\(\frac{\Lambda}{4\rho}(1+o(1))\to\infty\) when \(\Lambda\to\infty\) and \(\rho\to0\).
So
the whole bracket in \eqref{eq:climbing-uphill} is positive for all sufficiently large dimensions
and all \(0\le t<n\).

The fresh-test error moves in the opposite direction.
By the exact per-step test error
\eqref{eq:single-geometric-risk} of Lemma~\ref{lem:single-geometric},
\[
    \Ttest(\vec{w}_{t+1})-\Ttest(\vec{w}_t)
    =
    \mu^{2t}\left[-2\eta\gamma+(\eta\gamma)^{2}+\frac{\Lambda^{2}}{nR}\right]
    =
    -\mu^{2t}\,\eta\gamma\,\bigl(2+o(1)\bigr)<0 ,
\]
using \(\eta\gamma=\Lambda/n=o(1)\) and \(\Lambda/R=o(1)\).
Thus the full empirical training
error monotonically increases step by step while the fresh-test error decreases monotonically step by step.

\subsection{Is the logarithm penalty fundamental? The case of flat averaging.}
\label{subsec:flat-averaging}

Under the hypotheses of Theorem~\ref{thm:one-pass-sgd}, one pass of constant-rate SGD sits a factor
\(\tfrac14\log R\,(1+o(1))\) from the frontier on both axes.
No constant stable learning rate closes that test-error gap on the many-sample
branch \(\log R=o(n)\): the minimum of the test error over \(0<\eta\gamma<2\) is of
order \(\tfrac{\log R}{R}\).
That leaves the question of what is to blame.
The asymptotic logarithm factor could be the price of
the streaming constraint, or of not knowing \(\gamma\), or of not knowing where the spike lives.
In any of those cases every practically realizable span predictor would pay it too, and
the best-span error \(g_{\min}=1/(1+R)\) of \eqref{eq:frontier-endpoints-rho-R} would be a
loose oracle benchmark rather than an achievable one.

It is none of these.
This subsection exhibits a span predictor that is fully data-driven and
attains the best-span test error up to a \(1+o(1)\) factor.
It uses no knowledge of
\(\gamma\) and no knowledge of the spike coordinate.
Everything here is on the many-sample
branch \(\log R=o(n)\), which is what the comparison needs
and in particular forces \(n\to\infty\).

Average the first \(m=n-1\) training points and use the last remaining point \(\vec x_{n}\) only
to set the calibration by multiplying the average appropriately.
That held-out point is what
makes the construction data-driven.
It is therefore what makes this an answer to the question
above rather than an oracle comparison.
If \(\gamma\) were known the held-out point would be
unnecessary.
Define
\begin{equation}\label{eq:flat-average-def}
    \bar{\vec{x}}_m=\frac1m\sum_{i=1}^m \vec{x}_i.
\end{equation}
Since \(\vec{v}_{m+1}\) is orthogonal to \(\vec{v}_1,\ldots,\vec{v}_m\),
\begin{equation}\label{eq:flat-heldout-gram}
    \bar{\vec{x}}_m^{\top}\vec{x}_{m+1}=\gamma.
\end{equation}
Thus the held-out point reveals the unknown \(\gamma\) exactly.%
\footnote{For the first nontrivial flat average, \(m=n-1=1\), the identity \eqref{eq:flat-heldout-gram} reads \(\vec{x}_1^{\top}\vec{x}_2=\gamma\).
Thus a second
training point with a noiseless label reveals \(\gamma\) exactly.
With
the dimension \(d\) known, the finite-width test-optimal rate
\eqref{eq:one-sample-eta-opt} depends on the unknown model only
through \(\gamma\).
In the sensor construction of Subsection~\ref{subsec:sensor-misfitting}, \(m_{\rm sens}=d+1\) and \(\|\vec s_{\rm resp}\|^2=m_{\rm sens}\), so the revealed \(\gamma\), together with the known dimension \(d\), also determines the per-sensor noise variance \(\sigma^2=(d+1)/(\gamma-1)\) by \eqref{eq:sensor-spike-law}.
The exact dot-product equality in \eqref{eq:flat-heldout-gram} uses the stylized orthogonality.
What happens when the nuisance components of the two training vectors are only approximately orthogonal?
The same-distribution companion
\citep{ranade2026samedistribution} shows how two measurements at known
values estimate the per-sensor noise variance from their sensor
readings, with error tending to zero as the sensor count grows at
fixed noise variance.}
Dividing \(\bar{\vec{x}}_m\) by
this revealed \(\gamma\) calibrates its spike coordinate, so the predictor
\begin{equation}\label{eq:flat-predictor-def}
    \vec{w}_{\rm avg}
    =
    \frac{1}{\gamma m}\sum_{i=1}^m \vec{x}_i
\end{equation}
has exact spike calibration:
\[
    w_{\rm avg}[1]=\frac1{\sqrt\gamma}.
\]
In signal-processing language, this average acts as a data-calibrated
matched filter \cite{wozencraft1965principles,vanTrees1968detection}: its
prediction correlates a test input with the combined training signal and
normalizes by the spike calibration.
Only averaged nuisance remains, so
\begin{equation}\label{eq:flat-average-risk}
    \Ttest(\vec w_{\rm avg})=\frac{d}{\gamma^{2}m}
    =
    \frac{d}{\gamma^{2}(n-1)} .
\end{equation}
Applying \(g_{\min}=1/(1+\gamma^{2}m/d)\) from
\eqref{eq:frontier-endpoints-rho-R}, read at sample size \(m=n-1\), gives the
best-span test error \(g_{\min}=(d/\gamma^{2}m)(1+o(1))\) whenever
\(d\ll\gamma^{2}m\).
So
\eqref{eq:flat-average-risk} matches the frontier up to a \(1+o(1)\) factor.
The held-out point
costs a further factor, exactly \(\frac{n}{n-1}\bigl(1+\frac1R\bigr)\) with
the standing \(R=\gamma^{2}n/d\), against the frontier at the
full sample size \(n\), which is also
\(1+o(1)\) because \(n\to\infty\) and \(R\to\infty\) here --- at small \(n\) it would not be, and the comparison
would have to be stated at \(m\).
The answer to this
subsection's title is therefore that the logarithm penalty is entirely the constant-learning-rate SGD algorithm's fault.

The comparison also localizes the defect exactly.
Both procedures see the same \(n\) examples
and both reach calibration \(r=1+o(1)\).
The flat average \eqref{eq:flat-predictor-def} has
equal coefficients, so its participation is \(n_{\rm eff}=n-1\).
That is the largest value
\eqref{eq:neff-range} allows for a profile on those \(n-1\) points.
One pass of SGD at
\(\eta_{\rm pass}\) instead has participation \(\frac{4n}{\log R}(1+o(1))\) by
\eqref{eq:one-pass-neff-asymp}, because \eqref{eq:single-geometric-coeffs} makes its coefficients
decay geometrically.

By Remark~\ref{rem:participation}, calibration and participation are the only two things that
matter.
The entire ratio between the two risks --- a factor \(\tfrac14\log R\,(1+o(1))\) ---
is therefore due to participation.
The logarithm is attributable to the \emph{shape} of the coefficient
profile, and not to the streaming constraint, not to the span, and not to \(\gamma\) being
unknown.
A geometric profile participates a fixed number of decay lengths' worth of examples, no
matter how many it is shown.

\subsection{Mini-batches and linear scaling}
\label{subsec:minibatch-scaling}
In practice, increasing the mini-batch size from one example to \(B_{\rm mb}\) examples is often
accompanied by multiplying the learning rate by \(B_{\rm mb}\).
This is the \emph{linear
scaling rule} \cite{goyal2017accurate}.
This paper's account of why SGD can work is very
different from the standard one of approximating a path to an empirical risk minimizer, so it is
worth asking whether such a scaling rule still makes sense here.
The participation count
\eqref{eq:neff-def} answers that directly, and exactly.

Take the mini-batch loss to be averaged over the batch, which is the convention the linear rule is
stated for.
A single-example update applies the full learning rate to its one example, whereas
averaging divides each example's gradient by \(B_{\rm mb}\).
Restoring the same
\emph{per-example} learning rate is exactly what a factor of \(B_{\rm mb}\) in
\(\eta_{\rm mb}\) does.%
\footnote{Subsection~\ref{subsec:related-large-steps} compares this mechanism with the usual
linear-scaling, gradient-noise-scale, critical-batch-size, and stochastic differential equation
(SDE) accounts in the literature.}

All \(B_{\rm mb}\) examples of a batch are fresh and see the same spike residual, so a batch
writes the same coefficient on each of its members.
For convenience, assume \(B_{\rm mb}\)
divides \(n\), so the pass is \(T=n/B_{\rm mb}\) equal batches.
Writing
\(\bar\alpha_k\) for the common coefficient the \(k\)th batch writes on each of its members, we
have \(s=B_{\rm mb}\sum_{k}\bar\alpha_k\) and \(q=B_{\rm mb}\sum_{k}\bar\alpha_k^{2}\), so
\eqref{eq:neff-def} factors:
\begin{equation}\label{eq:minibatch-participation}
    n_{\rm eff}
    =
    B_{\rm mb}\cdot
    \frac{\bigl(\sum_{k=1}^{T}\bar\alpha_k\bigr)^{2}}{\sum_{k=1}^{T}\bar\alpha_k^{2}} .
\end{equation}
Mini-batching multiplies participation by the batch size, exactly, and the remaining factor is
the batch-level participation of the \(T\) batches.
The batch recursion is the single-example recursion of
Lemma~\ref{lem:single-geometric} with \(n\) replaced by \(T\) and \(\eta\) by \(\eta_{\rm mb}\),
so the exact one-pass participation formula for \(n_{\rm eff}\) in
\eqref{eq:one-pass-neff} applies to it.
Write \(\Lambda:=\eta_{\rm mb}\gamma T\) for the
batch-level calibration budget, which is \eqref{eq:Lambda-def} read at the batch level and
coincides with it exactly under the linear scaling \(\eta_{\rm mb}=B_{\rm mb}\eta\).
Let
\begin{equation}\label{eq:minibatch-mu}
    \mu_{\rm mb}
    :=
    1-\eta_{\rm mb}\gamma
    =
    1-\frac{B_{\rm mb}\Lambda}{n}
    \ \in(-1,1)
\end{equation}
be the batch analogue of the one-step multiplier \(\mu=1-\eta\gamma\) of \eqref{eq:mu-def}.
So one mini-batch multiplies
the shared spike residual by \(\mu_{\rm mb}\), exactly as one example multiplies it by \(\mu\).
This is the only new quantity the subsection needs, and it is directly readable: a small
\(|\mu_{\rm mb}|\) means a single mini-batch already makes a large dent in the spike residual on
its own.
Then \eqref{eq:minibatch-participation} and the one-pass \(n_{\rm eff}\) formula
\eqref{eq:one-pass-neff} give, exactly and
for every \(\mu_{\rm mb}\in(-1,1)\),
\begin{equation}\label{eq:minibatch-participation-exact}
    n_{\rm eff}
    =
    \frac{(1+\mu_{\rm mb})\,n}{\Lambda}
    \cdot\frac{1-\mu_{\rm mb}^{\,T}}{1+\mu_{\rm mb}^{\,T}} .
\end{equation}

Restrict now to \(0\le\mu_{\rm mb}<1\), where the batch coefficients
\(\bar\alpha_k=(\eta_{\rm mb}/B_{\rm mb})\mu_{\rm mb}^{\,k-1}\) are nonnegative, so that
the participation bounds \(1\le n_{\rm eff}\le n\) of \eqref{eq:neff-range-nonneg} apply and
\(n_{\rm eff}\) counts examples.
There
\(\mu_{\rm mb}^{\,T}\le e^{-(1-\mu_{\rm mb})T}=e^{-\Lambda}\) by \eqref{eq:minibatch-mu}.
So
\(\mu_{\rm mb}^{\,T}\to0\), and the correction factor in
\eqref{eq:minibatch-participation-exact} tends to one as soon as \(\Lambda\to\infty\):
\begin{equation}\label{eq:minibatch-participation-invariant}
    n_{\rm eff}=\frac{(1+\mu_{\rm mb})\,n}{\Lambda}\bigl(1+o(1)\bigr),
    \qquad
    r=1-o(1)
    \qquad (0\le\mu_{\rm mb}<1).
\end{equation}

Consider \eqref{eq:minibatch-participation-invariant} at fixed budget.
Holding \(\Lambda\) fixed
means \(\eta_{\rm mb}=B_{\rm mb}\Lambda/(\gamma n)\), that is \(B_{\rm mb}\) times the
single-example rate carrying the same budget.
At the tuned budget \(\Lambda=\tfrac12\log R\)
of \eqref{eq:eta-pass-def} this is exactly linear scaling,
\(\eta_{\rm mb}=B_{\rm mb}\,\eta_{\rm pass}\).
Fix that budget.
By
\eqref{eq:neff-frontier-form} and \eqref{eq:minibatch-participation-invariant} the two
nuisance-variance terms are \(\frac{\Lambda}{(1+\mu_{\rm mb})R}(1+o(1))\) and
\(\frac{\Lambda}{(1+\mu_{\rm mb})\rho^{2}}(1+o(1))\), while the spike term is at most
\(e^{-2\Lambda}=1/R\) and the training cross term is \(O(e^{-\Lambda}/\rho)\).
Since
\(1+\mu_{\rm mb}\in[1,2)\) on this branch, the nuisance terms dominate uniformly in
\(\mu_{\rm mb}\) as \(\Lambda\to\infty\) and \(\rho\to0\).
So linear scaling preserves the
\emph{order} of both errors at every batch size \emph{on this branch}, \(0\le\mu_{\rm mb}<1\).
It preserves their leading constants only
in the small-batch regime \(\mu_{\rm mb}=1-o(1)\), that is \(B_{\rm mb}=o(n/\Lambda)\), where
\eqref{eq:minibatch-participation-invariant} reduces to \(\frac{2n}{\Lambda}(1+o(1))\).

Comparing \eqref{eq:minibatch-participation-invariant} against that small-batch value isolates
the whole batch-size effect in one factor,
\begin{equation}\label{eq:minibatch-degradation}
    \frac{n_{\rm eff}}{2n/\Lambda}
    =
    \frac{1+\mu_{\rm mb}}{2}\,\bigl(1+o(1)\bigr)
    \qquad(0\le\mu_{\rm mb}<1),
\end{equation}
which is \(1-o(1)\) when \(\mu_{\rm mb}=1-o(1)\) and falls to \(\tfrac12\) at
\(\mu_{\rm mb}=0\).
At batch sizes
comparable to \(n/\Lambda\) the factor \eqref{eq:minibatch-degradation} therefore bites.
Participation falls, and by the participation form of the risks
\eqref{eq:neff-frontier-form} the nuisance contribution on each axis
rises by the reciprocal of \eqref{eq:minibatch-degradation}, a factor
\(2/(1+\mu_{\rm mb})\).

The exact participation formula \eqref{eq:minibatch-participation-exact} also locates where the
example-count reading stops.
On the nonnegative branch,
the participation bounds \(1\le n_{\rm eff}\le n\) of \eqref{eq:neff-range-nonneg} floor the
second factor of \eqref{eq:minibatch-participation} at one, so \(n_{\rm eff}\ge B_{\rm mb}\).
By \eqref{eq:minibatch-participation-exact} that floor
is reached exactly at \(\mu_{\rm mb}=0\) whenever the pass has \(T\ge2\) batches --- one
mini-batch on its own zeroing the spike residual --- that is at
\begin{equation}\label{eq:minibatch-saturation}
    B_{\rm mb}=\frac{n}{\Lambda}.
\end{equation}
Read in the continuous batch-size parametrization, \eqref{eq:minibatch-saturation} identifies the
saturation \emph{scale}.
It is attained exactly only when \(n/\Lambda\) is one of the admissible
batch sizes, since we assumed \(B_{\rm mb}\) divides \(n\).
At that scale participation has
already halved relative to the small-batch value, by
\eqref{eq:minibatch-degradation}.
Past \eqref{eq:minibatch-saturation} the batch coefficients
alternate in sign.
Then \eqref{eq:minibatch-participation-exact} goes on recording the
deterioration, and it is still exact, since the participation form of the risks
\eqref{eq:neff-frontier-form} does not care about signs.
But \(n_{\rm eff}\) no longer counts examples, and with it goes the intuition that
justifies linear scaling.
This saturation is a limit on the number of one-pass updates, not a
gradient-noise-scale limit.

Two things are worth saying about what this argument is and is not.

First, the argument reaches the per-example picture and not only the aggregate one.
By
the per-example prediction identity \(\vec u^{\top}\vec x_i=\gamma s+d\alpha_i\) of
\eqref{eq:train-prediction}, a batch member's own prediction moves by
\(\bar\alpha_k(d+B_{\rm mb}\gamma)\).
On the branch \(0\le\mu_{\rm mb}<1\),
\eqref{eq:minibatch-mu} gives \(B_{\rm mb}\Lambda/n\le1\), so by
\eqref{eq:scale-identities}
\[
    \frac{B_{\rm mb}\gamma}{d}
    =
    \frac{B_{\rm mb}\rho}{n}
    \le
    \frac{\rho}{\Lambda}
    =
    o(1) ,
\]
using \(\rho\to0\) and \(\Lambda\to\infty\).
So that movement is
\(\bar\alpha_k d\,(1+o(1))\).
What linear scaling at fixed budget holds fixed is the prefactor
\(\eta_{\rm mb}/B_{\rm mb}=\Lambda/(\gamma n)\), hence \(\bar\alpha_1\) exactly and each
later \(\bar\alpha_k\) up to its own \(\mu_{\rm mb}^{\,k-1}\).
It does \emph{not} hold the
whole profile fixed, and by \eqref{eq:minibatch-degradation} it does not hold participation fixed
either, except in the small-batch regime \(\mu_{\rm mb}=1-o(1)\).

Second, the reason the rule works here is not reduced gradient variance.
It is the need to hold
the per-example learning rate \(\eta_{\rm mb}/B_{\rm mb}\) fixed while a single pass builds up
the spike coefficient, which is what keeps the \emph{order} of both errors on the whole branch and
their leading constants in the small-batch regime.

\section{Rotation-aware coordinate robustness}
\label{app:rotation-aware}
\label{subsec:rotation-aware}

The RMS adversarial sensitivity calculation of Subsection~\ref{subsec:rms-sensitivity} is
coordinate-free.
For \(\ell_\infty\) perturbations, coordinate alignment matters.
In an
axis-aligned latent coordinate system, special block designs can make
\(\|\widetilde{\vec{v}}\|_1\) smaller or larger depending on how the nuisance residue is spread.
But after a generic random rotation, the coordinate story collapses back to the
Euclidean one up to constants.
The intuition is that a random rotation spreads
any fixed nuisance vector nearly evenly across coordinates, so the largest
prediction change from an allowed perturbation follows from its Euclidean norm.
This is the standard high-dimensional norm comparison in this setting: a
generic dense direction has
\(\ell_1\)-mass on the order of \(\sqrt d\) times its Euclidean norm, and
results on almost-Euclidean sections of convex bodies give a broader geometric
backdrop
\cite{kashin1977sections,figiel1977dimension,vershynin2018high}.

\begin{lemma}[Random rotation delocalizes coordinate mass]
\label{lem:random-rotation}
Let \(\bm{Q}\) be Haar-distributed, meaning uniformly distributed, on the
orthogonal group in dimension \(d+1\),
and let \(\vec{z}_0=\vec{z}_0(d)\ne\vec{0}\) be any deterministic vector.%
\footnote{In Lemma~\ref{lem:random-rotation} and its proof,
\(o_{\mathsf P}(1)\) and \(O_{\mathsf P}(1)\) denote convergence to zero
and boundedness in probability over the Haar rotation.}
As \(d\to\infty\),
\begin{equation}\label{eq:random-rotation-l1}
    \|\bm{Q}\vec{z}_0\|_1
    =
    \left(\sqrt{\frac2\pi}+o_{\mathsf P}(1)\right)
    \sqrt{d+1}\,\|\vec{z}_0\|_2.
\end{equation}
If \(\vec{z}_0\perp\vec{e}_1\), \(\vec{\theta}=\bm{Q}\vec{e}_1\), and
\(\vec{z}=\bm{Q}\vec{z}_0\), then, for every \(\delta\ge0\),
\begin{equation}\label{eq:random-rotation-sup}
    \sup_{\vec{\Delta}\perp\vec{\theta},\ \|\vec{\Delta}\|_\infty\le\delta}
    |\langle \vec{z},\vec{\Delta}\rangle|
    =
    \left(\sqrt{\frac2\pi}+o_{\mathsf P}(1)\right)
    \delta\sqrt{d+1}\,\|\vec{z}_0\|_2.
\end{equation}
\end{lemma}

\begin{proof}
The vector \(\bm{Q}\vec{z}_0/\|\vec{z}_0\|_2\) is uniform on the unit sphere, which has the same
law as \(\vec{g}/\|\vec{g}\|_2\) for \(\vec{g}\sim\calN(\vec{0},\bm{I}_{d+1})\).
The law of large numbers gives
the \(\ell_1\) formula \eqref{eq:random-rotation-l1}.
For this label-preserving perturbation set, the upper
bound is \(\delta\|\vec{z}\|_1\).
For the lower bound, use the coordinate sign
vector \(\vec{\omega}=\operatorname{sign}(\vec{z})\), with the convention
\(\operatorname{sign}(0)=0\).  
Conditional on \(\vec{z}\), Haar
invariance and \(\vec{z}_0\perp\vec{e}_1\) imply that \(\vec{\theta}\) is uniform on the unit
sphere in \(\vec{z}^{\perp}\).
Conditional on \(\vec{z}\), the covariance of a uniform
unit vector in \(\vec{z}^{\perp}\) is the orthogonal projector onto \(\vec{z}^{\perp}\),
divided by its dimension.
Since \(\|\vec{\omega}\|_2^2=d+1\), this gives
\[
    \operatorname{Var}(T\mid \vec{z})=O(1),
    \qquad
    T:=\langle\vec{\theta},\vec{\omega}\rangle=O_{\mathsf P}(1).
\]
Standard spherical concentration
\cite{vershynin2018high} also gives
\[
    \|\vec{\theta}\|_\infty=o_{\mathsf P}(1).
\]
The vector \(\vec{\omega}-T\vec{\theta}\) is orthogonal to \(\vec{\theta}\), has
\(\ell_\infty\) norm at most \(1+|T|\|\vec{\theta}\|_\infty=1+o_{\mathsf P}(1)\),
and has inner product \(\|\vec{z}\|_1\) with \(\vec{z}\) because \(\vec{z}\perp\vec{\theta}\).
After normalizing by this \(1+o_{\mathsf P}(1)\) factor and multiplying by
\(\delta\), it gives an allowed perturbation and the matching lower bound.
\end{proof}

Up to the harmless \(d\) versus \(d+1\) distinction,
a randomly rotated observed-coordinate \(\ell_\infty\) adversary sees the
same order of response as an RMS adversary.
After a generic rotation, define
the \emph{rotated-coordinate \(\ell_\infty\) adversarial sensitivity} by
\begin{equation}\label{eq:Ainfty-rot-def}
    A_\infty^{\rm rot}(\vec{u};\delta)
    :=
    \sup_{\vec{\Delta}\perp\vec{\theta},\ \|\vec{\Delta}\|_\infty\le\delta}
    |\langle \vec{u},\vec{\Delta}\rangle|.
\end{equation}
The spike component of \(\vec{u}\) drops out because every allowed perturbation
is orthogonal to \(\vec{\theta}\).
For a symmetric predictor with
small \(\Ttest\), applying Lemma~\ref{lem:random-rotation} to the nuisance
residue \(\widetilde{\vec{v}}\) gives
\begin{equation}\label{eq:Ainfty-rot-scale}
    A_\infty^{\rm rot}(\vec{u};\delta)
    =
    \left(\sqrt{\frac2\pi}+o_{\mathsf P}(1)\right)
    \delta\sqrt{d+1}\,\|\widetilde{\vec{v}}\|_2
\end{equation}
as \(d\to\infty\), where \(o_{\mathsf P}(1)\) is over the Haar rotation.
For a symmetric span predictor with nuisance residue
\(\|\widetilde{\vec{v}}\|_2=|r|\sqrt{d/(\gamma^2n)}\), this becomes
\begin{equation}\label{eq:Ainfty-rot-scale-symmetric}
    A_\infty^{\rm rot}(\vec{u};\delta)
    =
    \left(\sqrt{\frac2\pi}+o_{\mathsf P}(1)\right)
    \delta |r|\frac{d}{\gamma\sqrt n}.
\end{equation}
For fixed test-error targets, \(|r|=\Theta(1)\), giving the same size as
the RMS adversarial sensitivity.

The median construction in Appendix~\ref{app:improper} shows that nonlinear
aggregation can repair empirical fit by routing each training point around
the constituent predictor trained on it.
But it leaves the nuisance residue,
and hence the order of the RMS adversarial sensitivity, unchanged.
We do not
analyze the median's rotated-coordinate \(\ell_\infty\) sensitivity.

\section{Unknown orientation makes the training span complete for global linear predictors}
\label{app:latent-rotation-span-complete}

What does it mean to know nothing about the direction of the spike?
The
aligned coordinate convention in \eqref{eq:coordinate-convention} makes
the frontier calculation in Theorem~\ref{thm:pareto} transparent by fixing the spike axis in its
coordinate description.
For that same reason, it does not itself
represent uncertainty about the spike's ambient orientation.
Section~\ref{sec:one-sample} faced this question with a single labeled
sample.
There, the mirrored right triangles of
Figure~\ref{fig:rotspan-compatible-worlds} turn every off-span move into
a bet, and on average every bet loses.

This appendix formalizes the same ignorance for every sample size the
deterministic orthogonal model allows.
One unknown orthogonal
transformation sets the common orientation of the whole hidden frame, and
a mirror reflection (across the training span) pairs directions in the entire sphere of spike
directions that could have produced the observed training data.
The one-sample verdict survives
unchanged.
Averaging over that sphere or guarding against its worst
member leads to the same conclusion.
Projection onto the realized training span
never hurts, so the span train--test frontier of Theorem~\ref{thm:pareto}
is the Pareto frontier for all global linear predictors.

\subsection{Uniform averaging and worst-compatible risk formalize ignorance}

We formalize ignorance in two complementary ways.
The first uses symmetry.
Write
\[
    \mathrm O(d+1)
    :=
    \left\{
      \bm H\in\R^{(d+1)\times(d+1)}:
      \bm H^\top\bm H=\bm I_{d+1}
    \right\}
\]
for the full\footnote{We use the full group
because the handedness of the aligned coordinates is not side information
about the latent spike.
This distinction matters at the endpoint \(n=d\),
where a reflection exchanges the two compatible spike directions.} orthogonal group.
Its elements preserve Euclidean inner
products and include both rotations and reflections.

Ignorance%
\footnote{The use of invariance to formalize the absence of directional
knowledge is classical in statistical decision theory
\cite{kiefer1957invariance,eaton1989group,wijsman1990invariant}.
The
one-spike covariance spectrum with an isotropic bulk is standard in the
spiked-covariance literature
\cite{johnstone2001distribution,baik2005phase}.
The Haar draw used here
adds an explicit prior over its otherwise arbitrary orientation.}
about the spike direction means that all orientations are the same to us a priori.
To express this
mathematically using a prior, we require the prior over spike directions to be invariant under
every orthogonal change of coordinates.
On the sphere there is exactly one
such probability law.
It is the uniform law.

We implement this invariant prior by drawing the whole hidden frame from
normalized Haar measure on $\mathrm O(d+1)$.
This makes the spike direction
uniform.
Normalized Haar measure is the unique probability law invariant
under left multiplication.
On this compact group, it is also invariant
under right multiplication
\cite{eaton1989group,wijsman1990invariant}.
Requiring invariance is itself
the modeling decision.
Once that requirement is imposed, the uniform law
is forced, in the tradition of priors selected by invariance
\cite{jeffreys1946invariant,kass1996selection}.
Among probability laws on the sphere with densities relative to its
surface measure, the uniform law is also the maximum-entropy choice
\cite{jaynes1957information}.

The second approach does not average.
It retains every spike direction that
could have produced the observed data---the version space of the sample
\cite{mitchell1982generalization}---and evaluates a predictor in its worst
compatible world, as in optimal recovery
\cite{micchelli1977optimalrecovery}.
The first approach leads to the relevant test risk being a uniform
average over compatible worlds.
The second leads to worst-compatible risk.
We define both criteria in \eqref{eq:rotspan-criteria} after constructing the
compatible directions in
Proposition~\ref{prop:rotspan-compatible-directions}.
Under either
criterion, projection onto the realized training span never hurts.

\subsection{The unknown-orientation model}

Retain $1\le n\le d$ and the paper's standing assumption $\gamma>1$.
Let $\one\in\R^n$ denote the all-ones vector.
Let $\Xt_0=[\vec x_{0,1},\ldots,\vec x_{0,n}]$ denote the canonical training design matrix
$\Xt$ in \eqref{eq:training-data} with the spike in the first coordinate.
Let
$\bm{\Sigma}_0=\diag(\gamma,1,\ldots,1)$.
Let
$\vec e_1=(1,0,\ldots,0)^\top$.
For a world
$\bm H\in\mathrm O(d+1)$, define
\begin{equation}
\begin{aligned}
    \Xt_{\bm H}&:=\bm H\Xt_0,
    &\mathcal D_{\bm H}&:=(\Xt_{\bm H},\one),\\
    \vec\theta_{\bm H}&:=\bm H\vec e_1,
    &\bm{\Sigma}_{\bm H}
      &:=\bm H\bm{\Sigma}_0\bm H^\top
       =\bm I+(\gamma-1)
          \vec\theta_{\bm H}\vec\theta_{\bm H}^\top .
\end{aligned}
\label{eq:rotspan-family}
\end{equation}
In world $\bm H$, a fresh test input has law
$\calN(\vec 0,\bm{\Sigma}_{\bm H})$.
Its label is
$\vec\theta_{\bm H}^\top\vec x/\sqrt\gamma$.
Its spike oracle is
$\vec\theta_{\bm H}/\sqrt\gamma$.
The learner knows this family but only observes the specific training
data $\mathcal D_{\bm H}$.
The learner is not given $\bm H$ or
$\vec\theta_{\bm H}$.
Its final prediction is
$\vec u(\mathcal D_{\bm H})^\top\vec x$, with $\vec u$ chosen before the
fresh test input $\vec x$ is seen.
The learning map
$\mathcal D_{\bm H}\mapsto\vec u$ may otherwise be arbitrary.

\subsection{Compatible spike directions share the same projection onto the
training span and form a sphere around it}

Fix a realized ordered sample
$\mathcal D=\mathcal D_{\bm H}=(\bm X,\one)$ from
\eqref{eq:rotspan-family}.
Write
$\bm X=[\vec x_1,\ldots,\vec x_n]$.
Call a world
$\bm H'\in\mathrm O(d+1)$ compatible with the observed data $\mathcal D$ if
$\mathcal D_{\bm H'}=\mathcal D$ as ordered labeled samples.%
\footnote{The columns of $\Xt_{\bm H'}$ must match the observed design
matrix column by column.
Label equality is automatic because every label
vector in this family is $\one$.}
The spike direction
$\vec\theta_{\bm H'}$ of such a world is a compatible spike direction.%
\footnote{An equivalent definition works with the hidden nuisance legs
directly.
A unit vector $\vec\vartheta$ is a compatible spike direction
exactly when there exist $n$ mutually orthogonal nuisance vectors of
squared length $d$, each orthogonal to $\vec\vartheta$, with
$\sqrt\gamma\,\vec\vartheta$ plus the $i$th nuisance vector equal to the
$i$th training input.
Such vectors are forced to be
$\vec z_i=\vec x_i-\sqrt\gamma\,\vec\vartheta$, and the reverse inclusion
in the proof below verifies exactly this property in
\eqref{eq:rotspan-nuisance-orthogonal} and \eqref{eq:rotspan-nuisance-gram}.
Conversely, their
existence gives $\vec x_i^{\top}\vec\vartheta=\sqrt\gamma$ for every $i$.
For $n=1$ this is the right-triangle picture of
Section~\ref{sec:one-sample}.
The freedom in the unknown orientation, as
the data and the spike direction see it, is exactly the freedom to choose
these hidden legs.
The remaining freedom rotates coordinates that neither
the data nor the spike direction touch.}
It is convenient to define the spheres
\begin{align}
    \mathbb S^d&:=\{\vec v\in\R^{d+1}:\|\vec v\|_2=1\},
\label{eq:rotspan-unit-sphere}\\
    \mathbb S(V)&:=\{\vec v\in V:\|\vec v\|_2=1\}.
\label{eq:rotspan-subspace-sphere}
\end{align}
We write
\begin{equation}
    \operatorname{null}(\bm X^\top)
    :=
    \{\vec r\in\R^{d+1}:\bm X^\top\vec r=\vec0\}
\label{eq:rotspan-nullspace}
\end{equation}
for the nullspace of $\bm X^\top$.
Equivalently, the nullspace in
\eqref{eq:rotspan-nullspace} is $\operatorname{col}(\bm X)^\perp$.
We call the set of all compatible spike directions the compatible-direction
\emph{fiber}%
\footnote{A fiber is the set of hidden choices that lead to the same observed
value.
Here those choices are the orthogonal transformations $\bm H'$ that
produce $\mathcal D$.
The compatible-direction fiber records the spike
directions represented among them.}.

\begin{proposition}[Compatible spike directions share the same projection
onto the training span and form an affine sphere]
\label{prop:rotspan-compatible-directions}
For the realized ordered sample
$\mathcal D=(\bm X,\one)$ from \eqref{eq:rotspan-family} fixed above, the
design has Gram matrix
\begin{equation}
    \bm X^\top\bm X
    =\gamma\,\one\one^\top+d\,\bm I_n.
\label{eq:rotspan-compatible-gram}
\end{equation}
Consequently, $\bm X$ has full column rank.
Define
\begin{equation}
    \vec\theta_0
      :=\bm X(\bm X^\top\bm X)^{-1}\sqrt\gamma\,\one,
    \qquad
    \|\vec\theta_0\|_2^2
      =\frac{\gamma n}{d+\gamma n}<1.
\label{eq:rotspan-visible-spike}
\end{equation}
Then $\vec\theta_0$ is the unique vector in
$\operatorname{col}(\bm X)$ whose inner products with the training inputs
equal $\sqrt\gamma$.
Define the world-generated and constraint-defined
direction sets inside the unit sphere \eqref{eq:rotspan-unit-sphere} by
\begin{equation}
\begin{aligned}
    \Theta_{\mathcal D}^{\mathrm{world}}
    &:=
      \{\vec\theta_{\bm H'}:
          \mathcal D_{\bm H'}=\mathcal D\},\\
    \Theta_{\mathcal D}^{\mathrm{constraint}}
    &:=
      \{\vec\vartheta\in\mathbb S^d:
           \bm X^\top\vec\vartheta=\sqrt\gamma\,\one\}.
\end{aligned}
\label{eq:rotspan-compatible-sets}
\end{equation}
The two sets in \eqref{eq:rotspan-compatible-sets} are equal.
Write
$\Theta_{\mathcal D}$ for their common value.
Then, in the notation of
\eqref{eq:rotspan-subspace-sphere},
\begin{equation}
    \Theta_{\mathcal D}
    =
      \vec\theta_0+
      \sqrt{1-\|\vec\theta_0\|_2^2}\,
      \mathbb S\!\left(\operatorname{null}(\bm X^\top)\right).
\label{eq:rotspan-compatible-fiber}
\end{equation}
Thus $\vec\theta_0$ is the common projection of every compatible
spike direction onto $\operatorname{col}(\bm X)$, and
$\Theta_{\mathcal D}$ is an affine sphere of dimension $d-n$ centered at
$\vec\theta_0$.
At the endpoint $n=d$, this sphere consists of two points.
\end{proposition}

\begin{proof}
Because $\bm X$ is generated by \eqref{eq:rotspan-family}, orthogonality of
$\bm H$ and
\eqref{eq:orthogonal-vectors}--\eqref{eq:training-data} give
$\vec x_i^\top\vec x_j=\gamma+d\,\mathbf 1\{i=j\}$.
These inner products
prove \eqref{eq:rotspan-compatible-gram}.

We next show that $\bm X$ has full column rank.
For every nonzero
$\vec a\in\R^n$, the Gram identity \eqref{eq:rotspan-compatible-gram} gives
\[
    \vec a^\top\bm X^\top\bm X\vec a
      =d\|\vec a\|_2^2+\gamma(\one^\top\vec a)^2>0.
\]
Thus $\bm X^\top\bm X$ is positive definite, and $\bm X$ has full column
rank.
The definition of $\vec\theta_0$ in
\eqref{eq:rotspan-visible-spike} now gives
$\bm X^\top\vec\theta_0=\sqrt\gamma\,\one$.
It is the unique solution in
$\operatorname{col}(\bm X)$ because $\bm X^\top\bm X$ is invertible.
Since
$\bm X^\top\bm X\,\one=(d+\gamma n)\one$, and consequently
$\one^\top\bm X^\top\bm X\,\one=n(d+\gamma n)$, we also have
\[
    \vec\theta_0
      =\frac{\sqrt\gamma}{d+\gamma n}\,\bm X\one,
    \qquad
    \|\vec\theta_0\|_2^2
      =\frac{\gamma n}{d+\gamma n}.
\]
This proves the norm identity in \eqref{eq:rotspan-visible-spike}.

\emph{World-generated directions satisfy the observed constraint.}
We first prove
$\Theta_{\mathcal D}^{\mathrm{world}}
 \subseteq\Theta_{\mathcal D}^{\mathrm{constraint}}$.
Let $\vec\vartheta\in\Theta_{\mathcal D}^{\mathrm{world}}$.
By definition,
there is a $\bm H'\in\mathrm O(d+1)$ such that
$\mathcal D_{\bm H'}=\mathcal D$ and
$\vec\vartheta=\vec\theta_{\bm H'}$.
Equality of the ordered samples gives
$\Xt_{\bm H'}=\bm X$.
Therefore
\[
    \bm X^\top\vec\vartheta
    =
    (\bm H'\Xt_0)^\top\bm H'\vec e_1
    =
    \Xt_0^\top\vec e_1
    =
    \sqrt\gamma\,\one.
\]
Since $\vec\vartheta=\bm H'\vec e_1$ is a unit vector, this proves
$\vec\vartheta\in\Theta_{\mathcal D}^{\mathrm{constraint}}$.

\emph{Every direction satisfying the constraint comes from a compatible
world.}
We next prove
$\Theta_{\mathcal D}^{\mathrm{constraint}}
 \subseteq\Theta_{\mathcal D}^{\mathrm{world}}$.
Let
$\vec\vartheta\in\Theta_{\mathcal D}^{\mathrm{constraint}}$.
By definition,
$\|\vec\vartheta\|_2=1$ and
$\bm X^\top\vec\vartheta=\sqrt\gamma\,\one$.
Let
\[
    \vec z_i:=\vec x_i-\sqrt\gamma\,\vec\vartheta,
    \qquad
    \widetilde{\vec v}_i
      :=\begin{bmatrix}0\\ \vec v_i\end{bmatrix}\in\R^{d+1},
\]
where $\vec v_i$ is the canonical nuisance vector from
\eqref{eq:orthogonal-vectors}.
The definition of
$\Theta_{\mathcal D}^{\mathrm{constraint}}$ in
\eqref{eq:rotspan-compatible-sets} gives
\begin{equation}
    \vec\vartheta^\top\vec z_i
      =\sqrt\gamma-\sqrt\gamma=0.
\label{eq:rotspan-nuisance-orthogonal}
\end{equation}
The same definition and the Gram identity
\eqref{eq:rotspan-compatible-gram} give
\begin{equation}
\begin{aligned}
    \vec z_i^\top\vec z_j
      &=\vec x_i^\top\vec x_j
        -\sqrt\gamma\,\vec\vartheta^\top\vec x_i
        -\sqrt\gamma\,\vec\vartheta^\top\vec x_j
        +\gamma\|\vec\vartheta\|_2^2\\
      &=\vec x_i^\top\vec x_j-\gamma
       =d\,\mathbf 1\{i=j\}.
\end{aligned}
\label{eq:rotspan-nuisance-gram}
\end{equation}
Hence
\[
    \left(
      \vec e_1,
      \frac{\widetilde{\vec v}_1}{\sqrt d},
      \ldots,
      \frac{\widetilde{\vec v}_n}{\sqrt d}
    \right)
    \quad\text{and}\quad
    \left(
      \vec\vartheta,
      \frac{\vec z_1}{\sqrt d},
      \ldots,
      \frac{\vec z_n}{\sqrt d}
    \right)
\]
are orthonormal $(n+1)$-frames in $\R^{d+1}$.
When $n<d$, complete both
frames to orthonormal bases.%
\footnote{Neither the all-ones label vector nor
the diagonal nuisance Gram is essential to this geometric step.
Fix any
label vector $\vec y\in\R^n$ and any source nuisance matrix
$\bm V_0\in\R^{(d+1)\times n}$ satisfying
$\vec e_1^\top\bm V_0=\vec0^\top$ and
$\bm G:=\bm V_0^\top\bm V_0\succ0$, and put
$\bm X_0=\sqrt\gamma\,\vec e_1\vec y^\top+\bm V_0$.
A common rotation
produces $\bm X=\bm H\bm X_0$ and
$\bm X^\top\bm X=\gamma\vec y\vec y^\top+\bm G$.
A unit vector
$\vec\vartheta$ is compatible exactly when
$\bm X^\top\vec\vartheta=\sqrt\gamma\,\vec y$.
Indeed, for
$\bm Z:=\bm X-\sqrt\gamma\,\vec\vartheta\vec y^\top$, this constraint
gives $\vec\vartheta^\top\bm Z=\vec0^\top$ and
$\bm Z^\top\bm Z=\bm G$.
The source and target spike--nuisance tuples
therefore have the same joint Gram matrix and are related by an ambient
orthogonal map \cite[Section~7.3]{horn2013matrix}.
That map sends
$\vec e_1$ to $\vec\vartheta$ and $\bm X_0$ to $\bm X$, proving the converse
compatibility claim.
Consequently, the compatible spike directions again
form a positive-radius affine hypersphere.
The present choices
$\vec y=\one$ and $\bm G=d\bm I_n$ give the especially simple center, radius,
and risk formulas used in this paper.
The forthcoming second and third
preprints use essentially this extension, in forms adapted to their
models, to obtain the corresponding span-completeness results
\cite{ranade2026samedistribution,ranade2026multispike}.}
When $n=d$,
the displayed frames are already orthonormal
bases.
In either case, let $\bm H'\in\mathrm O(d+1)$ send the first basis
to the second.
Then $\bm H'\vec e_1=\vec\vartheta$, and
\[
    \bm H'\vec x_{0,i}
    =
    \sqrt\gamma\,\vec\vartheta+\vec z_i
    =
    \vec x_i.
\]
Consequently, $\bm H'\Xt_0=\bm X$ and
$\mathcal D_{\bm H'}=\mathcal D$.
Moreover,
$\vec\theta_{\bm H'}=\bm H'\vec e_1=\vec\vartheta$.
Hence
$\vec\vartheta\in\Theta_{\mathcal D}^{\mathrm{world}}$.
This proves
$\Theta_{\mathcal D}^{\mathrm{constraint}}
 \subseteq\Theta_{\mathcal D}^{\mathrm{world}}$.

Together, the two inclusions prove that the sets in
\eqref{eq:rotspan-compatible-sets} are equal.
Their common value
$\Theta_{\mathcal D}$ is therefore well defined.

It remains to prove the affine-sphere representation
\eqref{eq:rotspan-compatible-fiber}.
This representation
will also show that $\Theta_{\mathcal D}$ is nonempty.
Every
$\vec\vartheta\in\Theta_{\mathcal D}$ satisfies
\[
    \bm X^\top(\vec\vartheta-\vec\theta_0)=\vec0.
\]
Therefore $\vec r:=\vec\vartheta-\vec\theta_0$ belongs to
$\operatorname{null}(\bm X^\top)$, and
$\vec\vartheta=\vec\theta_0+\vec r$.
By
\eqref{eq:rotspan-nullspace}, these two components are orthogonal, so
$\vec\theta_0$ is the orthogonal projection of every compatible direction
onto $\operatorname{col}(\bm X)$.
The unit-norm condition in
\eqref{eq:rotspan-compatible-sets} gives
$\|\vec r\|_2=\sqrt{1-\|\vec\theta_0\|_2^2}$.
Conversely, every
$\vec r\in\operatorname{null}(\bm X^\top)$ with this norm makes
$\vec\theta_0+\vec r$ a unit vector and satisfies
$\bm X^\top(\vec\theta_0+\vec r)=\sqrt\gamma\,\one$.
By
\eqref{eq:rotspan-compatible-sets}, this vector belongs to
$\Theta_{\mathcal D}^{\mathrm{constraint}}$, and the equality proved above
places it in $\Theta_{\mathcal D}$.
This proves
\eqref{eq:rotspan-compatible-fiber}.
Since
$\bm X$ has $n$ linearly independent columns in $\R^{d+1}$,
$\dim\operatorname{null}(\bm X^\top)=d+1-n$.
Its unit sphere has dimension
$d-n$ and has two points when $n=d$.
\end{proof}

\subsection{Projection onto the training span never hurts either risk criterion}

For an ambient predictor $\vec u\in\R^{d+1}$, define its empirical loss
on the realized training data by
\begin{equation}
    L_{\mathcal D}(\vec u)
      :=\frac1n\|\bm X^\top\vec u-\one\|_2^2.
\label{eq:rotspan-ambient-training-loss}
\end{equation}

For a true spike direction $\vec\vartheta\in\Theta_{\mathcal D}$, the test
risk of a predictor $\vec u$ is
\begin{align}
    \calE_{\vec\vartheta}(\vec u)
      &:=
      \left(\vec u-\frac{\vec\vartheta}{\sqrt\gamma}\right)^\top
      \bm{\Sigma}_{\vec\vartheta}
      \left(\vec u-\frac{\vec\vartheta}{\sqrt\gamma}\right), \label{eq:rotspan-world-risk}\\
    \text{where}\qquad
    \bm{\Sigma}_{\vec\vartheta}
      &:=\bm I+(\gamma-1)
          \vec\vartheta\vec\vartheta^\top . \label{eq:rotspan-world-covariance}
\end{align}
The quadratic form $\calE_{\vec\vartheta}$ in \eqref{eq:rotspan-world-risk}
is the expected squared test error of the prediction $\vec u^\top\vec x$ in
any world whose spike direction is $\vec\vartheta$, so the test risk
depends on the world only through its spike direction.%
\footnote{For a true spike direction $\vec\vartheta$, a fresh test input
has law $\calN(\vec 0,\bm{\Sigma}_{\vec\vartheta})$ and label
$\vec\vartheta^\top\vec x/\sqrt\gamma$.
The expected squared test error
of the prediction $\vec u^\top\vec x$ is therefore
$\E[((\vec u-\vec\vartheta/\sqrt\gamma)^\top\vec x)^2]
=(\vec u-\vec\vartheta/\sqrt\gamma)^\top
\E[\vec x\vec x^\top]
(\vec u-\vec\vartheta/\sqrt\gamma)$, with
$\E[\vec x\vec x^\top]=\bm{\Sigma}_{\vec\vartheta}$.}
By Proposition~\ref{prop:rotspan-compatible-directions}, specifically
\eqref{eq:rotspan-compatible-fiber}, $\Theta_{\mathcal D}$ is an affine
sphere.
Let $\pi_{\mathcal D}$ be its normalized uniform surface measure,
with normalized counting measure when the sphere has two points.

When $\bm H$ is Haar distributed, observing $\mathcal D$ leaves the uniform
distribution over the compatible directions as the posterior.%
\footnote{Does the proof of span-completeness depend on this posterior
interpretation?
No.
The proof of Proposition~\ref{prop:rotspan-complete-class}
uses only the direct definition of $\mathcal B_{\mathcal D}$ through
$\pi_{\mathcal D}$ in \eqref{eq:rotspan-criteria}, so no conditional-law
calculation is required.}
Formally, the assignment
$\mathcal D\mapsto\pi_{\mathcal D}$ is an equivariant \emph{version}%
\footnote{A regular conditional distribution is determined only up to
changes on a probability-zero set of observed data.
A version is one choice
for those otherwise undetermined values.
It is equivariant when an
orthogonal transformation of the observed data induces the same
transformation of the conditional law.
The identification holds in
three steps.
Left invariance of the Haar prior makes an equivariant
version of the conditional law invariant under every orthogonal
transformation that fixes the observed data.
These transformations act
transitively on $\Theta_{\mathcal D}$.
Among probability laws on
$\Theta_{\mathcal D}$, only $\pi_{\mathcal D}$ is invariant under a
transitively acting group \cite{eaton1989group,wijsman1990invariant}.}
of the regular conditional distribution of $\vec\theta_{\bm H}$ given
$\mathcal D_{\bm H}$.
Define the uniform compatible-world average $\mathcal B_{\mathcal D}$ and
the worst-compatible risk $\mathcal M_{\mathcal D}$ by
\begin{equation}
    \mathcal B_{\mathcal D}(\vec u)
      :=\int_{\Theta_{\mathcal D}}
          \calE_{\vec\vartheta}(\vec u)\,
          \pi_{\mathcal D}(d\vec\vartheta),
    \qquad
    \mathcal M_{\mathcal D}(\vec u)
      :=\sup_{\vec\vartheta\in\Theta_{\mathcal D}}
          \calE_{\vec\vartheta}(\vec u).
\label{eq:rotspan-criteria}
\end{equation}
The criterion $\mathcal B_{\mathcal D}$ models expected risk under maximal
ignorance about which compatible world generated the data.
The criterion
$\mathcal M_{\mathcal D}$ protects against every compatible world.
Neither criterion asserts pointwise improvement in
one fixed world.

By Proposition~\ref{prop:rotspan-compatible-directions}, $\bm X$ has full
column rank, so the orthogonal projector onto the training span is
\begin{equation}
    \bm P_X:=\bm X(\bm X^\top\bm X)^{-1}\bm X^\top.
\label{eq:rotspan-projector}
\end{equation}
We can decompose an ambient vector $\vec u\in\R^{d+1}$ into its span part and its off-span
part:
\begin{equation}
    \vec u=\bar{\vec u}+\vec u_\perp,
    \qquad
    \bar{\vec u}:=\bm P_X\vec u,
    \qquad
    \vec u_\perp:=(\bm I-\bm P_X)\vec u.
\label{eq:rotspan-span-decomposition}
\end{equation}

\begin{proposition}[Projection removes every data-invisible linear bet]
\label{prop:rotspan-complete-class}
Let $\mathcal D=(\bm X,\one)$ be generated by
\eqref{eq:rotspan-family}.
For every $\vec u\in\R^{d+1}$, decomposed as
in \eqref{eq:rotspan-span-decomposition},
\begin{equation}
\begin{aligned}
    \mathcal B_{\mathcal D}(\vec u)
      &\ge
      \mathcal B_{\mathcal D}(\bar{\vec u})
      +\min\{1,\gamma\}\|\vec u_\perp\|_2^2,\\
    \mathcal M_{\mathcal D}(\vec u)
      &\ge
      \mathcal M_{\mathcal D}(\bar{\vec u})
      +\min\{1,\gamma\}\|\vec u_\perp\|_2^2,\\
    L_{\mathcal D}(\vec u)&=L_{\mathcal D}(\bar{\vec u}),\\
    \|\bar{\vec u}\|_2^2
      &=\|\vec u\|_2^2-\|\vec u_\perp\|_2^2.
\end{aligned}
\label{eq:rotspan-projection-dominance}
\end{equation}
Thus, under either criterion, every linear predictor outside the training
span is strictly dominated by its span projection.%
\footnote{Note that the proposition itself needs only $\gamma>0$.
The overshooting
frontier in Theorem~\ref{thm:pareto} retains its standing assumption
$\gamma>1$.}
\end{proposition}

\begin{proof}
Define the reflection across the training span by
\begin{equation}
    \bm R_X:=2\bm P_X-\bm I.
\label{eq:rotspan-reflection-def}
\end{equation}
The reflection equals the identity on $\operatorname{col}(\bm X)$ and
minus the identity on $\operatorname{null}(\bm X^\top)$ in
\eqref{eq:rotspan-nullspace}.
It therefore fixes $\vec\theta_0$ from \eqref{eq:rotspan-visible-spike}, and it
acts on the decomposition \eqref{eq:rotspan-span-decomposition} by
\begin{equation}
    \bm R_X\vec u=\bar{\vec u}-\vec u_\perp.
\label{eq:rotspan-reflected-coefficient}
\end{equation}
It is symmetric, and
\begin{equation}
  \bm R_X \bm R_X^\top = \bm R_X^{2}=4\bm P_X^{2}-4\bm P_X+\bm I=\bm I,
\label{eq:rotspan-reflection-involution}
\end{equation}
so $\bm R_X$ is orthogonal.
Since $\bm R_X\bm X=\bm X$, taking
transposes gives $\bm X^\top\bm R_X=\bm X^\top$, so every
$\vec\vartheta\in\Theta_{\mathcal D}$ satisfies
\begin{equation}
    \bm X^\top\bm R_X\vec\vartheta=\bm X^\top\vec\vartheta,
    \qquad
    \|\bm R_X\vec\vartheta\|_2=\|\vec\vartheta\|_2.
\label{eq:rotspan-fiber-preserved}
\end{equation}
By \eqref{eq:rotspan-fiber-preserved}, the reflection maps
$\Theta_{\mathcal D}$ into itself, and by the involution
\eqref{eq:rotspan-reflection-involution} it maps $\Theta_{\mathcal D}$
onto itself, pairing each compatible direction $\vec\vartheta$ with
$\bm R_X\vec\vartheta$.
It preserves the uniform law $\pi_{\mathcal D}$.
Computing the world covariance in \eqref{eq:rotspan-world-covariance} gives
\begin{equation}
    \bm{\Sigma}_{\bm R_X\vec\vartheta}
    =\bm I+(\gamma-1)\,\bm R_X\vec\vartheta\,(\bm R_X\vec\vartheta)^{\top}
    =\bm R_X\bm{\Sigma}_{\vec\vartheta}\bm R_X.
\label{eq:rotspan-covariance-conjugation}
\end{equation}

The first fact to establish is that, inside the test risk, reflecting the
spike direction is the same as reflecting the predictor.
Working directly with the quadratic
form \eqref{eq:rotspan-world-risk},
\begin{align}
    \calE_{\bm R_X\vec\vartheta}(\vec u)
    &=\Bigl(\vec u-\tfrac{\bm R_X\vec\vartheta}{\sqrt\gamma}\Bigr)^{\top}
      \bm R_X\bm{\Sigma}_{\vec\vartheta}\bm R_X
      \Bigl(\vec u-\tfrac{\bm R_X\vec\vartheta}{\sqrt\gamma}\Bigr)
    \label{eq:rotspan-swap-conjugate}\\
    &=\Bigl(\bm R_X\vec u-\tfrac{\vec\vartheta}{\sqrt\gamma}\Bigr)^{\top}
      \bm{\Sigma}_{\vec\vartheta}
      \Bigl(\bm R_X\vec u-\tfrac{\vec\vartheta}{\sqrt\gamma}\Bigr)
    \label{eq:rotspan-swap-move}\\
    &=\calE_{\vec\vartheta}(\bar{\vec u}-\vec u_\perp).
    \label{eq:rotspan-world-swap}
\end{align}
Equality \eqref{eq:rotspan-swap-conjugate} substitutes the conjugation
\eqref{eq:rotspan-covariance-conjugation} into
\eqref{eq:rotspan-world-risk}.
Equality \eqref{eq:rotspan-swap-move}
moves one factor of $\bm R_X$ onto each side of the quadratic form, using
symmetry of $\bm R_X$ and the involution
\eqref{eq:rotspan-reflection-involution}.
Equality
\eqref{eq:rotspan-world-swap} applies
\eqref{eq:rotspan-reflected-coefficient}.

Fix $\vec\vartheta\in\Theta_{\mathcal D}$ and write
$\vec\delta:=\bar{\vec u}-\vec\vartheta/\sqrt\gamma$ for the error vector of
the span part in a world with spike direction $\vec\vartheta$, so that
$\calE_{\vec\vartheta}(\bar{\vec u})
 =\vec\delta^{\top}\bm{\Sigma}_{\vec\vartheta}\vec\delta$.
The quantity to bound is the sum of the risks of $\vec u$ across the
reflected pair of spike directions $\{\vec\vartheta,\bm R_X\vec\vartheta\}$:
\begin{align}
    \calE_{\vec\vartheta}(\vec u)
    +\calE_{\bm R_X\vec\vartheta}(\vec u)
    &=\calE_{\vec\vartheta}(\bar{\vec u}+\vec u_\perp)
     +\calE_{\vec\vartheta}(\bar{\vec u}-\vec u_\perp)
    \label{eq:rotspan-pair-sum}\\
    &=(\vec\delta+\vec u_\perp)^{\top}\bm{\Sigma}_{\vec\vartheta}
      (\vec\delta+\vec u_\perp)
     +(\vec\delta-\vec u_\perp)^{\top}\bm{\Sigma}_{\vec\vartheta}
      (\vec\delta-\vec u_\perp)
    \label{eq:rotspan-pair-expand}\\
    &=2\,\vec\delta^{\top}\bm{\Sigma}_{\vec\vartheta}\vec\delta
     +\underbrace{2\,\vec u_\perp^{\top}\bm{\Sigma}_{\vec\vartheta}\vec\delta
     -2\,\vec u_\perp^{\top}\bm{\Sigma}_{\vec\vartheta}\vec\delta}_{
       \text{the cross terms cancel}}
     +2\,\vec u_\perp^{\top}\bm{\Sigma}_{\vec\vartheta}\vec u_\perp
    \label{eq:rotspan-pair-cancel}\\
    &=2\,\calE_{\vec\vartheta}(\bar{\vec u})
     +2\,\vec u_\perp^{\top}\bm{\Sigma}_{\vec\vartheta}\vec u_\perp .
    \label{eq:rotspan-pair-value}
\end{align}
The first equality \eqref{eq:rotspan-pair-sum} applies the swap
\eqref{eq:rotspan-world-swap} together with
$\vec u=\bar{\vec u}+\vec u_\perp$.
The second equality
\eqref{eq:rotspan-pair-expand} writes out \eqref{eq:rotspan-world-risk}
using $\bar{\vec u}\pm\vec u_\perp-\vec\vartheta/\sqrt\gamma
=\vec\delta\pm\vec u_\perp$.
The third equality
\eqref{eq:rotspan-pair-cancel} expands both quadratics, using the
symmetry of $\bm{\Sigma}_{\vec\vartheta}$.
The last
equality \eqref{eq:rotspan-pair-value} recognizes
$\vec\delta^{\top}\bm{\Sigma}_{\vec\vartheta}\vec\delta
 =\calE_{\vec\vartheta}(\bar{\vec u})$.

The eigenvalues of $\bm{\Sigma}_{\vec\vartheta}$ are $\gamma$ along
$\vec\vartheta$ and $1$ on its orthogonal complement, so
$\vec u_\perp^{\top}\bm{\Sigma}_{\vec\vartheta}\vec u_\perp
\ge\min\{1,\gamma\}\,\|\vec u_\perp\|_2^{2}$.
Dividing
\eqref{eq:rotspan-pair-value} by two,
\begin{align}
    \frac{\calE_{\vec\vartheta}(\vec u)
          +\calE_{\bm R_X\vec\vartheta}(\vec u)}{2}
    &=\calE_{\vec\vartheta}(\bar{\vec u})
     +\vec u_\perp^{\top}\bm{\Sigma}_{\vec\vartheta}\vec u_\perp
    \label{eq:rotspan-pair-average}\\
    &\ge\calE_{\vec\vartheta}(\bar{\vec u})
     +\min\{1,\gamma\}\,\|\vec u_\perp\|_2^{2}.
    \label{eq:rotspan-reflected-midpoint}
\end{align}

Since $\mathcal M_{\mathcal D}(\vec u)$ is at least both risks in the
reflected pair, it is at least their average
\eqref{eq:rotspan-pair-average}, and therefore at least the lower bound
\eqref{eq:rotspan-reflected-midpoint}.
Taking the supremum over
$\vec\vartheta\in\Theta_{\mathcal D}$ proves the inequality \eqref{eq:rotspan-projection-dominance} for
$\mathcal M_{\mathcal D}$.

Integrating the equality \eqref{eq:rotspan-pair-average} over
$\pi_{\mathcal D}$ and using
the reflection invariance of $\pi_{\mathcal D}$ gives
\begin{equation}
    \mathcal B_{\mathcal D}(\vec u)
    =\mathcal B_{\mathcal D}(\bar{\vec u})
     +\int_{\Theta_{\mathcal D}}
        \vec u_\perp^{\top}\bm{\Sigma}_{\vec\vartheta}\vec u_\perp\,
        \pi_{\mathcal D}(d\vec\vartheta),
\label{eq:rotspan-average-decomposition}
\end{equation}
which proves the inequality \eqref{eq:rotspan-projection-dominance} for $\mathcal B_{\mathcal D}$.

Finally,
$\bm X^{\top}\bm P_X=\bm X^{\top}$ proves the training-loss equality in
\eqref{eq:rotspan-projection-dominance}, and orthogonal Pythagoras proves
the norm identity.
\end{proof}

Figure~\ref{fig:rotspan-compatible-worlds} shows the one-sample geometry
behind Propositions~\ref{prop:rotspan-compatible-directions}
and~\ref{prop:rotspan-complete-class}.
Panel~A uses two exact
$9$--$40$--$41$ decompositions that produce the same observed vector.
Panel~B draws a circle obtained by restricting the affine sphere
\eqref{eq:rotspan-compatible-fiber} from
Proposition~\ref{prop:rotspan-compatible-directions} to a two-dimensional
hidden subspace.
Panel~C plots the two reflected-world risks and their
average \eqref{eq:rotspan-pair-average} as a signed off-span coordinate
$\zeta$ varies around two different span predictors in subpanels (i) and (ii).

\subsection{Projection recovers the canonical span frontier}

For a span predictor
$\bar{\vec u}=\bm X\vec\alpha$, use
$s=\sum_i\alpha_i$ and $q=\sum_i\alpha_i^2$ from
\eqref{eq:sq-defs}.
For every compatible spike direction $\vec\vartheta$,
Proposition~\ref{prop:rotspan-compatible-directions} identifies
$\Theta_{\mathcal D}$ with the constraint-defined set in
\eqref{eq:rotspan-compatible-sets}, so
$\bm X^\top\vec\vartheta=\sqrt\gamma\,\one$.
The Gram identity
\eqref{eq:rotspan-compatible-gram} from the same proposition gives
$\bm X^\top\bm X=\gamma\,\one\one^\top+d\,\bm I_n$.
Therefore
\[
    \bar{\vec u}^{\top}\vec\vartheta=\sqrt\gamma\,s,
    \qquad
    \|\bar{\vec u}\|_2^2=\gamma s^2+dq.
\]
The empirical loss is also exactly the canonical training error:
\begin{equation}
\begin{aligned}
    L_{\mathcal D}(\bar{\vec u})
    &=
    \frac1n\|\bm X^\top\bm X\vec\alpha-\one\|_2^2\\
    &=
    \frac1n\|\Xt_0^\top\Xt_0\vec\alpha-\one\|_2^2
    =\Ttrain(\Xt_0\vec\alpha).
\end{aligned}
\label{eq:rotspan-canonical-training-bridge}
\end{equation}
Expanding the compatible-world risk gives
\begin{align}
    \calE_{\vec\vartheta}(\bar{\vec u})
    &=
      \|\bar{\vec u}\|_2^2
      +(\gamma-1)
         (\bar{\vec u}^{\top}\vec\vartheta)^2
      -2\sqrt\gamma\,
         \bar{\vec u}^{\top}\vec\vartheta+1
      \nonumber\\
    &=(\gamma s-1)^2+dq
     =\Ttest(\Xt_0\vec\alpha).
\label{eq:rotspan-canonical-risk-bridge}
\end{align}
Thus a span predictor has the same test risk in every compatible world,
and its training and test errors are exactly the canonical quantities already
used in the frontier calculation.

\begin{corollary}[The span frontier is the ambient global-linear frontier]
\label{cor:rotspan-frontier-transfer}
Let $\mathcal D=(\bm X,\one)$ be generated by
\eqref{eq:rotspan-family}, and let $\mathcal J_{\mathcal D}$ denote either
$\mathcal B_{\mathcal D}$ or $\mathcal M_{\mathcal D}$.
For every $t\in\R$ and $C\ge0$,
\begin{equation}
\inf_{\substack{\vec u\in\R^{d+1}\\
                  L_{\mathcal D}(\vec u)\le t,\ \|\vec u\|_2^2\le C}}
       \mathcal J_{\mathcal D}(\vec u)
\qquad = \qquad
 \inf_{\substack{\vec u\in\operatorname{col}(\bm X)\\
                  L_{\mathcal D}(\vec u)\le t,\ \|\vec u\|_2^2\le C}}
       \mathcal J_{\mathcal D}(\vec u),
\label{eq:rotspan-constrained-frontier}
\end{equation}
with an empty infimum interpreted as $+\infty$.
The Pareto-minimal
objective values of
\[
    \bigl(
      \mathcal J_{\mathcal D}(\vec u),
      L_{\mathcal D}(\vec u),
      \|\vec u\|_2^2
    \bigr),
\]
with all three coordinates minimized, are likewise unchanged when the
action class is restricted to the training span.
The same is true after
omitting the norm coordinate.
By
\eqref{eq:rotspan-canonical-training-bridge} and
\eqref{eq:rotspan-canonical-risk-bridge}, the resulting train--test frontier
is precisely the curve in Theorem~\ref{thm:pareto}.
\end{corollary}

\begin{proof}
Projection preserves empirical loss, does not increase either test-risk
criterion, and contracts the norm by
Proposition~\ref{prop:rotspan-complete-class}.
It therefore sends every
feasible ambient predictor to a no-worse feasible span predictor.
Set
inclusion gives the reverse infimum inequality.

The same projection weakly improves $\mathcal J_{\mathcal D}(\vec u)$,
$L_{\mathcal D}(\vec u)$, and $\|\vec u\|_2^2$.
Hence
every ambient Pareto-minimal tuple has a span representative.
Conversely,
if an ambient predictor dominated a tuple that was Pareto-minimal in the
span, projecting that predictor would give the same domination inside the
span, a contradiction.
The Pareto-minimal objective values therefore agree.
Equations~\eqref{eq:rotspan-canonical-training-bridge} and
\eqref{eq:rotspan-canonical-risk-bridge} identify their train--test
specialization with the frontier already computed in
Theorem~\ref{thm:pareto}.
\end{proof}

The conclusion covers any method that outputs a single vector of coefficients
defining a linear predictor before the fresh test input is observed.
Training may be nonlinear,
adaptive, preconditioned, sparsity penalized, initialized away from zero, or
randomized.
For a randomized method, projection applies seed by seed.
Of course, the conclusion need not hold if the learner receives directional side
information correlated with the latent orientation.
Nonlinear prediction rules and rules that choose their coefficient vector using the input \(\vec{x}\) at prediction time remain outside this global-linear conclusion, and as Appendix~\ref{app:improper} shows, nonlinear rules {\em can} do substantially better.

\section{Improper median escape}
\label{app:improper}

\paragraph{Why nonlinear aggregation can help.}
The training-error lower bound \eqref{eq:fourth-quadrant-finite} in the main text is first stated for linear predictors in the training span.
Under the uniform-average and
worst-compatible criteria of
Appendix~\ref{app:latent-rotation-span-complete}, it extends to every procedure
that outputs one global linear predictor.
Once
nonlinear aggregation is allowed, the training-error obstruction can
disappear even though the constituent predictors carry the same nuisance
residue.
Split the data into piles, build one calibrated span predictor per
pile, and predict by their median.
The median shields each training point
from the one constituent predictor that was trained on it.

In this construction, the constituent predictors are linear functions, but
the resulting rule is improper relative to the global-linear class.
It does not output one fixed vector \(\vec{u}\) whose inner product \(\vec{u}^{\top}\vec{x}\) gives the prediction for every input \(\vec{x}\).
This places it near proper/improper separations
\cite{hanneke2016optimal,bousquet2020proper,montasser2019vc}, the
split-and-aggregate spirit of bagging and
median-of-means \cite{breiman1996bagging,nemirovsky1983problem,
jerrum1986random,alon1999space}, and robust-estimation variants based
on geometric medians and tournaments
\cite{audibert2011robust,minsker2015geometric,devroye2016subgaussian,
lugosi2019subgaussian,lugosi2019mean,lugosi2019tournaments}.
It is also
formally close to recent majority-of-three and bagging results for PAC learning
\cite{larsen2023bagging,adenali2024majority}.

It is worth being precise about why leaving the linear class helps here.
Here the data and the constituents are both clean.
The mechanism is instead structural.
Any linear predictor built from the
training points lies in their span, and within the span exact spike
calibration forces the nuisance residue onto every training point the
predictor was built from --- in effect, the training points become outliers
relative to the general test distribution as far as this linear predictor is
concerned.

Splitting the data and taking a median changes which predictor judges each
training point.
Each point is evaluated by the two constituents that did not
see it.
For those constituents, the point's nuisance direction is fresh and
orthogonal, so they fit it without overshoot, and the median follows them.

The improper rule does not suppress the constituent predictors' nuisance residue.
Its RMS adversarial
sensitivity is unchanged in order, as shown below.
Instead, it rearranges which
predictor is responsible for each point.

\paragraph{Split-and-median construction.}
For simplicity assume \(n\) is divisible by \(3\), and split the training
points into three deterministic piles \(\mathcal P_1,\mathcal P_2,\mathcal P_3\), each of size
\(m_k=n/3\).
Unequal piles with \(m_k\asymp n\) change only constants.
On
pile \(\mathcal P_k\), define the
calibrated flat predictor (the same spirit would work with SGD-learned predictors)
\begin{equation}\label{eq:improper-pile-predictors}
    f_k(\vec{x})=\vec{u}_k^\top \vec{x},
    \qquad
    \vec{u}_k=\frac{1}{\gamma m_k}\sum_{i\in\mathcal P_k}\vec{x}_i.
\end{equation}
For a general measurable predictor \(f\), write
\begin{equation}\label{eq:general-risk-defs}
    \Ttrain(f)=\frac1n\sum_{i=1}^n(f(\vec{x}_i)-1)^2,
    \qquad
    \Ttest(f)=\E[(f(\vec{x}_{\rm test})-y_{\rm test})^2].
\end{equation}
Each \(\vec{u}_k\) is a linear predictor in the span of its pile.
It has
exact spike calibration \(u_k[1]=1/\sqrt{\gamma}\) and test error
\[
    \Ttest(f_k)=\frac{d}{\gamma^2m_k}.
\]
Define the improper aggregate
\begin{equation}\label{eq:median-predictor-def}
    f_{\rm med}(\vec{x})=\median(f_1(\vec{x}),f_2(\vec{x}),f_3(\vec{x})).
\end{equation}

\paragraph{Why the median fits the training points.}
\begin{lemma}[Median shielding]
\label{lem:median-shield}
For any \(a,b,c,t\in\R\),
\begin{equation}\label{eq:median-shield}
    |\median(a,b,c)-t|\le \max\{|b-t|,\ |c-t|\}.
\end{equation}
\end{lemma}

\begin{proof}
Among any three real numbers, the median lies between any chosen two of them.
In particular, it lies in the closed interval with endpoints \(b\) and \(c\).
Every point in that interval is within
\(\max\{|b-t|,|c-t|\}\) of \(t\).
\end{proof}

\begin{proposition}[Improper median escape in the stylized model]
\label{prop:improper-stylized}
Assume \(n\le d\) and, for simplicity, \(3\mid n\).
The median-of-three predictor satisfies
\begin{equation}\label{eq:improper-train}
    \Ttrain(f_{\rm med})=0
\end{equation}
on the deterministic training set, while
\begin{equation}\label{eq:improper-test}
    \Ttest(f_{\rm med})
    \le
    \frac{6d}{\gamma^2 n}.
\end{equation}
Thus nonlinear median aggregation escapes the misfitting
obstruction.
\end{proposition}

\begin{proof}
If \(\vec{x}_i\in\mathcal P_1\), then \(f_2(\vec{x}_i)=f_3(\vec{x}_i)=1\), because the spike
contribution is calibrated and the nuisance directions from other piles are
orthogonal to \(\vec{v}_i\).
The value of \(f_1(\vec{x}_i)\) may be badly overshot, but
the median of \((f_1(\vec{x}_i),1,1)\) is \(1\).
The same argument applies to all
three piles, proving zero training error.

For a fresh test point, write \(f_k(\vec{x}_{\rm test})-y_{\rm test}=e_k\).
The
errors \(e_k\) are centered Gaussian nuisance projections with variances
\(d/(\gamma^2m_k)\), and the three piles use orthogonal nuisance subspaces.
Since adding the common value \(y_{\rm test}\) to all three arguments shifts
their median by \(y_{\rm test}\),
\[
    f_{\rm med}(\vec{x}_{\rm test})-y_{\rm test}
    =
    \median(e_1,e_2,e_3).
\]
By Lemma~\ref{lem:median-shield},
\[
    |\median(e_1,e_2,e_3)|\le \max\{|e_2|,|e_3|\},
\]
and hence
\[
    (\median(e_1,e_2,e_3))^2\le e_2^2+e_3^2
\]
after applying Lemma~\ref{lem:median-shield} with \(a=e_1\), \(b=e_2\), \(c=e_3\), and \(t=0\).
By symmetry the choice of the omitted pile is irrelevant.
Taking expectations,
\[
    \E\,(\median(e_1,e_2,e_3))^2
    \le
    \E e_2^2+\E e_3^2
    =
    2\,\frac{d}{\gamma^2(n/3)}
    =
    \frac{6d}{\gamma^2n}.
\]
\end{proof}

The construction is functionally improper because the median of three linear functions is
not generally linear.
It is a continuous piecewise-linear function, hence
representable by a small ReLU network.
Indeed,
\(\median(a,b,c)=a+b+c-\min(a,b,c)-\max(a,b,c)\), while maxima, minima, and
absolute values can be written using ReLU gates.
Thus the aggregate is a
small piecewise-linear nonlinear predictor built from three span predictors.
It is outside the class of span-restricted linear predictors, and this
is why it can have both good training and good test performance in a regime
where every linear span predictor with small test error misfits.

It is important to note that this is an expressivity escape, not an optimization theorem.
A small
ReLU network {\em could} represent the split-and-median rule, but that does not mean
that gradient descent from a standard initialization {\em will} find such a routing.
The span frontier \eqref{eq:frontier} remains relevant for algorithms whose implicit bias stays close to span-supported, minimum-norm, or otherwise simple linear combinations of the training data.
Section~\ref{subsec:related-large-steps} discusses related implicit-bias views
of SGD and GD.
This median construction should be read as a capacity
separation, not as a claim about what standard training will find.
Indeed, the answer to the latter question is presumably of great practical interest.

\paragraph{The escape does not remove adversarial sensitivity.}
Each constituent
predictor has nuisance component
\[
    \vec{a}_k=\frac{1}{\gamma m_k}\sum_{i\in\mathcal P_k}\vec{v}_i,
    \qquad
    \|\vec{a}_k\|_2^2=\frac{d}{\gamma^2m_k}.
\]
The three \(\vec{a}_k\)'s are mutually orthogonal and have equal norm under the
equal-pile convention.
The vector \(\sum_j \vec{a}_j\) lies in the nuisance
subspace, hence is orthogonal to the spike direction.
The perturbation below
is therefore label-preserving.
Aligning it with \(\sum_j \vec{a}_j\) gives, for
each \(k\),
\begin{equation}\label{eq:improper-median-adversarial-response}
    \left\langle
    \vec{a}_k,\,
    \delta\sqrt d\,\frac{\sum_j \vec{a}_j}{\|\sum_j \vec{a}_j\|_2}
    \right\rangle
    =
    \delta\sqrt d\,\frac{\|\vec{a}_k\|_2^2}
    {\left(\sum_j\|\vec{a}_j\|_2^2\right)^{1/2}}
    =
    \Theta\left(
    \delta\frac{d}{\gamma\sqrt n}
    \right),
\end{equation}
when \(m_k=n/3\).
Thus all three constituent predictions shift in the same
direction by equal amounts of the order in
\eqref{eq:improper-median-adversarial-response}.
The median therefore inherits the same order of adversarial sensitivity.%
\footnote{For a nonlinear predictor \(f\), adversarial sensitivity here means
the worst-case output change \(|f(\vec x+\vec{\Delta})-f(\vec x)|\) over base
points \(\vec x\) and perturbations \(\vec{\Delta}\perp\vec{\theta}\) with
\(\|\vec{\Delta}\|_2\le\delta\sqrt d\).
For a span predictor the response \(\langle\vec u,\vec{\Delta}\rangle\) does
not depend on the base point \(\vec x\), and this recovers \eqref{eq:A2-def}.
Both bounds in this paragraph hold at every base point.}
In the reverse direction, the median of three numbers moves by at most the
largest of the three moves, and each constituent's response to any
label-preserving perturbation with \(\|\vec\Delta\|_2\le\delta\sqrt d\) is at
most \(\delta\sqrt d\,\|\vec a_k\|_2=\sqrt3\,\delta d/(\gamma\sqrt n)\), so
this order is also an upper bound.
Improper aggregation shields each training point from the predictor trained on
it, but it does not erase the high-dimensional nuisance residue that an
adversary can align with.

The escape in this appendix is functional: the median leaves the class of
globally linear predictors.
Appendix~\ref{app:sco-comparison} discusses a different bounded-comparator-ball
perspective: in the benign-misfitting window, the sample-efficient linear
span predictor leaves the natural Euclidean norm ball containing the spike
oracle.

\section[Benign underfitting and bounded comparator balls]{Benign underfitting and comparator balls in stochastic optimization}
\label{app:sco-comparison}

This appendix has two aims.
First, it distinguishes the present
fourth-quadrant geometry from benign-underfitting results in stochastic
convex optimization (SCO): the SCO examples are path-dependent,
while the obstruction here is a train--test property of the linear training
span.
Second, it gives a bounded-comparator-ball observation that is
natural from the SCO perspective, where one typically fixes a norm ball and
asks for excess risk against the best point in it: within the span,
sample-efficient spike calibration must leave the Euclidean norm ball
containing the spike oracle.

\paragraph{Path dependence versus span geometry.}
Benign-underfitting examples are path-dependent, whereas the span obstruction is
geometric.
Because the distinction is partly mechanistic, we spell out the relevant
details of the SCO constructions in footnotes.
One close antecedent to benign misfitting is the benign-underfitting line in
stochastic convex optimization.
Koren, Livni, Mansour, and Sherman
\cite{koren2022benign} construct\footnote{%
Koren, Livni, Mansour, and Sherman's Theorem~1 establishes, via a careful
construction, a stochastic convex optimization problem with a $4$-Lipschitz
convex loss over a high-dimensional ($d\ge 2^{4n\log n}$) Euclidean ball.
For one-pass projected SGD from zero, their
bound implies that with the standard step size (learning rate)
$\eta=\Theta(1/\sqrt{n})$, the
expected population excess risk is $O(1/\sqrt{n})$, while the expected empirical
risk and generalization gap are both $\Omega(1)$.
Their proof uses a
constructed nonsmooth loss of the form
$f(\vec{w};\,\vec{z})=\|\vec{z}\odot \vec{w}\|_2+\psi(\vec{w};\,\vec{z})$.
The auxiliary term $\psi$ encodes which coordinates have appeared in the prefix of
the data stream and makes a valid subgradient point in a coordinate that is
likely to have zeros in the future suffix.
A step in that coordinate is
therefore not corrected by later examples, while earlier examples penalize it
when the empirical risk is evaluated at the end.
The phenomenon is thus tied to
a carefully constructed nonsmooth loss, a particular without-replacement SGD
path, and an adversarial but valid subgradient selection.
Later work reduces this dimensional requirement in related constructions;
see the surrounding paragraph.}
SCO problems in which one-pass
without-replacement SGD obtains the usual population excess-risk rate while
its empirical risk remains large.
This line of work sits alongside the classical
stability and uniform-convergence foundations of stochastic convex learning
\cite{bousquet2002stability,shalevshwartz2010sco} and the dimension-dependent limitations of ERM
\cite{feldman2016sco}.
Subsequent work sharpens and broadens the gradient and ERM picture:
\citet{schliserman2025dimension} give polynomial-dimensional ($d=O(n^2)$)
gradient-method constructions, \citet{vansover2025rapid} show rapid overfitting
under multi-pass SGD and obtain a nearly linear-dimensional ($d=\widetilde O(n)$)
one-pass generalization-gap lower bound, and other work develops
sample-complexity bounds for gradient descent and ERM and
adaptive-data-analysis viewpoints
\cite{carmon2024erm,livni2024false,burla2026erms}.
In particular,
\citet{livni2024gd} shows that full-batch GD in nonsmooth SCO can have
generalization error matching worst-case ERM behavior, so that GD inherits the
dimension-dependent penalty.

The SCO literature also gives a useful precedent for the warning that ERM can pay
for high-dimensional nuisance directions: in Feldman's lower bound\footnote{%
Feldman's Theorem~3.3 gives a clean version of the ERM dimension penalty.
For the Euclidean domain $\mathcal B_2^d$, he constructs an exponentially
large codebook $\mathcal W\subseteq\{-1,1\}^d$ of nearly orthogonal sign
vectors and functions
$g_{\mathcal V}(\vec{x})=\max\{1/2,\,\max_{\vec{w}\in\mathcal V}
\langle \vec{w}/\sqrt{d},\,\vec{x}\rangle\}$,
$\mathcal V\subseteq\mathcal W$.
Under the uniform distribution over these functions, if $n\le d/6$, then with
probability larger than $1/2$ the sample misses at least one direction
$\vec{w}\in\mathcal W$.
An adversarial
ERM tie-breaking rule can then choose this unseen direction: it attains the
empirical minimum $1/2$, but its population value is $3/4$, so its excess
risk is at least $1/4$.
Not every empirical minimizer fails.
Feldman explicitly notes (Remark~3.4) that another ERM, such as the origin, can
generalize well.
The lower bound shows that empirical minimization can pay for
hidden high-dimensional directions through its choice of minimizer.} and the
subsequent ERM sample-complexity work, the dimension-dependent penalty is a
property of optimizing the fixed empirical objective formed from the observed
training data, not of the population objective.

The contrast is structural.
In the SCO constructions, a particular loss,
order, and subgradient structure \emph{can} steer a procedure toward
empirical badness.
Here the exact train--test frontier \emph{forces} every
linear span predictor with small test error in the benign-misfitting window
to have large empirical error.
SGD supplies one route to the fourth quadrant,
but it is not the source of the obstruction.

There is nevertheless a useful rhyme between the mechanisms.
The
Koren--Livni--Mansour--Sherman construction is an empirical-risk-raising
mechanism: the chosen updates move into coordinates that are penalized by
earlier training examples when the final empirical risk is evaluated.
In the
present squared-loss geometry, the analogue is explicit:
equation~\eqref{eq:climbing-uphill} computes directly that the full empirical
training error increases step by step along the useful pass.

\paragraph{Thriving outside the ball: the norm-budget view.}
\begingroup
\setlength{\abovedisplayskip}{6pt}
\setlength{\belowdisplayskip}{6pt}
The following observation is a useful translation for readers interested in
bounded domains.
For the fixed-axis model in Subsection~\ref{subsec:setup}, it is a norm-budget comparison inside the linear training span, not a bounded-domain lower bound for all predictors.
The revealed-axis spike oracle itself lies in
any reasonable ball and has zero error.
Under the unknown-orientation
criteria of Appendix~\ref{app:latent-rotation-span-complete}, however,
Corollary~\ref{cor:rotspan-frontier-transfer} shows that the same constrained
value is obtained even when every ambient global-linear predictor is
allowed.
Nonlinear prediction rules and rules that choose their coefficient vector using the input \(\vec{x}\) at prediction time remain outside that statement.
In a bounded-domain formulation, the learner is often required to output a
point in a domain $\mathcal W$, typically a norm ball%
\footnote{For linear predictors, $\|\vec{w}\|_2$ is also the Euclidean
input-Lipschitz constant of the prediction map $\vec{x}\mapsto \vec{w}^\top \vec{x}$.
We will only use the norm comparison below.}
chosen to contain a comparator.
In the present model the spike oracle
$\vec{w}^\star=(1/\sqrt{\gamma})\,\vec{e}_1$ has
\[
    \|\vec{w}^\star\|_2^2=\frac1\gamma.
\]
For an arbitrary span predictor $\vec{u}=\Xt\vec{\alpha}$, write
$r=\gamma\sum_i\alpha_i$ and $q=\sum_i\alpha_i^2$ using our typical notation.
Then
\[
    \|\vec{u}\|_2^2
    =
    \frac{r^2}{\gamma}+dq
    \ge
    \frac{r^2}{\gamma}\left(1+\frac1\rho\right),
\]
where the inequality is Cauchy--Schwarz:
$q\ge(\sum_i\alpha_i)^2/n=r^2/(\gamma^2 n)$, and $\rho=\gamma n/d$.
Equality holds exactly for the symmetric coefficients
$\alpha_1=\cdots=\alpha_n$.
Thus the oracle norm ball
$\mathcal B_\star=\{\vec{u}:\|\vec{u}\|_2\le\|\vec{w}^\star\|_2\}$ forces
\[
    |r|\le\sqrt{\frac{\rho}{1+\rho}}\,.
\]
In the benign-misfitting window $\rho\to0$, this prevents spike calibration:
since $\Ttest(\vec{u})\ge(r-1)^2$, every span predictor inside
$\mathcal B_\star$ has
\[
    \Ttest(\vec{u})\ge
    \left(1-\sqrt{\frac{\rho}{1+\rho}}\right)^{\!2},
\]
which tends to the zero-predictor baseline rather than to zero.

Conversely, for \(0\le\varepsilon<1\), any span predictor with $\Ttest(\vec{u})\le\varepsilon$ must have
$r\ge1-\sqrt{\varepsilon}$, and hence
\[
    \|\vec{u}\|_2^2
    \ge
    \frac{(1-\sqrt{\varepsilon})^2}{\gamma}
    \left(1+\frac1\rho\right).
\]
Thus sample-efficient span prediction in the fourth-quadrant window requires
a larger norm budget than the one defined by the oracle ball: it must leave
that ball by a factor of order $1/\rho$ in squared norm.
The test-optimal span
predictor realizes this scale.
At the symmetric test-optimal calibration
$r^\star=\gamma\rho/(1+\gamma\rho)$,
\[
    \|\vec{u}^{\star}_{\rm span}\|_2^2
    =
    \frac{\gamma\rho(1+\rho)}{(1+\gamma\rho)^2}
    =
    (1+o(1))\frac{1}{\gamma\rho}
    =
    (1+o(1))\frac{d}{\gamma^2n}
\]
throughout the benign-misfitting window.
This is much larger than
$\|\vec{w}^\star\|_2^2=1/\gamma$.
By contrast, the unique span-restricted
interpolator has
\[
    \|\vec{u}_{\rm int}\|_2^2
    =
    \frac{n}{d+\gamma n}
    =
    \frac{\rho}{\gamma(1+\rho)}
    <
    \frac{1}{\gamma},
\]
and is the point selected among interpolators by zero-initialized
least-squares dynamics with a sufficiently small learning rate, in contrast
to the large one-pass rates studied in Section~\ref{sec:sgd}.
It remains
inside the oracle norm ball, but it
does not generalize well until the later interpolation threshold.
This is the
norm-budget lesson of the stylized geometry: the $d/\gamma^2$ sample
advantage is achieved by leaving the oracle norm ball, not merely
by choosing a different empirical minimizer inside it.
The three norms satisfy a compact identity: with
$g_{\min}=1/(1+\gamma\rho)$ from \eqref{eq:frontier-endpoints},
\[
    \|\vec{u}_{\rm int}\|_2\,\|\vec{u}^{\star}_{\rm span}\|_2
    =
    \frac{\rho}{1+\gamma\rho}
    =
    (1-g_{\min})\,\|\vec{w}^\star\|_2^2.
\]
Thus in the calibrated regime $\gamma\rho\to\infty$, the oracle squared
norm is, up to $1+o(1)$, the geometric mean of the interpolator and
test-optimal span-predictor squared norms.

The same norm-budget picture also shows why the usual empirical-risk
diagnostic is uninformative for the useful point: a statement of the form population risk $\le$ empirical risk
\emph{plus} a controlled gap cannot usefully certify the calibrated span
predictor, because its empirical risk is intentionally large.
This is the
norm-budget version of the same warning about empirical-risk diagnostics.
A close SCO precedent is the ``thinking outside the ball'' analysis of
\citet{amir2022thinking}: for generalized-linear SCO, unprojected
early-stopped gradient descent can learn relative to a unit-norm comparator
even though a naive origin-centered norm-ball analysis gives the wrong sample
complexity.
Their mechanism is different---the right analysis set is
distribution-dependent---but the shared lesson is that the comparator norm
ball need not be the right explanatory object.

\paragraph{Scope and terminology.}
The gap sizes and scopes also differ.
The SCO examples typically emphasize a
constant-order empirical/population discrepancy for a constructed problem.
In
the stylized span geometry here, test error can vanish while empirical error
exceeds the zero-predictor baseline and can diverge with the parameters.
The
price is scope: the orthogonal model is deliberately narrow.
The payoff is
auditability: the fourth-quadrant obstruction can be checked by elementary
linear algebra.
The dimension regime is also different: the
benign-misfitting window has $d\gg\gamma n$, so $n$ is deliberately far
below the interpolation threshold.

Our word ``misfitting'' is meant to emphasize that the useful span predictors
are not underpowered.
They lie past interpolation and must deliberately
overshoot the training labels on average.
``Benign underfitting'' describes what can
happen for a particular SCO algorithm; ``benign misfitting'' describes the
fundamental consequences of the geometry of the training set when one is
sample limited.
\endgroup

\bibliographystyle{plainnat}
\bibliography{itw_benign_misfitting/itw_references,BenignMisfittingStylized_selected}

\begin{thebibliography}{255}
\providecommand{\natexlab}[1]{#1}
\providecommand{\url}[1]{\texttt{#1}}
\expandafter\ifx\csname urlstyle\endcsname\relax
  \providecommand{\doi}[1]{doi: #1}\else
  \providecommand{\doi}{doi: \begingroup \urlstyle{rm}\Url}\fi

\bibitem[Abadi et~al.(2016)Abadi, Chu, Goodfellow, McMahan, Mironov, Talwar,
  and Zhang]{abadi2016deep}
Martin Abadi, Andy Chu, Ian Goodfellow, H.~Brendan McMahan, Ilya Mironov, Kunal
  Talwar, and Li~Zhang.
\newblock Deep learning with differential privacy.
\newblock In \emph{Proceedings of the 2016 ACM SIGSAC Conference on Computer
  and Communications Security}, pages 308--318, 2016.

\bibitem[Abbe et~al.(2022)Abbe, Adser{\`a}, and Misiakiewicz]{abbe2022merged}
Emmanuel Abbe, Enric~Boix Adser{\`a}, and Theodor Misiakiewicz.
\newblock The merged-staircase property: A necessary and nearly sufficient
  condition for {SGD} learning of sparse functions on two-layer neural
  networks.
\newblock In \emph{Proceedings of Thirty Fifth Conference on Learning Theory},
  volume 178 of \emph{Proceedings of Machine Learning Research}, pages
  4782--4887. PMLR, 2022.
\newblock URL \url{https://proceedings.mlr.press/v178/abbe22a.html}.

\bibitem[Abbe et~al.(2023)Abbe, Adser{\`a}, and Misiakiewicz]{abbe2023leap}
Emmanuel Abbe, Enric~Boix Adser{\`a}, and Theodor Misiakiewicz.
\newblock {SGD} learning on neural networks: Leap complexity and
  saddle-to-saddle dynamics.
\newblock In \emph{Proceedings of Thirty Sixth Conference on Learning Theory},
  volume 195 of \emph{Proceedings of Machine Learning Research}, pages
  2552--2623. PMLR, 2023.
\newblock URL \url{https://proceedings.mlr.press/v195/abbe23a.html}.

\bibitem[Aden-Ali et~al.(2024)Aden-Ali, H{\o}gsgaard, Larsen, and
  Zhivotovskiy]{adenali2024majority}
Ishaq Aden-Ali, Mikael~M{\o}ller H{\o}gsgaard, Kasper~Green Larsen, and Nikita
  Zhivotovskiy.
\newblock Majority-of-three: The simplest optimal learner?
\newblock In \emph{Proceedings of the Thirty Seventh Conference on Learning
  Theory}, volume 247 of \emph{Proceedings of Machine Learning Research}, pages
  22--45. PMLR, 2024.

\bibitem[Advani and Saxe(2020)]{advani2020highdim}
Madhu~S. Advani and Andrew~M. Saxe.
\newblock High-dimensional dynamics of generalization error in neural networks.
\newblock \emph{Neural Networks}, 132:\penalty0 428--446, 2020.
\newblock \doi{10.1016/j.neunet.2020.08.022}.

\bibitem[Ahn et~al.(2022)Ahn, Zhang, and Sra]{ahn2022understanding}
Kwangjun Ahn, Jingzhao Zhang, and Suvrit Sra.
\newblock Understanding the unstable convergence of gradient descent.
\newblock In \emph{Proceedings of the 39th International Conference on Machine
  Learning}, volume 162 of \emph{Proceedings of Machine Learning Research},
  pages 247--257. PMLR, 2022.
\newblock URL \url{https://proceedings.mlr.press/v162/ahn22a.html}.

\bibitem[Ali et~al.(2019)Ali, Kolter, and Tibshirani]{ali2019continuous}
Alnur Ali, J.~Zico Kolter, and Ryan~J. Tibshirani.
\newblock A continuous-time view of early stopping for least squares
  regression.
\newblock In \emph{Proceedings of the Twenty-Second International Conference on
  Artificial Intelligence and Statistics}, volume~89 of \emph{Proceedings of
  Machine Learning Research}, pages 1370--1378. PMLR, 2019.
\newblock URL \url{https://proceedings.mlr.press/v89/ali19a.html}.

\bibitem[Ali et~al.(2020)Ali, Dobriban, and Tibshirani]{ali2020implicit}
Alnur Ali, Edgar Dobriban, and Ryan Tibshirani.
\newblock The implicit regularization of stochastic gradient flow for least
  squares.
\newblock In \emph{Proceedings of the 37th International Conference on Machine
  Learning}, volume 119 of \emph{Proceedings of Machine Learning Research},
  pages 233--244. PMLR, 2020.
\newblock URL \url{https://proceedings.mlr.press/v119/ali20a.html}.

\bibitem[Alon et~al.(1999)Alon, Matias, and Szegedy]{alon1999space}
Noga Alon, Yossi Matias, and Mario Szegedy.
\newblock The space complexity of approximating the frequency moments.
\newblock \emph{Journal of Computer and System Sciences}, 58\penalty0
  (1):\penalty0 137--147, 1999.

\bibitem[Amir et~al.(2021{\natexlab{a}})Amir, Carmon, Koren, and
  Livni]{amir2021neverfullbatch}
Idan Amir, Yair Carmon, Tomer Koren, and Roi Livni.
\newblock Never go full batch (in stochastic convex optimization).
\newblock In \emph{Advances in Neural Information Processing Systems},
  volume~34, 2021{\natexlab{a}}.

\bibitem[Amir et~al.(2021{\natexlab{b}})Amir, Koren, and Livni]{amir2021sgd}
Idan Amir, Tomer Koren, and Roi Livni.
\newblock {SGD} generalizes better than {GD} (and regularization doesn't help).
\newblock In \emph{Proceedings of the Thirty Fourth Conference on Learning
  Theory}, volume 134 of \emph{Proceedings of Machine Learning Research}, pages
  63--92. PMLR, 2021{\natexlab{b}}.

\bibitem[Amir et~al.(2022)Amir, Livni, and Srebro]{amir2022thinking}
Idan Amir, Roi Livni, and Nathan Srebro.
\newblock Thinking outside the ball: Optimal learning with gradient descent for
  generalized linear stochastic convex optimization.
\newblock In \emph{Advances in Neural Information Processing Systems},
  volume~35, 2022.

\bibitem[Andreyev and Beneventano(2024)]{andreyev2024edge}
Arseniy Andreyev and Pierfrancesco Beneventano.
\newblock Edge of stochastic stability: Revisiting the edge of stability for
  {SGD}, 2024.
\newblock URL \url{https://arxiv.org/abs/2412.20553}.

\bibitem[Andriushchenko et~al.(2023)Andriushchenko, Varre, Pillaud-Vivien, and
  Flammarion]{andriushchenko2023sparse}
Maksym Andriushchenko, Aditya~Vardhan Varre, Loucas Pillaud-Vivien, and Nicolas
  Flammarion.
\newblock {SGD} with large step sizes learns sparse features.
\newblock In \emph{Proceedings of the 40th International Conference on Machine
  Learning}, volume 202 of \emph{Proceedings of Machine Learning Research},
  pages 903--925. PMLR, 2023.
\newblock URL \url{https://proceedings.mlr.press/v202/andriushchenko23b.html}.

\bibitem[Arora et~al.(2022)Arora, Li, and Panigrahi]{arora2022understandingeos}
Sanjeev Arora, Zhiyuan Li, and Abhishek Panigrahi.
\newblock Understanding gradient descent on the edge of stability in deep
  learning.
\newblock In \emph{International Conference on Machine Learning}, pages
  948--1024, 2022.

\bibitem[Attias et~al.(2024)Attias, Dziugaite, Haghifam, Livni, and
  Roy]{attias2024information}
Idan Attias, Gintare~Karolina Dziugaite, Mahdi Haghifam, Roi Livni, and
  Daniel~M. Roy.
\newblock Information complexity of stochastic convex optimization:
  Applications to generalization, memorization, and tracing.
\newblock In \emph{Proceedings of the 41st International Conference on Machine
  Learning}, 2024.
\newblock arXiv:2402.09327.

\bibitem[Audibert and Catoni(2011)]{audibert2011robust}
Jean-Yves Audibert and Olivier Catoni.
\newblock Robust linear least squares regression.
\newblock \emph{The Annals of Statistics}, 39\penalty0 (5):\penalty0
  2766--2794, 2011.

\bibitem[Avestimehr et~al.(2011)Avestimehr, Diggavi, and
  Tse]{avestimehr2011wireless}
Amir~Salman Avestimehr, Suhas~N. Diggavi, and David N.~C. Tse.
\newblock Wireless network information flow: A deterministic approach.
\newblock \emph{IEEE Transactions on Information Theory}, 57\penalty0
  (4):\penalty0 1872--1905, 2011.
\newblock \doi{10.1109/TIT.2011.2110110}.

\bibitem[Baik et~al.(2005)Baik, Ben~Arous, and P{\'e}ch{\'e}]{baik2005phase}
Jinho Baik, G{\'e}rard Ben~Arous, and Sandrine P{\'e}ch{\'e}.
\newblock Phase transition of the largest eigenvalue for nonnull complex sample
  covariance matrices.
\newblock \emph{The Annals of Probability}, 33\penalty0 (5):\penalty0
  1643--1697, 2005.
\newblock \doi{10.1214/009117905000000233}.

\bibitem[Barrett and Dherin(2021)]{barrett2021implicit}
David G.~T. Barrett and Benoit Dherin.
\newblock Implicit gradient regularization.
\newblock In \emph{International Conference on Learning Representations}, 2021.
\newblock URL \url{https://openreview.net/forum?id=3q5IqUrkcF}.

\bibitem[Bartlett et~al.(2020)Bartlett, Long, Lugosi, and
  Tsigler]{bartlett2020benign}
Peter~L. Bartlett, Philip~M. Long, G{\'a}bor Lugosi, and Alexander Tsigler.
\newblock Benign overfitting in linear regression.
\newblock \emph{Proceedings of the National Academy of Sciences}, 117\penalty0
  (48):\penalty0 30063--30070, 2020.
\newblock \doi{10.1073/pnas.1907378117}.

\bibitem[Belkin(2021)]{belkin2021fit}
Mikhail Belkin.
\newblock Fit without fear: Remarkable mathematical phenomena of deep learning
  through the prism of interpolation.
\newblock \emph{Acta Numerica}, 30:\penalty0 203--248, 2021.
\newblock \doi{10.1017/S0962492921000039}.

\bibitem[Belkin et~al.(2018{\natexlab{a}})Belkin, Hsu, and
  Mitra]{belkin2018overfitting}
Mikhail Belkin, Daniel~J. Hsu, and Partha~P. Mitra.
\newblock Overfitting or perfect fitting? risk bounds for classification and
  regression rules that interpolate.
\newblock In \emph{Advances in Neural Information Processing Systems 31}, pages
  2300--2311, 2018{\natexlab{a}}.
\newblock URL \url{https://papers.nips.cc/paper/7498}.

\bibitem[Belkin et~al.(2018{\natexlab{b}})Belkin, Ma, and
  Mandal]{belkin2018kernel}
Mikhail Belkin, Siyuan Ma, and Soumik Mandal.
\newblock To understand deep learning we need to understand kernel learning.
\newblock In \emph{Proceedings of the 35th International Conference on Machine
  Learning}, volume~80 of \emph{Proceedings of Machine Learning Research},
  pages 541--549. PMLR, 2018{\natexlab{b}}.
\newblock URL \url{https://proceedings.mlr.press/v80/belkin18a.html}.

\bibitem[Belkin et~al.(2019{\natexlab{a}})Belkin, Hsu, Ma, and
  Mandal]{belkin2019reconciling}
Mikhail Belkin, Daniel Hsu, Siyuan Ma, and Soumik Mandal.
\newblock Reconciling modern machine-learning practice and the classical
  bias--variance trade-off.
\newblock \emph{Proceedings of the National Academy of Sciences}, 116\penalty0
  (32):\penalty0 15849--15854, 2019{\natexlab{a}}.
\newblock \doi{10.1073/pnas.1903070116}.

\bibitem[Belkin et~al.(2019{\natexlab{b}})Belkin, Rakhlin, and
  Tsybakov]{belkin2019does}
Mikhail Belkin, Alexander Rakhlin, and Alexandre~B. Tsybakov.
\newblock Does data interpolation contradict statistical optimality?
\newblock In \emph{Proceedings of the Twenty-Second International Conference on
  Artificial Intelligence and Statistics}, volume~89 of \emph{Proceedings of
  Machine Learning Research}, pages 1611--1619. PMLR, 2019{\natexlab{b}}.
\newblock URL \url{https://proceedings.mlr.press/v89/belkin19a.html}.

\bibitem[Belkin et~al.(2020)Belkin, Hsu, and Xu]{belkin2020twomodels}
Mikhail Belkin, Daniel Hsu, and Ji~Xu.
\newblock Two models of double descent for weak features.
\newblock \emph{SIAM Journal on Mathematics of Data Science}, 2\penalty0
  (4):\penalty0 1167--1180, 2020.
\newblock \doi{10.1137/20M1336072}.

\bibitem[Bello and Nelin(1962)]{sp_bello1962}
Phillip Bello and Bert~D. Nelin.
\newblock Predetection diversity combining with selectively fading channels.
\newblock \emph{IRE Transactions on Communications Systems}, 10\penalty0
  (1):\penalty0 32--42, 1962.
\newblock \doi{10.1109/TCOM.1962.1088632}.

\bibitem[Bertsekas(2011)]{bertsekas2011incremental}
Dimitri~P. Bertsekas.
\newblock Incremental gradient, subgradient, and proximal methods for convex
  optimization: A survey.
\newblock In Suvrit Sra, Sebastian Nowozin, and Stephen~J. Wright, editors,
  \emph{Optimization for Machine Learning}, pages 85--120. MIT Press, 2011.
\newblock \doi{10.7551/mitpress/8996.003.0006}.
\newblock Also arXiv:1507.01030.

\bibitem[Beugnot et~al.(2022)Beugnot, Mairal, and Rudi]{beugnot2022large}
Gaspard Beugnot, Julien Mairal, and Alessandro Rudi.
\newblock On the benefits of large learning rates for kernel methods.
\newblock In \emph{Proceedings of Thirty Fifth Conference on Learning Theory},
  volume 178 of \emph{Proceedings of Machine Learning Research}, pages
  254--282. PMLR, 2022.
\newblock URL \url{https://proceedings.mlr.press/v178/beugnot22a.html}.

\bibitem[Bietti et~al.(2022)Bietti, Bruna, Sanford, and
  Song]{bietti2022singleindex}
Alberto Bietti, Joan Bruna, Clayton Sanford, and Min~Jae Song.
\newblock Learning single-index models with shallow neural networks.
\newblock In \emph{Advances in Neural Information Processing Systems 35}, pages
  9768--9783, 2022.
\newblock URL
  \url{https://papers.nips.cc/paper_files/paper/2022/hash/3fb6c52aeb11e09053c16eabee74dd7b-Abstract-Conference.html}.

\bibitem[Blanc et~al.(2020)Blanc, Gupta, Valiant, and
  Valiant]{blanc2020implicit}
Guy Blanc, Neha Gupta, Gregory Valiant, and Paul Valiant.
\newblock Implicit regularization for deep neural networks driven by an
  {Ornstein--Uhlenbeck} like process.
\newblock In \emph{Proceedings of the Thirty Third Conference on Learning
  Theory}, volume 125 of \emph{Proceedings of Machine Learning Research}, pages
  483--513. PMLR, 2020.
\newblock URL \url{https://proceedings.mlr.press/v125/blanc20a.html}.

\bibitem[Borade and Zheng(2008)]{borade2008euclidean}
Shashi Borade and Lizhong Zheng.
\newblock Euclidean information theory.
\newblock In \emph{International Zurich Seminar on Communications}, pages
  14--17, 2008.

\bibitem[Bottou and Bousquet(2007)]{bottou2008tradeoffs}
L{\'e}on Bottou and Olivier Bousquet.
\newblock The tradeoffs of large scale learning.
\newblock In \emph{Advances in Neural Information Processing Systems},
  volume~20, 2007.

\bibitem[Bottou et~al.(2018)Bottou, Curtis, and
  Nocedal]{bottou2018optimization}
L{\'e}on Bottou, Frank~E. Curtis, and Jorge Nocedal.
\newblock Optimization methods for large-scale machine learning.
\newblock \emph{SIAM Review}, 60\penalty0 (2):\penalty0 223--311, 2018.
\newblock \doi{10.1137/16M1080173}.

\bibitem[Bousquet and Elisseeff(2002)]{bousquet2002stability}
Olivier Bousquet and Andr{\'e} Elisseeff.
\newblock Stability and generalization.
\newblock \emph{Journal of Machine Learning Research}, 2:\penalty0 499--526,
  2002.

\bibitem[Bousquet et~al.(2020)Bousquet, Hanneke, Moran, and
  Zhivotovskiy]{bousquet2020proper}
Olivier Bousquet, Steve Hanneke, Shay Moran, and Nikita Zhivotovskiy.
\newblock Proper learning, {Helly} number, and an optimal {SVM} bound.
\newblock In \emph{Proceedings of the Thirty Third Conference on Learning
  Theory}, volume 125 of \emph{Proceedings of Machine Learning Research}, pages
  582--609. PMLR, 2020.

\bibitem[Breiman(1996)]{breiman1996bagging}
Leo Breiman.
\newblock Bagging predictors.
\newblock \emph{Machine Learning}, 24\penalty0 (2):\penalty0 123--140, 1996.

\bibitem[Brown et~al.(2021)Brown, Bun, Feldman, Smith, and
  Talwar]{brown2021memorization}
Gavin Brown, Mark Bun, Vitaly Feldman, Adam Smith, and Kunal Talwar.
\newblock When is memorization of irrelevant training data necessary for
  high-accuracy learning?
\newblock In \emph{Proceedings of the 53rd Annual ACM SIGACT Symposium on
  Theory of Computing}, STOC 2021, pages 123--132, 2021.
\newblock \doi{10.1145/3406325.3451131}.

\bibitem[Bruna and Hsu(2025)]{bruna2025multiindex}
Joan Bruna and Daniel Hsu.
\newblock Survey on algorithms for multi-index models.
\newblock \emph{Statistical Science}, 40\penalty0 (3):\penalty0 378--391, 2025.

\bibitem[Bu et~al.(2020)Bu, Zou, and Veeravalli]{bu2020tightening}
Yuheng Bu, Shaofeng Zou, and Venugopal~V. Veeravalli.
\newblock Tightening mutual information-based bounds on generalization error.
\newblock \emph{IEEE Journal on Selected Areas in Information Theory},
  1\penalty0 (1):\penalty0 121--130, 2020.
\newblock \doi{10.1109/JSAIT.2020.2991139}.

\bibitem[Bubeck and Sellke(2021)]{bubeck2021universalLaw}
S{\'e}bastien Bubeck and Mark Sellke.
\newblock A universal law of robustness via isoperimetry.
\newblock In \emph{Advances in Neural Information Processing Systems},
  volume~34, 2021.

\bibitem[Burla and Livni(2026)]{burla2026erms}
Tal Burla and Roi Livni.
\newblock All {ERM}s can fail in stochastic convex optimization lower bounds in
  linear dimension, 2026.

\bibitem[Cai et~al.(2024)Cai, Wu, Mei, Lindsey, and
  Bartlett]{cai2024nonhomogeneous}
Yuhang Cai, Jingfeng Wu, Song Mei, Michael Lindsey, and Peter~L. Bartlett.
\newblock Large stepsize gradient descent for non-homogeneous two-layer
  networks: Margin improvement and fast optimization.
\newblock In \emph{Advances in Neural Information Processing Systems},
  volume~37, pages 71306--71351, 2024.
\newblock \doi{10.52202/079017-2278}.
\newblock URL
  \url{https://papers.nips.cc/paper_files/paper/2024/hash/835a0185f61867a1ea0f86155489839a-Abstract-Conference.html}.

\bibitem[Carlini et~al.(2019)Carlini, Liu, Erlingsson, Kos, and
  Song]{carlini2019secret}
Nicholas Carlini, Chang Liu, {\'U}lfar Erlingsson, Jernej Kos, and Dawn Song.
\newblock The secret sharer: Evaluating and testing unintended memorization in
  neural networks.
\newblock In \emph{28th USENIX Security Symposium}, pages 267--284, 2019.

\bibitem[Carlini et~al.(2021)Carlini, Tram{\`e}r, Wallace, Jagielski,
  Herbert-Voss, Lee, Roberts, Brown, Song, Erlingsson, Oprea, and
  Raffel]{carlini2021extracting}
Nicholas Carlini, Florian Tram{\`e}r, Eric Wallace, Matthew Jagielski, Ariel
  Herbert-Voss, Katherine Lee, Adam Roberts, Tom Brown, Dawn Song, {\'U}lfar
  Erlingsson, Alina Oprea, and Colin Raffel.
\newblock Extracting training data from large language models.
\newblock In \emph{30th USENIX Security Symposium}, 2021.

\bibitem[Carlini et~al.(2022)Carlini, Chien, Nasr, Song, Terzis, and
  Tram{\`e}r]{carlini2022firstprinciples}
Nicholas Carlini, Steve Chien, Milad Nasr, Shuang Song, Andreas Terzis, and
  Florian Tram{\`e}r.
\newblock Membership inference attacks from first principles.
\newblock In \emph{IEEE Symposium on Security and Privacy (SP)}, pages
  1897--1914, 2022.
\newblock \doi{10.1109/SP46214.2022.9833649}.

\bibitem[Carmon et~al.(2024)Carmon, Livni, and Yehudayoff]{carmon2024erm}
Daniel Carmon, Roi Livni, and Amir Yehudayoff.
\newblock The sample complexity of {ERM}s in stochastic convex optimization.
\newblock In \emph{Proceedings of the 27th International Conference on
  Artificial Intelligence and Statistics}, volume 238 of \emph{Proceedings of
  Machine Learning Research}, pages 3799--3807. PMLR, 2024.

\bibitem[Cavers(1991)]{sp_cavers1991}
James~K. Cavers.
\newblock An analysis of pilot symbol assisted modulation for {Rayleigh} fading
  channels.
\newblock \emph{IEEE Transactions on Vehicular Technology}, 40\penalty0
  (4):\penalty0 686--693, 1991.
\newblock \doi{10.1109/25.108378}.

\bibitem[Chaudhuri et~al.(2011)Chaudhuri, Monteleoni, and
  Sarwate]{chaudhuri2011differentially}
Kamalika Chaudhuri, Claire Monteleoni, and Anand~D. Sarwate.
\newblock Differentially private empirical risk minimization.
\newblock \emph{Journal of Machine Learning Research}, 12:\penalty0 1069--1109,
  2011.

\bibitem[Chaudhuri et~al.(2013)Chaudhuri, Sarwate, and
  Sinha]{chaudhuri2013near}
Kamalika Chaudhuri, Anand~D. Sarwate, and Kaushik Sinha.
\newblock Near-optimal algorithms for differentially-private principal
  components.
\newblock In \emph{Advances in Neural Information Processing Systems},
  volume~26, 2013.

\bibitem[Chen et~al.(2020)Chen, Wu, and Hong]{chen2020understandingclipping}
Xiangyi Chen, Steven~Z Wu, and Mingyi Hong.
\newblock Understanding gradient clipping in private {SGD}: A geometric
  perspective.
\newblock In \emph{Advances in Neural Information Processing Systems},
  volume~33, 2020.

\bibitem[Chen et~al.(2026)Chen, Zhou, and Zhou]{chen2026sketched}
Ziyan Chen, Zhongzhu Zhou, and Ding-Xuan Zhou.
\newblock From one-pass {SGD} to data reuse: Mini-batch scaling laws in
  sketched linear regression, 2026.
\newblock arXiv:2605.24316.

\bibitem[Chinot et~al.(2022)Chinot, L{\"o}ffler, and van~de
  Geer]{chinot2022robustness}
Geoffrey Chinot, Matthias L{\"o}ffler, and Sara van~de Geer.
\newblock On the robustness of minimum norm interpolators and regularized
  empirical risk minimizers.
\newblock \emph{The Annals of Statistics}, 50\penalty0 (4):\penalty0
  2306--2333, 2022.
\newblock \doi{10.1214/22-AOS2190}.

\bibitem[Choquette-Choo et~al.(2021)Choquette-Choo, Tram{\`e}r, Carlini, and
  Papernot]{choquette2021labelonly}
Christopher~A. Choquette-Choo, Florian Tram{\`e}r, Nicholas Carlini, and
  Nicolas Papernot.
\newblock Label-only membership inference attacks.
\newblock In \emph{International Conference on Machine Learning}, volume 139 of
  \emph{Proceedings of Machine Learning Research}, pages 1964--1974, 2021.

\bibitem[Cohen et~al.(2021)Cohen, Kaur, Li, Kolter, and
  Talwalkar]{cohen2021gradient}
Jeremy~M. Cohen, Simran Kaur, Yuanzhi Li, J.~Zico Kolter, and Ameet Talwalkar.
\newblock Gradient descent on neural networks typically occurs at the edge of
  stability.
\newblock In \emph{International Conference on Learning Representations}, 2021.
\newblock URL \url{https://openreview.net/forum?id=jh-rTtvkGeM}.

\bibitem[Collins-Woodfin and Paquette(2024)]{collinswoodfin2024onepass}
Elizabeth Collins-Woodfin and Elliot Paquette.
\newblock High-dimensional limit of one-pass {SGD} on least squares.
\newblock \emph{Electronic Communications in Probability}, 29, 2024.
\newblock \doi{10.1214/23-ECP571}.
\newblock arXiv:2304.06847.

\bibitem[Cubuk et~al.(2019)Cubuk, Zoph, Mane, Vasudevan, and
  Le]{cubuk2019autoaugment}
Ekin~D. Cubuk, Barret Zoph, Dandelion Mane, Vijay Vasudevan, and Quoc~V. Le.
\newblock {AutoAugment}: Learning augmentation strategies from data.
\newblock In \emph{Proceedings of the IEEE/CVF Conference on Computer Vision
  and Pattern Recognition}, 2019.

\bibitem[Damian et~al.(2023)Damian, Nichani, and
  Lee]{damian2023selfstabilization}
Alex Damian, Eshaan Nichani, and Jason~D Lee.
\newblock Self-stabilization: The implicit bias of gradient descent at the edge
  of stability.
\newblock In \emph{International Conference on Learning Representations}, 2023.

\bibitem[Dar et~al.(2021)Dar, Muthukumar, and Baraniuk]{dar2021farewell}
Yehuda Dar, Vidya Muthukumar, and Richard~G. Baraniuk.
\newblock A farewell to the bias-variance tradeoff? an overview of the theory
  of overparameterized machine learning, 2021.
\newblock URL \url{https://arxiv.org/abs/2109.02355}.

\bibitem[Devroye et~al.(2016)Devroye, Lerasle, Lugosi, and
  Oliveira]{devroye2016subgaussian}
Luc Devroye, Matthieu Lerasle, G{\'a}bor Lugosi, and Roberto~I. Oliveira.
\newblock Sub-{G}aussian mean estimators.
\newblock \emph{The Annals of Statistics}, 44\penalty0 (6):\penalty0
  2695--2725, 2016.

\bibitem[Dinh et~al.(2017)Dinh, Pascanu, Bengio, and Bengio]{dinh2017sharp}
Laurent Dinh, Razvan Pascanu, Samy Bengio, and Yoshua Bengio.
\newblock Sharp minima can generalize for deep nets.
\newblock In \emph{Proceedings of the 34th International Conference on Machine
  Learning}, volume~70 of \emph{Proceedings of Machine Learning Research},
  pages 1019--1028. PMLR, 2017.
\newblock URL \url{https://proceedings.mlr.press/v70/dinh17b.html}.

\bibitem[Dohmatob(2019)]{dohmatob2019generalizedNoFreeLunch}
Elvis Dohmatob.
\newblock Generalized no free lunch theorem for adversarial robustness.
\newblock In \emph{Proceedings of the 36th International Conference on Machine
  Learning}, volume~97 of \emph{Proceedings of Machine Learning Research},
  pages 1646--1654, 2019.

\bibitem[Dohmatob and Scetbon(2023)]{dohmatob2023robustGeneral}
Elvis Dohmatob and Meyer Scetbon.
\newblock Robust linear regression: Phase-transitions and precise tradeoffs for
  general norms.
\newblock \emph{arXiv preprint arXiv:2308.00556}, 2023.

\bibitem[Draper and Van~Nostrand(1979)]{draper1979ridge}
Norman~R. Draper and R.~Craig Van~Nostrand.
\newblock Ridge regression and james--stein estimation: Review and comments.
\newblock \emph{Technometrics}, 21\penalty0 (4):\penalty0 451--466, 1979.
\newblock \doi{10.1080/00401706.1979.10489815}.

\bibitem[Dudeja and Hsu(2018)]{dudeja2018singleindex}
Rishabh Dudeja and Daniel Hsu.
\newblock Learning single-index models in gaussian space.
\newblock In \emph{Proceedings of the 31st Conference on Learning Theory},
  volume~75 of \emph{Proceedings of Machine Learning Research}, pages
  1887--1930. PMLR, 2018.
\newblock URL \url{https://proceedings.mlr.press/v75/dudeja18a.html}.

\bibitem[Dwork and Roth(2014)]{dwork2014algorithmic}
Cynthia Dwork and Aaron Roth.
\newblock The algorithmic foundations of differential privacy.
\newblock \emph{Foundations and Trends in Theoretical Computer Science},
  9\penalty0 (3--4):\penalty0 211--407, 2014.
\newblock \doi{10.1561/0400000042}.

\bibitem[Dwork et~al.(2006)Dwork, McSherry, Nissim, and
  Smith]{dwork2006calibrating}
Cynthia Dwork, Frank McSherry, Kobbi Nissim, and Adam Smith.
\newblock Calibrating noise to sensitivity in private data analysis.
\newblock In \emph{Theory of Cryptography Conference}, pages 265--284.
  Springer, 2006.

\bibitem[Dwork et~al.(2014)Dwork, Talwar, Thakurta, and
  Zhang]{dwork2014analyze}
Cynthia Dwork, Kunal Talwar, Abhradeep Thakurta, and Li~Zhang.
\newblock Analyze gauss: Optimal bounds for privacy-preserving principal
  component analysis.
\newblock In \emph{Proceedings of the 46th Annual ACM Symposium on Theory of
  Computing}, pages 11--20, 2014.

\bibitem[Eaton(1989)]{eaton1989group}
Morris~L. Eaton.
\newblock \emph{Group Invariance Applications in Statistics}, volume~1 of
  \emph{NSF-CBMS Regional Conference Series in Probability and Statistics}.
\newblock Institute of Mathematical Statistics and American Statistical
  Association, 1989.
\newblock ISBN 9780940600157.

\bibitem[Efron and Hastie(2016)]{efron2016jamesstein}
Bradley Efron and Trevor Hastie.
\newblock James--stein estimation and ridge regression.
\newblock In \emph{Computer Age Statistical Inference}, pages 91--107.
  Cambridge University Press, 2016.
\newblock \doi{10.1017/CBO9781316576533.008}.

\bibitem[Efron and Morris(1977)]{efron1977stein}
Bradley Efron and Carl Morris.
\newblock Stein's paradox in statistics.
\newblock \emph{Scientific American}, 236\penalty0 (5):\penalty0 119--127,
  1977.
\newblock \doi{10.1038/scientificamerican0577-119}.

\bibitem[Etkin et~al.(2008)Etkin, Tse, and Wang]{etkin2008gaussian}
Raul~H. Etkin, David N.~C. Tse, and Hua Wang.
\newblock Gaussian interference channel capacity to within one bit.
\newblock \emph{IEEE Transactions on Information Theory}, 54\penalty0
  (12):\penalty0 5534--5562, 2008.
\newblock \doi{10.1109/TIT.2008.2006447}.

\bibitem[Even et~al.(2023)Even, Pesme, Gunasekar, and Flammarion]{even2023sgd}
Mathieu Even, Scott Pesme, Suriya Gunasekar, and Nicolas Flammarion.
\newblock {(S)GD} over diagonal linear networks: Implicit bias, large stepsizes
  and edge of stability.
\newblock In \emph{Advances in Neural Information Processing Systems 36}, 2023.
\newblock URL \url{https://openreview.net/forum?id=uAyElhYKxg}.

\bibitem[Fawzi et~al.(2018)Fawzi, Fawzi, and Frossard]{fawzi2018analysis}
Alhussein Fawzi, Omar Fawzi, and Pascal Frossard.
\newblock Analysis of classifiers' robustness to adversarial perturbations.
\newblock \emph{Machine Learning}, 107:\penalty0 481--508, 2018.

\bibitem[Feldman(2016)]{feldman2016sco}
Vitaly Feldman.
\newblock Generalization of {ERM} in stochastic convex optimization: The
  dimension strikes back.
\newblock In \emph{Advances in Neural Information Processing Systems},
  volume~29, 2016.

\bibitem[Feldman(2020)]{feldman2020memorization}
Vitaly Feldman.
\newblock Does learning require memorization? a short tale about a long tail.
\newblock In \emph{Proceedings of the 52nd Annual ACM SIGACT Symposium on
  Theory of Computing}, STOC 2020, pages 954--959, 2020.
\newblock \doi{10.1145/3357713.3384290}.

\bibitem[Figiel et~al.(1977)Figiel, Lindenstrauss, and
  Milman]{figiel1977dimension}
T.~Figiel, J.~Lindenstrauss, and V.~D. Milman.
\newblock The dimension of almost spherical sections of convex bodies.
\newblock \emph{Acta Mathematica}, 139\penalty0 (1--2):\penalty0 53--94, 1977.

\bibitem[Foster et~al.(2018)Foster, Kale, Luo, Mohri, and
  Sridharan]{foster2018logistic}
Dylan~J. Foster, Satyen Kale, Haipeng Luo, Mehryar Mohri, and Karthik
  Sridharan.
\newblock Logistic regression: The importance of being improper.
\newblock In \emph{Proceedings of the 31st Conference on Learning Theory},
  volume~75 of \emph{Proceedings of Machine Learning Research}, pages 167--208.
  PMLR, 2018.

\bibitem[Freeman(2025)]{freeman2026shrinkage}
Jake Freeman.
\newblock Shrinkage to infinity: Reducing test error by inflating the minimum
  norm interpolator in linear models, 2025.
\newblock arXiv:2510.19206; v2 of April 30, 2026.

\bibitem[Frei et~al.(2023)Frei, Vardi, Bartlett, and
  Srebro]{frei2023doubleedged}
Spencer Frei, Gal Vardi, Peter~L. Bartlett, and Nati Srebro.
\newblock The double-edged sword of implicit bias: Generalization vs.
  robustness in {ReLU} networks.
\newblock In \emph{Advances in Neural Information Processing Systems},
  volume~36, 2023.
\newblock URL
  \url{https://proceedings.neurips.cc/paper_files/paper/2023/hash/1c26c389d60ec419fd24b5fee5b35796-Abstract-Conference.html}.

\bibitem[Frost and Thompson(2000)]{eiv_frost2000}
Chris Frost and Simon~G. Thompson.
\newblock Correcting for regression dilution bias: Comparison of methods for a
  single predictor variable.
\newblock \emph{Journal of the Royal Statistical Society: Series A (Statistics
  in Society)}, 163\penalty0 (2):\penalty0 173--189, 2000.
\newblock \doi{10.1111/1467-985X.00164}.

\bibitem[Fuller(1987)]{fuller1987measurement}
Wayne~A. Fuller.
\newblock \emph{Measurement Error Models}.
\newblock John Wiley \& Sons, New York, 1987.
\newblock ISBN 9780471861874.
\newblock \doi{10.1002/9780470316665}.

\bibitem[Gallager(1968)]{gallager1968information}
Robert~G. Gallager.
\newblock \emph{Information Theory and Reliable Communication}.
\newblock John Wiley \& Sons, New York, 1968.

\bibitem[Ge et~al.(2015)Ge, Huang, Jin, and Yuan]{ge2015escaping}
Rong Ge, Furong Huang, Chi Jin, and Yang Yuan.
\newblock Escaping from saddle points---online stochastic gradient for tensor
  decomposition.
\newblock In \emph{Proceedings of the 28th Conference on Learning Theory},
  pages 797--842, 2015.

\bibitem[Gilmer et~al.(2018)Gilmer, Metz, Faghri, Schoenholz, Raghu,
  Wattenberg, and Goodfellow]{gilmer2018adversarial}
Justin Gilmer, Luke Metz, Fartash Faghri, Samuel~S. Schoenholz, Maithra Raghu,
  Martin Wattenberg, and Ian Goodfellow.
\newblock Adversarial spheres.
\newblock In \emph{International Conference on Learning Representations
  Workshop}, 2018.

\bibitem[Gilmer et~al.(2022)Gilmer, Ghorbani, Garg, Kudugunta, Neyshabur,
  Cardoze, Dahl, Nado, and Firat]{gilmer2022loss}
Justin Gilmer, Behrooz Ghorbani, Ankush Garg, Sneha Kudugunta, Behnam
  Neyshabur, David Cardoze, George~Edward Dahl, Zachary Nado, and Orhan Firat.
\newblock A loss curvature perspective on training instabilities of deep
  learning models.
\newblock In \emph{International Conference on Learning Representations}, 2022.
\newblock URL \url{https://openreview.net/forum?id=OcKMT-36vUs}.

\bibitem[Golub and Van~Loan(1980)]{nr_golub1980}
Gene~H. Golub and Charles~F. Van~Loan.
\newblock An analysis of the total least squares problem.
\newblock \emph{SIAM Journal on Numerical Analysis}, 17\penalty0 (6):\penalty0
  883--893, 1980.
\newblock \doi{10.1137/0717073}.

\bibitem[Golub et~al.(1999)Golub, Hansen, and O'Leary]{nr_golub1999}
Gene~H. Golub, Per~Christian Hansen, and Dianne~P. O'Leary.
\newblock Tikhonov regularization and total least squares.
\newblock \emph{SIAM Journal on Matrix Analysis and Applications}, 21\penalty0
  (1):\penalty0 185--194, 1999.
\newblock \doi{10.1137/S0895479897326432}.

\bibitem[Goodfellow et~al.(2015)Goodfellow, Shlens, and
  Szegedy]{goodfellow2015explaining}
Ian~J. Goodfellow, Jonathon Shlens, and Christian Szegedy.
\newblock Explaining and harnessing adversarial examples.
\newblock In \emph{International Conference on Learning Representations}, 2015.

\bibitem[Goyal et~al.(2017)Goyal, Doll{\'a}r, Girshick, Noordhuis, Wesolowski,
  Kyrola, Tulloch, Jia, and He]{goyal2017accurate}
Priya Goyal, Piotr Doll{\'a}r, Ross Girshick, Pieter Noordhuis, Lukasz
  Wesolowski, Aapo Kyrola, Andrew Tulloch, Yangqing Jia, and Kaiming He.
\newblock Accurate, large minibatch {SGD}: Training {ImageNet} in 1 hour, 2017.
\newblock URL \url{https://arxiv.org/abs/1706.02677}.

\bibitem[Gunasekar et~al.(2018)Gunasekar, Lee, Soudry, and
  Srebro]{gunasekar2018characterizing}
Suriya Gunasekar, Jason Lee, Daniel Soudry, and Nathan Srebro.
\newblock Characterizing implicit bias in terms of optimization geometry.
\newblock In \emph{Proceedings of the 35th International Conference on Machine
  Learning}, volume~80 of \emph{Proceedings of Machine Learning Research},
  pages 1832--1841. PMLR, 2018.
\newblock URL \url{https://proceedings.mlr.press/v80/gunasekar18a.html}.

\bibitem[Hanneke(2016)]{hanneke2016optimal}
Steve Hanneke.
\newblock The optimal sample complexity of {PAC} learning.
\newblock \emph{Journal of Machine Learning Research}, 17\penalty0
  (38):\penalty0 1--15, 2016.

\bibitem[Hao and Zhang(2024)]{hao2024surprising}
Yifan Hao and Tong Zhang.
\newblock The surprising harmfulness of benign overfitting for adversarial
  robustness.
\newblock \emph{arXiv preprint arXiv:2401.12236}, 2024.

\bibitem[HaoChen et~al.(2021)HaoChen, Wei, Lee, and Ma]{haochen2021shape}
Jeff~Z. HaoChen, Colin Wei, Jason~D. Lee, and Tengyu Ma.
\newblock Shape matters: Understanding the implicit bias of the noise
  covariance.
\newblock In \emph{Proceedings of the Thirty Fourth Conference on Learning
  Theory}, volume 134 of \emph{Proceedings of Machine Learning Research}, pages
  2315--2357. PMLR, 2021.

\bibitem[Hardt and Talwar(2010)]{hardt2010geometry}
Moritz Hardt and Kunal Talwar.
\newblock On the geometry of differential privacy.
\newblock In \emph{Proceedings of the Forty-Second ACM Symposium on Theory of
  Computing}, pages 705--714, 2010.
\newblock \doi{10.1145/1806689.1806786}.

\bibitem[Hardt et~al.(2016)Hardt, Recht, and Singer]{hardt2016train}
Moritz Hardt, Benjamin Recht, and Yoram Singer.
\newblock Train faster, generalize better: Stability of stochastic gradient
  descent.
\newblock In \emph{Proceedings of the 33rd International Conference on Machine
  Learning}, pages 1225--1234, 2016.

\bibitem[Hassibi and Hochwald(2003)]{hassibi2003training}
Babak Hassibi and Bertrand~M. Hochwald.
\newblock How much training is needed in multiple-antenna wireless links?
\newblock \emph{IEEE Transactions on Information Theory}, 49\penalty0
  (4):\penalty0 951--963, 2003.
\newblock \doi{10.1109/TIT.2003.809594}.

\bibitem[Hastie et~al.(2022)Hastie, Montanari, Rosset, and
  Tibshirani]{hastie2022surprises}
Trevor Hastie, Andrea Montanari, Saharon Rosset, and Ryan~J. Tibshirani.
\newblock Surprises in high-dimensional ridgeless least squares interpolation.
\newblock \emph{The Annals of Statistics}, 50\penalty0 (2):\penalty0 949--986,
  2022.
\newblock \doi{10.1214/21-AOS2133}.

\bibitem[Herman(2009)]{herman2009fundamentals}
Gabor~T. Herman.
\newblock \emph{Fundamentals of Computerized Tomography: Image Reconstruction
  from Projections}.
\newblock Advances in Pattern Recognition. Springer, London, 2nd edition, 2009.
\newblock \doi{10.1007/978-1-84628-723-7}.
\newblock Algebraic reconstruction techniques and relaxation: Chapter~11,
  pp.~193--216.

\bibitem[Hochreiter and Schmidhuber(1997)]{hochreiter1997flat}
Sepp Hochreiter and J{\"u}rgen Schmidhuber.
\newblock Flat minima.
\newblock \emph{Neural Computation}, 9\penalty0 (1):\penalty0 1--42, 1997.
\newblock \doi{10.1162/neco.1997.9.1.1}.

\bibitem[Horn and Johnson(2013)]{horn2013matrix}
Roger~A. Horn and Charles~R. Johnson.
\newblock \emph{Matrix Analysis}.
\newblock Cambridge University Press, Cambridge, second edition, 2013.

\bibitem[Huang et~al.(2015)Huang, Suh, and Zheng]{huang2015euclideanNetworks}
Shao-Lun Huang, Changho Suh, and Lizhong Zheng.
\newblock Euclidean information theory of networks.
\newblock \emph{IEEE Transactions on Information Theory}, 61\penalty0
  (12):\penalty0 6795--6814, 2015.
\newblock \doi{10.1109/TIT.2015.2484066}.

\bibitem[Huang et~al.(2024)Huang, Makur, Wornell, and
  Zheng]{huang2024universal}
Shao-Lun Huang, Anuran Makur, Gregory~W. Wornell, and Lizhong Zheng.
\newblock Universal features for high-dimensional learning and inference:
  Information theoretic and geometric perspectives.
\newblock \emph{Foundations and Trends in Communications and Information
  Theory}, 21\penalty0 (1--2):\penalty0 1--299, 2024.
\newblock \doi{10.1561/0100000107}.

\bibitem[Hussain and Fang(2025)]{hussain2025optimizer}
Ayana Hussain and Ricky Fang.
\newblock Optimizer dynamics at the edge of stability with differential
  privacy, 2025.
\newblock URL \url{https://arxiv.org/abs/2512.19019}.

\bibitem[Ilyas et~al.(2019)Ilyas, Santurkar, Tsipras, Engstrom, Tran, and
  Madry]{ilyas2019adversarial}
Andrew Ilyas, Shibani Santurkar, Dimitris Tsipras, Logan Engstrom, Brandon
  Tran, and Aleksander Madry.
\newblock Adversarial examples are not bugs, they are features.
\newblock In \emph{Advances in Neural Information Processing Systems},
  volume~32, 2019.

\bibitem[Jain et~al.(2018)Jain, Kakade, Kidambi, Netrapalli, and
  Sidford]{jain2018parallelizing}
Prateek Jain, Sham~M. Kakade, Rahul Kidambi, Praneeth Netrapalli, and Aaron
  Sidford.
\newblock Parallelizing stochastic gradient descent for least squares
  regression: Mini-batching, averaging, and model misspecification.
\newblock \emph{Journal of Machine Learning Research}, 18\penalty0
  (223):\penalty0 1--42, 2018.
\newblock URL \url{http://jmlr.org/papers/v18/16-595.html}.

\bibitem[James and Stein(1961)]{james1961estimation}
W.~James and Charles Stein.
\newblock Estimation with quadratic loss.
\newblock In \emph{Proceedings of the Fourth Berkeley Symposium on Mathematical
  Statistics and Probability}, volume~1, pages 361--379. University of
  California Press, 1961.

\bibitem[Jastrzebski et~al.(2017)Jastrzebski, Kenton, Arpit, Ballas, Fischer,
  Bengio, and Storkey]{jastrzebski2017three}
Stanislaw Jastrzebski, Zachary Kenton, Devansh Arpit, Nicolas Ballas, Asja
  Fischer, Yoshua Bengio, and Amos Storkey.
\newblock Three factors influencing minima in {SGD}, 2017.
\newblock URL \url{https://arxiv.org/abs/1711.04623}.

\bibitem[Javanmard et~al.(2020)Javanmard, Soltanolkotabi, and
  Hassani]{javanmard2020precise}
Adel Javanmard, Mahdi Soltanolkotabi, and Hamed Hassani.
\newblock Precise tradeoffs in adversarial training for linear regression.
\newblock In \emph{Proceedings of Thirty Third Conference on Learning Theory},
  volume 125 of \emph{Proceedings of Machine Learning Research}, pages
  2034--2078, 2020.

\bibitem[Jaynes(1957)]{jaynes1957information}
Edwin~T. Jaynes.
\newblock Information theory and statistical mechanics.
\newblock \emph{Physical Review}, 106\penalty0 (4):\penalty0 620--630, 1957.
\newblock \doi{10.1103/PhysRev.106.620}.

\bibitem[Jeffreys(1946)]{jeffreys1946invariant}
Harold Jeffreys.
\newblock An invariant form for the prior probability in estimation problems.
\newblock \emph{Proceedings of the Royal Society of London. Series A,
  Mathematical and Physical Sciences}, 186\penalty0 (1007):\penalty0 453--461,
  1946.
\newblock \doi{10.1098/rspa.1946.0056}.

\bibitem[Jerrum et~al.(1986)Jerrum, Valiant, and Vazirani]{jerrum1986random}
Mark~R. Jerrum, Leslie~G. Valiant, and Vijay~V. Vazirani.
\newblock Random generation of combinatorial structures from a uniform
  distribution.
\newblock \emph{Theoretical Computer Science}, 43:\penalty0 169--188, 1986.

\bibitem[Jin et~al.(2017)Jin, Ge, Netrapalli, Kakade, and
  Jordan]{jin2017escape}
Chi Jin, Rong Ge, Praneeth Netrapalli, Sham~M. Kakade, and Michael~I. Jordan.
\newblock How to escape saddle points efficiently.
\newblock In \emph{Proceedings of the 34th International Conference on Machine
  Learning}, pages 1724--1732, 2017.

\bibitem[Johnstone(2001)]{johnstone2001distribution}
Iain~M. Johnstone.
\newblock On the distribution of the largest eigenvalue in principal components
  analysis.
\newblock \emph{The Annals of Statistics}, 29\penalty0 (2):\penalty0 295--327,
  2001.

\bibitem[Kaczmarz(1937)]{kaczmarz1937angenaeherte}
Stefan Kaczmarz.
\newblock Angen{\"a}herte {A}ufl{\"o}sung von {S}ystemen linearer
  {G}leichungen.
\newblock \emph{Bulletin International de l'Acad{\'e}mie Polonaise des Sciences
  et des Lettres, Classe A}, 35:\penalty0 355--357, 1937.

\bibitem[Kale et~al.(2021)Kale, Sekhari, and Sridharan]{sekhari2021sgd}
Satyen Kale, Ayush Sekhari, and Karthik Sridharan.
\newblock {SGD}: The role of implicit regularization, batch-size and multiple
  epochs.
\newblock In \emph{Advances in Neural Information Processing Systems},
  volume~34, pages 27422--27433, 2021.

\bibitem[Kamath et~al.(2022)Kamath, Mouzakis, and Singhal]{kamath2022new}
Gautam Kamath, Argyris Mouzakis, and Vikrant Singhal.
\newblock New lower bounds for private estimation and a generalized
  fingerprinting lemma.
\newblock In \emph{Advances in Neural Information Processing Systems},
  volume~35, 2022.
\newblock URL
  \url{https://proceedings.neurips.cc/paper_files/paper/2022/hash/9a6b278218966499194491f55ccf8b75-Abstract-Conference.html}.

\bibitem[Kashin(1977)]{kashin1977sections}
B.~S. Kashin.
\newblock The widths of certain finite-dimensional sets and classes of smooth
  functions.
\newblock \emph{Izvestiya Akademii Nauk SSSR. Seriya Matematicheskaya},
  41:\penalty0 334--351, 1977.

\bibitem[Kass and Wasserman(1996)]{kass1996selection}
Robert~E. Kass and Larry Wasserman.
\newblock The selection of prior distributions by formal rules.
\newblock \emph{Journal of the American Statistical Association}, 91\penalty0
  (435):\penalty0 1343--1370, 1996.

\bibitem[Kattis and Nikolov(2017)]{kattis2017lower}
Assimakis Kattis and Aleksandar Nikolov.
\newblock Lower bounds for differential privacy from gaussian width.
\newblock In \emph{33rd International Symposium on Computational Geometry},
  volume~77 of \emph{Leibniz International Proceedings in Informatics}, pages
  45:1--45:16, 2017.
\newblock \doi{10.4230/LIPIcs.SoCG.2017.45}.

\bibitem[Keskar et~al.(2017)Keskar, Mudigere, Nocedal, Smelyanskiy, and
  Tang]{keskar2017largebatch}
Nitish~Shirish Keskar, Dheevatsa Mudigere, Jorge Nocedal, Mikhail Smelyanskiy,
  and Ping Tak~Peter Tang.
\newblock On large-batch training for deep learning: Generalization gap and
  sharp minima.
\newblock In \emph{International Conference on Learning Representations}, 2017.
\newblock URL \url{https://openreview.net/forum?id=H1oyRlYgg}.

\bibitem[Khalil et~al.(2026)Khalil, Blanco-Justicia, Jebreel, and
  Domingo-Ferrer]{khalil2025membership}
Mona Khalil, Alberto Blanco-Justicia, Najeeb Jebreel, and Josep Domingo-Ferrer.
\newblock Membership inference attacks beyond overfitting.
\newblock In \emph{Computer Security. {ESORICS} 2025 International Workshops},
  volume 16231 of \emph{Lecture Notes in Computer Science}, pages 32--48.
  Springer, 2026.
\newblock \doi{10.1007/978-3-032-16089-8_3}.
\newblock arXiv:2511.16792.

\bibitem[Kiefer(1957)]{kiefer1957invariance}
Jack Kiefer.
\newblock Invariance, minimax sequential estimation, and continuous time
  processes.
\newblock \emph{The Annals of Mathematical Statistics}, 28\penalty0
  (3):\penalty0 573--601, 1957.
\newblock \doi{10.1214/aoms/1177706874}.

\bibitem[Kimeldorf and Wahba(1971)]{kimeldorf1971some}
George~S. Kimeldorf and Grace Wahba.
\newblock Some results on {Tchebycheffian} spline functions.
\newblock \emph{Journal of Mathematical Analysis and Applications}, 33\penalty0
  (1):\penalty0 82--95, 1971.
\newblock \doi{10.1016/0022-247X(71)90184-3}.

\bibitem[Kish(1965)]{kish1965survey}
Leslie Kish.
\newblock \emph{Survey Sampling}.
\newblock John Wiley \& Sons, New York, 1965.

\bibitem[Kobak et~al.(2020)Kobak, Lomond, and Sanchez]{nr_kobak2020}
Dmitry Kobak, Jonathan Lomond, and Benoit Sanchez.
\newblock The optimal ridge penalty for real-world high-dimensional data can be
  zero or negative due to the implicit ridge regularization.
\newblock \emph{Journal of Machine Learning Research}, 21\penalty0
  (169):\penalty0 1--16, 2020.

\bibitem[Koehler et~al.(2021)Koehler, Zhou, Sutherland, and
  Srebro]{koehler2021uniform}
Frederic Koehler, Lijia Zhou, Danica~J. Sutherland, and Nathan Srebro.
\newblock Uniform convergence of interpolators: Gaussian width, norm bounds and
  benign overfitting.
\newblock In \emph{Advances in Neural Information Processing Systems 34}, 2021.

\bibitem[Koren et~al.(2022)Koren, Livni, Mansour, and Sherman]{koren2022benign}
Tomer Koren, Roi Livni, Yishay Mansour, and Uri Sherman.
\newblock Benign underfitting of stochastic gradient descent.
\newblock In \emph{Advances in Neural Information Processing Systems},
  volume~35, pages 19605--19617, 2022.
\newblock arXiv:2202.13361.

\bibitem[Kotel'nikov(1959)]{kotelnikov1959optimum}
V.~A. Kotel'nikov.
\newblock \emph{The Theory of Optimum Noise Immunity}.
\newblock McGraw--Hill, New York, 1959.
\newblock {English} translation of the Russian edition; based on the author's
  1947 doctoral dissertation.

\bibitem[Larsen(2023)]{larsen2023bagging}
Kasper~Green Larsen.
\newblock Bagging is an optimal {PAC} learner.
\newblock In \emph{Proceedings of the Thirty Sixth Conference on Learning
  Theory}, volume 195 of \emph{Proceedings of Machine Learning Research}, pages
  450--468. PMLR, 2023.

\bibitem[Lee and Jang(2023)]{leejang2023new}
Sungyoon Lee and Cheongjae Jang.
\newblock A new characterization of the edge of stability based on a sharpness
  measure aware of batch gradient distribution.
\newblock In \emph{International Conference on Learning Representations}, 2023.
\newblock URL \url{https://openreview.net/forum?id=bH-kCY6LdKg}.

\bibitem[Lewkowycz et~al.(2020)Lewkowycz, Bahri, Dyer, Sohl-Dickstein, and
  Gur-Ari]{lewkowycz2020large}
Aitor Lewkowycz, Yasaman Bahri, Ethan Dyer, Jascha Sohl-Dickstein, and Guy
  Gur-Ari.
\newblock The large learning rate phase of deep learning: The catapult
  mechanism, 2020.
\newblock URL \url{https://arxiv.org/abs/2003.02218}.

\bibitem[Li et~al.(2025)Li, Pan, Lyu, and Li]{li2025featureAveraging}
Binghui Li, Zhixuan Pan, Kaifeng Lyu, and Jian Li.
\newblock Feature averaging: An implicit bias of gradient descent leading to
  non-robustness in neural networks.
\newblock In \emph{International Conference on Learning Representations}, 2025.
\newblock URL \url{https://openreview.net/forum?id=zPHra4V5Mc}.

\bibitem[Li et~al.(2019)Li, Wei, and Ma]{li2019towards}
Yuanzhi Li, Colin Wei, and Tengyu Ma.
\newblock Towards explaining the regularization effect of initial large
  learning rate in training neural networks.
\newblock In \emph{Advances in Neural Information Processing Systems 32}, 2019.

\bibitem[Lin et~al.(2024)Lin, Wu, Kakade, Bartlett, and Lee]{lin2024scaling}
Licong Lin, Jingfeng Wu, Sham~M. Kakade, Peter~L. Bartlett, and Jason~D. Lee.
\newblock Scaling laws in linear regression: Compute, parameters, and data.
\newblock In \emph{Advances in Neural Information Processing Systems},
  volume~37, 2024.
\newblock arXiv:2406.08466.

\bibitem[Lin et~al.(2025)Lin, Wu, and Bartlett]{lin2025datareuse}
Licong Lin, Jingfeng Wu, and Peter~L. Bartlett.
\newblock Improved scaling laws in linear regression via data reuse.
\newblock In \emph{Advances in Neural Information Processing Systems}, 2025.
\newblock arXiv:2506.08415.

\bibitem[Livni(2024{\natexlab{a}})]{livni2024false}
Roi Livni.
\newblock Making progress based on false discoveries.
\newblock In \emph{15th Innovations in Theoretical Computer Science Conference
  (ITCS 2024)}, volume 287 of \emph{Leibniz International Proceedings in
  Informatics (LIPIcs)}, pages 76:1--76:18, 2024{\natexlab{a}}.
\newblock \doi{10.4230/LIPIcs.ITCS.2024.76}.

\bibitem[Livni(2024{\natexlab{b}})]{livni2024gd}
Roi Livni.
\newblock The sample complexity of gradient descent in stochastic convex
  optimization.
\newblock In \emph{Advances in Neural Information Processing Systems},
  volume~37, 2024{\natexlab{b}}.
\newblock arXiv:2404.04931.

\bibitem[Loog et~al.(2020)Loog, Viering, Mey, Krijthe, and
  Tax]{loog2020prehistory}
Marco Loog, Tom Viering, Alexander Mey, Jesse~H. Krijthe, and David M.~J. Tax.
\newblock A brief prehistory of double descent.
\newblock \emph{Proceedings of the National Academy of Sciences}, 117\penalty0
  (20):\penalty0 10625--10626, 2020.
\newblock \doi{10.1073/pnas.2001875117}.

\bibitem[Lu et~al.(2024)Lu, Wu, Yang, and Zou]{lu2024benignoscillation}
Miao Lu, Beining Wu, Xiaodong Yang, and Difan Zou.
\newblock Benign oscillation of stochastic gradient descent with large learning
  rate.
\newblock In \emph{International Conference on Learning Representations}, pages
  45874--45935, 2024.
\newblock URL
  \url{https://proceedings.iclr.cc/paper_files/paper/2024/hash/c87df228922b47a8cd0f3303410e88ca-Abstract-Conference.html}.

\bibitem[Lugosi and Mendelson(2019{\natexlab{a}})]{lugosi2019mean}
G{\'a}bor Lugosi and Shahar Mendelson.
\newblock Mean estimation and regression under heavy-tailed distributions: A
  survey.
\newblock \emph{Foundations of Computational Mathematics}, 19\penalty0
  (5):\penalty0 1145--1190, 2019{\natexlab{a}}.

\bibitem[Lugosi and Mendelson(2019{\natexlab{b}})]{lugosi2019subgaussian}
G{\'a}bor Lugosi and Shahar Mendelson.
\newblock Sub-{G}aussian estimators of the mean of a random vector.
\newblock \emph{The Annals of Statistics}, 47\penalty0 (2):\penalty0 783--794,
  2019{\natexlab{b}}.

\bibitem[Lugosi and Mendelson(2019{\natexlab{c}})]{lugosi2019tournaments}
G{\'a}bor Lugosi and Shahar Mendelson.
\newblock Regularization, sparse recovery, and median-of-means tournaments.
\newblock \emph{Bernoulli}, 25\penalty0 (3):\penalty0 2075--2106,
  2019{\natexlab{c}}.

\bibitem[Madry et~al.(2018)Madry, Makelov, Schmidt, Tsipras, and
  Vladu]{madry2018towards}
Aleksander Madry, Aleksandar Makelov, Ludwig Schmidt, Dimitris Tsipras, and
  Adrian Vladu.
\newblock Towards deep learning models resistant to adversarial attacks.
\newblock In \emph{International Conference on Learning Representations}, 2018.

\bibitem[Malladi et~al.(2022)Malladi, Lyu, Panigrahi, and
  Arora]{malladi2022sdes}
Sadhika Malladi, Kaifeng Lyu, Abhishek Panigrahi, and Sanjeev Arora.
\newblock On the {SDE}s and scaling rules for adaptive gradient algorithms.
\newblock In \emph{Advances in Neural Information Processing Systems},
  volume~35, 2022.
\newblock URL
  \url{https://papers.neurips.cc/paper_files/paper/2022/hash/32ac710102f0620d0f28d5d05a44fe08-Abstract-Conference.html}.

\bibitem[Mallinar et~al.(2022)Mallinar, Simon, Abedsoltan, Pandit, Belkin, and
  Nakkiran]{mallinar2022taxonomy}
Neil Mallinar, James~B. Simon, Amirhesam Abedsoltan, Parthe Pandit, Misha
  Belkin, and Preetum Nakkiran.
\newblock Benign, tempered, or catastrophic: Toward a refined taxonomy of
  overfitting.
\newblock In \emph{Advances in Neural Information Processing Systems 35}, 2022.
\newblock URL
  \url{https://proceedings.neurips.cc/paper_files/paper/2022/hash/08342dc6ab69f23167b4123086ad4d38-Abstract-Conference.html}.

\bibitem[Mandt et~al.(2017)Mandt, Hoffman, and Blei]{mandt2017sgd}
Stephan Mandt, Matthew~D. Hoffman, and David~M. Blei.
\newblock Stochastic gradient descent as approximate {Bayesian} inference.
\newblock \emph{Journal of Machine Learning Research}, 18\penalty0
  (134):\penalty0 1--35, 2017.
\newblock URL \url{https://jmlr.org/papers/v18/17-214.html}.

\bibitem[Markovsky and Van~Huffel(2007)]{nr_markovsky2007}
Ivan Markovsky and Sabine Van~Huffel.
\newblock Overview of total least-squares methods.
\newblock \emph{Signal Processing}, 87\penalty0 (10):\penalty0 2283--2302,
  2007.
\newblock \doi{10.1016/j.sigpro.2007.04.004}.

\bibitem[McCandlish et~al.(2018)McCandlish, Kaplan, Amodei, and {OpenAI Dota
  Team}]{mccandlish2018empirical}
Sam McCandlish, Jared Kaplan, Dario Amodei, and {OpenAI Dota Team}.
\newblock An empirical model of large-batch training, 2018.
\newblock URL \url{https://arxiv.org/abs/1812.06162}.

\bibitem[{McRae} et~al.(2022){McRae}, Karnik, Davenport, and
  Muthukumar]{mcrae2022structured}
Andrew~D. {McRae}, Santhosh Karnik, Mark Davenport, and Vidya~K. Muthukumar.
\newblock Harmless interpolation in regression and classification with
  structured features.
\newblock In \emph{Proceedings of The 25th International Conference on
  Artificial Intelligence and Statistics}, volume 151 of \emph{Proceedings of
  Machine Learning Research}, pages 5853--5875. PMLR, 2022.
\newblock URL \url{https://proceedings.mlr.press/v151/mcrae22a.html}.

\bibitem[M{\'e}dard and Gallager(2002)]{medard2002bandwidth}
Muriel M{\'e}dard and Robert~G. Gallager.
\newblock Bandwidth scaling for fading multipath channels.
\newblock \emph{IEEE Transactions on Information Theory}, 48\penalty0
  (4):\penalty0 840--852, 2002.
\newblock \doi{10.1109/18.992769}.

\bibitem[Meng et~al.(2024)Meng, Orvieto, Cao, and De~Sa]{meng2024nonseparable}
Si~Yi Meng, Antonio Orvieto, Daniel~Yiming Cao, and Christopher De~Sa.
\newblock Gradient descent on logistic regression with non-separable data and
  large step sizes, 2024.
\newblock URL \url{https://arxiv.org/abs/2406.05033v3}.
\newblock arXiv:2406.05033; v3 of September 2, 2026.

\bibitem[Micchelli and Rivlin(1977)]{micchelli1977optimalrecovery}
Charles~A. Micchelli and Theodore~J. Rivlin.
\newblock A survey of optimal recovery.
\newblock In Charles~A. Micchelli and Theodore~J. Rivlin, editors,
  \emph{Optimal Estimation in Approximation Theory}, pages 1--54. Plenum Press,
  New York, 1977.

\bibitem[Min and Vidal(2024)]{min2024implicit}
Hancheng Min and Ren{\'e} Vidal.
\newblock Can implicit bias imply adversarial robustness?
\newblock In \emph{Proceedings of the 41st International Conference on Machine
  Learning}, volume 235 of \emph{Proceedings of Machine Learning Research},
  pages 35687--35718. PMLR, 2024.
\newblock URL \url{https://proceedings.mlr.press/v235/min24a.html}.

\bibitem[Minsker(2015)]{minsker2015geometric}
Stanislav Minsker.
\newblock Geometric median and robust estimation in {B}anach spaces.
\newblock \emph{Bernoulli}, 21\penalty0 (4):\penalty0 2308--2335, 2015.

\bibitem[Mishchenko et~al.(2020)Mishchenko, Khaled, and
  Richt{\'a}rik]{mishchenko2020random}
Konstantin Mishchenko, Ahmed Khaled, and Peter Richt{\'a}rik.
\newblock Random reshuffling: Simple analysis with vast improvements.
\newblock In \emph{Advances in Neural Information Processing Systems},
  volume~33, pages 17309--17320, 2020.

\bibitem[Mitchell(1982)]{mitchell1982generalization}
Tom~M. Mitchell.
\newblock Generalization as search.
\newblock \emph{Artificial Intelligence}, 18\penalty0 (2):\penalty0 203--226,
  1982.
\newblock \doi{10.1016/0004-3702(82)90040-6}.

\bibitem[Montasser et~al.(2019)Montasser, Hanneke, and Srebro]{montasser2019vc}
Omar Montasser, Steve Hanneke, and Nathan Srebro.
\newblock {VC} classes are adversarially robustly learnable, but only
  improperly.
\newblock In \emph{Proceedings of the Thirty-Second Conference on Learning
  Theory}, volume~99 of \emph{Proceedings of Machine Learning Research}, pages
  2512--2530. PMLR, 2019.

\bibitem[Muthukumar et~al.(2019)Muthukumar, Vodrahalli, and
  Sahai]{muthukumar2019harmless}
Vidya Muthukumar, Kailas Vodrahalli, and Anant Sahai.
\newblock Harmless interpolation of noisy data in regression.
\newblock In \emph{IEEE International Symposium on Information Theory}, pages
  2299--2303, 2019.
\newblock \doi{10.1109/ISIT.2019.8849614}.

\bibitem[Muthukumar et~al.(2020)Muthukumar, Vodrahalli, Subramanian, and
  Sahai]{muthukumar2020harmless}
Vidya Muthukumar, Kailas Vodrahalli, Vignesh Subramanian, and Anant Sahai.
\newblock Harmless interpolation of noisy data in regression.
\newblock \emph{IEEE Journal on Selected Areas in Information Theory},
  1\penalty0 (1):\penalty0 67--83, 2020.
\newblock \doi{10.1109/JSAIT.2020.2984716}.

\bibitem[Muthukumar et~al.(2021)Muthukumar, Narang, Subramanian, Belkin, Hsu,
  and Sahai]{muthukumar2021classification}
Vidya Muthukumar, Adhyyan Narang, Vignesh Subramanian, Mikhail Belkin, Daniel
  Hsu, and Anant Sahai.
\newblock Classification vs regression in overparameterized regimes: Does the
  loss function matter?
\newblock \emph{Journal of Machine Learning Research}, 22\penalty0
  (222):\penalty0 1--69, 2021.
\newblock URL \url{https://jmlr.org/papers/v22/20-603.html}.

\bibitem[Nacson et~al.(2022)Nacson, Ravichandran, Srebro, and
  Soudry]{nacson2022implicit}
Mor~Shpigel Nacson, Kavya Ravichandran, Nathan Srebro, and Daniel Soudry.
\newblock Implicit bias of the step size in linear diagonal neural networks.
\newblock In \emph{Proceedings of the 39th International Conference on Machine
  Learning}, volume 162 of \emph{Proceedings of Machine Learning Research},
  pages 16270--16295. PMLR, 2022.
\newblock URL \url{https://proceedings.mlr.press/v162/nacson22a.html}.

\bibitem[Nagarajan and Kolter(2019)]{nagarajan2019uniform}
Vaishnavh Nagarajan and J.~Zico Kolter.
\newblock Uniform convergence may be unable to explain generalization in deep
  learning.
\newblock In \emph{Advances in Neural Information Processing Systems},
  volume~32, 2019.

\bibitem[Nakkiran(2020)]{nakkiran2020annealing}
Preetum Nakkiran.
\newblock Learning rate annealing can provably help generalization, even for
  convex problems, 2020.
\newblock URL \url{https://www.opt-ml.org/papers/2020/paper_35.pdf}.
\newblock arXiv:2005.07360; presented at {OPT2020}, the 12th Annual Workshop on
  Optimization for Machine Learning at {NeurIPS}.

\bibitem[Nakkiran et~al.(2020)Nakkiran, Kaplun, Bansal, Yang, Barak, and
  Sutskever]{nakkiran2020deep}
Preetum Nakkiran, Gal Kaplun, Yamini Bansal, Tristan Yang, Boaz Barak, and Ilya
  Sutskever.
\newblock Deep double descent: Where bigger models and more data hurt.
\newblock In \emph{International Conference on Learning Representations}, 2020.
\newblock URL \url{https://openreview.net/forum?id=B1g5sA4twr}.

\bibitem[Nakkiran et~al.(2021)Nakkiran, Neyshabur, and
  Sedghi]{nakkiran2021deepbootstrap}
Preetum Nakkiran, Behnam Neyshabur, and Hanie Sedghi.
\newblock The deep bootstrap framework: Good online learners are good offline
  generalizers.
\newblock In \emph{International Conference on Learning Representations}, 2021.

\bibitem[Narang et~al.(2021)Narang, Muthukumar, and
  Sahai]{narang2021classification}
Adhyyan Narang, Vidya Muthukumar, and Anant Sahai.
\newblock Classification and adversarial examples in an overparameterized
  linear model: A signal processing perspective, 2021.
\newblock URL \url{https://arxiv.org/abs/2109.13215}.

\bibitem[Needell et~al.(2014)Needell, Srebro, and Ward]{needell2014kaczmarz}
Deanna Needell, Nathan Srebro, and Rachel Ward.
\newblock Stochastic gradient descent, weighted sampling, and the randomized
  {K}aczmarz algorithm.
\newblock In \emph{Advances in Neural Information Processing Systems},
  volume~27, 2014.

\bibitem[Nemirovski et~al.(2009)Nemirovski, Juditsky, Lan, and
  Shapiro]{nemirovski2009robust}
Arkadi Nemirovski, Anatoli Juditsky, Guanghui Lan, and Alexander Shapiro.
\newblock Robust stochastic approximation approach to stochastic programming.
\newblock \emph{SIAM Journal on Optimization}, 19\penalty0 (4):\penalty0
  1574--1609, 2009.
\newblock \doi{10.1137/070704277}.

\bibitem[Nemirovsky and Yudin(1983)]{nemirovsky1983problem}
Arkadi~S. Nemirovsky and David~B. Yudin.
\newblock \emph{Problem Complexity and Method Efficiency in Optimization}.
\newblock Wiley-Interscience Series in Discrete Mathematics. John Wiley \&
  Sons, New York, 1983.

\bibitem[Oja(1982)]{oja1982simplified}
Erkki Oja.
\newblock Simplified neuron model as a principal component analyzer.
\newblock \emph{Journal of Mathematical Biology}, 15\penalty0 (3):\penalty0
  267--273, 1982.
\newblock \doi{10.1007/BF00275687}.

\bibitem[Paul et~al.(2021)Paul, Ganguli, and Dziugaite]{paul2021dataDiet}
Mansheej Paul, Surya Ganguli, and Gintare~Karolina Dziugaite.
\newblock Deep learning on a data diet: Finding important examples early in
  training.
\newblock In \emph{Advances in Neural Information Processing Systems 34}, 2021.

\bibitem[Pesme et~al.(2021)Pesme, Pillaud-Vivien, and
  Flammarion]{pesme2021implicit}
Scott Pesme, Loucas Pillaud-Vivien, and Nicolas Flammarion.
\newblock Implicit bias of {SGD} for diagonal linear networks: a provable
  benefit of stochasticity.
\newblock In \emph{Advances in Neural Information Processing Systems 34}, 2021.

\bibitem[Pillaud-Vivien et~al.(2018)Pillaud-Vivien, Rudi, and
  Bach]{pillaud2018statistical}
Loucas Pillaud-Vivien, Alessandro Rudi, and Francis Bach.
\newblock Statistical optimality of stochastic gradient descent on hard
  learning problems through multiple passes.
\newblock In \emph{Advances in Neural Information Processing Systems},
  volume~31, 2018.

\bibitem[Polyak(1987)]{polyak1987introduction}
Boris~T. Polyak.
\newblock \emph{Introduction to Optimization}.
\newblock Optimization Software, Inc., New York, 1987.

\bibitem[Polyak and Juditsky(1992)]{polyak1992acceleration}
Boris~T. Polyak and Anatoli~B. Juditsky.
\newblock Acceleration of stochastic approximation by averaging.
\newblock \emph{SIAM Journal on Control and Optimization}, 30\penalty0
  (4):\penalty0 838--855, 1992.

\bibitem[Raghunathan et~al.(2020)Raghunathan, Xie, Yang, Duchi, and
  Liang]{raghunathan2020understanding}
Aditi Raghunathan, Sang~Michael Xie, Fanny Yang, John Duchi, and Percy Liang.
\newblock Understanding and mitigating the tradeoff between robustness and
  accuracy.
\newblock In \emph{Proceedings of the 37th International Conference on Machine
  Learning}, 2020.

\bibitem[Ranade and Sahai(2026{\natexlab{a}})]{ranade2026multispike}
Gireeja Ranade and Anant Sahai.
\newblock The fourth quadrant: Benign misfitting across multiple spikes,
  2026{\natexlab{a}}.
\newblock In preparation.

\bibitem[Ranade and Sahai(2026{\natexlab{b}})]{ranade2026oja}
Gireeja Ranade and Anant Sahai.
\newblock The fourth quadrant: One-pass direction recovery at the {BBP} scale,
  2026{\natexlab{b}}.
\newblock In preparation.

\bibitem[Ranade and Sahai(2026{\natexlab{c}})]{ranade2026samedistribution}
Gireeja Ranade and Anant Sahai.
\newblock The fourth quadrant: Benign misfitting under same-distribution
  sampling, 2026{\natexlab{c}}.
\newblock In preparation.

\bibitem[Ren et~al.(2024)Ren, Ma, and Ying]{ren2024understanding}
Yinuo Ren, Chao Ma, and Lexing Ying.
\newblock Understanding the generalization benefits of late learning rate
  decay.
\newblock In \emph{Proceedings of The 27th International Conference on
  Artificial Intelligence and Statistics}, volume 238 of \emph{Proceedings of
  Machine Learning Research}, pages 4465--4473. PMLR, 2024.
\newblock URL \url{https://proceedings.mlr.press/v238/ren24c.html}.

\bibitem[Ribeiro and Sch{\"o}n(2023)]{ribeiro2023overparameterized}
Ant{\^o}nio~H. Ribeiro and Thomas~B. Sch{\"o}n.
\newblock Overparameterized linear regression under adversarial attacks.
\newblock \emph{IEEE Transactions on Signal Processing}, 71:\penalty0 601--614,
  2023.
\newblock \doi{10.1109/TSP.2023.3246228}.

\bibitem[Robbins and Monro(1951)]{robbins1951stochastic}
Herbert Robbins and Sutton Monro.
\newblock A stochastic approximation method.
\newblock \emph{The Annals of Mathematical Statistics}, 22\penalty0
  (3):\penalty0 400--407, 1951.
\newblock \doi{10.1214/aoms/1177729586}.

\bibitem[Roe and White(1961)]{sp_roe1961}
G.~M. Roe and G.~M. White.
\newblock Probability density functions for correlators with noisy reference
  signals.
\newblock \emph{IRE Transactions on Information Theory}, 7\penalty0
  (1):\penalty0 13--18, 1961.
\newblock \doi{10.1109/TIT.1961.1057613}.

\bibitem[Ruppert(1988)]{ruppert1988efficient}
David Ruppert.
\newblock Efficient estimations from a slowly convergent {R}obbins--{M}onro
  process.
\newblock Technical Report 781, Cornell University, School of Operations
  Research and Industrial Engineering, 1988.

\bibitem[Rushforth(1964)]{sp_rush1964}
C.~K. Rushforth.
\newblock Transmitted-reference techniques for random or unknown channels.
\newblock \emph{IEEE Transactions on Information Theory}, 10\penalty0
  (1):\penalty0 39--42, 1964.
\newblock \doi{10.1109/TIT.1964.1053641}.

\bibitem[Russo and Zou(2016)]{russo2016controlling}
Daniel Russo and James Zou.
\newblock Controlling bias in adaptive data analysis using information theory.
\newblock In \emph{Proceedings of the 19th International Conference on
  Artificial Intelligence and Statistics}, volume~51 of \emph{Proceedings of
  Machine Learning Research}, pages 1232--1240. PMLR, 2016.

\bibitem[Sagawa et~al.(2020)Sagawa, Koh, Hashimoto, and
  Liang]{sagawa2020groupDRO}
Shiori Sagawa, Pang~Wei Koh, Tatsunori~B. Hashimoto, and Percy Liang.
\newblock Distributionally robust neural networks for group shifts: On the
  importance of regularization for worst-case generalization.
\newblock In \emph{International Conference on Learning Representations}, 2020.

\bibitem[Sankar et~al.(2013)Sankar, Rajagopalan, and Poor]{sankar2013utility}
Lalitha Sankar, S.~Raj Rajagopalan, and H.~Vincent Poor.
\newblock Utility-privacy tradeoffs in databases: An information-theoretic
  approach.
\newblock \emph{IEEE Transactions on Information Forensics and Security},
  8\penalty0 (6):\penalty0 838--852, 2013.

\bibitem[Saxe et~al.(2014)Saxe, McClelland, and Ganguli]{saxe2014exact}
Andrew~M. Saxe, James~L. McClelland, and Surya Ganguli.
\newblock Exact solutions to the nonlinear dynamics of learning in deep linear
  neural networks.
\newblock In \emph{International Conference on Learning Representations}, 2014.

\bibitem[Scetbon and Dohmatob(2023)]{scetbon2023robust}
Meyer Scetbon and Elvis Dohmatob.
\newblock Robust linear regression: Gradient-descent, early-stopping, and
  beyond.
\newblock In \emph{Proceedings of the 26th International Conference on
  Artificial Intelligence and Statistics}, volume 206 of \emph{Proceedings of
  Machine Learning Research}. PMLR, 2023.

\bibitem[Schliserman et~al.(2025)Schliserman, Sherman, and
  Koren]{schliserman2025dimension}
Matan Schliserman, Uri Sherman, and Tomer Koren.
\newblock The dimension strikes back with gradients: Generalization of gradient
  methods in stochastic convex optimization.
\newblock In \emph{Proceedings of the 36th International Conference on
  Algorithmic Learning Theory}, volume 272 of \emph{Proceedings of Machine
  Learning Research}, pages 1041--1107. PMLR, 2025.
\newblock arXiv:2401.12058.

\bibitem[Schmidt et~al.(2018)Schmidt, Santurkar, Tsipras, Talwar, and
  Madry]{schmidt2018adversarial}
Ludwig Schmidt, Shibani Santurkar, Dimitris Tsipras, Kunal Talwar, and
  Aleksander Madry.
\newblock Adversarially robust generalization requires more data.
\newblock In \emph{Advances in Neural Information Processing Systems},
  volume~31, 2018.

\bibitem[Sch{\"o}lkopf et~al.(2001)Sch{\"o}lkopf, Herbrich, and
  Smola]{schoelkopf2001generalized}
Bernhard Sch{\"o}lkopf, Ralf Herbrich, and Alex~J. Smola.
\newblock A generalized representer theorem.
\newblock In \emph{Computational Learning Theory}, volume 2111 of \emph{Lecture
  Notes in Computer Science}, pages 416--426, Berlin, Heidelberg, 2001.
  Springer.
\newblock \doi{10.1007/3-540-44581-1_27}.

\bibitem[Shah et~al.(2020)Shah, Tamuly, Raghunathan, Jain, and
  Netrapalli]{shah2020pitfalls}
Harshay Shah, Kaustav Tamuly, Aditi Raghunathan, Prateek Jain, and Praneeth
  Netrapalli.
\newblock The pitfalls of simplicity bias in neural networks.
\newblock In \emph{Advances in Neural Information Processing Systems},
  volume~33, pages 9573--9585, 2020.
\newblock URL
  \url{https://proceedings.neurips.cc/paper_files/paper/2020/hash/6cfe0e6127fa25df2a0ef2ae1067d915-Abstract.html}.

\bibitem[Shalev-Shwartz et~al.(2010)Shalev-Shwartz, Shamir, Srebro, and
  Sridharan]{shalevshwartz2010sco}
Shai Shalev-Shwartz, Ohad Shamir, Nathan Srebro, and Karthik Sridharan.
\newblock Learnability, stability and uniform convergence.
\newblock \emph{Journal of Machine Learning Research}, 11:\penalty0 2635--2670,
  2010.

\bibitem[Shallue et~al.(2018)Shallue, Lee, Antognini, Sohl-Dickstein, Frostig,
  and Dahl]{shallue2018measuring}
Christopher~J. Shallue, Jaehoon Lee, Joseph Antognini, Jascha Sohl-Dickstein,
  Roy Frostig, and George~E. Dahl.
\newblock Measuring the effects of data parallelism on neural network training,
  2018.
\newblock URL \url{https://arxiv.org/abs/1811.03600}.

\bibitem[Shannon(1949)]{shannon1949communication}
Claude~E. Shannon.
\newblock Communication in the presence of noise.
\newblock \emph{Proceedings of the IRE}, 37\penalty0 (1):\penalty0 10--21,
  1949.
\newblock \doi{10.1109/JRPROC.1949.232969}.

\bibitem[Shokri et~al.(2017)Shokri, Stronati, Song, and
  Shmatikov]{shokri2017membership}
Reza Shokri, Marco Stronati, Congzheng Song, and Vitaly Shmatikov.
\newblock Membership inference attacks against machine learning models.
\newblock In \emph{2017 IEEE Symposium on Security and Privacy (SP)}, pages
  3--18. IEEE, 2017.

\bibitem[Sinha et~al.(2018)Sinha, Namkoong, and Duchi]{sinha2018certifying}
Aman Sinha, Hongseok Namkoong, and John Duchi.
\newblock Certifying some distributional robustness with principled adversarial
  training.
\newblock In \emph{International Conference on Learning Representations}, 2018.

\bibitem[Smith(2017)]{smith2017cyclical}
Leslie~N. Smith.
\newblock Cyclical learning rates for training neural networks.
\newblock In \emph{2017 IEEE Winter Conference on Applications of Computer
  Vision (WACV)}, pages 464--472, 2017.
\newblock \doi{10.1109/WACV.2017.58}.

\bibitem[Smith and Topin(2019)]{smith2019super}
Leslie~N. Smith and Nicholay Topin.
\newblock Super-convergence: Very fast training of neural networks using large
  learning rates.
\newblock In \emph{Artificial Intelligence and Machine Learning for
  Multi-Domain Operations Applications}, volume 11006, page 1100612. SPIE,
  2019.
\newblock \doi{10.1117/12.2520589}.

\bibitem[Smith and Le(2018)]{smith2018bayesian}
Samuel~L. Smith and Quoc~V. Le.
\newblock A {Bayesian} perspective on generalization and stochastic gradient
  descent.
\newblock In \emph{International Conference on Learning Representations}, 2018.
\newblock URL \url{https://openreview.net/forum?id=BJij4yg0Z}.

\bibitem[Smith et~al.(2018)Smith, Kindermans, Ying, and Le]{smith2018dont}
Samuel~L. Smith, Pieter-Jan Kindermans, Chris Ying, and Quoc~V. Le.
\newblock Don't decay the learning rate, increase the batch size.
\newblock In \emph{International Conference on Learning Representations}, 2018.
\newblock URL \url{https://openreview.net/forum?id=B1Yy1BxCZ}.

\bibitem[Smith et~al.(2021)Smith, Dherin, Barrett, and De]{smith2021origin}
Samuel~L. Smith, Benoit Dherin, David G.~T. Barrett, and Soham De.
\newblock On the origin of implicit regularization in stochastic gradient
  descent.
\newblock In \emph{International Conference on Learning Representations}, 2021.
\newblock URL \url{https://openreview.net/forum?id=rq_Qr0c1Hyo}.

\bibitem[Song et~al.(2021)Song, Steinke, Thakkar, and
  Thakurta]{song2020evading}
Shuang Song, Thomas Steinke, Om~Thakkar, and Abhradeep Thakurta.
\newblock Evading the curse of dimensionality in unconstrained private {GLMs}.
\newblock In \emph{International Conference on Artificial Intelligence and
  Statistics}, pages 2223--2234, 2021.

\bibitem[Sorscher et~al.(2022)Sorscher, Geirhos, Shekhar, Ganguli, and
  Morcos]{sorscher2022beyond}
Ben Sorscher, Robert Geirhos, Shashank Shekhar, Surya Ganguli, and Ari~S.
  Morcos.
\newblock Beyond neural scaling laws: Beating power law scaling via data
  pruning.
\newblock In \emph{Advances in Neural Information Processing Systems 35}, 2022.

\bibitem[Soudry et~al.(2018)Soudry, Hoffer, Nacson, Gunasekar, and
  Srebro]{soudry2018implicit}
Daniel Soudry, Elad Hoffer, Mor~Shpigel Nacson, Suriya Gunasekar, and Nathan
  Srebro.
\newblock The implicit bias of gradient descent on separable data.
\newblock \emph{Journal of Machine Learning Research}, 19\penalty0
  (70):\penalty0 1--57, 2018.
\newblock URL \url{https://www.jmlr.org/papers/v19/18-188.html}.

\bibitem[Spigler et~al.(2019)Spigler, Geiger, d'Ascoli, Sagun, Biroli, and
  Wyart]{spigler2019jamming}
Stefano Spigler, Mario Geiger, St{\'e}phane d'Ascoli, Levent Sagun, Giulio
  Biroli, and Matthieu Wyart.
\newblock A jamming transition from under- to over-parametrization affects
  generalization in deep learning.
\newblock \emph{Journal of Physics A: Mathematical and Theoretical},
  52\penalty0 (47):\penalty0 474001, 2019.
\newblock \doi{10.1088/1751-8121/ab45e3}.

\bibitem[Stein(1956)]{stein1956inadmissibility}
Charles Stein.
\newblock Inadmissibility of the usual estimator for the mean of a multivariate
  normal distribution.
\newblock In \emph{Proceedings of the Third Berkeley Symposium on Mathematical
  Statistics and Probability}, volume~1, pages 197--206. University of
  California Press, 1956.

\bibitem[Steinke and Zakynthinou(2020)]{steinke2020conditional}
Thomas Steinke and Lydia Zakynthinou.
\newblock Reasoning about generalization via conditional mutual information.
\newblock In \emph{Proceedings of Thirty Third Conference on Learning Theory},
  volume 125 of \emph{Proceedings of Machine Learning Research}, pages
  3437--3452. PMLR, 2020.

\bibitem[Strohmer and Vershynin(2009)]{strohmer2009randomized}
Thomas Strohmer and Roman Vershynin.
\newblock A randomized {Kaczmarz} algorithm with exponential convergence.
\newblock \emph{Journal of Fourier Analysis and Applications}, 15\penalty0
  (2):\penalty0 262--278, 2009.
\newblock \doi{10.1007/s00041-008-9030-4}.

\bibitem[Subramanian et~al.(2022)Subramanian, Arya, and
  Sahai]{subramanian2022multiclass}
Vignesh Subramanian, Rahul Arya, and Anant Sahai.
\newblock Generalization for multiclass classification with overparameterized
  linear models.
\newblock In \emph{Advances in Neural Information Processing Systems 35}, 2022.
\newblock URL \url{https://openreview.net/forum?id=ikWvMRVQBWW}.

\bibitem[Subramanian and Hajek(2002)]{subramanian2002broadband}
Vijay~G. Subramanian and Bruce Hajek.
\newblock Broad-band fading channels: Signal burstiness and capacity.
\newblock \emph{IEEE Transactions on Information Theory}, 48\penalty0
  (4):\penalty0 809--827, 2002.
\newblock \doi{10.1109/18.992768}.

\bibitem[Szegedy et~al.(2014)Szegedy, Zaremba, Sutskever, Bruna, Erhan,
  Goodfellow, and Fergus]{szegedy2014intriguing}
Christian Szegedy, Wojciech Zaremba, Ilya Sutskever, Joan Bruna, Dumitru Erhan,
  Ian Goodfellow, and Rob Fergus.
\newblock Intriguing properties of neural networks.
\newblock In \emph{International Conference on Learning Representations}, 2014.

\bibitem[Telatar and Tse(2000)]{telatar2000capacityWideband}
Emre Telatar and David N.~C. Tse.
\newblock Capacity and mutual information of wideband multipath fading
  channels.
\newblock \emph{IEEE Transactions on Information Theory}, 46\penalty0
  (4):\penalty0 1384--1400, 2000.
\newblock \doi{10.1109/18.850674}.

\bibitem[Toneva et~al.(2019)Toneva, Sordoni, Tachet~des Combes, Trischler,
  Bengio, and Gordon]{toneva2019empirical}
Mariya Toneva, Alessandro Sordoni, Remi Tachet~des Combes, Adam Trischler,
  Yoshua Bengio, and Geoffrey~J. Gordon.
\newblock An empirical study of example forgetting during deep neural network
  learning.
\newblock In \emph{International Conference on Learning Representations}, 2019.

\bibitem[Touvron et~al.(2019)Touvron, Vedaldi, Douze, and
  J{\'e}gou]{touvron2019fixing}
Hugo Touvron, Andrea Vedaldi, Matthijs Douze, and Herv{\'e} J{\'e}gou.
\newblock Fixing the train-test resolution discrepancy.
\newblock In \emph{Advances in Neural Information Processing Systems 32}, 2019.

\bibitem[Tsigler and Bartlett(2023)]{nr_tsigler2023}
Alexander Tsigler and Peter~L. Bartlett.
\newblock Benign overfitting in ridge regression.
\newblock \emph{Journal of Machine Learning Research}, 24\penalty0
  (123):\penalty0 1--76, 2023.

\bibitem[Tsilivis et~al.(2024)Tsilivis, Frank, Srebro, and
  Kempe]{tsilivis2024price}
Nikolaos Tsilivis, Natalie Frank, Nathan Srebro, and Julia Kempe.
\newblock The price of implicit bias in adversarially robust generalization,
  2024.
\newblock URL \url{https://arxiv.org/abs/2406.04981}.

\bibitem[Tsipras et~al.(2019)Tsipras, Santurkar, Engstrom, Turner, and
  Madry]{tsipras2019robustness}
Dimitris Tsipras, Shibani Santurkar, Logan Engstrom, Alexander Turner, and
  Aleksander Madry.
\newblock Robustness may be at odds with accuracy.
\newblock In \emph{International Conference on Learning Representations}, 2019.

\bibitem[Van~Trees(1968)]{vanTrees1968detection}
Harry~L. Van~Trees.
\newblock \emph{Detection, Estimation, and Modulation Theory, Part {I}}.
\newblock John Wiley \& Sons, New York, 1968.

\bibitem[Vansover-Hager et~al.(2025)Vansover-Hager, Koren, and
  Livni]{vansover2025rapid}
Shira Vansover-Hager, Tomer Koren, and Roi Livni.
\newblock Rapid overfitting of multi-pass {SGD} in stochastic convex
  optimization.
\newblock In \emph{Proceedings of the 42nd International Conference on Machine
  Learning}, volume 267 of \emph{Proceedings of Machine Learning Research},
  pages 60905--60923. PMLR, 2025.
\newblock arXiv:2505.08306.

\bibitem[Verd{\'u}(2002)]{verdu2002spectral}
Sergio Verd{\'u}.
\newblock Spectral efficiency in the wideband regime.
\newblock \emph{IEEE Transactions on Information Theory}, 48\penalty0
  (6):\penalty0 1319--1343, 2002.
\newblock \doi{10.1109/TIT.2002.1003824}.

\bibitem[Vershynin(2018)]{vershynin2018high}
Roman Vershynin.
\newblock \emph{High-Dimensional Probability: An Introduction with Applications
  in Data Science}.
\newblock Cambridge University Press, 2018.

\bibitem[Wang et~al.(2021)Wang, Muthukumar, and
  Thrampoulidis]{wang2021multiclass}
Ke~Wang, Vidya Muthukumar, and Christos Thrampoulidis.
\newblock Benign overfitting in multiclass classification: All roads lead to
  interpolation.
\newblock In \emph{Advances in Neural Information Processing Systems 34}, pages
  24164--24179, 2021.
\newblock URL
  \url{https://proceedings.neurips.cc/paper/2021/hash/caaa29eab72b231b0af62fbdff89bfce-Abstract.html}.

\bibitem[Wang et~al.(2025)Wang, Xu, Zhao, and Tao]{wang2025goodregularity}
Yuqing Wang, Zhenghao Xu, Tuo Zhao, and Molei Tao.
\newblock Good regularity creates large learning rate implicit biases: Edge of
  stability, balancing, and catapult.
\newblock \emph{Journal of Machine Learning Research}, 26\penalty0
  (273):\penalty0 1--68, 2025.
\newblock URL \url{https://jmlr.org/papers/v26/23-1691.html}.

\bibitem[Wang et~al.(2024)Wang, Sonthalia, and Hu]{wang2024nearinterp}
Yutong Wang, Rishi Sonthalia, and Wei Hu.
\newblock Near-interpolators: Rapid norm growth and the trade-off between
  interpolation and generalization.
\newblock In \emph{International Conference on Artificial Intelligence and
  Statistics}, volume 238 of \emph{Proceedings of Machine Learning Research},
  2024.
\newblock arXiv:2403.07264.

\bibitem[Widrow and Stearns(1985)]{widrow1985adaptive}
Bernard Widrow and Samuel~D. Stearns.
\newblock \emph{Adaptive Signal Processing}.
\newblock Prentice-Hall, Englewood Cliffs, NJ, 1985.

\bibitem[Wijsman(1990)]{wijsman1990invariant}
Robert~A. Wijsman.
\newblock \emph{Invariant Measures on Groups and Their Use in Statistics},
  volume~14 of \emph{IMS Lecture Notes--Monograph Series}.
\newblock Institute of Mathematical Statistics, Hayward, CA, 1990.
\newblock ISBN 9780940600195.

\bibitem[Wozencraft and Jacobs(1965)]{wozencraft1965principles}
John~M. Wozencraft and Irwin~Mark Jacobs.
\newblock \emph{Principles of Communication Engineering}.
\newblock John Wiley \& Sons, New York, 1965.

\bibitem[Wu and Sahai(2023)]{wu2023multiclass}
David~Xing Wu and Anant Sahai.
\newblock Precise asymptotic generalization for multiclass classification with
  overparameterized linear models.
\newblock In \emph{Advances in Neural Information Processing Systems 36}, 2023.
\newblock URL \url{https://openreview.net/forum?id=cRGINXQWem}.

\bibitem[Wu and Xu(2020)]{nr_wu2020}
Denny Wu and Ji~Xu.
\newblock On the optimal weighted {$\ell_2$} regularization in
  overparameterized linear regression.
\newblock In \emph{Advances in Neural Information Processing Systems},
  volume~33, pages 10112--10123, 2020.

\bibitem[Wu et~al.(2021)Wu, Zou, Braverman, and Gu]{wu2021direction}
Jingfeng Wu, Difan Zou, Vladimir Braverman, and Quanquan Gu.
\newblock Direction matters: On the implicit bias of stochastic gradient
  descent with moderate learning rate.
\newblock In \emph{International Conference on Learning Representations}, 2021.
\newblock URL \url{https://openreview.net/forum?id=3X64RLgzY6O}.

\bibitem[Wu et~al.(2022)Wu, Zou, Braverman, Gu, and Kakade]{wu2022last}
Jingfeng Wu, Difan Zou, Vladimir Braverman, Quanquan Gu, and Sham~M. Kakade.
\newblock Last iterate risk bounds of {SGD} with decaying stepsize for
  overparameterized linear regression.
\newblock In \emph{International Conference on Machine Learning (ICML)}, 2022.
\newblock arXiv:2110.06198.

\bibitem[Wu et~al.(2023)Wu, Braverman, and Lee]{wu2023logisticedge}
Jingfeng Wu, Vladimir Braverman, and Jason~D. Lee.
\newblock Implicit bias of gradient descent for logistic regression at the edge
  of stability.
\newblock In \emph{Advances in Neural Information Processing Systems},
  volume~36, pages 74229--74256, 2023.
\newblock \doi{10.52202/075280-3247}.
\newblock URL
  \url{https://papers.nips.cc/paper_files/paper/2023/hash/eb189151ced0ff808abafd16a51fec92-Abstract-Conference.html}.

\bibitem[Wu et~al.(2024)Wu, Bartlett, Telgarsky, and Yu]{wu2024largelogistic}
Jingfeng Wu, Peter~L. Bartlett, Matus Telgarsky, and Bin Yu.
\newblock Large stepsize gradient descent for logistic loss: Non-monotonicity
  of the loss improves optimization efficiency.
\newblock In \emph{Proceedings of Thirty Seventh Conference on Learning
  Theory}, volume 247 of \emph{Proceedings of Machine Learning Research}, pages
  5019--5073. PMLR, 2024.
\newblock URL \url{https://proceedings.mlr.press/v247/wu24b.html}.

\bibitem[Wu et~al.(2025{\natexlab{a}})Wu, Bartlett, Telgarsky, and
  Yu]{wu2025earlylogistic}
Jingfeng Wu, Peter~L. Bartlett, Matus Telgarsky, and Bin Yu.
\newblock Benefits of early stopping in gradient descent for overparameterized
  logistic regression.
\newblock In \emph{Proceedings of the 42nd International Conference on Machine
  Learning}, volume 267 of \emph{Proceedings of Machine Learning Research},
  pages 67081--67110. PMLR, 2025{\natexlab{a}}.
\newblock URL \url{https://proceedings.mlr.press/v267/wu25b.html}.

\bibitem[Wu et~al.(2025{\natexlab{b}})Wu, Marion, and
  Bartlett]{wu2025regularizedlogistic}
Jingfeng Wu, Pierre Marion, and Peter~L. Bartlett.
\newblock Large stepsizes accelerate gradient descent for regularized logistic
  regression.
\newblock In \emph{Advances in Neural Information Processing Systems},
  volume~38, pages 104485--104525, 2025{\natexlab{b}}.
\newblock \doi{10.52202/085713-3490}.
\newblock URL
  \url{https://papers.nips.cc/paper_files/paper/2025/hash/96b8167534ef3cc30c230bbeb55a524d-Abstract-Conference.html}.

\bibitem[Wu et~al.(2026)Wu, Bartlett, Kakade, Lee, and
  Yu]{wu2026riskcomparisons}
Jingfeng Wu, Peter~L. Bartlett, Sham~M. Kakade, Jason~D. Lee, and Bin Yu.
\newblock Risk comparisons in linear regression: Implicit regularization
  dominates explicit regularization (extended abstract).
\newblock In \emph{Proceedings of Thirty Ninth Conference on Learning Theory},
  volume 336 of \emph{Proceedings of Machine Learning Research}, pages
  6849--6851. PMLR, 2026.
\newblock URL \url{https://proceedings.mlr.press/v336/wu26a.html}.
\newblock Full version: arXiv:2509.17251.

\bibitem[Wu and Su(2023)]{wu2023implicit}
Lei Wu and Weijie~J. Su.
\newblock The implicit regularization of dynamical stability in stochastic
  gradient descent.
\newblock In \emph{Proceedings of the 40th International Conference on Machine
  Learning}, volume 202 of \emph{Proceedings of Machine Learning Research},
  pages 37656--37684. PMLR, 2023.
\newblock URL \url{https://proceedings.mlr.press/v202/wu23r.html}.

\bibitem[Wu et~al.(2018)Wu, Ma, and {E}]{wu2018sgd}
Lei Wu, Chao Ma, and Weinan {E}.
\newblock How {SGD} selects the global minima in over-parameterized learning: A
  dynamical stability perspective.
\newblock In \emph{Advances in Neural Information Processing Systems 31}, pages
  8279--8288, 2018.

\bibitem[Xing et~al.(2021)Xing, Zhang, and Cheng]{xing2021adversarially}
Yue Xing, Ruizhi Zhang, and Guang Cheng.
\newblock Adversarially robust estimate and risk analysis in linear regression.
\newblock In \emph{Proceedings of the 24th International Conference on
  Artificial Intelligence and Statistics}, pages 514--522, 2021.

\bibitem[Xu and Raginsky(2017)]{xu2017information}
Aolin Xu and Maxim Raginsky.
\newblock Information-theoretic analysis of generalization capability of
  learning algorithms.
\newblock In \emph{Advances in Neural Information Processing Systems},
  volume~30, pages 2524--2533, 2017.

\bibitem[Yeom et~al.(2018)Yeom, Giacomelli, Fredrikson, and
  Jha]{yeom2018privacy}
Samuel Yeom, Irene Giacomelli, Matt Fredrikson, and Somesh Jha.
\newblock Privacy risk in machine learning: Analyzing the connection to
  overfitting.
\newblock In \emph{IEEE 31st Computer Security Foundations Symposium (CSF)},
  pages 268--282. IEEE, 2018.

\bibitem[Yun et~al.(2019)Yun, Han, Oh, Chun, Choe, and Yoo]{yun2019cutmix}
Sangdoo Yun, Dongyoon Han, Seong~Joon Oh, Sanghyuk Chun, Junsuk Choe, and
  Youngjoon Yoo.
\newblock {CutMix}: Regularization strategy to train strong classifiers with
  localizable features.
\newblock In \emph{Proceedings of the IEEE/CVF International Conference on
  Computer Vision}, 2019.

\bibitem[Zhang et~al.(2018)Zhang, Cisse, Dauphin, and
  Lopez-Paz]{zhang2018mixup}
Hongyi Zhang, Moustapha Cisse, Yann~N. Dauphin, and David Lopez-Paz.
\newblock mixup: Beyond empirical risk minimization.
\newblock In \emph{International Conference on Learning Representations}, 2018.

\bibitem[Zhang et~al.(2026)Zhang, Wu, Lin, and Bartlett]{zhang2026adaptive}
Ruiqi Zhang, Jingfeng Wu, Licong Lin, and Peter~L. Bartlett.
\newblock Minimax optimal convergence of gradient descent in logistic
  regression via large and adaptive stepsizes.
\newblock \emph{Journal of Machine Learning Research}, 27\penalty0
  (124):\penalty0 1--31, 2026.
\newblock URL \url{https://jmlr.org/papers/v27/25-1941.html}.

\bibitem[Zheng and Tse(2002)]{zheng2002grassmann}
Lizhong Zheng and David N.~C. Tse.
\newblock Communication on the {Grassmann} manifold: A geometric approach to
  the noncoherent multiple-antenna channel.
\newblock \emph{IEEE Transactions on Information Theory}, 48\penalty0
  (2):\penalty0 359--383, 2002.
\newblock \doi{10.1109/18.978730}.

\bibitem[Zhou et~al.(2024)Zhou, Koehler, Sutherland, and
  Srebro]{zhou2024optimistic}
Lijia Zhou, Frederic Koehler, Danica~J. Sutherland, and Nathan Srebro.
\newblock Optimistic rates: A unifying theory for interpolation learning and
  regularization in linear regression.
\newblock \emph{ACM/IMS Journal of Data Science}, 1\penalty0 (2):\penalty0
  1--51, 2024.
\newblock \doi{10.1145/3594234}.

\bibitem[Zhu et~al.(2019)Zhu, Wu, Yu, Wu, and Ma]{zhu2019anisotropic}
Zhanxing Zhu, Jingfeng Wu, Bing Yu, Lei Wu, and Jinwen Ma.
\newblock The anisotropic noise in stochastic gradient descent: Its behavior of
  escaping from sharp minima and regularization effects.
\newblock In \emph{Proceedings of the 36th International Conference on Machine
  Learning}, volume~97 of \emph{Proceedings of Machine Learning Research},
  pages 7654--7663. PMLR, 2019.

\bibitem[Zou et~al.(2021)Zou, Wu, Braverman, Gu, Foster, and
  Kakade]{zou2021benefits}
Difan Zou, Jingfeng Wu, Vladimir Braverman, Quanquan Gu, Dean~P. Foster, and
  Sham~M. Kakade.
\newblock The benefits of implicit regularization from {SGD} in least squares
  problems.
\newblock In \emph{Advances in Neural Information Processing Systems},
  volume~34, pages 5456--5468, 2021.
\newblock URL
  \url{https://papers.nips.cc/paper/2021/hash/2b6bb5354a56ce256116b6b307a1ea10-Abstract.html}.

\bibitem[Zou et~al.(2022)Zou, Wu, Braverman, Gu, and Kakade]{zou2022multipass}
Difan Zou, Jingfeng Wu, Vladimir Braverman, Quanquan Gu, and Sham~M. Kakade.
\newblock Risk bounds of multi-pass {SGD} for least squares in the
  interpolation regime.
\newblock In \emph{Advances in Neural Information Processing Systems},
  volume~35, pages 12909--12920, 2022.
\newblock \doi{10.52202/068431-0938}.
\newblock URL
  \url{https://papers.nips.cc/paper_files/paper/2022/hash/543924fdf260ba990f2ef84f940f3db2-Abstract-Conference.html}.

\bibitem[Zou et~al.(2023)Zou, Wu, Braverman, Gu, and Kakade]{zou2023benign}
Difan Zou, Jingfeng Wu, Vladimir Braverman, Quanquan Gu, and Sham~M. Kakade.
\newblock Benign overfitting of constant-stepsize {SGD} for linear regression.
\newblock \emph{Journal of Machine Learning Research}, 24\penalty0
  (326):\penalty0 1--58, 2023.
\newblock URL \url{http://jmlr.org/papers/v24/21-1297.html}.
\newblock Short version in Proc.\ 34th Conference on Learning Theory (COLT),
  PMLR 134:4633--4635, 2021.

\end{thebibliography}

\clearpage
\phantomsection
\addcontentsline{toc}{section}{Notation Index}
\section*{Notation Index}

Table~\ref{tab:notation} collects the principal notation used in the paper.
Proof-local symbols are defined where they occur.
All logarithms in this paper are natural logarithms.  

Five cross-section reuses
are worth flagging:
\begin{itemize}
\item \(t\): training-error coordinate on the frontier.
It is also used as the SGD update index.
\item \(g\): test-error coordinate on the frontier; \(\vec{g}\) is a Gaussian vector in the
  rotation proof.
\item \(a,b\): polynomial exponents in the benign-misfitting-window footnote of
  the introduction, and local scalar variables in a few proofs.
\item \(c\): scaled-interpolator multiplier in
  Section~\ref{subsec:scaled-interpolation}, where it is also the common training prediction by
  \eqref{eq:c-r-relation}.
  Elsewhere \(c\) is a local scalar, introduced
  where used: the learning-rate misspecification factor
  \(\eta=c\,\eta_{\rm pass}\) in Remark~\ref{rem:constant-factor-lr}, and a
  local real number in the median-shielding lemma.
  These carry no relation to
  the multiplier and are local to their statements.
\item \(T\): the number of mini-batch steps in one pass in
  Subsection~\ref{subsec:minibatch-scaling}; a scalar projection in the
  random-rotation proof.
\end{itemize}


{\small
\begin{longtable}{@{}p{0.18\textwidth}p{0.58\textwidth}p{0.18\textwidth}@{}}
\caption{Principal notation used in this paper.
Entries are grouped by topic;
the Reference column points to the first definition or principal use.}
\label{tab:notation} \\
\hline
\textbf{Symbol} & \textbf{Description} & \textbf{Reference} \\
\hline
\endfirsthead
\multicolumn{3}{@{}l}{\textit{Table~\ref{tab:notation} continued.}}\\
\hline
\textbf{Symbol} & \textbf{Description} & \textbf{Reference} \\
\hline
\endhead
\hline
\endfoot

\notationgroupplain{General operators and conventions}
\(\R\)
  & Real numbers.
  & Preamble \\
\(\E\), \(\E_\xi\), \(\Pr_\xi\)
  & Expectation and probability; the subscript \(\xi\) denotes averaging or
    probability over label noise.
  & \eqref{eq:risk-defs}, \eqref{eq:noisy-full-pass-risk-main} \\
\(\calN(\cdot,\cdot)\)
  & Gaussian distribution.
  & \eqref{eq:test-law} \\
\(\diag(\cdot)\)
  & Diagonal matrix with the listed entries.
  & \eqref{eq:test-law} \\
\(\median(\cdot)\)
  & Median of three scalar predictions.
  & \eqref{eq:median-predictor-def} \\
\(\bm{I}_d\), \(\bm{I}_n\), \(\bm{I}_{d+1}\)
  & Identity matrices of the indicated sizes.
  & \eqref{eq:one-sample-test-error}, \eqref{eq:gram-inverse} \\
\(\vec{0}\)
  & Zero vector of the dimension determined by context.
  & \eqref{eq:test-law}, \eqref{eq:sgd-update} \\
\(\one\)
  & All-ones vector in \(\R^n\).
  & \eqref{eq:gram-inverse} \\
\(\mathbf{1}\{\cdot\}\)
  & Indicator of the event in braces.
  & \eqref{eq:rotation-aware-data}, Subsec.~\ref{subsec:minibatch-scaling} of App.~\ref{app:sgd-refinements} \\
\((\cdot)^\top\)
  & Transpose.
  & \eqref{eq:orthogonal-vectors} \\
\(\lambda_{\max}(\cdot)\)
  & Largest eigenvalue, used in the stability comparisons.
  Three regimes are
    distinguished: a long stream of \emph{fresh} examples is governed by the
    population covariance, as in \eqref{eq:stability-fresh}; a
    \emph{reused} example is governed by its own per-example boundary
    \(\eta\|\vec{x}_i\|_2^2<2\), as in \eqref{eq:stability-reuse}; and
    full-batch gradient descent on the summed objective is governed by
    \(\lambda_{\max}(\Xt\Xt^{\top})=d+\gamma n\)
    (Subsection~\ref{subsec:related-large-steps}).
  & \eqref{eq:stability-fresh}, \eqref{eq:stability-reuse},
    Sec.~\ref{sec:sgd} \\
\(z[1]\)
  & Scalar spike coordinate extracted from \(\vec{z}\in\R^{d+1}\).
  & \eqref{eq:coordinate-convention} \\
\(\vec{z}_{-1}\)
  & Nuisance-coordinate subvector of \(\vec{z}\in\R^{d+1}\).
  & \eqref{eq:coordinate-convention} \\
\(\|\cdot\|_2,\|\cdot\|_1,\|\cdot\|_\infty\)
  & Euclidean, \(\ell_1\), and \(\ell_\infty\) norms.
  & \eqref{eq:one-sample-data}, Lemma~\ref{lem:random-rotation} \\
\(\|\cdot\|_{\rm op}\)
  & Operator norm, \(\|\bm{A}\|_{\rm op}:=\sup_{\|\vec{x}\|_2=1}\|\bm{A}\vec{x}\|_2\);
    the largest absolute eigenvalue for symmetric \(\bm{A}\), the largest
    eigenvalue for a covariance.
  & \eqref{eq:flat-effective-ranks}, Sec.~\ref{par:nonflat-tails} \\
\begin{tabular}{@{}l@{}}
\(\ll,\gg,o(1),O(\cdot)\)\\
\(\asymp,\gtrsim,\lesssim\)\\
\(o_{\mathsf P}(1),O_{\mathsf P}(\cdot)\)
\end{tabular}
  & Standard deterministic and probabilistic asymptotic notation.
  & Throughout \\
\(\lfloor\cdot\rfloor,\lceil\cdot\rceil\)
  & Integer floor and ceiling.
  Integer sample counts are obtained from
    real-valued thresholds by rounding up with \(\lceil\cdot\rceil\).
  & Cor.~\ref{cor:sample-complexity} \\[4pt]

\notationgroup{sec:setup-frontier}
  {Core deterministic model (Section~\ref*{sec:setup-frontier})}
\(d\)
  & Nuisance dimension.
  The full ambient dimension is \(d+1\).
  & \eqref{eq:orthogonal-vectors} \\
\(n\)
  & Number of training samples.
  & \eqref{eq:training-data} \\
\(\gamma\)
  & Spike-to-nuisance variance ratio after nuisance-variance normalization;
    the main fourth-quadrant theory assumes \(\gamma>1\).
  & \eqref{eq:training-data}, Thm.~\ref{thm:pareto} \\
\(i,j,k\)
  & Sample or summation indices.
  & \eqref{eq:orthogonal-vectors} \\
\(\vec{v}_i\)
  & Orthogonal nuisance component of training sample \(i\).
  & \eqref{eq:orthogonal-vectors} \\
\(\vec{x}_i\)
  & Deterministic training vector.
  & \eqref{eq:training-data} \\
\(\Xt\)
  & Data matrix \([\vec{x}_1,\ldots,\vec{x}_n]\).
  & \eqref{eq:training-data}, Sec.~\ref{sec:sgd} \\
\(y_i\)
  & Training label; equal to \(1\) in the clean deterministic model and
    perturbed in the noisy-label model of Subsection~\ref{subsec:one-complete-pass}.
  & \eqref{eq:training-labels}, \eqref{eq:noisy-full-pass-risk-main} \\
\(\vec{x}_{\rm test}\)
  & Gaussian test point.
  & \eqref{eq:test-law} \\
\(y_{\rm test}\)
  & Scalar normalized-spike test label.
  & \eqref{eq:test-law} \\
\(\bm{\Sigma}\)
  & Population test covariance \(\diag(\gamma,1,\ldots,1)\).
  & \eqref{eq:test-law}, \eqref{eq:stability-fresh} \\
\(\vec{u}\)
  & Linear predictor, usually in the training span.
  & \eqref{eq:risk-defs} \\
\(\vec{\alpha}\)
  & Coefficient vector of a span predictor \(\vec{u}=\Xt\vec{\alpha}\).
  & \eqref{eq:risk-defs} \\
\(\alpha_i\)
  & Scalar entry \(i\) of \(\vec{\alpha}\).
  & \eqref{eq:sq-defs} \\
\(\Ttrain,\Ttest\)
  & Empirical training and clean test squared errors; the prose also calls
    these quantities risks, as declared after \eqref{eq:risk-defs}.
    In the
    one-sample warm-up they are also written as functions of \(\eta\).
  & \eqref{eq:risk-defs}, \eqref{eq:general-risk-defs} \\
\(s,q\)
  & Coefficient sum and coefficient energy.
  & \eqref{eq:sq-defs} \\
\(\vec{w}^{\star}\)
  & Spike oracle \((1/\sqrt\gamma,0,\ldots,0)^\top\).
  & \eqref{eq:oracle-def} \\
\(\rho\)
  & Sample ratio \(\rho=\gamma n/d=n/(d/\gamma)\), measuring \(n\)
    against the interpolation threshold.
  & \eqref{eq:scale-identities} \\
\(R\)
  & Signal-to-nuisance ratio
    \(R=\gamma^2n/d=\gamma\rho=n/(d/\gamma^2)\), measuring \(n\)
    against the first-useful-span threshold; the minimum span test error is
    \(g_{\min}=1/(1+R)\).
    Label noise
    changes the SGD risk but not \(R\).
  & \eqref{eq:scale-identities} \\
\(t,g\)
  & Training-error and test-error coordinates on the frontier.
  & \eqref{eq:pareto-set} \\
\(g_+(t)\)
  & Test error along the train--test tradeoff curve, as a function of training error.
  & Thm.~\ref{thm:pareto}, \eqref{eq:frontier-endpoint-form} \\
\(t^\star,g_{\min},g_{\rm int}\)
  & Frontier endpoint training error, best span test error, and interpolator test error.
  & \eqref{eq:frontier-endpoints} \\
\(h,\widetilde g(h)\)
  & Signed training residual \(h=r(1+1/\rho)-1=(\gamma+d/n)s-1\) and the
    signed-residual test-risk curve; \(g_\pm(t)=\widetilde g(\pm\sqrt t)\)
    are the overshooting and undershooting branches.
  & \eqref{eq:frontier-h-def}, \eqref{eq:frontier-signed-risk},
    \eqref{eq:frontier-branch-defs} \\[4pt]

\pagebreak
\notationgroup{sec:one-sample}
  {One-sample warm-up (Section~\ref*{sec:one-sample})}
\(\vec{x},\vec{v}\)
  & Single training point and its nuisance component.
  & \eqref{eq:one-sample-data} \\
\(y\)
  & Scalar label of the single training point.
  & \eqref{eq:one-sample-data} \\
\(\eta\)
  & Learning rate of the SGD update \eqref{eq:sgd-update}; also the scalar
    multiplying a training point in the one-sample warm-up.
  & \eqref{eq:one-sample-train-loss}, \eqref{eq:sgd-update} \\
\(\vec{x}_{\rm test},\vec{z}_{\rm test}\)
  & Fresh one-sample test point and its nuisance component.
  & \eqref{eq:one-sample-test-error} \\
\(y_{\rm test}\)
  & Test label (the normalized spike coordinate) of the fresh one-sample test point.
  & \eqref{eq:one-sample-test-error} \\
\begin{tabular}{@{}l@{}}
\(\eta_{\rm opt},\eta_{\rm interp}\)\\
\(\eta_{\rm cal},\eta_{\rm edge}\)
\end{tabular}
  & Test-optimal, interpolating, calibrated, and repeated-example
    stability-boundary learning rates.
  & \eqref{eq:one-sample-eta-opt}, \eqref{eq:one-sample-rates} \\[4pt]

\(Y\)
  & Physical quantity of interest in the sensor model.
  & Rem.~\ref{rem:sensor-model} \\
\(m_{\rm sens}\)
  & Number of sensors in the sensor model.
  & Rem.~\ref{rem:sensor-model} \\
\(\vec s_{\rm resp}\)
  & Fixed sensor responses to a physical value of one, before noise
    is added, with \(\|\vec s_{\rm resp}\|^2=m_{\rm sens}\).
    Unit responses mean that every entry of \(\vec s_{\rm resp}\) equals one.
  & Subsec.~\ref{subsec:sensor-misfitting} and
    Footnote~\ref{fn:one-sample-sensor-reflection} \\
\(\vec S_0\), \(\vec S\)
  & Sensor readings at the known value one; readings at a fresh value
    \(Y\).
  & \eqref{eq:sensor-product-average} \\
\(\sigma^2\)
  & Per-sensor noise variance in the sensor model.
  & Rem.~\ref{rem:sensor-model} \\

\notationgroup{subsec:frontier}
  {Frontier, interpolation, and ridge (Sections~\ref*{subsec:frontier}--\ref*{subsec:scaled-interpolation})}
\(\vec{u}_{\rm int},\vec{\alpha}_{\rm int}\)
  & Minimum-norm interpolator and its coefficient vector.
  & \eqref{eq:interpolator-def} \\
\(\vec{u}_c\)
  & Scaled interpolator.
  & \eqref{eq:scaled-train-error} \\
\(c\)
  & Scaled-interpolator ray multiplier and common training prediction;
    \(h=c-1\) and \(r=c\,r_{\rm int}\).
  & \eqref{eq:scaled-train-error}, \eqref{eq:c-r-relation} \\
\(c^\star\)
  & Test-optimal scaled-interpolator multiplier.
  & \eqref{eq:cstar-def} \\
\(\lambda,\lambda^\star\)
  & Kernel ridge penalty and the negative penalty reaching the test-optimal point.
  & \eqref{eq:ridge-scale}, \eqref{eq:lambda-star} \\
\(\bm{K},\bm{W},\Xt_{-1}\)
  & Training Gram matrix, nuisance Gram matrix, and nuisance-coordinate data
    matrix used in the ridge and residue-amplification comparisons.
  & \eqref{eq:kernel-ridge-objective}, Sec.~\ref{par:residue-amplification} \\
\(r,r_{\rm int}\)
  & Spike calibration \(r=\gamma\one^\top\vec{\alpha}\); \(r_{\rm int}\) is the
    minimum-norm interpolator's calibration.
  & \eqref{eq:calibrated-frontier-form}, \eqref{eq:c-r-relation} \\
\(\mathrm{SU}(\vec u),\mathrm{CN}(\vec u)\)
  & Signal-survival factor and fresh-test contamination standard deviation.
  & \eqref{eq:survival-contamination}, \eqref{eq:survival-risk-decomposition}, \eqref{eq:best-span-survival-contamination} \\
\(r_k(\bm{\Sigma}),R_k(\bm{\Sigma})\)
  & Bartlett mass-to-peak and participation effective-rank families; the
    subscripts distinguish them from the un-subscripted calibration \(r\)
    and ratio \(R=\gamma^2n/d\).
  & \eqref{eq:effective-rank-families} \\
\begin{tabular}{@{}l@{}}
\(r_0(\bm{\Sigma}),R_0(\bm{\Sigma})\)\\
\(\operatorname{tr}(\bm{\Sigma}),\|\bm{\Sigma}\|_{\rm op}\)
\end{tabular}
  & The \(k=0\) full-spectrum members used in the benign-overfitting
    comparison, and their spectral ingredients.
    The two sample-count
    thresholds they encode are displayed separately from the window
    inequality those thresholds bracket; the non-flat version is summarized
    in the related-work discussion.
  & Flat: \eqref{eq:flat-effective-ranks},
    \eqref{eq:flat-effective-rank-thresholds},
    \eqref{eq:flat-effective-rank-window}.
    Non-flat: Sec.~\ref{par:nonflat-tails} \\
\(\varepsilon\)
  & Target test error in sample-complexity, certified learning-rate, and
    adversarial-sensitivity calculations.
  & Cor.~\ref{cor:sample-complexity}, Subsec.~\ref{subsec:one-complete-pass} \\
\(n_{\rm mis},\rho_{\rm mis},n_{\rm int},\rho_{\rm int}\)
  & Misfitting sample threshold, the corresponding value of \(\rho\),
    the interpolation sample threshold, and the corresponding positive root.
  & Cor.~\ref{cor:sample-complexity} \\[4pt]

\notationgroup{sec:sgd}
  {One-pass SGD (Section~\ref*{sec:sgd})}
\(\vec{w}_i,\vec{w}_t,\vec{w}_n\)
  & SGD iterate after \(i\), \(t\), or a full pass of \(n\) examples; the
    exact final-iterate risk is \eqref{eq:one-pass-exact-risks}.
  & \eqref{eq:sgd-update}, \eqref{eq:sgd-span-rep},
    \eqref{eq:one-pass-exact-risks} \\
\(\alpha_i^{(t)},s_t,b_t\)
  & Coefficient on example \(i\) after \(t\) updates, coefficient sum, and spike residual.
  & \eqref{eq:sgd-span-rep}, \eqref{eq:sgd-state-defs},
    \eqref{eq:single-geometric-coeffs} \\
\(q_t\)
  & Coefficient energy after \(t\) updates, \(q_t=\sum_i(\alpha_i^{(t)})^{2}\), the pass's value
    of \(q\) in \eqref{eq:sq-defs}.
  & \eqref{eq:sq-defs}, \eqref{eq:single-geometric-energy} \\
\(r_n\)
  & Calibration of the final iterate, \(r_n=\gamma s_n=1-\mu^{n}\); it is the pass's value of the
    calibration ratio \(r\) in \eqref{eq:neff-frontier-form}, and reaches \(1\) exactly when the
    spike residual is driven to zero.
  & Prop.~\ref{prop:one-pass-risk-summary},
    \eqref{eq:one-pass-calibration}, \eqref{eq:neff-frontier-form} \\
\(h_i\)
  & End-of-pass signed residual on training example \(i\),
    \(h_i=\vec{w}_n^{\top}\vec{x}_i-1=-\mu^{n}+d\eta\,\mu^{\,i-1}\); positive
    --- the example is overshot --- whenever \(1/d<\eta<1/\gamma\).
  & \eqref{eq:reuse-residual-onepass}, App.~\ref{app:stability-boundaries} \\
\(\mu=1-\eta\gamma\)
  & Spike-residual multiplier: the signed factor by which one fresh update multiplies the
    remaining spike residual.
    Its magnitude is below one exactly on \(0<\eta\gamma<2\), it is
    zero at \(\eta\gamma=1\), and it is expansive for \(\eta\gamma>2\).
  & Lemma~\ref{lem:single-geometric},
    Prop.~\ref{prop:one-pass-risk-summary} \\
\(\Lambda=\eta\gamma n\)
  & The pass's \emph{calibration budget}, and --- on the small-rate branch only --- the spike
    contraction exponent accumulated over one pass.
    The exact exponent is
    \(-n\log|1-\eta\gamma|\), of which \(\Lambda\) is the first-order approximation; see
    \eqref{eq:eta-pass-step}.
    It controls the
    exact one-pass spike-residual factor
    \((1-\eta\gamma)^n\), which equals \(e^{-\Lambda}(1+o(1))\) when
    \(n(\eta\gamma)^2=\Lambda\eta\gamma=o(1)\); the spike bias after the pass
    is then \(e^{-2\Lambda}(1+o(1))\).
    Label noise changes the risk
    decomposition, not this definition.
  & Prop.~\ref{prop:one-pass-risk-summary},
    Cor.~\ref{cor:learning-rate-robustness},
    \eqref{eq:noisy-lr-bound-main} \\
\(n_{\rm eff}\)
  & Participation count \eqref{eq:neff-def}, \(s^2/q\).
  For a nonnegative coefficient profile
    it counts how many training examples effectively share the load, and \(n/n_{\rm eff}\) is the
    participation deficit.
    Unrelated to the effective ranks \(r_k,R_k\) of
    \eqref{eq:effective-rank-families}.
  & \eqref{eq:neff-def}, Rem.~\ref{rem:participation},
    Prop.~\ref{prop:one-pass-risk-summary},
    Rem.~\ref{rem:one-pass-log-participation}, \eqref{eq:one-pass-neff} \\
\(\bar{\vec{x}}_m,\vec{w}_{\rm avg}\)
  & Average of the first \(m\) samples and the calibrated averaged predictor.
  & \eqref{eq:flat-average-def}, \eqref{eq:flat-predictor-def} \\
\(m\)
  & Number of samples in the flat average.
  & \eqref{eq:flat-average-def}, \eqref{eq:flat-predictor-def} \\
\(\Lambda_{\rm pass}\)
  & The clean-label tuned budget \(\Lambda_{\rm pass}=\tfrac12\log R\), the budget carried by
    \(\eta_{\rm pass}\) of \eqref{eq:eta-pass-def}.
  & \eqref{eq:Lambda-def}, \eqref{eq:eta-pass-def}, \eqref{eq:eta-pass-step} \\
\(\eta_{\rm mb},B_{\rm mb},T\)
  & Mini-batch learning rate, scalar batch size, and number of batches in one pass,
    \(T=n/B_{\rm mb}\).
  & Subsec.~\ref{subsec:minibatch-scaling} of App.~\ref{app:sgd-refinements} \\
\(\mu_{\rm mb},\bar\alpha_k\)
  & Per-batch spike-residual multiplier \eqref{eq:minibatch-mu}, the batch analogue of
    \(\mu\), and the common coefficient the \(k\)th batch writes on each of its members.
  & \eqref{eq:minibatch-mu}, \eqref{eq:minibatch-participation} \\

\notationgroup{sec:robustness}
  {Adversarial sensitivity (Section~\ref*{sec:robustness}) and random rotation
   (Appendix~\ref*{app:rotation-aware})}
\(\vec{\theta}\)
  & Unit spike direction; equals \(\vec{e}_1\) in the aligned coordinates.
  & \eqref{eq:rotation-aware-data} \\
\(\vec{e}_1\)
  & First standard basis vector.
  & Lemma~\ref{lem:random-rotation} \\
\(\widetilde{\vec{v}}\)
  & Nuisance residue of a span predictor at spike calibration \(r\).
  & \eqref{eq:calibration-residue}, \eqref{eq:rotation-decomposition} \\
\(\vec{\Delta}\)
  & Adversarial perturbation vector.
  & \eqref{eq:A2-def}, Thm.~\ref{thm:absolute-rms} \\
\(\delta,\tau\)
  & Scalar perturbation-radius parameter and output tolerance.
  & \eqref{eq:A2-def}, Thm.~\ref{thm:absolute-rms} \\
\(A_2(\vec{u};\delta),A_\infty^{\rm rot}(\vec{u};\delta)\)
  & Adversarial sensitivity under RMS and randomly rotated coordinate
    perturbations.
  & \eqref{eq:A2-def}, \eqref{eq:A2-response}, \eqref{eq:Ainfty-rot-def}, \eqref{eq:Ainfty-rot-scale} \\
\(\sqrt{2/\pi}\)
  & Explicit asymptotic constant in the rotated-coordinate robustness formula.
  & \eqref{eq:Ainfty-rot-scale}, \eqref{eq:Ainfty-rot-scale-symmetric} \\
\begin{tabular}{@{}l@{}}
\(n_{\rm clean},n_{\rm int}\)\\
\(n_{\rm adv}(\varepsilon,\delta,\tau)\)
\end{tabular}
  & Clean-prediction, interpolation, and adversarial-sensitivity sample
    thresholds.
    In the robustness ladder \(n_{\rm clean}\) and
    \(n_{\rm int}\) are the exact real-valued thresholds
    \(n_{\rm mis}(\varepsilon)\) and \(n_{\rm int}(\varepsilon)\) of
    Corollary~\ref{cor:sample-complexity}, before integer rounding.
  & Fig.~\ref{fig:robustness-phases}, \eqref{eq:robustness-thresholds} \\
\(\rho^{\star}\)
  & Sample-ratio threshold \((1-\sqrt\varepsilon)/(1+\sqrt\varepsilon)\) for a
    clean-test target \(\varepsilon\).
    While \(\rho<\rho^{\star}\), no span
    predictor meeting the target lies in the good-fit quadrant.
  & Footnote~\ref{fn:rho-star}, \eqref{eq:clean-calibration-floor},
    \eqref{eq:calibrated-frontier-form} \\
\(\bm{Q}\)
  & Haar random rotation matrix.
  & Lemma~\ref{lem:random-rotation} \\
\(\vec{z}_0,\vec{z}\)
  & Fixed pre-rotation vector and rotated vector.
  & Lemma~\ref{lem:random-rotation} \\
\(\vec{g},\vec{\omega}\)
  & Gaussian vector and coordinate-sign vector.
  & Proof of Lemma~\ref{lem:random-rotation} \\
\(T\)
  & Scalar projection in the random-rotation proof.
  & Proof of Lemma~\ref{lem:random-rotation} \\[4pt]

\notationgroup{subsec:one-complete-pass}
  {Noisy training labels (Subsection~\ref*{subsec:one-complete-pass})}
\(\sigma_\xi,\xi_i\)
  & Label-noise level and independent standard Gaussian label noise.
  & \eqref{eq:noisy-full-pass-risk-main} \\
\(\Ttest^{\rm clean}(\vec{w}_n)\)
  & Full-pass clean SGD test risk without label noise.
  & \eqref{eq:noisy-full-pass-risk-main} \\[4pt]

\notationgroup{par:nonflat-tails}
  {Non-flat nuisance tails (related work)}
\(\bm{\Omega}\)
  & Non-flat nuisance test covariance, normalized to
    \(\operatorname{tr}(\bm{\Omega})=d\).
  & Sec.~\ref{par:nonflat-tails} \\
\(\bm{V},\bm{G}_{\Omega},M_2,\delta_{\rm tail}\)
  & Training nuisance matrix, covariance-weighted nuisance Gram
    \(\bm{G}_{\Omega}=\bm{V}^{\top}\bm{\Omega}\bm{V}\), its squared spectral
    mass \(M_2=\operatorname{tr}(\bm{\Omega}^2)\), and the relative-isotropy
    tolerance.
  & Sec.~\ref{par:nonflat-tails} \\[4pt]

\notationgroup{app:latent-rotation-span-complete}
  {Unknown-orientation span completeness
   (Appendix~\ref*{app:latent-rotation-span-complete})}
\(\mathrm O(d+1),\bm H,\mathcal D_{\bm H}\)
  & Full orthogonal group, latent common orientation, and resulting ordered
    labeled sample.
  & \eqref{eq:rotspan-family} \\
\(\mathcal D=(\bm X,\one),\vec\theta_0\)
  & Ordered realized labeled sample and the data-visible projection shared
    by all of its compatible spike directions.
  & \eqref{eq:rotspan-visible-spike} \\
\(\Theta_{\mathcal D}^{\mathrm{world}}\)
  & Spike directions generated by hidden worlds that produce
    $\mathcal D$.
  & \eqref{eq:rotspan-compatible-sets} \\
\(\Theta_{\mathcal D}^{\mathrm{constraint}}\)
  & Unit directions satisfying all observed inner-product equations.
  & \eqref{eq:rotspan-compatible-sets} \\
\(\Theta_{\mathcal D}\)
  & Common value of the world-generated and constraint-defined sets, equal
    to the compatible-direction fiber.
  & Proposition~\ref{prop:rotspan-compatible-directions} and
    \eqref{eq:rotspan-compatible-fiber} \\
\(\operatorname{null}(\bm X^\top)\)
  & Nullspace of $\bm X^\top$, consisting of directions that
    change no training prediction.
  & \eqref{eq:rotspan-nullspace} \\
\(\pi_{\mathcal D},\mathcal B_{\mathcal D},\mathcal M_{\mathcal D}\)
  & Uniform law on the compatible-direction fiber, uniform compatible-world average
    risk, and worst-compatible risk.
  & \eqref{eq:rotspan-criteria} \\
\(\bm P_X,\bm R_X\)
  & Orthogonal projector onto the realized training span and the
    data-fixing reflection \(\bm R_X=2\bm P_X-\bm I\).
  & \eqref{eq:rotspan-projector} and \eqref{eq:rotspan-reflection-def} \\
\(L_{\mathcal D}(\vec u)\)
  & Empirical squared loss of an arbitrary ambient global-linear
    predictor in the unknown-orientation model.
  & \eqref{eq:rotspan-ambient-training-loss} \\[4pt]

\notationgroup{app:improper}
  {Improper median escape (Appendix~\ref*{app:improper})}
\(\mathcal P_k\)
  & Deterministic pile \(k\).
  & \eqref{eq:improper-pile-predictors} \\
\(m_k\)
  & Scalar size of pile \(\mathcal P_k\).
  & \eqref{eq:improper-pile-predictors} \\
\(f_k\)
  & Pile-wise calibrated scalar-valued linear predictor.
  & \eqref{eq:improper-pile-predictors} \\
\(\vec{u}_k\)
  & Weight vector of the pile-wise predictor.
  & \eqref{eq:improper-pile-predictors} \\
\(\vec{a}_k\)
  & Nuisance component of constituent predictor \(f_k\).
  & \eqref{eq:improper-median-adversarial-response} \\
\(f_{\rm med}\)
  & Median-of-three improper aggregate.
  & \eqref{eq:median-predictor-def} \\
\(f\)
  & General measurable predictor used to extend train/test risk notation.
  & \eqref{eq:general-risk-defs} \\
\(e_k\)
  & Fresh-test residual of constituent predictor \(f_k\) in the proof.
  & Proof of Prop.~\ref{prop:improper-stylized} \\
\(a,b,c,t\)
  & Local real numbers in the median-shielding lemma.
  & Lemma~\ref{lem:median-shield} \\[4pt]
\notationgroup{app:sco-comparison}
  {Norm budget and SCO comparison (Appendix~\ref*{app:sco-comparison})}
\(\mathcal W,\mathcal B_\star\)
  & Bounded domain notation and oracle norm ball in the norm-budget
    comparison.
  & App.~\ref{app:sco-comparison} \\
\(r^\star\)
  & Test-optimal scalar span calibration in the norm-budget comparison.
  & App.~\ref{app:sco-comparison} \\
\(\vec{u}^{\star}_{\rm span}\)
  & Corresponding test-optimal span predictor.
  & App.~\ref{app:sco-comparison} \\
\end{longtable}
}

\section*{Tooling}
The primary editor used in the preparation of this work is Antigravity, 
and the Gemini, ChatGPT, and Claude family of tools and associated harnesses were 
used extensively for drafting, editing, and checking math carefully 
as well as helping maintain notational consistency, the editing of figures,
and running/checking simulations.

\section*{Acknowledgements}

The authors would like to acknowledge funding support from:
NSF CAREER grant ECCS-2240031, NSF grant AST-2132700,
Renaissance Philanthropy and the XTX Markets AI for Math Fund 2025, 
an OpenAI Superalignment grant,
the 2023 CITRIS Institute Seed Grant,
the UC Berkeley Innovation MicroGrant Program,
Google, and Qualcomm.

We would like to thank Misha Belkin, Daniel Hsu, Vidya Muthukumar, 
Vignesh Subramanian, and Adhyyan Narang --- the initial seed for this line of work was planted while collaborating with them. Thanks also to conversations with David Wu. Finally, the comments from reviewers on previous papers also influenced this result.

\end{document}